\documentclass[11pt, reqno]{amsart}
\usepackage{txfonts}
\usepackage{amsmath,amssymb,amsthm,mathrsfs,enumerate,bm,xcolor,multirow,pbox}
\usepackage{mathrsfs}
\usepackage{algorithm,algorithmic,float,placeins}
\usepackage{graphicx,color,framed,tikz,caption,subcaption}
\usepackage{enumitem}
\setlist{leftmargin=9mm}
\usepackage[colorlinks,linkcolor=black,citecolor=black,urlcolor=black]{hyperref}
\allowdisplaybreaks[4]
\numberwithin{equation}{section}
\newcommand{\N}{\mathbb{N}}
\newcommand{\R}{\mathbb{R}}

\newcommand{\pnorm}[2]{\lVert #1\rVert_{#2}}
\newcommand{\bigpnorm}[2]{\big\lVert#1\big\rVert_{#2}}
\newcommand{\biggpnorm}[2]{\bigg\lVert#1\bigg\rVert_{#2}}
\newcommand{\abs}[1]{\lvert#1\rvert}
\newcommand{\bigabs}[1]{\big\lvert#1\big\rvert}
\newcommand{\biggabs}[1]{\bigg\lvert#1\bigg\rvert}
\newcommand{\iprod}[2]{\langle#1,#2\rangle}
\newcommand{\bigiprod}[2]{\big\langle#1,#2\big\rangle}

\renewcommand{\epsilon}{\varepsilon}

\renewcommand{\d}[1]{\mathrm{d}#1}

\newcommand{\smallo}{\mathfrak{o}}
\newcommand{\bigo}{\mathcal{O}}

\renewcommand{\tilde}{\widetilde}

\DeclareMathOperator{\E}{\mathbb{E}}
\DeclareMathOperator{\Prob}{\mathbb{P}}

\DeclareMathOperator{\var}{Var}
\DeclareMathOperator{\cov}{Cov}

\DeclareMathOperator{\op}{op}

\DeclareMathOperator{\prox}{\mathsf{prox}}

\DeclareMathOperator*{\argmin}{arg\,min\,}

\newcommand{\red}[1]{{\color{red}#1}}
\newcommand{\blue}[1]{{\color{blue}#1}}

\usepackage{soul}

\theoremstyle{definition}\newtheorem{problem}{Problem}[section]
\theoremstyle{definition}\newtheorem{definition}[problem]{Definition}
\theoremstyle{definition}\newtheorem{algorithmdef}{Algorithm}
\theoremstyle{remark}\newtheorem{assumption}{Assumption}

\theoremstyle{remark}\newtheorem{remark}{Remark}
\theoremstyle{definition}
\theoremstyle{plain}\newtheorem{theorem}[problem]{Theorem}
\theoremstyle{plain}
\theoremstyle{plain}\newtheorem{lemma}[problem]{Lemma}
\theoremstyle{plain}\newtheorem{proposition}[problem]{Proposition}
\theoremstyle{plain}
\theoremstyle{plain}

\AtBeginDocument{%
	\def\MR#1{}
}

\begin{document}

\title[Gradient descent training dynamics for neural networks]{Precise gradient descent training dynamics for finite-width multi-layer neural networks}

\author[Q. Han]{Qiyang Han}

\address[Q. Han]{
Department of Statistics, Rutgers University, Piscataway, NJ 08854, USA.
}
\email{qh85@stat.rutgers.edu}

\author[M. Imaizumi]{Masaaki Imaizumi}
\address[M. Imaizumi]{Graduate School of Arts and Sciences, The University of Tokyo, Meguro, Tokyo 153-8902, Japan. RIKEN Center for Advanced Intelligence Project, Chuo, Tokyo 103-0027, Japan.}
\email{imaizumi@g.ecc.u-tokyo.ac.jp}

\date{\today}

\keywords{empirical risk minimization, generalization error, gradient descent, neural networks, state evolution}
\subjclass[2000]{60E15, 60G15}

\begin{abstract}
Understanding the training dynamics of gradient descent in realistic, finite-width multi-layer neural networks remains a central challenge in deep learning theory. 
In this paper, we provide the first precise distributional characterization of gradient descent iterates for general multi-layer neural networks under the canonical single-index regression model, in the so-called `finite-width proportional regime', where the sample size and feature dimension grow proportionally while the network width and depth remain bounded. Our non-asymptotic state evolution theory captures Gaussian fluctuations in first-layer weights and deterministic concentration in deeper-layer weights, and remains valid for non-Gaussian features.

Our theory differs from popular frameworks such as the neural tangent kernel (NTK), mean-field (MF), and tensor program (TP) in several key aspects. First, our theory operates intrinsically in the finite-width regime consistent with practical architectures, whereas these existing theories are fundamentally infinite-width. Second, our theory allows weights to evolve from individual initializations beyond the lazy training regime, whereas NTK and MF are either frozen at or only weakly sensitive to initialization, and TP relies on special initialization schemes. Third, our theory characterizes both training and generalization errors for general multi-layer neural networks and reveals a clear generalization gap beyond the classical uniform convergence regime, whereas existing theories study generalization almost exclusively in two-layer settings.

As a statistical application, we show that vanilla gradient descent can be augmented with a few closed-form computations to yield consistent estimates of the generalization error at each iteration. Crucially, our method remains valid without requiring either algorithmic convergence or any knowledge of the underlying link function and the signal, and may therefore inform practical decisions such as early stopping and hyperparameter tuning. As a further theoretical implication, we show that despite model misspecification, the model learned by gradient descent retains the structure of a single-index function, with an effective signal determined by a linear combination of the true signal and the initialization.

The proof relies on an iterative reduction scheme that maps the gradient descent iterates to a sequence of matrix-variate general first order methods (GFOMs), for which the entrywise, non-asymptotic GFOM state evolution theory recently developed in \cite{han2024entrywise} is employed in its full strength.
\end{abstract}

\maketitle

\setcounter{tocdepth}{1}
\tableofcontents

\sloppy

\section{Introduction}

\subsection{Overview}
Consider the class of feed-forward $L$-layer ($L\geq 2$) neural networks model that consists of functions $f_{\bm{W}}:\R^n\to \R$ defined by
\begin{align}\label{def:NN_model}
f_{\bm{W}}(x)\equiv W_L^\top \sigma\big(W_{L-1}^\top\cdots \sigma(W_1^\top x)\big),\quad \hbox{for all } x \in \R^n.
\end{align}
Here (i) $\bm{W}=(W_1,\ldots,W_L) \in \R^{n\times q}\times (\R^{q\times q})^{[2:L]}$ denotes a network parameter\footnote{With slight abuse of notation, we shall identify $W_L$ and its only non-trivial first column in the introduction.} with $W_L=[\ast\,|\, 0_{q\times (q-1)}]$ and (ii) $\sigma:\R\to \R$ is a (non-linear) activation function applied entrywise. We assume access to training data consisting of feature-response pairs $\{(X_i,Y_i)\}_{i \in [m]} \subset \R^{n}\times \R$ from a standard single-index regression model:
\begin{align}\label{def:model}
Y_i = \varphi_\ast\big(\iprod{X_i}{\mu_\ast}\big)+\xi_i,\quad i \in [m],
\end{align}
where $\varphi_\ast:\R\to \R$ is an unknown link function, $\{\xi_i\}_{i \in [m]}$ are (random) errors, and $m$ and $n$ denote the sample size and feature dimension, respectively. For definiteness, we consider normalization with $\var(X_{ij})\asymp 1$ and $\pnorm{\mu_\ast}{}=\bigo(1)$. 

We will be interested in gradient descent training for multi-layer neural network models (\ref{def:NN_model}) that has witnessed enormous success in modern machine learning. The simplest possible gradient descent over the empirical squared loss 
\begin{align}\label{def:ERM_squared_loss}
\mathsf{L}(\bm{W}) = \frac{1}{2m}\sum_{i \in [m]} \big(Y_i-f_{\bm{W}}(X_i)\big)^2
\end{align}
proceeds as follows: with initialization $\bm{W}^{(0)}$ and a learning rate $\eta>0$, for $t=1,2,\ldots$, the gradient descent iteratively updates the network parameter $\bm{W}^{(t)}$ by
\begin{align}\label{def:grad_descent_intro}
\bm{W}^{(t)}\equiv \bm{W}^{(t-1)}- \eta\cdot  \nabla_{W} \mathsf{L}(\bm{W}^{(t-1)}).
\end{align}
A central object of statistical interest for the gradient descent $\bm{W}^{(t)}$ in (\ref{def:grad_descent_intro}) is the generalization capacity of the learned model $f_{\bm{W}^{(t)}}$ measured by the \emph{generalization/test error}, formally defined as
\begin{align}\label{def:test_error}
\mathscr{E}_{\texttt{test}}^{(t)}(X,Y)\equiv \E \big[ \big(Y_{\textrm{new}}- f_{\bm{W}^{(t)}}( X_{\textrm{new}} ) \big)^2 | \{(X_i,Y_i)\}_{i \in [m]} \big].
\end{align}
Here the expectation in (\ref{def:test_error}) is taken with respect to a pair of new test data $(X_{\textrm{new}},Y_{\textrm{new}})\in \R^n\times \R$ from the regression model (\ref{def:model}).

\vspace{0.5em}
\noindent \textbf{Question.} \emph{Can we characterize the behavior of the gradient descent iterates $\bm{W}^{(t)}$ in (\ref{def:grad_descent_intro}) and its generalization error (\ref{def:test_error}) for the general class of multi-layer neural network models (\ref{def:NN_model})?}
\vspace{0.5em}

As the landscape of the loss function $\bm{W}\mapsto\mathsf{L}(\bm{W})$ in (\ref{def:ERM_squared_loss}) is highly non-convex, gradient descent over neural network models is not apriori guaranteed to converge to a near-global optimum. Therefore, standard statistical theory does not readily apply to provide quantitative characterizations for its generalization error (\ref{def:test_error}).

Several rigorous theoretical frameworks have been proposed in the literature to shed light on the Question (some of which mainly consider the closely related stochastic gradient descent setting):
\begin{enumerate}
	\item[(T1)] In the neural tangent kernel (NTK) theory~\cite{jacot2018neural}, the gradient descent dynamics are approximated by a linearized model around the initialization, which becomes exact in the infinite-width limit under suitable scaling and small learning rates; see, e.g.,~\cite{du2019gradient,allen2019convergence,chizat2019lazy,oymak2019overparameterized,ji2020oolylogarithmic,zou2020gradient,bartlett2021deep}.
	
	\item[(T2)] In the mean-field (MF) theory~\cite{chizat2018global,mei2018mean,sirignano2020mean,rotskoff2022trainability,nguyen2023rigorous}, typically for two-layer networks, the evolution of the empirical distribution of hidden-layer weights is described by a continuous-time partial differential equation in the infinite-width limit, interpreted as a Wasserstein gradient flow on the space of probability measures.
	
	\item[(T3)] The tensor program (TP) framework~\cite{yang2019scaling,yang2020tensor2,yang2020tensor3,yang2020feature,yang2021tuning,golikov2022non,yang2024tensor} establishes the joint weak convergence of the empirical distributions of neuron-level scalar random variables (e.g., (pre-)activations, weights) across layers, under i.i.d.\ initialization and infinite-width scaling. 
\end{enumerate}
While the aforementioned rigorous theories have led to significant progress in understanding the behavior of the gradient descent iterates $\bm{W}^{(t)}$ in~(\ref{def:grad_descent_intro}), they exhibit limitations in several important aspects:
\begin{enumerate}
	\item[(L1)] (\emph{Infinite width}). All NTK, MF, and TP theories are formulated under an essentially infinite-width setting, or aim to quantify deviations from it. In contrast, many widely used deep learning models such as \texttt{ResNet}~\cite{he2016deep,zagoruyko2016wide}, \texttt{EfficientNet}~\cite{tan2019efficientnet}, \texttt{ViT}~\cite{dosovitskiy2020image}, and \texttt{GPT-3}~\cite{brown2020language}, are trained on large, high-dimensional datasets using networks with relatively modest width and depth (cf.~Table~\ref{tab:model-summary}), and the empirical training dynamics are heavily influenced by finite-width effects~\cite{lee2020finite}.
	
	\item[(L2)] (\emph{Impact of initialization}). 
    The NTK theory constrains the gradient updates $\bm{W}^{(t)}$ in the lazy training regime that remains close to the initialization $\bm{W}^{(0)}$ in magnitude~\cite{chizat2019lazy}. The MF theory incorporates initialization through its limiting empirical distribution, whereas empirical training dynamics often exhibit sensitivity to specific realizations of the initialization~\cite{sutskever2013importance,hanin2018start}. Moreover, when the number of layer is at least $4$, the MF theory must utilize specially designed initialization schemes to avoid strong degeneracy \cite{nguyen2023rigorous}. The TP theory requires the i.i.d. initialization scheme in an essential way, and does not currently accommodate more structured initialization schemes (e.g.,  orthogonal  initializations~\cite{saxe2013exact}). 
	
	\item[(L3)] (\emph{Generalization error}). The NTK theory provides sharp generalization characterizations in the lazy training regime via kernel ridge (or ridgeless) estimators, but is primarily limited to two-layer neural networks; cf.~\cite{wu2020optimal,ghorbani2021linearized,mei2022generalization,mei2022generalizationb,hastie2022surprises,montanari2022interpolation,han2023distribution}. Both MF and TP theories currently fall short of providing useful characterizations of the generalization error~(\ref{def:test_error}) even in the two-layer case.
\end{enumerate}

\begin{table}
	\centering
	\small
	\begin{tabular}{l|c|c|c|c|l}
		\textbf{Model} & \textbf{Dep. $L$} & \textbf{Wid. $q$} & \textbf{Sample Size $m$} & \textbf{Feat. Dim. $n$} & \textbf{Reference} \\
		\hline
		\hline
		\texttt{ResNet-152} & 152 & 2048 & ImageNet: 1.3M & $\sim$150K & \cite{he2016deep} \\
		\hline
		\texttt{EfficientNet-B7} & $\sim$66 & $\sim$640 & ImageNet: 1.3M & $\sim$1M & \cite{tan2019efficientnet} \\
		\hline
		\texttt{ViT-L/16} & 24 & 1024 & ImageNet: 1.3M  & $\sim$200K & \cite{dosovitskiy2020image} \\
		\hline
		\texttt{GPT-3} & 96 & $\sim$12K & $\sim$300B tokens & up to 25M &  \cite{brown2020language} 
	\end{tabular}
    \vspace{0.5em}
	\caption{Some examples of popular neural network models. Here 1K = $10^3$, 1M = $10^6$ and 1B = $10^9$.}
	\label{tab:model-summary}
\end{table}

The main goal of this paper is to introduce a new theoretical framework that addresses some of the aforementioned limitations (L1)-(L3) within the existing theories (T1)-(T3), and thereby providing a new perspective on the Question. We will do so by characterizing, for every iteration $t$, the precise distributional behavior of the gradient updates $\bm{W}^{(t)}$ and related quantities in the data model (\ref{def:model}), in what we call the \emph{finite-width proportional regime}, where
\begin{equation}\label{def:intro_prop_regime}
\begin{minipage}{0.85\textwidth}
the sample size $m$ and the feature dimension $n$ grow proportionally, while the network width $q$ and depth $L$ remain bounded.
\end{minipage}
\end{equation}
Our choice of the regime~(\ref{def:intro_prop_regime}) places us in the finite-width setting, which is of greater practical relevance for many popular neural network architectures (cf.~Table~\ref{tab:model-summary}) compared to the infinite-width regime (L1). As will be clear below, our theory in the regime (\ref{def:intro_prop_regime}) retains non-trivial evolution of gradient descent away from initialization beyond the lazy training regime in the NTK theory (see, e.g., Eqn (\ref{eqn:weight_evolution}) for a precise mathematical description), and depends on concrete realizations of the initialization that go beyond the MF and TP theories as in (L2). Furthermore, despite the outstanding non-convexity of the loss landscape, the regime (\ref{def:intro_prop_regime}) allows for a precise understanding of a large array of statistical properties of the gradient updates $\bm{W}^{(t)}$ in (\ref{def:grad_descent_intro}), including their generalization error (\ref{def:test_error}) in (L3).

Before presenting the details of our framework, we highlight a distinct line of theoretical work~\cite{bordelon2023self,bordelon2024dynamics} that employs techniques from the so-called Dynamic Mean Field Theory (DMFT), originally developed in statistical physics. This framework provides invaluable physics-based insights into the dynamics of gradient descent in certain infinite-width neural network training settings, and some of its predictions have been formally validated in the setting of \emph{layer-wise} gradient descent for two-layer neural networks; see~\cite{celentano2021high}. While this line of work has been successful in capturing certain qualitative trends in generalization behavior, its focus differs from our objective of developing a mathematical theory that applies beyond the two-layer case.

\subsection{Our new theoretical framework}

 Assuming a simple data model (\ref{def:model}) where $X_i$'s have independent, mean-zero and sub-gaussian entries with unit variance, we show in Theorem \ref{thm:se_gd_NN} that the following hold for each iteration $t=1,2,\ldots$ in the finite-width proportional regime (\ref{def:intro_prop_regime}):
\begin{itemize}
	\item  (\emph{First-layer weight $W_1^{(t)}\in \R^{n\times q}$}). There exist linear functions $\big\{\Delta_{\ell}^{(t)}:\R^{q[1:t]}\to \R^q\big\}_{\ell \in [n]}$, (typically) non-linear functions $\big\{\Theta_{k}^{(t+1)}: \R^{q[0:t+1]}\to \R^q\big\}_{k \in [m]}$, and Gaussian vectors $\mathfrak{U}^{(0)},\ldots,\mathfrak{U}^{(t+1)}, \mathfrak{V}^{(1)},\ldots,\mathfrak{V}^{(t)}\in \R^q$ such that the distributional approximations
	\begin{align}\label{eqn:gd_dist_1}
	n^{1/2}W_{1;\ell\cdot}^{(t)} \stackrel{d}{\approx} \Delta_{\ell}^{(t)}(\mathfrak{V}^{([1:t])}),\quad
	\big(X W_1^{(t)}\big)_{k\cdot}\stackrel{d}{\approx} \Theta_{k}^{(t+1)}(\mathfrak{U}^{([0:t+1])})
	\end{align}
	hold in an averaged sense across $k \in [m]$ and $\ell \in [n]$.
	\item (\emph{Deeper-layer weights $\{W_\alpha^{(t)}\in \R^{q\times q}\}_{\alpha \in [2:L]}$}). There exist deterministic matrices $\{V_\alpha^{(t)} \in \R^{q\times q}\}_{\alpha \in [2:L]}$ such that the following concentration holds:
	\begin{align}\label{eqn:gd_dist_2}
	W_\alpha^{(t)} \approx V_\alpha^{(t)},\quad \alpha = 2,\ldots, L.
	\end{align}
\end{itemize}
For any fixed initialization $\bm{W}^{(0)}$, the deterministic functions $\{\Delta_{\ell}^{(t)}\}$, $\{\Theta_{k}^{(t+1)}\}$, the matrices $\{V_\alpha^{(t)}\}$ and the Gaussian laws of $(\mathfrak{U}^{([0:t+1])},\mathfrak{V}^{([1:t])})$ are defined recursively in Definition \ref{def:gd_NN_se} via a theoretical formalism known as \emph{state evolution}, and typically evolve non-trivially for each iteration $t$. As a result, in contrast to the existing theories as described in (L2), the gradient descent training dynamics resulted from (\ref{eqn:gd_dist_1})-(\ref{eqn:gd_dist_2}) typically evolve non-trivially from the initialization, and retain modest dependence on the concrete realization of $\bm{W}^{(0)}$ that is consistent with empirical observations as mentioned above.

Our theory in (\ref{eqn:gd_dist_1})-(\ref{eqn:gd_dist_2}) addresses some limitations in (L3) of the aforementioned frameworks. Specifically, we characterize in Theorem \ref{thm:err_gd_NN}, up to the leading order, both the generalization error $\mathscr{E}_{\texttt{test}}^{(t)}(X,Y)$ in (\ref{def:test_error}) and the training error $\mathscr{E}_{\texttt{train}}^{(t)}(X,Y)$ (to be defined below in Definition \ref{def:training_test_error}) for general multi-layer neural network models (\ref{def:NN_model}) via low-dimensional Gaussian integrals: For each iteration $t=1,2,\ldots$, with high probability, 
\begin{align}\label{eqn:gen_err_char_intro}
\begin{cases}
\mathscr{E}_{\texttt{train}}^{(t)}(X,Y)\approx m^{-1} \sum_{k \in [m]}\E \pnorm{\mathscr{R}_{V^{(t)};k}\big(\Theta_{k}^{(t+1)}(\mathfrak{U}^{([0:t+1])}),\mathfrak{U}^{(0)} \big)}{}^2,\\
\mathscr{E}_{\texttt{test}}^{(t)}(X,Y)\approx m^{-1} \sum_{k \in [m]} \E \pnorm{\mathscr{R}_{V^{(t)};k}\big(\mathfrak{U}^{(t+1)},\mathfrak{U}^{(0)} \big)}{}^2.
\end{cases}
\end{align}
Here $\big\{\mathscr{R}_{\cdot;k}:\R^q\times \R^q\to \R^q\big\}_{k \in [m]}$ are deterministic functions defined non-recursively (cf. Definition \ref{def:mean_field_fcn}). Since $\Theta^{(t)}(\mathfrak{u}^{([0:t])}) \neq \mathfrak{u}^{(t)}$ is generally non-linear, our characterization (\ref{eqn:gen_err_char_intro}) quantifies the non-trivial generalization gap, defined as the difference between the training error and the generalization/test error. This result is useful for describing gradient descent training of neural networks beyond the classical uniform converegence regime, since the generalization gap arises ubiquitously is this regime \cite{bartlett2021deep}. 

While we focus on the single-index regression model (\ref{def:model}) and the vanilla gradient descent iterate (\ref{def:grad_descent_intro}) in this paper, it is straightforward to generalize our theory (\ref{eqn:gd_dist_1})-(\ref{eqn:gd_dist_2}) to the multi-index regression model, and to other gradient descent variants such as stochastic gradient descent with large mini-batches (proportionally to $m$ or $n$) or noisy gradient descent. We omit these formal extensions to avoid overly complicated notation system.

\subsection{A statistical application: Algorithmic estimation of generalization error}

As a statistical application of our theory in (\ref{eqn:gd_dist_1})–(\ref{eqn:gd_dist_2}) and the characterization in (\ref{eqn:gen_err_char_intro}), we show that the vanilla gradient descent algorithm (\ref{def:grad_descent_intro}) can be augmented (cf. Algorithm \ref{alg:aug_gd_nn}) to simultaneously output, at each iteration $t$, (i) the gradient update $\bm{W}^{(t)}$ and (ii) a consistent estimate $\hat{\mathscr{E}}_{\texttt{test}}^{(t)}(X,Y)$ of the (unknown) generalization error $\mathscr{E}_{\texttt{test}}^{(t)}(X,Y)$. The augmented algorithm requires only a small number of closed-form computations of auxiliary statistics beyond standard gradient evaluations and, crucially, requires neither the convergence of gradient descent nor any ad-hoc tuning or prior knowledge of the unknown link function $\varphi_\ast$ or the underlying signal $\mu_\ast$. Since training neural networks is typically computationally expensive and need not be run until convergence, having a reliable, data-driven estimate of the generalization error at each iteration can serve various practical purposes, such as guiding early stopping or tuning hyperparameters effectively.

In the broader literature, the estimation of generalization error via iterative algorithms was initiated in \cite{bellec2024uncertainty,tan2024estimating} in the context of gradient descent over convex losses in linear models with Gaussian data. \cite{han2024gradient} proposes an alternative, unified algorithmic method that applies to a much broader class of models, including linear and logistic regression with non-convex losses and non-Gaussian data. The algorithmic estimation method for the generalization error of neural network gradient descent training in this paper draws conceptual inspiration from this line of work, but extends substantially beyond simple linear/logistic models to address the full complexity of non-convex multi-layer neural networks (\ref{def:NN_model}).

\subsection{A further theoretical implication: Structure of the learned model $f_{\textbf{W}^{(t)}} $}

As a further implication of our new theoretical framework (\ref{eqn:gd_dist_1})-(\ref{eqn:gd_dist_2}), we provide a precise structural understanding for the learned model $f_{\bm{W}^{(t)}}$ beyond the generalization error (\ref{eqn:gen_err_char_intro}). We show in Theorem \ref{thm:represent_learning} that in a proper sense, 
\begin{align}\label{ineq:represent_learning_1}
f_{\bm{W}^{(t)}}(x)\stackrel{d}{\approx} \mathfrak{h}_{V^{(t)}_{[2:L]}}\big(U_{\ast,\texttt{eff}}^{(t),\top} x+ \mathsf{Z}_{\texttt{hd}}^{(t)}\big),\quad x \in \R^n,
\end{align}
where $\mathfrak{h}_{V^{(t)}_{[2:L]}}:\R^q \to \R^q$ is an \emph{effective link function}, $U_{\ast,\texttt{eff}}^{(t)}\in \R^{n\times q}$ is an \emph{effective signal} formed as a linear combination of the true signal $\mu_\ast$ and the initialization $W_1^{(0)}$, and $\mathsf{Z}_{\texttt{hd}}^{(t)}\in \R^q$ is a Gaussian noise due to the high dimensionality in the finite-width proportional regime (\ref{def:intro_prop_regime}). 

The representation (\ref{ineq:represent_learning_1}) sheds new light on the inner working of the feature learning mechanism of gradient descent training. In particular, despite that the single-index regression function $x\mapsto \varphi_\ast(\iprod{x}{\mu_\ast})$ is intrinsically misspecified by the neural network model (\ref{def:NN_model}), the learned model $x\mapsto f_{\bm{W}^{(t)}}(x)$ is qualitatively again a single-index function. Particularly, the single-index function has the prescribed effective link function and depends on the input $x \in \R^n$ only through its linear projection via the effective signal $U_{\ast,\texttt{eff}}^{(t)}$. 
As a result, when regression data comes from the single-index model (\ref{def:model}), the learned model $f_{\bm{W}^{(t)}}(x)$ depends on the unknown signal $\mu_\ast$ only via its linear projection $\iprod{x}{\mu_\ast}$ and its signal strength $\pnorm{\mu_\ast}{}$ (in a proper sense).

From a different perspective, (\ref{ineq:represent_learning_1}) is reminiscent of mean-field characterizations of regularized regression estimators in high-dimensional statistics \cite{bayati2012lasso,donoho2016high,elkaroui2018impact,thrampoulidis2018precise,sur2019modern,miolane2021distribution,celentano2023lasso,han2023noisy,han2023universality}. For instance, in linear regression, regularized estimators are distributed as effective functions applied to signals corrupted by Gaussian noise (cf. Eqn (\ref{eqn:rep_est_linear_model})). Our representation (\ref{ineq:represent_learning_1}) differs from this line of work in two key aspects: (i) it holds algorithmically at each gradient descent iteration, and (ii) it remains valid even under model misspecification of the neural network model with respect to the single-index structure.

\subsection{Proof techniques} \label{sec:proof_techniques}

The proof of our theory (\ref{eqn:gd_dist_1})-(\ref{eqn:gd_dist_2}) employs a matrix-variate version of the non-asymptotic, entrywise distribution theory for general first order methods (GFOMs) recently developed in \cite{han2024entrywise}\footnote{The matrix version of the GFOM theory in \cite{han2024entrywise} along with some other useful results are included in Appendix \ref{section:GFOM_theory} for the convenience of the readers.}. While it is known that gradient descent in simple models can be directly converted into the canonical GFOM form (cf. \cite[Section 5]{han2024entrywise}, and also \cite{celentano2020estimation,celentano2021high,gerbelot2024rigorous}), a major technical difficulty for (\ref{def:grad_descent_intro}) lies in the fact that weights at deeper layers $\{W_\alpha^{(t)}\}_{\alpha \in [2:L]}$ depend on $\{XW_1^{(s)}\}_{s \in [1:t-1]}$ in a highly non-separable way, and therefore (\ref{def:grad_descent_intro}) cannot be converted to the canonical GFOM directly. 

Under the finite-width proportional regime (\ref{def:intro_prop_regime}), we develop an \emph{iterative reduction scheme} to address this technical challenge. At a high level, we show that the gradient descent (\ref{def:grad_descent_intro}) are, in a suitable sense, close to a sequence of successively constructed GFOM, referred to as the \emph{auxiliary GFOM} in our proof. The construction of the auxiliary GFOM at a given iteration relies on the precise distributional characterizations of all preceding iterates obtained via state evolution. In other words, the iterative reduction scheme alternates between applying the GFOM theory and constructing the next auxiliary GFOM iterate. A technical overview of this scheme is provided in Section \ref{section:proof_sketch}.

Executing this iterative reduction scheme requires leveraging the GFOM theory developed in \cite{han2024entrywise} in its full strength with entrywise, non-asymptotic guarantees under heterogeneous random matrix ensembles. Specifically, the auxiliary GFOM is built using a heterogeneous data matrix; controlling the error between the auxiliary GFOM and the gradient descent necessitates apriori delocalization/$\ell_\infty$ estimates, which in turn require GFOM error bounds that are non-asymptotic and polynomial in $n$. Moreover, the concentration estimate (\ref{eqn:gd_dist_2}) for weights in deeper layers requires an entrywise GFOM theory.

\subsection{Further related literature}\label{subsection:further_literature}

As the literature is broad, here we only briefly review two closely related aspects.

\subsubsection{Other precise characterizations of the generalization error}

Precise characterizations of the generalization error (\ref{def:test_error}) for (variants of) gradient descent beyond the NTK regime remain limited mostly to wide two-layer networks. \cite{ba2023learning} analyzed generalization after a single gradient step, showing potential improvements over kernel methods; see also \cite{moniri2023theory,dandi2024random} for extensions to larger or maximal learning rates (in the sense of \cite{yang2020feature} as $q \to \infty$). In another line of work, \cite{berthier2024learning} characterized the incremental learning phenomena in generalization with different time scales under certain gradient flow dynamics, while \cite{montanari2025dynamical} identified separate time scales where non-monotonic generalization may emerge.

Our result (\ref{eqn:gen_err_char_intro}) characterizes the generalization error in the finite-width proportional regime (\ref{def:intro_prop_regime}) that differs significantly from the above works, and applies to networks of arbitrary depth. Understanding qualitatively how the behavior of the right-hand side of (\ref{eqn:gen_err_char_intro}) varies with iteration $t$, width $q$, and depth $L$ along the lines of the above works is an interesting direction for future research.

\subsubsection{Sample complexity for learning single-index functions}
For the single-index data model (\ref{def:model}), a significant line of recent work focuses on the sample complexity theory for (stochastic) gradient training of learning $\varphi_\ast$ in the two-layer neural network model. In this line of works, the complexity of $\varphi_\ast$ is typically quantified by the so-called \emph{information exponent} \cite{benarous2021online} $\texttt{IE}(\varphi_\ast)$ defined as the order of the first non-trivial Hermite coefficient of $\varphi_\ast$, whereas the sample complexity is usually measured via the estimation error of the learned model $f_{\bm{W}^{(t)}}$ \cite{abbe2022merged,abbe2023sgd,damian2022neural,damian2023smoothing,ba2023learning,mahankali2023beyond,lee2024neural} and/or the correlation of $\bm{W}^{(t)}$ with the signal $\mu_\ast$ \cite{bietti2022learning,mousavi2023neural,mousavi2023gradient,arnaboldi2024repetita,dandi2024how}. The main thrust of this line of work brings down the sample complexity of variants of gradient descent (in respective senses) near the information-theoretic limit $\bigo(n)$ as opposed to the earlier belief $\bigo(n^{\Theta(\texttt{IE}(\varphi_\ast))})$. The readers are also referred to \cite{nichani2023provable,ren2023depth} and references therein for related results in certain specialized settings beyond the two-layer neural networks. 

While our theory (\ref{eqn:gd_dist_1})-(\ref{eqn:gd_dist_2}) does not directly address the sample complexity of the gradient descent iterate (\ref{def:grad_descent_intro}), we believe that further analysis of the representation (\ref{ineq:represent_learning_1}) may yield useful insights in this direction; interested readers are referred to the discussion following Theorem \ref{thm:represent_learning_large_sample} for more details.

\subsection{Organization}
The rest of the paper is organized as follows. Section \ref{section:gradient_mapping} presents the gradient formulae and associated `theoretical functions'. Our main theory (\ref{eqn:gd_dist_1})–(\ref{eqn:gd_dist_2}) is detailed in Section \ref{section:gd_precise_dynamics}. Section \ref{section:gd_inference} formalizes the generalization error characterization (\ref{eqn:gen_err_char_intro}) and describes the algorithmic estimation method. Section \ref{section:gaussian_representation_f_W} presents the rigorous version of the structural representation (\ref{ineq:represent_learning_1}). Numerical results for the proposed algorithmic estimation method are reported in Section \ref{section:simulation}. A proof sketch appears in Section \ref{section:proof_sketch}, with full details in Sections \ref{section:proof_preliminaries}-\ref{section:remaining_proof}. Appendix \ref{section:notation} summarizes notation; Appendix \ref{section:GFOM_theory} presents the matrix-variate GFOM theory and auxiliary results; Appendix \ref{section:simulation_additional} contains additional simulations.

\section{Preliminaries: Gradient mappings}\label{section:gradient_mapping}

Fix the number of layer $L \in \mathbb{Z}_{\geq 2}$ and network width $q \in \N$. The first layer weight matrix is $W_1 \in \R^{n\times q}$, while for $\alpha \in [2:L]$, the $\alpha$-th layer weight matrix is $W_\alpha \in \R^{q\times q} =\mathbb{M}_q$. We slightly alter the definition of the network parameter $\bm{W}$ for notational convenience, hereafter defined as
\begin{align}\label{eqn:homo_nn_def_1}
\bm{W}\equiv \big(W_1^\top,\ldots,W_L^\top\big)^\top \in \R^{n_{L,q}\times q},\quad n_{L,q}\equiv n+(L-1)q.
\end{align}
The neural network model (\ref{eqn:homo_nn_def_1}) accommodates networks with varying widths $\{W_\alpha \in \R^{q_{\alpha-1}\times q_\alpha}\}_{\alpha \in [1:L]}$ where $q_0=n$ and typically $q_L=1$. In this case, we may identify the neural network with the same width $q\equiv \pnorm{q_{[1:L]}}{\infty}$ by $W_1 \in \R^{n\times q_1}\to 
\big(W_1\,\, 0\big)
\in \R^{n\times q}$ and $W_\alpha \in \R^{q_{\alpha-1}\times q_\alpha} \to 
\Big(\begin{smallmatrix} W_\alpha & 0 \\ 0 & 0 \end{smallmatrix}\Big)  \in \R^{q\times q}$ for $\alpha \in [2:L]$. In the sequel, we shall therefore focus on the network parametrization (\ref{eqn:homo_nn_def_1}).

For notational compatibility, we write $(Y_i)_{[q]}\equiv \big(Y_i\,|\,0_{q-1}^\top\big)^\top \in \R^q$,  $(\xi_i)_{[q]}\equiv \big(\xi_i\,|\,0_{q-1}^\top\big)^\top \in \R^q$ and let the matrices $\mu_{\ast,[q]} \in \R^{n\times q}$ and $Y_{[q]}, \xi_{[q]}\in \R^{m\times q}$ be
\begin{align}\label{eqn:homo_nn_def_2}
\mu_{\ast,[q]}\equiv \big(\mu_\ast\,|\, 0_{n\times (q-1)}\big),\,
Y_{[q]}\equiv \big(Y\,|\, 0_{m\times (q-1)}\big),\,
\xi_{[q]}\equiv  \big(\xi\,|\, 0_{m\times (q-1)}\big).
\end{align}
\subsection{Algorithmic gradient mappings}

We consider a slightly more general class of multi-layer neural networks than (\ref{def:NN_model}) which allows for possibly different activation functions across layers.

\begin{definition}\label{def:grad_descent_fcn}
Fix a sequence of functions $\bm{\sigma}\equiv \{\sigma_\alpha:\R \to \R\}_{\alpha \in [1:L]}$ with $\sigma_0\equiv \sigma_L\equiv \mathrm{id}$.
\begin{enumerate}
	\item (\emph{Feed-forward mappings (vector version)}). Let $h_{\bm{W};0}\equiv \mathrm{id}(\R^n)$ and $\{h_{\bm{W};\alpha} : \R^n\to \R^q\}_{\alpha \in [1:L]}$ be defined recursively via 
	\begin{align*}
	h_{\bm{W};\alpha}(x)\equiv W_\alpha^\top \sigma_{\alpha-1}\big(h_{\bm{W};\alpha-1}(x)\big),\quad \alpha \in [1:L].
	\end{align*}
    The \emph{neural network function} $f_{\bm{W}}:\R^n\to \R^q$ is defined as $
    f_{\bm{W}}(x)\equiv h_{\bm{W};L}(x)$. 
	\item (\emph{Feed-forward mappings}). Let $H_{\bm{W};0}\equiv \mathrm{id}(\R^{m\times n})$ and $\big\{H_{\bm{W};\alpha}: \R^{m\times n} \to \R^{m\times q}\big\}_{\alpha \in [1:L]}$ be defined recursively via
	\begin{align*}
	H_{\bm{W};\alpha}(\cdot)\equiv \sigma_{\alpha-1}\big(H_{\bm{W};\alpha-1}(\cdot)\big)W_\alpha,\quad \alpha \in [1:L].
	\end{align*} 
	For $\alpha \in [0:L]$, let $
	G_{\bm{W};\alpha}(\cdot)\equiv \sigma_\alpha\big(H_{\bm{W};\alpha}(\cdot)\big), G_{\bm{W};\alpha}'(\cdot)\equiv \sigma_\alpha'\big(H_{\bm{W};\alpha}(\cdot)\big)$.
	\item (\emph{Layer-wise gradient mappings}). With $X \in \R^{m\times n}$, let $\big\{\mathsf{P}_{X,\bm{W};\alpha}: \R^{m \times q}\to \R^{m \times q}\big\}_{\alpha \in [2:L]}$ be defined by
	\begin{align*}
	\mathsf{P}_{X,\bm{W};\alpha}(z)\equiv \big(z\odot G_{\bm{W};\alpha}'(X)\big) W_\alpha^\top.
	\end{align*}
\end{enumerate}
\end{definition}
In the above definition, $\sigma_\alpha(\cdot)$ acts entrywise when applied to vectors/matrices. With the neural network models in Definition \ref{def:grad_descent_fcn}, the empirical squared loss (\ref{def:ERM_squared_loss}) can be generalized as
\begin{align}\label{def:loss_fcn_general}
\mathsf{L}(\bm{W})\equiv \frac{1}{2m}\pnorm{Y_{[q]}- G_{\bm{W};L}(X) }{}^2.
\end{align}
We consider below a slightly more general gradient descent algorithm than (\ref{def:grad_descent_intro}) with possible different learning rates $\{\eta_\alpha^{(t)}: \alpha \in [1:L], t\geq 1\}\subset \R_{>0}$: initialized with $\bm{W}^{(0)}$, for $t=1,2,\ldots$, we iteratively compute
\begin{align}\label{def:grad_descent_general}
W_\alpha^{(t)}\equiv W_\alpha^{(t-1)}- \eta_\alpha^{(t-1)}\cdot  \nabla_{W_\alpha^{(t-1)}} \mathsf{L}(\bm{W}^{(t-1)}),\quad \alpha \in [1:L].
\end{align}
The following proposition provides a formula for the gradients $\nabla_{W_\alpha} \mathsf{L}$; its proof can be found in Section \ref{subsection:proof_grad_formula}.
\begin{proposition}\label{prop:grad_formula}
The following gradient formula holds for $\alpha \in [1:L]$:
\begin{align*}
\nabla_{W_\alpha} \mathsf{L}(\bm{W}) &= \frac{1}{m}\cdot  G_{\bm{W};\alpha-1}(X)^\top \big[\mathsf{P}_{X,\bm{W}}^{(\alpha:L]}\big(G_{\bm{W};L}(X)-Y_{[q]}\big)\odot G_{\bm{W};\alpha}'(X)\big].
\end{align*}
Here $\mathsf{P}_{X,\bm{W}}^{(\alpha:L]}:\R^{m\times q}\to \R^{m\times q}$ is a \emph{back-propagation mapping} defined via
\begin{align*}
\mathsf{P}_{X,\bm{W}}^{(\alpha:L]}\equiv 
\begin{cases}
\mathsf{P}_{X,\bm{W};\alpha+1}\circ \cdots \circ \mathsf{P}_{X,\bm{W};L}, & \alpha \in [1:L-1];\\
\mathrm{id}_{\R^{m\times q}}, & \alpha = L.
\end{cases}
\end{align*}
\end{proposition}
Note that here $\nabla_{W_1} \mathsf{L}(\bm{W})\in \R^{n\times q}$ and $\{\nabla_{W_\alpha} \mathsf{L}(\bm{W})\}_{\alpha \in [2:L]} \subset \R^{q\times q}$. Consequently, with initialization $\bm{W}^{(0)}$, and suppose that $\bm{W}^{(1)},\ldots,\bm{W}^{(t-1)}$ have been computed, the gradient descent update (\ref{def:grad_descent_general}) reads as follows: for $\alpha \in [1:L]$,
\begin{align*}
W_\alpha^{(t)}= W_\alpha^{(t-1)}-\eta_\alpha^{(t-1)}\cdot  G_{\bm{W}^{(t-1)};\alpha-1}(X)^\top \big[\mathsf{P}_{X,\bm{W}^{(t-1)} }^{(\alpha:L]}\big(G_{\bm{W}^{(t-1)};L}(X)-Y_{[q]}\big)\odot G_{\bm{W}^{(t-1)};\alpha}'(X)\big].
\end{align*}
Moreover, $\nabla_{W_1} \mathsf{L}(\bm{W})$ takes a simpler form:
\begin{align*}
\nabla_{W_1} \mathsf{L}(\bm{W}) &= \frac{1}{m}\cdot  X^\top \big[\mathsf{P}_{X,\bm{W}}^{(1:L]}\big(G_{\bm{W};L}(X)-Y_{[q]}\big)\odot \sigma_1'(XW_1)\big] \in \R^{n\times q}.
\end{align*}

\subsection{Theoretical gradient mappings}

Now we shall provide `theoretical' analogues of the functions in Definition \ref{def:grad_descent_fcn}. Roughly speaking, these theoretical functions take $XW_1 \in \R^{m\times q}$ rather than $X \in \R^{m\times n}$ as the input.

\begin{definition}\label{def:mean_field_fcn}
Fix $v\equiv (v_2,\ldots,v_L) \in (\mathbb{M}_q)^{[2:L]}$.
\begin{enumerate}
	\item (\emph{Theoretical feed-forward mappings}). Let $\big\{\mathscr{H}_{v_{[2:\alpha]}}: \R^{m\times q}\to \R^{m\times q}\big\}_{\alpha \in [1:L]}$ be defined recursively as follows:
	\begin{enumerate}
		\item Initialize with $\mathscr{H}_{v_{[2:1]}}\equiv \mathrm{id}(\R^{m\times q})$.
		\item For $\alpha \in [2:L]$, let $\mathscr{H}_{v_{[2:\alpha]}}(u)\equiv \sigma_{\alpha-1}\big(\mathscr{H}_{v_{[2:\alpha-1]}}(u)\big) v_\alpha$.
	\end{enumerate}
	For $\alpha \in [1:L]$, let $\mathscr{G}_{v_{[2:\alpha]}}(\cdot)\equiv \sigma_\alpha\big(\mathscr{H}_{v_{[2:\alpha]}}(\cdot)\big)$, $\mathscr{G}_{v_{[2:\alpha]} }'(\cdot)\equiv \sigma_\alpha'\big(\mathscr{H}_{v_{[2:\alpha]}}(\cdot)\big)$.
	\item (\emph{Theoretical layer-wise gradient mappings}). For $\alpha\in [2:L]$ and $u \in \R^{m\times q}$, let $\mathscr{P}_{u,v_{[2:\alpha]} }:  \R^{m \times q}\to \R^{m \times q }$ be defined by
	\begin{align*}
	\mathscr{P}_{u, v_{[2:\alpha]} }(z)\equiv \big(z\odot \mathscr{G}_{v_{[2:\alpha]} }'(u)\big) v_\alpha^\top.
	\end{align*}
	\item (\emph{Theoretical back-propagation mappings}). For $\alpha \in [1:L]$, let $\mathscr{P}_{v}^{(\alpha:L]}:\R^{m\times q}\times \R^{m\times q}\to \R^{m\times q}$ be defined via
	\begin{align*}
	\mathscr{P}_{v}^{(\alpha:L]}(u,z)\equiv 
	\begin{cases}
	\big[\odot_{\beta \in (\alpha:L]}\mathscr{P}_{u, v_{[2:\beta]}}\big](z), & \alpha \in [1:L-1];\\
	z, & \alpha = L.
	\end{cases}
	\end{align*}
	\item (\emph{Theoretical residual mapping}). Let $\mathscr{R}_{v}:\R^{m\times q}\times \R^{m\times q}\to \R^{m\times q}$ be 
	\begin{align*}
	\mathscr{R}_{v}(u,w)\equiv \mathscr{G}_{v_{[2:L]}}(u)-\varphi_\ast(w)-\xi_{[q]}.
	\end{align*}
	Here $\varphi_\ast$ is understood as applied entrywise.
\end{enumerate} 
\end{definition}

Comparing the functions in Definitions \ref{def:grad_descent_fcn} and \ref{def:mean_field_fcn}, it is easy to see that given a network parameter $\bm{W}=\big(W_1^\top,\ldots,W_L^\top\big)^\top\in \R^{n_{L,q}\times q}$ as in (\ref{eqn:homo_nn_def_1}), with input data matrix $X \in \R^{m\times n}$, the following correspondence holds:
\begin{table}[H]
	\centering
	\begin{tabular}{c||c|c|c}
		\textbf{Alg. grad. map.} & $H_{\bm{W};\alpha}(X)$ & $\mathsf{P}_{X,\bm{W}}^{(\alpha:L]}(z)$ &  $G_{\bm{W};L}(X)-Y_{[q]}$  \\ \hline
		\textbf{Theo. grad. map.} & $\mathscr{H}_{\bm{W}_{[2:\alpha]} }(XW_1)$ & $\mathscr{P}_{\bm{W}_{[2:L]}}^{(\alpha:L]}(XW_1,z)$ &   $\mathscr{R}_{\bm{W}_{[2:L]}}\big(XW_1,X\mu_{\ast,[q]}\big)$ 
	\end{tabular}
\end{table}
The pair $(H,\mathscr{H})$ in the table above can be replaced by either $(G,\mathscr{G})$ or $(G',\mathscr{G}')$. Moreover, the loss function~(\ref{def:loss_fcn_general}) can alternatively be written as
\begin{align}\label{eqn:loss_mean_field}
\mathsf{L}(\bm{W}) = \frac{1}{2m} \pnorm{\mathscr{R}_{\bm{W}_{[2:L]}}\big(XW_1, X\mu_{\ast,[q]}\big)}{}^2.
\end{align}
Recall the pre-gradients $\big\{\mathsf{P}_{X,\bm{W}}^{(\alpha:L]}\big(G_{\bm{W};L}(X) - Y_{[q]}\big) \odot G_{\bm{W};\alpha}'(X)\big\}_{\alpha \in [1:L]}$ appearing in Proposition~\ref{prop:grad_formula}. Below we introduce a theoretical version of these pre-gradients.

\begin{definition}\label{def:mean_field_S_T}
With the notation used above, let the \emph{theoretical pre-gradient maps} $\mathfrak{S}, \{\mathfrak{T}_{\alpha}\}_{\alpha \in [2:L]}: \R^{m\times q}\times \R^{m\times q}\times (\mathbb{M}_q)^{[2:L]}\to \R^{m\times q}$ be defined as
\begin{align*}
\mathfrak{S}\big(u,w, v\big)&\equiv \mathscr{P}_{v}^{(1:L]}\big(u,\mathscr{R}_{v}(u,w)\big)\odot \sigma_1'(u),\\
\mathfrak{T}_{\alpha}\big(u,w, v\big)&\equiv \mathscr{P}_{v}^{(\alpha:L]}\big(u,\mathscr{R}_{v}(u,w)\big)\odot \mathscr{G}_{v_{[2:\alpha]}}'(u),\quad \alpha \in [2:L].
\end{align*}
\end{definition}
Using the theoretical functions in Definitions \ref{def:mean_field_fcn} and \ref{def:mean_field_S_T} along with the gradient formula in Proposition \ref{prop:grad_formula}, we may rewrite $\{\nabla_{W_\alpha} \mathsf{L}(\bm{W}) \}$ as
\begin{align}\label{eqn:grad_mean_field_1}
\nabla_{W_\alpha} \mathsf{L}(\bm{W}) = \frac{1}{m}\cdot
 \begin{cases}
  X^\top \mathfrak{S}\big(XW_1, X\mu_{\ast,[q]}, \bm{W}_{[2:L]}\big),& \alpha =1;\\
  \mathscr{G}_{\bm{W}_{[2:\alpha-1]} }(XW_1)^\top\mathfrak{T}_{\alpha}\big(XW_1, X\mu_{\ast,[q]}, \bm{W}_{[2:L]}\big), & \alpha \in [2:L].
 \end{cases}
\end{align}
Note that both (\ref{eqn:loss_mean_field}) and (\ref{eqn:grad_mean_field_1}) depend on the unknown signal $\mu_\ast$, and thus are intended primarily for theoretical analysis.

\begin{remark}
	Two notational conventions will be adopted:
	\begin{enumerate}
		\item For any $k \in [m]$, as $\mathscr{H}_{v_{[2:\alpha]};k}(u)$, $\mathscr{G}_{v_{[2:\alpha]};k}(u)$, $\mathscr{G}'_{v_{[2:\alpha]};k}(u)$,  $\mathscr{R}_{v;k}(u,w)$ and $\mathscr{P}_{v;k}^{(\alpha:L]}(u,z)$ in Definition \ref{def:mean_field_fcn} depend on $u,w,z \in \R^{m\times q}$ only through their corresponding $k$-th rows $u_{k\cdot},w_{k\cdot},z_{k\cdot}\in \R^q$, we may identify $\mathscr{H}_{v_{[2:\alpha]};k}, \mathscr{G}_{v_{[2:\alpha]};k}, \mathscr{G}'_{v_{[2:\alpha]};k}: \R^q\to \R^q$ and $\mathscr{R}_{v;k}, \mathscr{P}_{v;k}^{(\alpha:L]}:\R^q\times \R^q\to \R^q$ as
		\begin{align}\label{eqn:row_sep_mean_field_fcn}
		\#_{v_{[2:\alpha]};k}(u) &= \#_{v_{[2:\alpha]};k}(u_{k\cdot}),\quad \# \in \{\mathscr{H},\mathscr{G},\mathscr{G}'\},  \nonumber\\
		\mathscr{R}_{v;k}(u,w) &=\mathscr{R}_{v;k}(u_{k\cdot},w_{k\cdot}),\,\mathscr{P}_{v;k}^{(\alpha:L]}(u,z)=\mathscr{P}_{v;k}^{(\alpha:L]}(u_{k\cdot},z_{k\cdot}).
		\end{align}
		
		By (\ref{eqn:row_sep_mean_field_fcn}), for any $k \in [m]$,  the function $\mathfrak{S}_{k}\big(u,w, v\big)$ in Definition \ref{def:mean_field_S_T} must depend on $u,w \in \R^{m\times q}$ only through their corresponding $k$-th rows $u_{k\cdot},w_{k\cdot}\in \R^q$ as well. Consequently, we may identify $\mathfrak{S}_{k}: \R^q\times \R^q\times (\mathbb{M}_q)^{[2:L]}\to \R^{q}$ as
		\begin{align*}
		\mathfrak{S}_{k}\big(u,w, v\big)\equiv \mathscr{P}_{v;k}^{(1:L]}\big(u_{k\cdot},\mathscr{R}_{v;k}(u_{k\cdot},w_{k\cdot})\big)\odot \sigma_1'(u_{k\cdot}).
		\end{align*}
		A similar notational convention applies to $(\mathfrak{T}_{\alpha})_k$.
		\item For derivatives of the theoretical functions, we will write
		\begin{align*}
		&\partial_{u_{k\ell}} \mathscr{H}_{v_{[2:\alpha]}}(u')\equiv ({\partial}/{\partial u_{k\ell}}) \mathscr{H}_{v_{[2:\alpha]}}(u)|_{u=u'},\\
		&\partial_{\#} \mathscr{P}_{v}^{(\alpha:L]}(u',z')\equiv ({\partial}/{\partial \#}) \mathscr{P}_{v}^{(\alpha:L]}(u,z) \big|_{(u,z)=(u',z')},\quad \# \in \{u_{k\ell},z_{k\ell}\}_{k\in [m], \ell \in [q]},\\
		&\partial_{\#} \mathscr{R}_{v}(u',w')  \equiv ({\partial}/{\partial \#}) \mathscr{R}_{v}(u,w) \big|_{(u,w)=(u',w')},\quad \# \in \{u_{k\ell},w_{k\ell}\}_{k \in [m], \ell \in [q]},\\
		&\partial_{\#} \mathfrak{S}\big(u',w', v'\big)\equiv ({\partial}/{\partial \#}) \mathfrak{S}\big(u,w,v\big) \big|_{(u,w,v)=(u',w',v')},\quad \# \in \{u_{k\ell},w_{k\ell}\}_{k \in [m], \ell \in [q]}.
		\end{align*}
		A similar notational convention applies to $\partial_{u_{k\ell}} \mathscr{G}_{v_{[2:\alpha]}}(u')$ and  $\partial_\#\mathfrak{T}_{\alpha}(u',w',v')$. We sometimes write $\partial_{u_{\ell}} \mathscr{H}_{v_{[2:\alpha]}}(u')=\partial_{u_{k\ell}} \mathscr{H}_{v_{[2:\alpha]};k}(u')$ for notational convenience (and similarly for other theoretical functions). 
	\end{enumerate}
\end{remark}

\section{Meta theory: Precise gradient descent training dynamics}\label{section:gd_precise_dynamics}

\subsection{Assumptions and some further notation}

We shall work with the following assumptions in the sequel unless otherwise specified.

\begin{assumption}\label{assump:gd_NN}
	Suppose that the following hold for some $K,\Lambda\geq 2$ and $r_0\geq 0$.
	\begin{enumerate}
		\item[(A1)] The aspect ratio $\phi\equiv m/n$ satisfies $
		1/K\leq \phi\leq K$.
		\item[(A2)] The random data  matrix $X \in \R^{m\times n}$ has independent mean zero, variance $1$ entries with sub-gaussian tails satisfying\footnote{Here $\pnorm{\cdot}{\psi_2}$ is the standard Orlicz-2/subgaussian norm; see, e.g., \cite[Section 2.1]{van1996weak} for a precise definition.} $
		\max_{i\in [m], j\in [n]} \pnorm{X_{ij}}{\psi_2}\leq K$.
		\item[(A3)] The learning rates satisfy $
		\max_{s \in \mathbb{Z}_{\geq 0}}\max_{\alpha \in [1:L]}\eta_\alpha^{(s)}\leq \Lambda$.
		\item[(A4)] The activation functions $\{\sigma_\alpha\}_{\alpha \in [1:L]}\subset C^4(\R)$ and satisfy 
		\begin{align*}
		\max\limits_{p\in [4]} \max\limits_{\alpha \in [1:L]} \big(\pnorm{\sigma_\alpha^{(p)}}{\infty}+\abs{\sigma_\alpha(0)}\big)\leq \Lambda.
		\end{align*}
		\item[(A5)] The regression function $\varphi_\ast \in C^3(\R)$ and there exists some $r_0>0$ such that 
		\begin{align*}
		\max_{p \in [0:3]} \sup_{x \in \R} \big\{ {\abs{\varphi_\ast^{(p)}(x)}}/{(1+\abs{x})^{r_0}}\big\}\leq \Lambda.
		\end{align*}
	\end{enumerate}
\end{assumption}

\begin{remark}
Some remarks on conditions (A1)-(A5):
\begin{enumerate}
	\item (A1) formalizes the finite-width proportional regime (\ref{def:intro_prop_regime}), and is related to the \emph{proportional regime} in mean-field high-dimensional statistics (literature cited in the Introduction).
	
	\item (A2) describes the data model. Note that we do \emph{not} assume Gaussian features, so our results apply universally to data matrices satisfying (A2).
	
	\item (A3) requires the learning rates in gradient descent (\ref{def:grad_descent_general}) to be of order $\bigo(1)$, which ensures non-trivial updates of $\bm{W}^{(t)}$ from initialization. We note that different learning rate scalings give rise to NTK or MF theories in the infinite-width regime \cite{yang2020feature}, whereas such distinctions disappear in the finite-width proportional regime (\ref{def:intro_prop_regime}).
	
	\item (A4) assumes smooth activation functions $\{\sigma_\alpha\}_{\alpha \in [1:L]}$. While this excludes ReLU $\sigma(x)=(x)_+$, we expect that some of our results extend to ReLU with additional technical work.
	
	\item (A5) imposes mild growth and smoothness conditions on the link function $\varphi_\ast$ and its derivatives (up to third order). As with (A4), our results likely hold under weaker assumptions with further technical work.
\end{enumerate}
\end{remark}

With initialization $\bm{W}^{(0)} \in \R^{n_{L,q}\times q}$ that is independent of the data matrix $X \in \R^{m\times n}$ (which can also be deterministic), we define
\begin{align}\label{def:kappa_ast}
\varkappa_{\ast}\equiv 1+\pnorm{n^{1/2}\mu_\ast}{\infty}+\pnorm{\xi}{\infty}+\pnorm{n^{1/2}W_1^{(0)}}{\infty}+\max_{\alpha \in [2:L]}\pnorm{W^{(0)}_\alpha}{\op}.
\end{align}
Our theory below allows $\varkappa_{\ast}$ to scale slowly with $n$, for instance $\varkappa_{\ast}\lesssim n^\epsilon$ for some small enough $\epsilon>0$. This in particular applies to the Gaussian initialization where $n^{1/2} W_1^{(0)} \in \R^{n\times q}$ and $\{q^{1/2}W_\alpha^{(0)}\}_{\alpha \in [2:L]} \subset \R^{q\times q}$ all have i.i.d. $\mathcal{N}(0,1)$ entries. 
\subsection{State evolution}

We adopt the following notational conventions:
\begin{itemize}
	\item Let $\pi_m$ (resp. $\pi_n$) denote the uniform distribution on $[1:m]$ (resp. $[1:n]$), independent of all other variables.
	\item The expectation $\E^{(0)}[\cdot]\equiv \E[\cdot\,|\,\mu_\ast,\bm{W}^{(0)},\xi]$ is taken with respect to all sources of randomness except for the true signal $\mu_\ast$, the initialization $\bm{W}^{(0)}$ and the error term $\xi$.
\end{itemize} 

Now we may introduce the state evolution.
\begin{definition}\label{def:gd_NN_se}
Recall $\phi=m/n$. Initialize with:
\begin{itemize}
	\item $V^{(0)}=(V_{\alpha}^{(0)})_{\alpha \in [2:L]}\equiv \big(W_\alpha^{(0)}\big)_{\alpha \in [2:L]}\in (\mathbb{M}_q)^{[2:L]}$,
	\item $\Omega_{\cdot,0}=\Omega_{\cdot,-1}\equiv 0_{q\times q}$,
	\item $\mathfrak{U}^{(0)}\equiv \big(\sigma_{\mu_\ast}\mathsf{Z}\,|\,0_{1\times (q-1)}\big)^\top \in \R^q$ where $\sigma_{\mu_\ast}^2\equiv \pnorm{\mu_\ast}{}^2$ and $\mathsf{Z}\sim \mathcal{N}(0,1)$,
	\item $D_{-1}\equiv n^{1/2}\mu_{\ast,[q]}\in \R^{n\times q}$ and $D_0\equiv n^{1/2} W_1^{(0)} \in \R^{n\times q}$.
\end{itemize}
For $t=1,2,\ldots$, we compute recursively as follows:  
\begin{enumerate}
	\item[(S1)] Let $\Theta^{(t)}:\R^{m\times q[0:t]}\to \R^{m\times q}$ be defined via
	\begin{align*}
	\Theta^{(t)}(\mathfrak{u}^{([0:t])}) &\equiv \mathfrak{u}^{(t)}-\phi^{-1} \sum_{s \in [1:t-1]} \eta^{(s-1)}_1\cdot \mathfrak{S}\big(\Theta^{(s)}(\mathfrak{u}^{([0:s])}) ,\mathfrak{u}^{(0)},V^{(s-1)}\big)\rho_{t-1,s}^\top.
	\end{align*}
	We further write $
	\Upsilon^{(t)}(\mathfrak{u}^{([0:t])})\equiv \phi^{-1}\eta_1^{(t-1)} \mathfrak{S}\big(\Theta^{(t)}(\mathfrak{u}^{([0:t])}) ,\mathfrak{u}^{(0)},V^{(t-1)}\big)$.
	\item[(S2)] Let the Gaussian law of $\mathfrak{U}^{(t)} \in \R^q$ be determined via the following correlation specification: for $s \in [0:t]$, 
	\begin{align*}
	\mathrm{Cov}(\mathfrak{U}^{(t)}, \mathfrak{U}^{(s)})\equiv  \Omega_{t-1,s-1}+ n^{-1}\cdot D_{t-1}^\top  D_{s-1 } \in \mathbb{M}_q.
	\end{align*}
	\item[(S3)]  Compute $\{\tau_{t,s}\}, \{\rho_{t,s}\}, \{\Sigma_{t,s}\}, \{\Omega_{t,s}\}\subset \mathbb{M}_q$ as follows: for $s \in [1:t]$,
	\begin{align*}
	\tau_{t,s} &\equiv -  \phi \cdot \E^{(0)} \nabla_{\mathfrak{U}^{(s)}} \Upsilon_{\pi_m}^{(t)}(\mathfrak{U}^{([0:t])})\in \mathbb{M}_q,\\
	\rho_{t,s} &\equiv I_q \bm{1}_{s=t}+\sum\limits_{r \in [s+1:t]}  \big(\tau_{t,r}+I_q\bm{1}_{r=t}\big) \rho_{r-1,s}\in \mathbb{M}_q,\\
	\Sigma_{t,s} &\equiv \phi\cdot \E^{(0)}   \Upsilon_{\pi_m}^{(t)}(\mathfrak{U}^{([0:t])}) \Upsilon_{\pi_m}^{(s)}(\mathfrak{U}^{([0:s])})^\top \in \mathbb{M}_q,\\
	\Omega_{t,s} &\equiv  \sum\limits_{r\in [1:t],r' \in [1:s]} \rho_{t,r} \Sigma_{r,r'} \rho_{s,r'}^\top\in \mathbb{M}_q.
	\end{align*}
	Compute $\delta_t \in \mathbb{M}_q$ as follows:
	\begin{align*}
	\delta_t\equiv -  \phi \cdot \E^{(0)} \nabla_{\mathfrak{U}^{(0)}} \Upsilon_{\pi_m}^{(t)}(\mathfrak{U}^{([0:t])})\in \mathbb{M}_q.
	\end{align*}
	\item[(S4)] Compute $D_t \in \R^{n\times q}$ as follows:
	\begin{align*}
	D_t&\equiv D_{-1}\delta_t^\top+\sum_{r \in [1:t]} D_{r-1} \big(\tau_{t,r}^\top+I_q\bm{1}_{r=t}\big)\in \R^{n\times q}.
	\end{align*}
	\item[(S5)] Compute $V^{(t)}  = (V_{\alpha}^{(t)})_{\alpha \in [2:L]}\in (\mathbb{M}_q)^{[2:L]}$ as follows: for $\alpha \in [2:L]$,
	\begin{align*}
	V_{\alpha}^{(t)}&\equiv V_\alpha^{(t-1)}-\eta_\alpha^{(t-1)} \cdot \E^{(0)}\mathscr{G}_{V^{(t-1)}_{[2:\alpha-1]};\pi_m }\big(\Theta_{\pi_m}^{(t)}(\mathfrak{U}^{([0:t])}) \big) \\
	&\qquad\qquad \times  \big(\mathfrak{T}_{\alpha}\big)_{\pi_m}\big(\Theta_{\pi_m}^{(t)}(\mathfrak{U}^{([0:t])}) ,\mathfrak{U}^{(0)},V^{(t-1)}\big)^\top \in \mathbb{M}_q.
	\end{align*}
\end{enumerate}
\end{definition}
With the state evolution parameters above, we further define the following:
\begin{itemize}
	\item Let $(\mathfrak{V}^{(1)},\ldots,\mathfrak{V}^{(t)})\in (\R^q)^{[1:t]}$ be a centered Gaussian matrix whose law is determined by the following correlation specification: for $r,s \in [1:t]$,
	\begin{align*}
	\cov(\mathfrak{V}^{(r)},\mathfrak{V}^{(s)})\equiv\Sigma_{r,s} \in \mathbb{M}_q.
	\end{align*}
	\item Let $\Delta^{(t)}:\R^{n\times q[0:t]}\to \R^{n\times q}$ be defined via
	\begin{align}\label{def:Delta}
	\Delta^{(t)}(\mathfrak{v}^{([1:t])} ) = \sum_{s \in [1:t]} \mathfrak{v}^{(s)} \rho_{t,s}^\top + D_t \in \R^{n\times q}.
	\end{align}
\end{itemize}
For notational consistency, we let $\Theta^{(0)}\equiv \mathrm{id}(\R^{m\times q})$ and $\Delta^{(0)}\equiv D_0 \in \R^{n\times q}$. 

The state evolution parameters introduced in Definition~\ref{def:gd_NN_se} play a central role in our theoretical development. Table~\ref{table:data_se} summarizes the correspondence between key observable random quantities and their counterparts under state evolution.

\begin{table}[t]
	\centering
	\begin{tabular}{c||c|c|c|c}
		\textbf{Random data} & $X \mu_\ast$ & $X W_1^{(t)}$ &  $n^{1/2}W_1^{(t)}$ & $\{W_\alpha^{(t)}\}_{\alpha \in [2:L]}$ \\ \hline
		\textbf{State evolution} & $\mathfrak{U}_1^{(0)}$ & $\Theta_\cdot^{(t)}(\mathfrak{U}^{[0:t]})$ &   $\Delta_\cdot^{(t)}(\mathfrak{V}^{([1:t])} )$ & $\{V_\alpha^{(t)}\}_{\alpha \in [2:L]}$ 
	\end{tabular}
	\vspace{0.5em}
	\caption{Correspondence between data and state evolution. }
	\label{table:data_se}
\end{table}

It is instructive to highlight the special role of $\{\tau_{t,s}\}$ (and consequently $\{\rho_{t,s}\}$) in condition (S3). As will become evident from Theorem~\ref{thm:se_gd_NN} below, the matrices $\{\rho_{t,s}\}_{s \in [1:t]}$ serve as key debiasing coefficients that restore the approximate normality of $XW_1^{(t)}$ by correcting the bias along the directions of the pre-gradients
\begin{align*}
\Big\{\mathsf{P}_{X,\bm{W}^{(s-1)}}^{(1:L]}\big(G_{\bm{W}^{(s-1)};L}(X)-Y_{[q]}\big) \odot \sigma_1'\big(XW_1^{(s-1)}\big)\Big\}_{s \in [1:t]}.
\end{align*}
For this reason, we refer to $\{\tau_{r,s}\}_{r,s \in [1:t]}$ and $\{\rho_{r,s}\}_{r,s \in [1:t]}$ as \emph{matrix-variate Onsager correction matrices}. In gradient descent training for much simpler models such as linear or logistic regression, analogous correction coefficients arise in scalar form; see, e.g., \cite[Section 2]{han2024gradient}.

\begin{remark}
	Some technical remarks are in order:
	\begin{enumerate}
		\item We note that $\delta_t=[*\,|\, 0_{q\times (q-1)}]$ and $V_L^{(t)}=[*\,|\, 0_{q\times (q-1)}]$, but we will use the full matrix form as above for notational simplicity.
		\item The state evolution parameters depend on the signal $\mu_\ast$ only through its signal strength $\sigma_{\mu_\ast}^2$.
		\item It is possible to write $\{\rho_{t,s}\}$ and $D_t$ entirely in terms of $\{\tau_{r,s}\}_{r,s \in [t]}$ by iterating (S3). We keep the current form for notational simplicity.
		\item With
		\begin{align}\label{def:tau_rho_matrix}
		&\bm{\tau}^{[t]}\equiv (\tau_{r,s})_{r,s \in [1:t]} \in (\mathbb{M}_q)^{t\times t},\quad \bm{\rho}^{[t]}\equiv (\rho_{r,s})_{r,s \in [1:t]} \in (\mathbb{M}_q)^{t\times t},\nonumber\\
		&\bm{\Sigma}^{[t]}\equiv (\Sigma_{r,s})_{r,s \in [1:t]} \in (\mathbb{M}_q)^{t\times t},\quad\bm{\Omega}^{[t]}\equiv (\Omega_{r,s})_{r,s \in [1:t]} \in (\mathbb{M}_q)^{t\times t},
		\end{align}
		and the notation $\mathfrak{O}_{\mathbb{M}_q;t}$ defined in (\ref{def:mat_O}), we may write the second and fourth lines in (S3) more compactly:
		\begin{align}\label{def:rho_se_matrix}
		\bm{\rho}^{[t]}&= (I_{\mathbb{M}_q})_t+ \big(\bm{\tau}^{[t]} + (I_{\mathbb{M}_q})_t\big) \mathfrak{O}_{\mathbb{M}_q;t}(\bm{\rho}^{[t-1]})  \in (\mathbb{M}_q)^{t\times t},\nonumber\\
		\bm{\Omega}^{[t]}&= \bm{\rho}^{[t]}\bm{\Sigma} \bm{\rho}_\top^{[t]}  \in (\mathbb{M}_q)^{t\times t}.
		\end{align}
	\end{enumerate}
\end{remark}

\subsection{Distributional characterizations in the finite-width proportional regime}

The following theorem provides a formal statement for (\ref{eqn:gd_dist_1})-(\ref{eqn:gd_dist_2}) with a more general joint distributional characterization including a debiased statistics $U^{(t)}$ defined as follows: for $t=1,2,\ldots$, let
\begin{align}\label{def:U_debiased}
U^{(t)}&\equiv X W_1^{(t)} +\phi^{-1}\sum_{s \in [1:t]} \eta^{(s-1)}_1\nonumber\\
&\qquad \times  \Big[\mathsf{P}_{X,\bm{W}^{(s-1)}}^{(1:L]}\big(G_{\bm{W}^{(s-1)};L}(X)-Y_{[q]}\big)\odot \sigma_1'(XW_1^{(s-1)})\Big]\, \rho_{t,s}^\top\in \R^{m\times q}.
\end{align}
Recall pseudo-Lipschitz functions defined in (\ref{cond:pseudo_lip}) in Appendix \ref{section:notation}.
\begin{theorem}\label{thm:se_gd_NN}
	Fix $t\in \N$. Suppose Assumption \ref{assump:gd_NN} holds for some $K,\Lambda\geq 2$ and $r_0\geq 0$.
	Fix a sequence of $\Lambda$-pseudo-Lipschitz functions $\{\psi_k:\R^{1+2q[1:t]} \to \R\}_{k \in [m]}$ and $\{\phi_\ell:\R^{q[0:t]} \to \R\}_{\ell \in [n]}$ of order $2$. Then for any $\mathfrak{r}\geq 1$, there exists some $c_t=c_t(t,q,L,\mathfrak{r},r_0)>0$ such that 
	\begin{align*}
	&\E^{(0)} \bigg[\biggabs{\frac{1}{m}\sum_{k \in [m]} \psi_k\Big((X\mu_{\ast})_{k\cdot},\big\{(XW_1^{(s-1)})_{k\cdot}, U_{k\cdot}^{(s-1)}\big\}_{s \in [1:t]}\Big)\nonumber\\
	&\qquad \qquad\qquad  - \frac{1}{m}\sum_{k \in [m]}  \E^{(0)}  \psi_k\Big(\mathfrak{U}^{(0)}_1, \big\{\Theta_{k}^{(s)}(\mathfrak{U}^{([0:s])}),\mathfrak{U}^{(s)}\big\}_{s \in [1:t]}\Big)  }^{\mathfrak{r}}\bigg]\\
	&\qquad + \E^{(0)}  \bigg[\biggabs{\frac{1}{n}\sum_{\ell \in [n]} \phi_\ell\Big(\big\{n^{1/2}W_{1;\ell\cdot}^{(s)}\big\}_{s \in [0:t]}\Big) - \frac{1}{n}\sum_{\ell \in [n]}  \E^{(0)}  \phi_\ell\Big(\big\{\Delta_{\ell}^{(s)}(\mathfrak{V}^{([1:s])} )\big\}_{s \in [0:t]}\Big)   }^{\mathfrak{r}}\bigg]  \\
	&\qquad +\max_{\alpha \in [2:L]} \max_{s \in [1:t]}  \E^{(0)} \pnorm{W_\alpha^{(s)}-V^{(s)}_\alpha}{\infty}^{\mathfrak{r}}\\
	& \leq  \big(K\Lambda\varkappa_{\ast}\big)^{c_{t}}\cdot n^{-1/c_t}.
	\end{align*}
\end{theorem}

The proof of Theorem \ref{thm:se_gd_NN} involves substantial technical preparations in Sections \ref{section:proof_preliminaries} and \ref{section:proof_gfom}, and is completed in Section \ref{subsection:proof_se_gd_NN}. Due to its complicated nature, we provide a high-level proof sketch in Section \ref{section:proof_sketch}.

It is easy to see from Theorem \ref{thm:se_gd_NN} that for any $t\geq 1$, with high probability,
\begin{align}\label{eqn:weight_evolution}
\pnorm{W_1^{(t)}-W_1^{(0)} }{}^2&\approx  n^{-1}\pnorm{D_t-D_0}{}^2+\Omega_{t,t},\nonumber\\
\pnorm{W_\alpha^{(t)}-W_\alpha^{(0)}  }{}^2& \approx \pnorm{V_\alpha^{(t)}-W_\alpha^{(0)}  }{}^2,\quad \alpha \in [2:L].
\end{align}
As the initialization $\bm{W}^{(0)}$ typically has magnitude $\pnorm{W_\alpha^{(0)}}{}\asymp 1$ for all $\alpha \in [1:L]$, (\ref{eqn:weight_evolution}) shows that the weights $\big(W_\alpha^{(t)}\big)_{\alpha \in [1:L]}$ typically move non-trivially from initialization $\bm{W}^{(0)}$ relative to its magnitude.

\begin{remark}\label{rmk:thm_se_gd_NN}
Some technical remarks for Theorem \ref{thm:se_gd_NN} are in order:
\begin{enumerate}
	\item It is possible to explicitly track the dependence of $c_t$ on $t,q,L$ in the proof. We refrain from doing so to keep the presentation of the proof clean, as this dependence is likely to be highly sub-optimal.
	\item The smoothness conditions (A4) on $\{\sigma_\alpha\}$'s are likely stronger than necessary. We conjecture that the minimal regularity required for Theorem \ref{thm:se_gd_NN} is the Lipschitz continuity of $\{\sigma_\alpha\}$'s. 
    \item The single-index regression model (\ref{def:model}) can be generalized to the multi-index regression model with further technical work.
\end{enumerate}
\end{remark}

\subsection{Gradient descent over the population risk}

For several practical neural network models in Table \ref{tab:model-summary}, it is of interest to understand the behavior of the gradient updates $\bm{W}^{(t)}$ in the regime $\phi\gg 1$. As we will see, in this regime $\bm{W}^{(t)}$ closely follows the `theoretical gradient descent iterates', which are often used as proxies for the actual gradient descent trajectory. Formally, let the population risk be  
\begin{align}\label{def:population_risk}
\overline{\mathsf{L}}(\bm{W})
& = \frac{1}{2}\E^{(0)}\big(\varphi_\ast(\iprod{\mu_\ast}{\mathsf{Z}_n})+\xi_{\pi_m}-f_{\bm{W}}(\mathsf{Z}_n)\big)^2,\quad \mathsf{Z}_n\sim \mathcal{N}(0,I_n).
\end{align}
It is easy to check that  $\overline{\mathsf{L}}(\bm{W})=\E^{(0)}_{X\sim \mathsf{Z}_{m\times n}}\pnorm{Y_{[q]}- G_{\bm{W};L}(X) }{}^2/(2m)$, where $\mathsf{Z}_{m\times n} \in \R^{m\times n}$ has i.i.d. $\mathcal{N}(0,1)$ entries.  

The theoretical gradient descent iterates on the population risk $\overline{\mathsf{L}}(\bm{W})$ are defined by replacing the empirical loss $\mathsf{L}$ in (\ref{def:loss_fcn_general}) by the population loss $\overline{\mathsf{L}}$ in (\ref{def:population_risk}):
initialized with $\bar{\bm{W}}^{(0)}\equiv \bm{W}^{(0)}$, for $t=1,2,\ldots$, we iteratively compute
\begin{align}\label{def:grad_descent_population}
\bar{W}_\alpha^{(t)}\equiv \bar{W}_\alpha^{(t-1)}- \eta_\alpha^{(t-1)}\cdot  \nabla_{\bar{W}_\alpha^{(t-1)}} \overline{\mathsf{L}}(\bar{\bm{W}}^{(t-1)}),\quad \alpha \in [1:L].
\end{align} 
Note that $\{\bar{W}_\alpha^{(t)}\}$ are deterministic matrices. Intuitively, we expect $ \bm{W}^{(t)}\approx \bar{\bm{W}}^{(t)}$ if and only if the empirical loss $\mathsf{L}$ concentrates around $\overline{\mathsf{L}}$. Clearly this does not occur in the regime $\phi\asymp 1$ as studied in Theorem \ref{thm:se_gd_NN}. Interestingly, the same Theorem \ref{thm:se_gd_NN} otherwise validates this concentration phenomenon in the optimal regime $\phi\gg 1$. 

The key is the following deterministic, recursive characterization for the dynamics of $\{\bar{\bm{W}}^{(t)}\}$; its proof can be found in Section \ref{subsection:proof_theoretical_gd_se}.

\begin{proposition}\label{prop:theoretical_gd_se}
For $t=1,2,\ldots$, the following hold:
\begin{enumerate}
	\item With 
	\begin{itemize}[label={}]
		\item $\bar{\tau}_{t}\equiv - \eta_1^{(t-1)} \E^{(0)} \nabla_{u}  \mathfrak{S}_{\pi_m}\big(\bar{W}_1^{(t-1),\top}\mathsf{Z}_n, \mu_{\ast,[q]}^\top \mathsf{Z}_n, \bar{W}_{[2:L]}^{(t-1)}\big) \in \mathbb{M}_q$,
		\item $\bar{\delta}_{t}\equiv - \eta_1^{(t-1)}\E^{(0)} \nabla_{w}  \mathfrak{S}_{\pi_m}\big(\bar{W}_1^{(t-1),\top}\mathsf{Z}_n, \mu_{\ast,[q]}^\top \mathsf{Z}_n, \bar{W}_{[2:L]}^{(t-1)}\big)\in \mathbb{M}_q$,
	\end{itemize}
    we have
    \begin{align*}
    \bar{W}_1^{(t)}\equiv \bar{W}_1^{(t-1)} \big(\bar{\tau}_{t}^\top+I_q\big)+\mu_{\ast,[q]}\bar{\delta}_{t}^\top\in \R^{n\times q}.
    \end{align*}
    \item For $\alpha \in [2:L]$,
    \begin{align*}
    \bar{W}_{\alpha}^{(t)}&\equiv \bar{W}_\alpha^{(t-1)}-\eta_\alpha^{(t-1)} \cdot \Big[\E^{(0)}\mathscr{G}_{\bar{W}^{(t-1)}_{[2:\alpha-1]};\pi_m }\big(\bar{W}_1^{(t-1),\top}\mathsf{Z}_n\big) \\
    &\qquad \times \big(\mathfrak{T}_{\alpha}\big)_{\pi_m}\big(\bar{W}_1^{(t-1),\top}\mathsf{Z}_n, \mu_{\ast,[q]}^\top \mathsf{Z}_n, \bar{W}_{[2:L]}^{(t-1)}\big)^\top\Big] \in \mathbb{M}_q.
    \end{align*}
\end{enumerate}
Here the expectation $\E^{(0)}$ is taken only with respect to $\mathsf{Z}_n\sim \mathcal{N}(0,I_n)$ and $\pi_m$.
\end{proposition}

The recursive characterization in Proposition \ref{prop:theoretical_gd_se}  has a natural connection to the state evolution in Definition \ref{def:gd_NN_se} in the regime $\phi\gg 1$. Suppose all state evolution parameters remain `stable' in Definition \ref{def:gd_NN_se}. Then by (S1), we expect $	\Theta^{(t)}(\mathfrak{u}^{([0:t])}) \approx \mathfrak{u}^{(t)}$  in the regime $\phi\gg 1$. So by (S3) we expect $\tau_{t,s}\approx 0_{\mathbb{M}_q}$ for $s \in [1:t-1]$, $\Sigma_{t,s}\approx 0_{\mathbb{M}_q}$ and $\Omega_{t,s}\approx 0_{\mathbb{M}_q}$ for $s \in [1:t]$; by (S4) we expect $D_t\approx D_{t-1}(\tau_{t,t}^\top+I_q)+D_{-1}\delta_t^\top$. With these reductions, we may then roughly identify the representation in Proposition \ref{prop:theoretical_gd_se} with the state evolution parameters as $n^{1/2}\bar{W}_1^{(t)}\approx D_t$ and $\bar{W}_\alpha^{(t)}\approx V_\alpha^{(t)}$ ($\alpha \in [2:L]$). These heuristics are formalized in the following theorem.

\begin{theorem}\label{thm:se_gd_NN_large_sample}
	Fix $t\in \N$. Suppose $\phi^{-1}\leq K$ and Assumptions (A3)-(A5) in Assumption \ref{assump:gd_NN} hold for some $K,\Lambda\geq 2$ and $r_0\geq 0$. Then for any $\mathfrak{r}\geq 1$, there exists some $c_t=c_t(t,q,L,\mathfrak{r},r_0)>0$ such that
	\begin{align*}
	&\max_{s \in [1:t]} \max_{\alpha \in [1:L]}  \E^{(0)} \pnorm{W_\alpha^{(s)}-\bar{W}_\alpha^{(s)}}{}^{\mathfrak{r}} \leq  \big(K\Lambda\varkappa_{\ast}\big)^{c_t}\cdot \big[(1\vee\phi)^{c_t}\cdot n^{-1/c_t}+\phi^{-1/c_t}\big].
	\end{align*}
\end{theorem}

A proof of the above theorem can be found in Section \ref{subsection:proof_se_gd_NN_large_sample}. We note that in the asymptotic formulation, the above theorem is valid by first taking limit $m\wedge n\to \infty$ with $m/n=\phi$, and then followed by $\phi \to \infty$ (assuming $K,\Lambda,\varkappa_{\ast}=\bigo(1)$).

\section{Characterization and algorithmic estimation of the training and generalization/test errors}\label{section:gd_inference}

In this section, we shall use Theorem \ref{thm:se_gd_NN} to characterize the training and generalization/test error as in (\ref{eqn:gen_err_char_intro}), and use it to develop a consistent algorithmic estimation method for the generalization error.

\subsection{Training and generalization/test errors}

We shall first provide a slight variation of the training and generalization/test errors as compared to (\ref{def:test_error}). 

\begin{definition}\label{def:training_test_error}
	Suppose we are given data $(X,Y) \in \R^{m\times n}\times \R^m$.
    \begin{enumerate}
    \item The \emph{training error} is defined as 
	\begin{align*}
	\mathscr{E}_{\texttt{train}}^{(t)}(X,Y) &\equiv \frac{1}{m} \bigpnorm{Y_{[q]}- G_{\bm{W}^{(t)};L}(X)}{}^2.
	\end{align*}
    \item The \emph{generalization/test error} is defined as
	\begin{align*}
	\mathscr{E}_{\texttt{test}}^{(t)}(X,Y)&\equiv \E^{(0)}\big[ \big(\varphi_\ast(\iprod{X_{\textrm{new}}}{\mu_\ast})+\xi_{\pi_m}- f_{\bm{W}^{(t)}}( X_{\textrm{new}} ) \big)^2 | (X,Y) \big].
	\end{align*}
	Here the expectation in the above display is taken over new data $X_{\textrm{new}}\in \R^n$ with the same law as $X_1\in \R^n$. 
    \end{enumerate}
\end{definition}

We have the following characterization for $\mathscr{E}_{\texttt{train}}^{(t)}(X,Y)$ and $\mathscr{E}_{\texttt{test}}^{(t)}(X,Y)$; its proof can be found in Section \ref{subsection:proof_err_gd_NN}.

\begin{theorem}\label{thm:err_gd_NN}
	Fix $t\in \N$. Suppose Assumption \ref{assump:gd_NN} holds for some $K,\Lambda\geq 2$ and $r_0\geq 0$. Then for any $\mathfrak{r}\geq 1$, there exists some $c_t=c_t(t,q,L,\mathfrak{r},r_0)>0$ such that  
	\begin{align*}
	&\E^{(0)}\bigabs{\mathscr{E}_{\texttt{train}}^{(t)}(X,Y)- \E^{(0)} \bigpnorm{\mathscr{R}_{V^{(t)};\pi_m}\big(\Theta_{\pi_m}^{(t+1)}(\mathfrak{U}^{([0:t+1])}),\mathfrak{U}^{(0)} \big)}{}^2   }^{\mathfrak{r}}\\
	&\qquad + \E^{(0)}\bigabs{\mathscr{E}_{\texttt{test}}^{(t)}(X,Y)- \E^{(0)} \bigpnorm{\mathscr{R}_{V^{(t)};\pi_m}\big(\mathfrak{U}^{(t+1)},\mathfrak{U}^{(0)} \big)}{}^2}^{\mathfrak{r}}\\
	&\leq \big(K\Lambda\varkappa_{\ast}\big)^{c_t}\cdot n^{-1/c_t}.
	\end{align*}
\end{theorem}
From Theorem \ref{thm:err_gd_NN}, we may immediately characterize the \emph{generalization gap}
\begin{align*}
&\texttt{Gap}^{(t)}(X,Y)\equiv \bigabs{\mathscr{E}_{\texttt{test}}^{(t)}(X,Y)-\mathscr{E}_{\texttt{train}}^{(t)}(X,Y)}\\
&\approx\bigabs{\E^{(0)} \bigpnorm{\mathscr{R}_{V^{(t)};\pi_m}\big(\Theta_{\pi_m}^{(t+1)}(\mathfrak{U}^{([0:t+1])}),\mathfrak{U}^{(0)} \big)}{}^2- \E^{(0)} \bigpnorm{\mathscr{R}_{V^{(t)};\pi_m}\big(\mathfrak{U}^{(t+1)},\mathfrak{U}^{(0)} \big)}{}^2}.
\end{align*}
As $\Theta^{(t)}(\mathfrak{u}^{([0:t])}) \neq \mathfrak{u}^{(t)}$ is typically non-linear, the above display quantifies a non-trivial generalization gap $\E^{(0)}\texttt{Gap}^{(t)}(X,Y)=\Theta(1)$ that arise ubiquitously in gradient descent training for neural network models beyond the classical uniform convergence regime in learning theory; the readers are referred to \cite{bartlett2021deep} for a comprehensive review. On the other hand, as $\phi \to \infty$, as $\Theta^{(t)}(\mathfrak{u}^{([0:t])}) \approx \mathfrak{u}^{(t)}$ (see the discussion preceding Theorem \ref{thm:se_gd_NN_large_sample}), the generalization gap $\texttt{Gap}^{(t)}(X,Y)$ vanishes correspondingly .

\begin{remark}
Theorem \ref{thm:err_gd_NN} is valid conditional on the noise $\xi$. By assuming suitable tail conditions on $\xi_i$'s, we may easily prove further concentration of $\E^{(0)} \bigpnorm{\mathscr{R}_{V^{(t)};\pi_m}\big(\Theta_{\pi_m}^{(t+1)}(\mathfrak{U}^{([0:t+1])}),\mathfrak{U}^{(0)} \big)}{}^2  $ and $\E^{(0)} \pnorm{\mathscr{R}_{V^{(t)};\pi_m}\big(\mathfrak{U}^{(t+1)},\mathfrak{U}^{(0)} \big)}{}^2$. We omit these routine details.
\end{remark}

\subsection{Augmented gradient descent algorithm}

Next we present an augmented gradient descent algorithm that aims at simultaneously outputting the gradient update $\bm{W}^{(t)}$ and a consistent estimator $\hat{\mathscr{E}}_{\texttt{test}}^{(t)}(X,Y)$ for the unknown generalization error $\mathscr{E}_{\texttt{test}}^{(t)}(X,Y)$. Recall the following notation in Appendix \ref{section:notation}: (i) $\mathfrak{O}_{\mathbb{M}_q;t+1}(N)$ defined for any (block) matrix $N \in (\mathbb{M}_q)^{t\times t}$ in (\ref{def:mat_O}); (ii) $\mathrm{diag}(M)\equiv (M_{ii}\bm{1}_{i=j})_{i,j \in [n]} \in \R^{n\times n}$ and $\mathrm{vecdiag}(M)\equiv (M_{ii})_{i \in [n]}\in \R^n$ defined for a square matrix $M\in \R^{n\times n}$; (iii)  $(I_{\mathbb{M}_q})_t=\mathrm{diag}(I_{q},\ldots,I_{q})\in (\mathbb{M}_q)^{t\times t}$. We also recall $\mathscr{G}_\cdot$ defined in Definition \ref{def:mean_field_fcn}.

\begin{algorithmdef}\label{alg:aug_gd_nn}
Take the following input data and initialization:
\begin{itemize}[label={}]
	\item \textbf{Input data}: (i) observations $(X,Y) \in \R^{m\times n}\times \R^{n}$; (ii) number of layers $L$; (iii) network width $q$; (iv) learning rates $\{\eta_\alpha^{(\cdot)}\}_{\alpha \in [1:L]}\subset \R_{>0}$; (v) nonlinearities $\{\sigma_\alpha: \R\to \R\}_{\alpha \in [1:L]}$ with $\sigma_0=\sigma_L\equiv \mathrm{id}$.
	\item \textbf{Initialization}: (i) weight matrices $W_1^{(0)}\in \R^{n\times q}$ and $\{W_\alpha^{(0)}\}_{\alpha \in [2:L]}\subset \R^{q\times q}$, where $W_L^{(0)}=[\ast\,|\, 0_{q\times (q-1)}]$; (ii) $\hat{\bm{\rho}}^{[t]}=\emptyset$. 
\end{itemize}
\textbf{For $ t = 1,2,\ldots$}: 
\begin{enumerate}
\item \emph{Forward propagation}: For $\alpha = 1,2,\ldots,L$ and $\ell \in [q]$,
\begin{align*}
&\red{\star}\quad \hat{H}_\alpha^{(t-1)} \equiv  \sigma_{\alpha-1}( \hat{H}_{\alpha-1}^{(t-1)} ) W_\alpha^{(t-1)}\in \R^{m\times q},\hbox{ where }\hat{H}_{0}^{(t-1)}=X\in \R^{m\times n},\\
&\blue{\triangle}\quad \partial_{\ell} \hat{H}_\alpha^{(t-1)}\equiv E_{\cdot \ell} \bm{1}_{\alpha=1}+ \big[\sigma_{\alpha-1}'( \hat{H}_{\alpha-1}^{(t-1)} )\odot \partial_{\ell} \hat{H}_{\alpha-1}^{(t-1)}\big] W_\alpha^{(t-1)}\bm{1}_{\alpha=2,\cdots,L} \in \R^{m\times q}.
\end{align*}
Here $E_{\cdot \ell}=\bm{1}_m e_\ell^\top \in \R^{m\times q}$ has all $0$ entries except $1$'s in the $\ell$-th row.
\item \emph{Backward propagation}: For $\alpha=L,L-1,\ldots,1$ and $\ell \in [q]$,
\begin{flalign*}
&\red{\star}\quad  \hat{P}_\alpha^{(t-1)}\equiv \big(\hat{H}_L^{(t-1)} -Y_{[q]}\big)\bm{1}_{\alpha=L}\\
&\qquad\qquad + \Big(\hat{P}_{\alpha+1}^{(t-1)}\odot \sigma_{\alpha+1}'(\hat{H}_{\alpha+1}^{(t-1)} )\Big) \big(W_{\alpha+1}^{(t-1)}\big)^{\top}\bm{1}_{\alpha=L-1,\cdots,1} \in \R^{m\times q},\\
& \blue{\triangle}\quad \binom{
	\partial^{(1)}_{\ell}\hat{P}_\alpha^{(t-1)}}{
	\partial^{(2)}_{\ell}\hat{P}_\alpha^{(t-1)}}
\equiv  \binom{ 0_{m\times q}  }{ E_{\cdot \ell} } \bm{1}_{\alpha=L}+ \Bigg[  \binom{ \hat{P}_{\alpha+1}^{(t-1)}\odot \sigma_{\alpha+1}''\big(\hat{H}_{\alpha+1}^{(t-1)}\big)\odot \partial_{\ell} \hat{H}_{\alpha+1}^{(t-1)}}{0_{m\times q}}\\
&\qquad\qquad + \binom{ \sigma_{\alpha+1}'\big(\hat{H}_{\alpha+1}^{(t-1)}\big) }{\sigma_{\alpha+1}'\big(\hat{H}_{\alpha+1}^{(t-1)}\big)}\odot \binom{
	\partial^{(1)}_{\ell}\hat{P}_{\alpha+1}^{(t-1)}}{
	\partial^{(2)}_{\ell}\hat{P}_{\alpha+1}^{(t-1)}} 
\Bigg] \big(W_{\alpha+1}^{(t-1)}\big)^{\top}\bm{1}_{\alpha=L-1,\cdots,1}\in \R^{2m\times q}.
\end{flalign*}
\item \emph{Compute pre-gradient derivative estimates}: For $\ell \in [q]$,
\begin{align*}
& \blue{\triangle}\quad \hat{Q}_\ell^{(t-1)}\equiv \sum_{r \in [q]}\mathrm{vecdiag}\,\bigg(\sum_{d \in [q]}  \partial^{(2)}_{d}\hat{P}_1^{(t-1)} e_r e_d^\top \partial_{\ell}\hat{H}_L^{(t-1),\top}\bigg) \,e_r^\top \in \R^{m\times q},\\
& \blue{\triangle}\quad \partial_{\ell} \hat{ \mathfrak{S} }^{(t-1)}\equiv \big(\hat{P}_1^{(t-1)}\odot \sigma_1''(XW_1^{(t-1)})\big) \, e_\ell e_\ell^\top \\
&\qquad \qquad\qquad\qquad+ \big(\partial^{(1)}_{\ell}\hat{P}_1^{(t-1)}+\hat{Q}_\ell^{(t-1)}\big)\odot \sigma_1'(XW_1^{(t-1)})\in \R^{m\times q}.
\end{align*}
\item \emph{Compute matrix-variate Onsager correction coefficients}: 
\begin{align*}
& \blue{\triangle}\quad \hat{\bm{L}}_k^{[t-1]}\equiv \mathrm{diag}\bigg( \bigg\{ \eta^{(s-1)} \sum_{\ell \in [q]}\big(\partial_{\ell} \hat{ \mathfrak{S} }^{(s-1)}\big)^\top e_k e_\ell^\top \bigg\}_{s \in [1:t]} \bigg) \in (\mathbb{M}_q)^{t\times t},\, \forall k \in [m],\\
& \blue{\triangle}\quad \hat{\bm{\tau}}^{[t]} \equiv -\frac{1}{m}\sum_{k \in [m]} \Big((I_{\mathbb{M}_q})_t+\phi^{-1} \hat{\bm{L}}_k^{[t-1]}\mathfrak{O}_{\mathbb{M}_q;t}(\hat{\bm{\rho}}^{[t-1]})\Big)^{-1}\hat{\bm{L}}_k^{[t-1]} \in (\mathbb{M}_q)^{t\times t},\\
&\blue{\triangle}\quad \hat{\bm{\rho}}^{[t]}\equiv (I_{\mathbb{M}_q})_t+ \big(\hat{\bm{\tau}}^{[t]} + (I_{\mathbb{M}_q})_t\big) \mathfrak{O}_{\mathbb{M}_q;t}(\hat{\bm{\rho}}^{[t-1]})  \in (\mathbb{M}_q)^{t\times t}.
\end{align*}
\item \emph{Compute the gradient updates}: For $\alpha \in [1:L]$, 
\begin{align*}
\red{\star}\quad  W_\alpha^{(t)}&\equiv W_\alpha^{(t-1)}- \frac{\eta_\alpha^{(t-1)}}{m} \Big(\sigma_{\alpha-1}(\hat{H}_{\alpha-1}^{(t-1)} )\Big)^\top \Big(\hat{P}_\alpha^{(t-1)}\odot \sigma_\alpha'(\hat{H}_\alpha^{(t-1)} )\Big).
\end{align*}
Here $W_1^{(t)}\in \R^{n\times q}$ and $\{W_\alpha^{(t)}\}_{\alpha \in [2:L]} \subset \R^{q\times q}$.
\item \emph{Compute the generalization error estimate}:
\begin{align*}
&\blue{\triangle}\quad \hat{U}^{(t)}\equiv X W_1^{(t)}+\phi^{-1}\sum_{s \in [1:t]} \eta^{(s-1)}_1 \Big[\hat{P}_1^{(s-1)}\odot \sigma_1'(XW_1^{(s-1)})\Big] \,\hat{\rho}_{t,s}^\top \in \R^{m\times q},\\
&\blue{\triangle}\quad \hat{\mathscr{E}}_{\texttt{test}}^{(t)}(X,Y)\equiv 
\frac{1}{m} \bigpnorm{Y_{[q]}- \mathscr{G}_{W_{[2:L]}^{(t)}}(\hat{U}^{(t)})}{ }^2 \in \R_{\geq 0}.
\end{align*}
\end{enumerate}
\end{algorithmdef}

In Algorithm \ref{alg:aug_gd_nn} above, the symbol $\red{\star}$ denotes standard gradient computations that can be implemented via auto-differentiation in the deep learning libraries such as \texttt{PyTorch} and \texttt{TensorFlow}, whereas $\blue{\triangle}$ highlights the additional computations necessary for constructing the generalization error estimate $\hat{\mathscr{E}}_{\texttt{test}}^{(t)}(X,Y)$. Importantly, none of the computations in Algorithm \ref{alg:aug_gd_nn} require any ad-hoc tuning or knowledge of the underlying link function $\varphi_\ast$ and the signal $\mu_\ast$, so our proposal is fundamentally different from a simulation method for the state evolution in Definition \ref{def:gd_NN_se}.

The following theorem shows that $\hat{\mathscr{E}}_{\texttt{test}}^{(t)}(X,Y)$ is indeed a consistent estimator at each iteration; its proof can be found in Section \ref{subsection:proof_est_test_err}.

\begin{theorem}\label{thm:est_test_err}
	Suppose Assumption \ref{assump:gd_NN} holds for some $K,\Lambda\geq 2$ and $r_0\geq 0$. Then for any $\mathfrak{r}\geq 1$, there exists some $c_t=c_t(t,q,L,\mathfrak{r},r_0)>0$ such that 
	\begin{align*}
	\E^{(0)} \abs{\hat{\mathscr{E}}_{\texttt{test}}^{(t)}(X,Y)- \#  }^{\mathfrak{r}}\leq \big(K\Lambda\varkappa_{\ast}\big)^{c_{t}}\cdot n^{-1/c_t}.
	\end{align*}
	Here $\#$ can be either $\mathscr{E}_{\texttt{test}}^{(t)}(X,Y)$ or  $\E^{(0)} \pnorm{\mathscr{R}_{V^{(t)};\pi_m}\big(\mathfrak{U}^{(t+1)},\mathfrak{U}^{(0)} \big)}{}^2$.
\end{theorem}

The consistent estimator $\hat{\mathscr{E}}_{\texttt{test}}^{(t)}(X,Y)$ for the generalization error at each iteration can be used for various practical purposes, such as tuning hyperparameters (e.g., learning rates) and determining early stopping in gradient descent training, all with rigorous theoretical guarantees. We present small-scale experiments in Section~\ref{section:simulation}, and leave large-scale empirical validation to future work.

\begin{remark}\label{rmk:alg_smooth}
Some technical remarks for Theorem \ref{thm:est_test_err} are in order:
\begin{enumerate}
    \item Similar to Remark~\ref{rmk:thm_se_gd_NN}-(2), the smoothness conditions (A4) are likely stronger than necessary. However, since $\{\sigma_\alpha''\}$ play a crucial role in the design of Algorithm~\ref{alg:aug_gd_nn}, we conjecture that the minimal regularity required for $\{\sigma_\alpha\}$ in Theorem~\ref{thm:est_test_err} is the Lipschitz continuity of their first derivatives $\{\sigma_\alpha'\}$. Some numerical evidence in support of this conjecture is provided in Section~\ref{subsection:simulation_caution}.
    
    \item Remark \ref{rmk:thm_se_gd_NN}-(3) also applies to Theorem \ref{thm:est_test_err}. Moreover, since Algorithm \ref{alg:aug_gd_nn} does not utilize any information about the single/multi-index modeling, we conjecture that Theorem \ref{thm:est_test_err} holds for a broader class of regression functions. Determining the maximal set of such regression functions remains open.
\end{enumerate}
\end{remark}

\subsection{Heuristics for algorithmic estimation without knowledge of $\mu_\ast$}

We develop some heuristics to understand the rationale behind Algorithm \ref{alg:aug_gd_nn} from the lens of our theory in Theorem \ref{thm:se_gd_NN}. The key principle of Algorithm \ref{alg:aug_gd_nn} is that it produces consistent estimates for the matrix-variate Onsager correction matrices $\bm{\tau}^{[t]},\bm{\rho}^{[t]}$:
\begin{align}\label{eqn:tau_rho_est_consist}
\hat{\bm{\tau}}^{[t]}\approx \bm{\tau}^{[t]},\quad \hat{\bm{\rho}}^{[t]}\approx \bm{\rho}^{[t]}.
\end{align}
As $\hat{\bm{\rho}}^{[t]}, \bm{\rho}^{[t]}$ are defined in terms of $\hat{\bm{\tau}}^{[t]}, \bm{\tau}^{[t]}$ and $\hat{\bm{\rho}}^{[t-1]}, \bm{\rho}^{[t-1]}$, it remains to understand the reason that $\hat{\bm{\tau}}^{[t]}\approx \bm{\tau}^{[t]}$. 

To this end, a crucial step is the following alternative representation of $\bm{\tau}^{[t]}$. 

\begin{proposition}\label{prop:tau_repres}
	For $k \in [m]$, let
	\begin{align}\label{def:L_u_V}
	&\bm{\mathfrak{L}}^{(t)}_k\big(u^{([0:t])},V^{([0:t-1])}\big)\nonumber\\
	&\equiv \mathrm{diag} \bigg(\bigg\{\eta^{(s-1)}_1 \sum_{\ell \in [q]} \Big[\partial_{u_\ell} \mathfrak{S}_{k}\big(u^{(s)},u^{(0)},V^{(s-1)}\big) \Big] e_\ell^\top\bigg\}_{s \in [1:t]}\bigg) \in (\mathbb{M}_q)^{t\times t}.
	\end{align}
	Then we have 
	\begin{align*}
	\bm{\tau}^{[t]} &= - \E^{(0)} \bigg[ \bigg( (I_{\mathbb{M}_q})_t+ \phi^{-1} \bm{\mathfrak{L}}^{(t)}_{\pi_m}\Big(\Big\{\Theta^{(s)}_{\pi_m}\big(\mathfrak{U}^{([0:s])}\big)\Big\}_{s \in [0:t]},V^{([0:t-1])}\Big) \mathfrak{O}_{\mathbb{M}_q;t}(\bm{\rho}^{[t-1]})\bigg)^{-1}\\
	&\qquad\qquad\qquad \times  \bm{\mathfrak{L}}^{(t)}_{\pi_m}\Big(\Big\{\Theta^{(s)}_{\pi_m}\big(\mathfrak{U}^{([0:s])}\big)\Big\}_{s \in [0:t]},V^{([0:t-1])}\Big)\bigg]\in (\mathbb{M}_q)^{t\times t}.
	\end{align*}
\end{proposition}
The proof of Proposition \ref{prop:tau_repres} can be found in Section \ref{subsection:proof_tau_repres}. In principle, the representation in Proposition \ref{prop:tau_repres} suggests that in order to obtain a data-driven estimate for $\bm{\tau}^{[t]}$, it suffices to obtain such an estimate for $\bm{\mathfrak{L}}^{(t)}_{k}\big(\big\{\Theta^{(s)}_{k}(\mathfrak{U}^{([0:s])})\big\}_{s \in [0:t]},V^{([0:t-1])}\big)$. In view of (\ref{def:L_u_V}), it then suffices to construct data-driven estimates for
\begin{align}\label{eqn:S_L_mat}
\Big[\partial_{u_\ell} \mathfrak{S}_{k}\Big(\Theta^{(s)}_{k}(\mathfrak{U}^{([0:s])}),\mathfrak{U}^{(0)},V^{(s-1)}\Big) \Big]_{\ell \in [q]} \in \mathbb{M}_q,\quad s \in [1:t].
\end{align}
Using our state evolution theory in Theorem \ref{thm:se_gd_NN}, it is plausible to consider
\begin{align}\label{eqn:S_L_mat_1}
\Big[\partial_{u_\ell} \mathfrak{S}_{k}\Big((XW_1^{(s-1)})_{k\cdot},(X\mu_{\ast,[q]})_{k\cdot},W_{[2:L]}^{(s-1)}\Big) \Big]_{\ell \in [q]} \in \mathbb{M}_q,\quad s \in [1:t].
\end{align}
Unfortunately, (\ref{eqn:S_L_mat_1}) is not a valid statistical estimator due to the existence of the unknown signal $\mu_\ast$. The key from here is the derivative formulae in Proposition \ref{prop:derivative_formula} for the theoretical functions in Definition \ref{def:mean_field_fcn}, from which the following identifications for the quantities in Algorithm \ref{alg:aug_gd_nn} hold for all $k \in [m]$:
\begin{itemize}[label={}]
	\item $\hat{H}_\alpha^{(t-1)} =\mathscr{H}_{W_{[2:L]}^{(t-1)}}(X W_1^{(t-1)})$,
	\item $e_k^\top\partial_\ell \hat{H}_\alpha^{(t-1)} =e_k^\top \partial_{k\ell}\mathscr{H}_{W_{[2:L]}^{(t-1)}}(X W_1^{(t-1)})$,
	\item $\hat{P}_\alpha^{(t-1)} =\mathscr{P}^{(\alpha:L]}_{W_{[2:L]}^{(t-1)}}\Big(X W_1^{(t-1)}, \mathscr{R}_{W_{[2:L]}^{(t-1)}}\big(XW_1^{(t-1)},X\mu_{\ast, [q]}\big)\Big)$,
	\item $e_k^\top\partial_{\ell}^{(1)}\hat{P}_\alpha^{(t-1)}  = e_k^\top\partial_{u_{k\ell}}\mathscr{P}^{(\alpha:L]}_{W_{[2:L]}^{(t-1)}}\Big(X W_1^{(t-1)}, \mathscr{R}_{W_{[2:L]}^{(t-1)}}\big(XW_1^{(t-1)},X\mu_{\ast, [q]}\big)\Big)$,
	\item $e_k^\top\partial^{(2)}_{\ell}\hat{P}_\alpha^{(t-1)}  = e_k^\top\partial_{z_{k\ell}}\mathscr{P}^{(\alpha:L]}_{W_{[2:L]}^{(t-1)}}\Big(X W_1^{(t-1)}, \mathscr{R}_{W_{[2:L]}^{(t-1)}}\big(XW_1^{(t-1)},X\mu_{\ast, [q]}\big)\Big)$.
\end{itemize}
With these identifications, we will show in Lemma \ref{lem:equiv_alg} by straightforward calculations that $\big(\partial_{\ell} \hat{ \mathfrak{S} }^{(s-1)}\big)^\top e_k$ in Algorithm \ref{alg:aug_gd_nn} exactly matches 
\begin{align*}
& \bigg[ \hat{P}_1^{(s-1)}\odot \sigma_1''(XW_1^{(s-1)})\odot e_ke_\ell^\top \\
&\qquad + \bigg(\partial^{(1)}_{\ell}\hat{P}_1^{(s-1)}+\sum_{r \in [q]}  \partial^{(2)}_{r}\hat{P}_1^{(s-1)} \big(\partial_{\ell} \hat{H}_L^{(s-1)}\big)_{kr}\bigg)\odot \sigma_1'(XW_1^{(s-1)})\bigg]^\top e_k,
\end{align*}
which, by the derivative formula in Proposition \ref{prop:derivative_formula}-(4), can be viewed as a fully data-driven estimator for the matrix in (\ref{eqn:S_L_mat}) and (\ref{eqn:S_L_mat_1}).

A formal proof of the above heuristics for (\ref{eqn:tau_rho_est_consist}) can be found in Section \ref{subsection:proof_tau_rho_est_err}.

With (\ref{eqn:tau_rho_est_consist}), we may expect $\hat{U}^{(t)}\approx U^{(t)}$, so the state evolution theory in Theorem \ref{thm:se_gd_NN} entails that 
\begin{align*}
\hat{\mathscr{E}}_{\texttt{test}}^{(t)}(X,Y)
&\stackrel{(\ref{eqn:tau_rho_est_consist})}{\approx} m^{-1} \bigpnorm{\varphi_\ast(X\mu_{\ast,[q]})+\xi- \mathscr{G}_{V_{[2:L]}^{(t)}}(U^{(t)})}{ }^2 = m^{-1} \bigpnorm{\mathscr{R}_{V_{[2:L]}^{(t)} }\big(U^{(t)},X\mu_{\ast,[q]}\big)  }{ }\\
&\stackrel{\textrm{Thm.}  \ref{thm:se_gd_NN}}{ \approx } \E^{(0)} \pnorm{\mathscr{R}_{V^{(t)};\pi_m}\big(\mathfrak{U}^{(t+1)},\mathfrak{U}^{(0)} \big)}{}^2 \stackrel{\textrm{Thm.}  \ref{thm:err_gd_NN}}{ \approx } {\mathscr{E}}_{\texttt{test}}^{(t)}(X,Y).
\end{align*}
A formal proof of the above display that leads to Theorem \ref{thm:est_test_err} can be found in Section \ref{subsection:proof_est_test_err}.

\section{A further implication: Structure of the learned model $f_{\emph{\textbf{W}}^{(t)}} $}\label{section:gaussian_representation_f_W}

In this section, we shall use Theorem \ref{thm:se_gd_NN} to provide a formal statement of (\ref{ineq:represent_learning_1}), which shows the structure of the learned model $f_{\bm{W}^{(t)}}(x)$.

\subsection{Proportional regime $\phi\asymp 1$}

To proceed with our discussion, for any $v_{[2:L]} \in (\mathbb{M}_q)^{[2:L]}$, let $\mathfrak{h}_{v_{[2:L]}}:\R^q\to \R^q$ be defined as
\begin{align}\label{def:h_v}
\mathfrak{h}_{v_{[2:L]}}(u)\equiv \mathscr{H}_{v_{[2:L]};k}(u),\quad \hbox{for all } u \in \R^q.
\end{align}
Here any $k \in [m]$ gives the same mapping in the above display. Recall $\{\tau_{t,s}\}$ and $\{\delta_t\}$ defined in the state evolution in Definition \ref{def:gd_NN_se}. 

\begin{definition}\label{def:represent_learning_notation}
	\begin{enumerate}
		\item Let $\mathfrak{m}_W^{(t)} \in \R^q$ and $\mathfrak{M}_W^{(t)} \in \mathbb{M}_q$ be determined recursively via
		\begin{align}\label{eqn:feature_learning_grad_2}
		\begin{cases}
		\mathfrak{m}_W^{(t)} = \delta_t e_1+ \sum_{s \in [1:t]}  \big(\tau_{t,s}+I_q\bm{1}_{s=t}\big) \mathfrak{m}_W^{(s-1)} \in \R^q,\\
		\mathfrak{M}_W^{(t)} = \sum_{s \in [1:t]}  \big(\tau_{t,s}+I_q\bm{1}_{s=t}\big) \mathfrak{M}_W^{(s-1)} \in \mathbb{M}_q,
		\end{cases}
		\end{align}
		with initialization $\mathfrak{m}_W^{(0)}=0_q$ and $\mathfrak{M}_W^{(0)}=I_q$.
		\item Let the \emph{effective signal} $U_{\ast,\texttt{eff}}^{(t)}$ (at iteration $t$) be defined as 
		\begin{align}\label{def:effective_signal}
		U_{\ast,\texttt{eff}}^{(t)}\equiv \mu_\ast \mathfrak{m}_W^{(t),\top}+W_1^{(0)}\mathfrak{M}_W^{(t),\top} \in \R^{n\times q}.
		\end{align}
	\end{enumerate}
\end{definition}

With the notation in Definition \ref{def:represent_learning_notation} above, the following theorem provides a precise quantification for the informal statement (\ref{ineq:represent_learning_1}); its proof can be found in Section \ref{subsection:proof_represent_learning}. 

\begin{theorem}\label{thm:represent_learning}
	Fix $t\in \N$. Suppose Assumption \ref{assump:gd_NN} holds for some $K,\Lambda\geq 2$ and $r_0\geq 0$. Then for any $\mathfrak{r}\geq 1$, there exists some $c_t=c_t(t,q,L,\mathfrak{r},r_0)>0$ such that  
	\begin{align*}
	& \E^{(0)}\Big[\mathfrak{d}_{\mathrm{BL}}^{(X,Y)}\Big(f_{\bm{W}^{(t)}}(x), \mathfrak{h}_{V^{(t)}_{[2:L]}}\big(U_{\ast,\texttt{eff}}^{(t),\top}x + \Omega_{t,t}^{1/2}\mathsf{Z}_q\big)\Big)\Big]^{\mathfrak{r}}\leq \big(K\Lambda\varkappa_{\ast}\big)^{c_t}\cdot n^{-1/c_t}.
	\end{align*}
	Here $\mathfrak{d}_{\mathrm{BL}}^{(X,Y)}(\cdot,\cdot)$, defined conditionally on the data $(X,Y)$, is the bounded-Lipschitz metric between pairs of $\R^q$-valued random variables defined over $(x,\mathsf{Z}_q) \sim \mathcal{N}(0,I_n)\otimes \mathcal{N}(0,I_q)$ that are independent of all other variables.
\end{theorem}

As mentioned in the Introduction, the representation in Theorem~\ref{thm:represent_learning} bears a close conceptual connection to the mean-field behavior of a broad class of regularized regression estimators, which we now elaborate. For concreteness, let us consider linear regression (i.e., $\varphi_\ast=\mathrm{id}$) and the regularized estimator
\begin{align*}
\hat{\mu}\in \argmin_{\mu \in \R^n } \bigg\{\frac{1}{2m}\sum_{i \in [m]}(Y_i-\iprod{X_i}{\mu})^2+ \frac{\lambda}{n}\sum_{j \in [n]} \mathsf{f}(n^{1/2}\mu_j)\bigg\},
\end{align*}
where $\mathsf{f}:\R\to \R_{\geq 0}$ is a convex regularizer. For a large class of such convex regularizers, it is known that in the `mean-field/proportional' regime $m\asymp n$, under the same design condition (A2) the fitted model $f_{\hat{\mu}}(x)\equiv \iprod{\hat{\mu}}{x}$ satisfies
\begin{align}\label{eqn:rep_est_linear_model}
f_{\hat{\mu}}(x) \stackrel{d}{\approx} \prox_{\alpha_{\texttt{eff}}\cdot \mathsf{f}} \big(\iprod{\mu_{\ast}}{x}+\sigma_{\texttt{eff}}\cdot\mathsf{Z}\big),\quad \mathsf{Z}\sim \mathcal{N}(0,1),
\end{align}
where the `effective signal' in the linear model reduces to $\mu_\ast$, and the pair of scalars $(\alpha_{\texttt{eff}},\sigma_{\texttt{eff}})\in \R_{>0}^2$ is determined via a system of two equations; see, e.g., \cite{thrampoulidis2018precise,han2023universality}. Clearly, the representation in~\eqref{eqn:rep_est_linear_model} shares the same structural form as that in Theorem~\ref{thm:represent_learning}, but with two crucial differences discussed earlier in the Introduction.

\subsection{Large sample regime $\phi\gg 1$}
Similar to Theorem \ref{thm:se_gd_NN_large_sample},  in the large sample regime $\phi\gg 1$ we may get a simplified representation of $f_{\bm{W}^{(t)}}$. Recall $\{\bar{\tau}_t\}$ and $\{\bar{\delta}_t\}$ defined in Proposition \ref{prop:theoretical_gd_se}. 

\begin{definition}
	\begin{enumerate}
		\item Let $\bar{\mathfrak{m}}_W^{(t)} \in \R^q$ and $\bar{\mathfrak{M}}_W^{(t)} \in \mathbb{M}_q$ be determined recursively via
		\begin{align}\label{eqn:feature_learning_grad_large_sample}
		\bar{\mathfrak{m}}_W^{(t)} = \bar{\delta}_t e_1+ \big(\bar{\tau}_{t}+I_q\big) \bar{\mathfrak{m}}_W^{(t-1)} \in \R^q,\quad \bar{\mathfrak{M}}_W^{(t)} =  \big(\bar{\tau}_{t}+I_q\big) \bar{\mathfrak{M}}_W^{(t-1)} \in \mathbb{M}_q,
		\end{align}
		with initialization $\bar{\mathfrak{m}}_W^{(0)}=0_q$ and $\bar{\mathfrak{M}}_W^{(0)}=I_q$. 
		\item Let the \emph{effective signal} $\bar{U}_{\ast,\texttt{eff}}^{(t)}$ (at iteration $t$) be defined as
		\begin{align}\label{def:effective_signal_large_sample}
		\bar{U}_{\ast,\texttt{eff}}^{(t)}\equiv \mu_\ast \bar{\mathfrak{m}}_W^{(t),\top}+W_1^{(0)}\bar{\mathfrak{M}}_W^{(t),\top} \in \R^{n\times q}.
		\end{align}
	\end{enumerate}
\end{definition}

With the notation above, we have the following simplified version of Theorem \ref{thm:represent_learning} in the regime $\phi \gg 1$; its proof can be found in Section \ref{subsection:proof_represent_learning_large_sample}.
\begin{theorem}\label{thm:represent_learning_large_sample}
	Fix $t\in \N$. Suppose $\phi^{-1}\leq K$ and Assumptions (A3)-(A5) in Assumption \ref{assump:gd_NN} hold for some $K,\Lambda\geq 2$ and $r_0\geq 0$. Then for any $\mathfrak{r}\geq 1$, there exists some $c_t=c_t(t,q,L,\mathfrak{r},r_0)>0$ such that  
	\begin{align*}
	 \E^{(0)}\Big[\mathfrak{d}_{\mathrm{BL}}^{(X,Y)}\Big(f_{\bm{W}^{(t)}}(x), \mathfrak{h}_{\bar{W}^{(t)}_{[2:L]}}\big(\bar{U}_{\ast,\texttt{eff}}^{(t),\top}x\big)\Big)\Big]^{\mathfrak{r}}\leq \big(K\Lambda\varkappa_{\ast}\big)^{c_t}\cdot \big[(1\vee \phi)^{c_t}\cdot n^{-1/c_t}+\phi^{-1/c_t}\big].
	\end{align*}
	Here the conditional bounded-Lipschitz metric $\mathfrak{d}_{\mathrm{BL}}^{(X,Y)}(\cdot,\cdot)$ is defined for pairs of $\R^q$-valued random variables over $x\sim \mathcal{N}(0,I_n)$ independent of all other variables.
\end{theorem}

Simply put, the above theorem shows that in the large sample regime $\phi\gg 1$, the high dimensional Gaussian noise in (\ref{ineq:represent_learning_1}) vanishes and the effective signal ${U}_{\ast,\texttt{eff}}^{(t)}$  therein is replaced by $\bar{U}_{\ast,\texttt{eff}}^{(t)}$ defined in (\ref{def:effective_signal_large_sample}).

As noted in the Introduction, Theorem~\ref{thm:represent_learning_large_sample} may offer insight into the sample complexity of gradient descent training. In the optimal regime with sample complexity $\bigo(n)$ (i.e., $\phi = m/n \gg 1$), the additional Gaussian noise term $\mathsf{Z}_{\texttt{hd}}^{(t)}$ in (\ref{ineq:represent_learning_1}) becomes negligible, and the remaining challenge is to understand how the \emph{deterministic} effective link and effective signal approximate the unknown single-index function as the iteration $t \to \infty$. We leave this question to future work.

\section{Simulation studies}\label{section:simulation}

\subsection{Common simulation parameters}\label{subsection:simulation_parameter}
Some numerical experiments in this section share the following common simulation parameters:
\begin{itemize}
    \item The single-index function is taken as $\varphi_*(\cdot) = \tanh(\cdot)$.
    \item The sample size is $m=300$ and the feature dimension is $n=600$.
    \item The signal $\mu_* \in \mathbb{R}^n$ is drawn from $\mathcal{N}(0,I_n/n)$ and is fixed throughout the simulation. The errors $\xi_i$ are i.i.d. samples from $\mathcal{N}(0,\sigma_\xi^2)$.
    \item The data matrix $X$ has i.i.d. entries following either $\mathcal{N}(0,1)$ (green) or a normalized \textit{t}-distribution with $10$ degrees of freedom (red) with unit variance.
    \item Algorithm \ref{alg:aug_gd_nn} is initialized with $n^{1/2} W_1^{(0)}$ and $\{q^{1/2} W_\alpha^{(0)}\}_{\alpha \in [2:L]}$ having i.i.d. $\mathcal{N}(0,1)$ entries and employs a step size of $\eta=2$.
\end{itemize}
The neural network architectures will be specified below in each simulation.

\subsection{Algorithmic estimation of the generalization error} \label{sec:estimation_error}\label{subsection:alg_est_gen_err}

\begin{figure}[!t]
	\centering
	\includegraphics[width=0.99\linewidth]{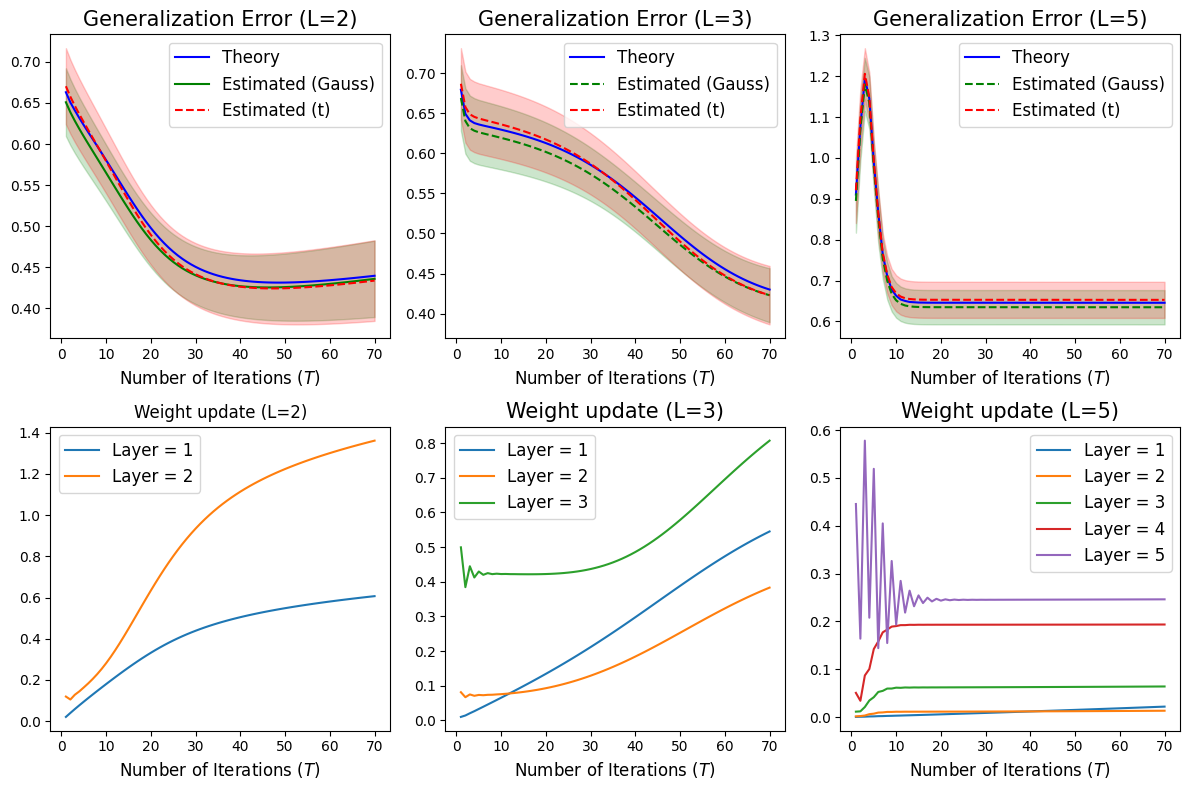}
	\caption{Algorithmic estimation of the generalization error.}
	\label{fig:estimation_error}
\end{figure}

In Figure \ref{fig:estimation_error}, we examine the numerical performance of $\hat{\mathscr{E}}_{\texttt{test}}^{(t)}(X,Y)$  proposed in Algorithm \ref{alg:aug_gd_nn}. We consider the setting in Section \ref{subsection:simulation_parameter} and neural networks with activation functions $\{\sigma_\alpha(x) = 1/(1 +e^{-x})\}_{\alpha \in [1:L-1]}$, width $q=10$, and the number of layers varying among $L=2$ (left), $L=3$ (middle), and $L=5$ (right). The noise level is set as $\sigma_\xi^2=1/4$ and Algorithm \ref{alg:aug_gd_nn} is run for $70$ iterations with $120$ Monte Carlo repetitions.

We report two sets of numerical results in Figure~\ref{fig:estimation_error}:
\begin{itemize}
    \item (\emph{Upper row}). We plot the generalization error estimate $\hat{\mathscr{E}}_{\texttt{test}}^{(t)}(X,Y)$ computed by Algorithm \ref{alg:aug_gd_nn} against its theoretical value $\mathscr{E}_{\texttt{test}}^{(t)}(X,Y)$.
    \item (\emph{Lower row}). We plot the (Monte Carlo averaged) relative distance $\pnorm{W_\alpha^{(t)}-W_\alpha^{(0)}}{}/\pnorm{W_\alpha^{(0)}}{}$ for each layer $\alpha \in [1:L]$.
\end{itemize}

From Figure \ref{fig:estimation_error}, we observe the following:
\begin{enumerate}
    \item In each case with $L=2,3,5$ in the upper row, the algorithmic estimate $\hat{\mathscr{E}}_{\texttt{test}}^{(t)}(X,Y)$ closely matches the theoretical generalization error at every iteration $t$. This phenomenon holds universally for both Gaussian and non-Gaussian data.
    \item For the two-layer case $L=2$, the theoretical generalization error begins to increase during training (around iteration $40$). The proposed algorithmic estimate $\hat{\mathscr{E}}_{\texttt{test}}^{(t)}(X,Y)$ accurately captures this trend and can be used to determine an early stopping point.
    \item In each case with $L=2,3,5$ in the lower row, the weights of the first (and other shallow) layers move non-trivially from their initialization beyond the lazy training regime. The apparently smaller update observed when $L=5$ is likely due to instability in the fifth-layer updates, making them visually minor relative to those in other layers.
\end{enumerate}

In Figure \ref{fig:estimation_error_smallnoise} in Appendix \ref{section:simulation_additional}, we report simulations with noise level $\sigma_\xi^2=0$ (noiseless). We observe that the accuracy (error bars) of Algorithm \ref{alg:aug_gd_nn} can be  substantially improved in the noiseless setting (= a large signal-to-noise ratio).

\subsection{Robustness of Algorithm \ref{alg:aug_gd_nn} for multi-index models}\label{subsection:alg_est_gen_err_multi}

\begin{figure}
    \centering
    \includegraphics[width=0.99\linewidth]{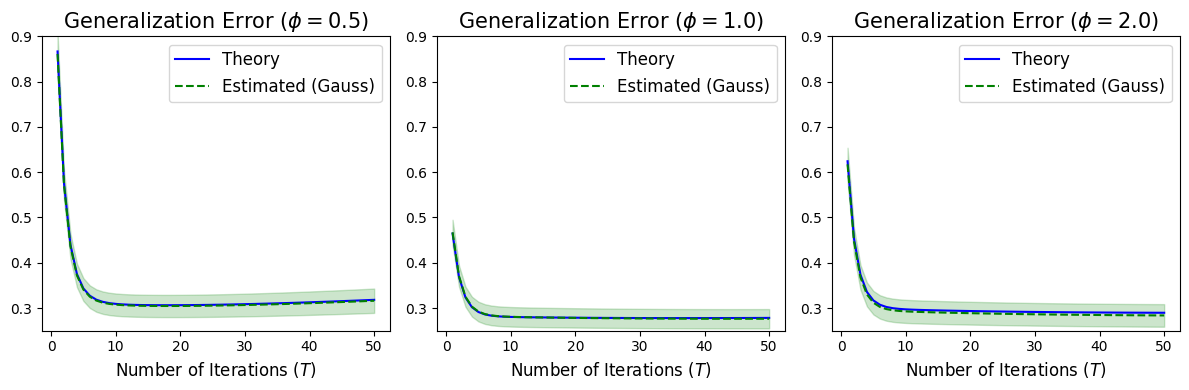}
    \caption{Algorithmic estimation of the generalization error with a multi-index model.}
    \label{fig:multi-index}
\end{figure}

In Figure~\ref{fig:multi-index}, we examine the robustness of Algorithm~\ref{alg:aug_gd_nn} when the regression model (\ref{def:model}) is misspecified. As mentioned in the Introduction, Theorem~\ref{thm:est_test_err} can be generalized to  the class of multi-index models with further technical work. Here we numerically verify this claim by considering a multi-index model $Y_i=\tanh(\pnorm{U_\ast X_i}{})+\xi_i$, $i \in [m]$, where the entries signal $U_\ast \in \R^{10\times n}$ are i.i.d. drawn from $\mathcal{N}(0,1/n)$ and fixed. We consider a two-layer neural network with activation function $\sigma_1(x) = 1/(1 + e^{-x})$ and width $q=10$. We set the sample size at $m=300$ with three aspect ratios $\phi = 0.5$ (left), $\phi = 1$ (middle) and $\phi = 2$ (right).  The noise level is set as $\sigma_\xi^2=1/4$. All other simulation parameters follow Section \ref{subsection:simulation_parameter}. For simplicity we only plot in Figure~\ref{fig:multi-index} the Gaussian feature case, where Algorithm~\ref{alg:aug_gd_nn} is run for $70$ iterations with $120$ Monte Carlo repetitions.

We observe that, similar to Figure~\ref{fig:estimation_error}, the algorithmic estimate $\hat{\mathscr{E}}_{\texttt{test}}^{(t)}(X,Y)$ remains valid for the prescribed multi-index model at each iteration. In Figure \ref{fig:multi-index_noiseless} in Appendix \ref{section:simulation_additional}, we report the same simulations in the noiseless case and observe substantially further reduced error bars in similar spirit to Figure \ref{fig:estimation_error_smallnoise} mentioned above.

\subsection{Some cautions of Algorithm \ref{alg:aug_gd_nn}}\label{subsection:simulation_caution}

We present two practical cautions regarding Algorithm~\ref{alg:aug_gd_nn} when the assumptions of Theorem~\ref{thm:est_test_err} are violated. All simulation parameters follows Section \ref{subsection:simulation_parameter} unless otherwise specified. For simplicity we only present simulations for Gaussian features.

\subsubsection{Effect of wide networks}

In Figure~\ref{fig:wide_network}, we examine the performance of Algorithm~\ref{alg:aug_gd_nn} for wide neural networks. Specifically, we consider two-layer neural networks with activation functions $\{\sigma_\alpha(x) = 1/(1 +e^{-x})\}_{\alpha \in [1:L-1]}$, with width $q/m=0.1$ (left), $q/m=0.2$ (middle), and $q/m=0.5$ (right). The noise level is set as $\sigma_\xi^2=1/4$, and Algorithm \ref{alg:aug_gd_nn} is run for $50$ iterations with fewer $30$ Monte Carlo repetitions due to large computation cost for large $q$'s.

Figure~\ref{fig:wide_network} shows that as the ratio $q/m$ increases, the theoretical generalization error and its algorithmic estimate $\hat{\mathscr{E}}_{\texttt{test}}^{(t)}(X,Y)$ increasingly deviate from each other. This indicates that the validity of Algorithm \ref{alg:aug_gd_nn} is intrinsic to finite width neural networks and essential further corrections will be necessary for wide neural networks. We again note that in practical examples (cf. Table \ref{tab:model-summary}), the width of neural networks is relatively small, so our numerical findings here are mainly of theoretical interest.

\begin{figure}[!t]
	\centering
	\includegraphics[width=0.99\linewidth]{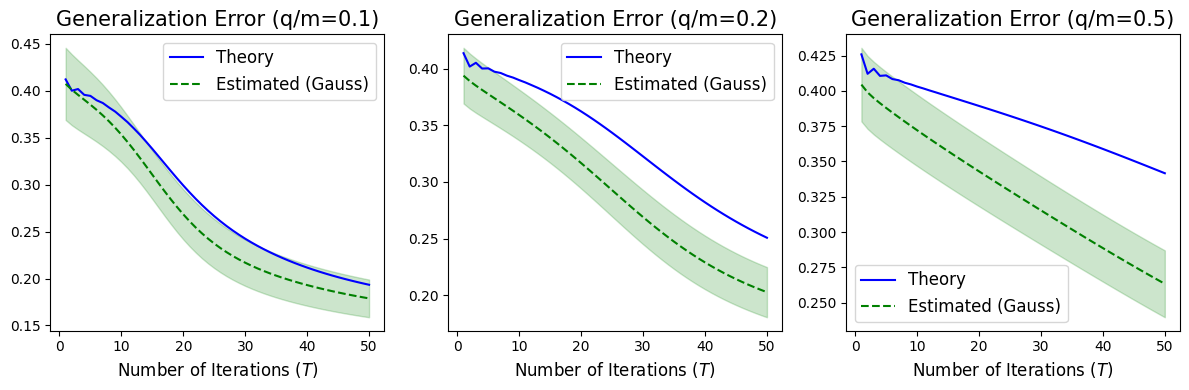}
	\caption{Effect of wide networks.}
	\label{fig:wide_network}
\end{figure}

\subsubsection{Effect of non-regular activations}

In Figure~\ref{fig:relu_activation}, we examine the effect of regularity of the activations on Algorithm~\ref{alg:aug_gd_nn} for two-layer neural networks:
\begin{itemize}
    \item In the left panel, we consider the ReLU activation $\sigma_1(x) = (x)_+$ which only admits a weak first derivative. We set $\sigma_1''(\cdot)\equiv 0$ in the implementation of Algorithm~\ref{alg:aug_gd_nn}.
    \item In the right panel, we consider the smoothed ReLU activation $\sigma_2(x) = (x^2/2)\bm{1}_{x \in (0,1)} + (x-1/2)\bm{1}_{x \geq 1}$ with weak second derivative $\sigma_1''(x) = \bm{1}_{x \in (0,1)}$.
\end{itemize}
In both panels, we set $\sigma_\xi^2=0$ to remove the effect of noises, and Algorithm \ref{alg:aug_gd_nn} is run for $50$ iterations with $120$ Monte Carlo repetitions.

As shown in Figure~\ref{fig:relu_activation}, for the ReLU activation, Algorithm~\ref{alg:aug_gd_nn} tends to become unstable with large fluctuations, and the Monte Carlo averages gradually drift away from the theoretical generalization error as the number of iterations increases. This suggests that it is invalid to simply set $\sigma_\alpha''(\cdot)\equiv 0$ for the ReLU activation, and that further essential corrections are needed. On the other hand, for activations with `minimal regularity' such as the smoothed ReLU with a weak second derivative as conjectured in Remark~\ref{rmk:alg_smooth}, Algorithm~\ref{alg:aug_gd_nn} appears to remain valid although Theorem~\ref{thm:est_test_err} formally assumes more regularity.

\section{Proof sketch of Theorem \ref{thm:se_gd_NN}}\label{section:proof_sketch}

Recall $n_{L,q}= n+(L-1)q$. We define some further notation:
\begin{align}\label{def:GFOM_notation}
A &\equiv 
\begin{pmatrix}
X/\sqrt{n} & 0_{m\times (L-1)q}
\end{pmatrix}
\in \R^{m\times n_{L,q}},\nonumber\\
\overline{\mu}_{\ast,n}  &\equiv 
\begin{pmatrix}
n^{1/2}\mu_\ast & 0_{n\times (q-1)}\\
0_{(L-1)q\times 1} & 0_{(L-1)q\times (q-1)}
\end{pmatrix}
\in \R^{n_{L,q}\times q},\nonumber\\
\bm{W}_n^{(0)} &\equiv \big[ \big( (\bm{1}_{\alpha\neq 1}+\bm{1}_{\alpha=1}n^{1/2})\cdot  W^{(0)}_\alpha\big)_{\alpha \in [1:L]} \big]\in \R^{n_{L,q}\times q}.
\end{align}

The proof of Theorem \ref{thm:se_gd_NN} relies on an extension of a theoretical machinery for characterizing the mean-field behavior of a large class of general first order methods (GFOMs) developed in  \cite{han2024entrywise}. Specifically, we need a version of the GFOM theory in \cite{han2024entrywise} that deals with matrix-variate iterations. We provide these extensions along with some other related delocalization results in Appendix \ref{section:GFOM_theory} for the convenience of the readers.

\begin{figure}[!t]
	\centering
	\includegraphics[width=0.79\linewidth]{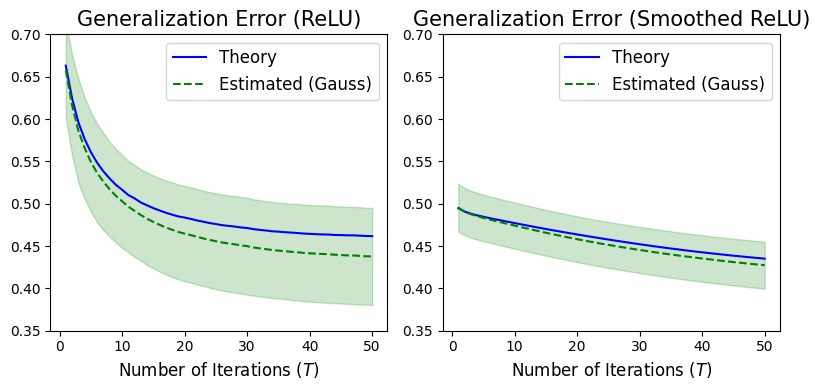}
	\caption{Effect of non-regular activations.}
	\label{fig:relu_activation}
\end{figure}

\subsection{Reformulation of gradient descent}

We shall first reformulate the gradient descent iteration. Let the initialization be
\begin{align*}
\hat{u}^{(-1)} = 0_{m\times q},\quad \hat{v}^{(-1)}\equiv \overline{\mu}_{\ast,n}
\in \R^{n_{L,q}\times q}.
\end{align*}
For $t=0$, let
\begin{align*}
\hat{u}^{(0)}& \equiv A \hat{v}^{(-1)} \in \R^{m\times q},\quad \hat{v}^{(0)} \equiv \bm{W}_n^{(0)}\in \R^{n_{L,q}\times q}.
\end{align*}
Given $\{(u^{(s)},v^{(s)})\}_{s \in [-1:t-1]}$ computed up to iteration $t-1$, by writing 
\begin{align}\label{def:hat_P_G_R}
\mathscr{P}_{  \hat{v}^{(s)}_{[2:L]} }^{(\alpha:L]}\equiv \hat{\mathscr{P}}_{(s)}^{(\alpha:L]},\quad \mathscr{G}_{\hat{v}^{(s)}_{[2:\alpha]} }\equiv \hat{\mathscr{G}}_{(s);\alpha},\quad \mathscr{R}_{\hat{v}^{(s)}_{[2:L]} }\equiv \hat{\mathscr{R}}_{(s)},
\end{align}
for iteration $t\geq 1$, let
\begin{align}\label{def:GD_GFOM_NN}
&\hat{u}^{(t)}\equiv A \hat{v}^{(t-1)} \in \R^{m\times q},\\
&\hat{v}^{(t)}
= A^\top 
\Big[
-\phi^{-1}\eta^{(t-1)}_1\cdot \hat{\mathscr{P}}_{ (t-1) }^{(1:L]}\Big(\hat{u}^{(t)},\hat{\mathscr{R}}_{(t-1)}(\hat{u}^{(t)},\hat{u}^{(0)})\Big) \odot \sigma_1'(\hat{u}^{(t)})\Big]+\hat{v}^{(t-1)}\nonumber\\
&\, -
\begin{pmatrix}
0_{n\times q}\\
\Big\{\eta^{(t-1)}_{\alpha}\cdot m^{-1}\hat{\mathscr{G}}_{(t-1);\alpha-1}(\hat{u}^{(t)})^\top \Big[\hat{\mathscr{P}}_{ (t-1) }^{(\alpha:L]}\Big(\hat{u}^{(t)},\hat{\mathscr{R}}_{(t-1)}(\hat{u}^{(t)},\hat{u}^{(0)})\Big) \odot \hat{\mathscr{G}}_{(t-1);\alpha}'(\hat{u}^{(t)})\Big]\Big\}_{\alpha \in [2:L]}
\end{pmatrix}.\nonumber
\end{align}
We may then identify,
\begin{align}\label{def:GD_GFOM_NN_v}
\hat{v}^{(t)} \equiv \big[ \big( (\bm{1}_{\alpha\neq 1}+\bm{1}_{\alpha=1}n^{1/2})\cdot  W^{(t)}_\alpha\big)_{\alpha \in [1:L]} \big].
\end{align}
Unfortunately the form in (\ref{def:GD_GFOM_NN}) does not allow for a direct application of the matrix-variate GFOM theory in Appendix \ref{section:GFOM_theory} due to the complicated, non-separable dependence structure in the second term of $\hat{v}^{(t)}$ in (\ref{def:GD_GFOM_NN}). 

The key idea from here is to construct a sequence of auxiliary GFOM iterates $\{(u^{(t)},v^{(t)})\}$ that is close to $\{(\hat{u}^{(t)},\hat{v}^{(t)})\}$ in (\ref{def:GD_GFOM_NN}), while enjoying a state evolution characterization via the GFOM theory.

\subsection{Auxiliary GFOM iterates}
For $t\geq 0$, let
\begin{align}\label{def:auxiliary_GFOM_F_G}
&\mathsf{F}_t^{\langle 1\rangle}(v^{([-1:t-1])})\equiv v^{(t-1)}\bm{1}_{t\geq 1}+ \overline{\mu}_{\ast,n}\bm{1}_{t=0},\nonumber \\
& \mathsf{F}_t^{\langle 2\rangle}(v^{([-1:t-1])})\equiv  \bigg[v^{(t-1)}-
\begin{pmatrix}
0_{n\times q}\\
\big\{\eta_\alpha^{(t-1)} M_\alpha^{(t-1)}\big\}_{\alpha \in [2:L]}
\end{pmatrix}\bigg] \bm{1}_{t\geq 1}+ \bm{W}_n^{(0)} \bm{1}_{t=0},\nonumber\\
& \mathsf{G}_{t}^{\langle 1\rangle}(u^{([-1:t-1])})\equiv 0_{m\times q},\nonumber\\
& \mathsf{G}_t^{\langle 2\rangle}(u^{([-1:t])})\equiv -\phi^{-1}\eta^{(t-1)}_1\cdot \mathfrak{S}\big(u^{(t)},u^{(0)}, V^{(t-1)} \big)\bm{1}_{t\geq 1}.
\end{align}
Here $V^{(t)}\equiv (V^{(t)}_\alpha)_{\alpha \in [2:L]}$ is defined as in the state evolution in Definition \ref{def:gd_NN_se}, and for $\alpha \in [2:L]$ and $t\geq 1$, we further let 
\begin{align}\label{def:M_aux_GFOM}
M_{\alpha}^{(t-1)}&\equiv  \E^{(0)}\mathscr{G}_{V^{(t-1)}_{[2:\alpha-1]};\pi_m }\big(\Theta_{\pi_m}^{(t)}(\mathfrak{U}^{([0:t])}) \big) \nonumber\\
&\qquad\qquad \times  \big(\mathfrak{T}_{\alpha}\big)_{\pi_m}\big(\Theta_{\pi_m}^{(t)}(\mathfrak{U}^{([0:t])}) ,\mathfrak{U}^{(0)},V^{(t-1)}\big)^\top \in \mathbb{M}_q.
\end{align}
 We define the \emph{auxiliary GFOM iteration} $\{(u^{(t)},v^{(t)})\}_{t\in \mathbb{Z}_{\geq -1}}$ as follows: with initialization $u^{(-1)} = 0_{m\times q}$, $v^{(-1)}\equiv \overline{\mu}_{\ast,n}
\in \R^{n_{L,q}\times q}$, for $t=0,1,2,\ldots,$ let 
\begin{align}\label{def:GFOM_auxiliary_NN}
\begin{cases}
u^{(t)} = A \mathsf{F}_t^{\langle 1\rangle}(v^{([-1:t-1])})+ \mathsf{G}_{t}^{\langle 1\rangle}(u^{([-1:t-1])}) \in \R^{m\times q},\\
v^{(t)}= A^\top \mathsf{G}_t^{\langle 2\rangle}(u^{([-1:t])})+\mathsf{F}_t^{\langle 2\rangle}(v^{([-1:t-1])})\in \R^{n_{L,q}\times q}.
\end{cases}
\end{align}
In other words, for $t=0$,
\begin{align*}
u^{(0)}& \equiv A v^{(-1)} \in \R^{m\times q},\quad v^{(0)} \equiv \bm{W}_n^{(0)}\in \R^{n_{L,q}\times q},
\end{align*}
and for iteration $t\geq 1$,
\begin{align}\label{def:GFOM_auxiliary_NN_1}
u^{(t)}&= A v^{(t-1)},\\
v^{(t)}& =  A^\top 
\Big[
-\phi^{-1}\eta^{(t-1)}_1\cdot  \mathfrak{S}\big(u^{(t)},u^{(0)},V^{(t-1)}\big) \Big]+v^{(t-1)}-
\begin{pmatrix}
0_{n\times q}\\
\Big\{\eta^{(t-1)}_{\alpha} M_\alpha^{(t-1)}\Big\}_{\alpha \in [2:L]}
\end{pmatrix}.\nonumber
\end{align}
As will be detailed below, the iterates $\{(u^{(t)},v^{(t)})\}$ defined above satisfy the desired properties. An important feature of (\ref{def:GFOM_auxiliary_NN}) is that the iteration $v^{(t)}$ is defined only if the empirical distribution of $u^{(t)}$ is determined in state evolution theory via $\mathfrak{U}^{(t)}$. This successive defining nature will play an important role in our iterative reduction scheme to be described below.

\subsection{Proof sketch of Theorem \ref{thm:se_gd_NN}: Iterative reduction scheme}

\subsubsection{State evolution for the auxiliary GFOM}
In the first step, we prove that the auxiliary GFOM iterates $\{(u^{(t)},v^{(t)})\}$ are amenable to an exact state evolution characterization. The details of the proof can be found in Section \ref{subsection:proof_se_aux_GFOM}.

\begin{proposition}\label{prop:se_aux_GFOM}
	Suppose Assumption \ref{assump:gd_NN} holds for some $K,\Lambda\geq 2$ and $r_0\geq 0$. Fix a sequence of $\Lambda_\psi$-pseudo-Lipschitz functions $\{\psi_k:\R^{q[0:t]} \to \R\}_{k \in [m\vee n]}$ of order $\mathfrak{p}$, where $\Lambda_\psi\geq 2$. Then for any $\mathfrak{r}\geq 1$, there exists some $c_t=c_t(t,\mathfrak{p},q,L,\mathfrak{r},r_0)>0$ such that 
	\begin{align*}
	&\E^{(0)} \bigg[\biggabs{\frac{1}{m}\sum_{k \in [m]} \psi_k\big(u_k^{([0:t])}\big) - \frac{1}{m}\sum_{k \in [m]}  \E^{(0)}  \psi_k\Big(\big\{\Theta_{k}^{(s)}(\mathfrak{U}^{([0:s])})\big\}_{s \in [0:t]}\Big)  }^{\mathfrak{r}}\bigg]\\
	&\qquad + \E^{(0)}  \bigg[\biggabs{\frac{1}{n}\sum_{\ell \in [n]} \psi_\ell\big(v_{1;\ell}^{([0:t])}\big) - \frac{1}{n}\sum_{\ell \in [n]}  \E^{(0)}  \psi_\ell\Big(\big\{\Delta_{\ell}^{(s)}(\mathfrak{V}^{([1:s])} )\big\}_{s \in [0:t]}\Big)   }^{\mathfrak{r}}\bigg]  \\
	&\qquad +\max_{\alpha \in [2:L]} \max_{s \in [1:t]}  \E^{(0)} \pnorm{v_\alpha^{(s)}-V^{(s)}_\alpha}{\infty}^{\mathfrak{r}} \\
	&\leq \big(K\Lambda\varkappa_{\ast}\big)^{c_t}\cdot n^{-1/c_t}. 
	\end{align*}
\end{proposition}

A major technical challenge in applying the meta GFOM state evolution Theorems \ref{thm:GFOM_se_asym} and \ref{thm:GFOM_se_asym_avg} to the auxiliary GFOM iterates $\{(u^{(t)}, v^{(t)})\}$ defined in (\ref{def:GFOM_auxiliary_NN_1}) lies in the weak regularity of the functions $\mathsf{G}_t^{(2)}$ in (\ref{def:auxiliary_GFOM_F_G}), which are not globally Lipschitz. Fortunately, as shown in Proposition \ref{prop:apr_est_det_fcn} below, each $\mathsf{G}_t^{(2)}$ satisfies a suitable \emph{pseudo-Lipschitz} condition as defined in (\ref{cond:pseudo_lip}). In particular, when the iterates $\{(u^{(t)}, v^{(t)})\}$ are delocalized, $\mathsf{G}_t^{(2)}$ can be effectively treated as `almost' Lipschitz continuous. We establish the required delocalization estimates in Proposition \ref{prop:GD_linfty_est}, which in turn enables the proof of Proposition \ref{prop:se_aux_GFOM}. The proof proceeds by introducing a truncated and smoothed version of the GFOM iterates in (\ref{def:GFOM_auxiliary_NN}), which remains close to the original iterates while allowing direct application of the meta GFOM state evolution Theorems \ref{thm:GFOM_se_asym} and \ref{thm:GFOM_se_asym_avg}.

The proof relies crucially on both the non-asymptotic and entrywise nature of the matrix-variate GFOM theory in Theorems \ref{thm:GFOM_se_asym} and \ref{thm:GFOM_se_asym_avg}. First, since the delocalization estimates for $\{(u^{(t)}, v^{(t)})\}$ necessarily involve poly-logarithmic factors in $n$, a non-asymptotic GFOM theory with polynomial-in-$n$ error bounds is essential. Second, proving $v^{(s)}_\alpha\approx V_\alpha^{(s)}$ requires the GFOM state evolution theory in its strongest, \emph{entrywise} form.

\subsubsection{Error control between the GD and the auxiliary GFOM}

Recall the gradient descent iterate $\{(\hat{u}^{(t)},\hat{v}^{(t)})\}$ in (\ref{def:GD_GFOM_NN}), and we write 
\begin{align*}
\Delta \hat{u}^{(t)}\equiv \hat{u}^{(t)}- u^{(t)},\quad \Delta \hat{v}^{(t)}\equiv \hat{v}^{(t)}- v^{(t)}.
\end{align*}
In the second step, we will use an iterative reduction scheme to prove the following error control between the gradient descent iterates (\ref{def:GD_GFOM_NN}) and the auxiliary GFOM iterates (\ref{def:GFOM_auxiliary_NN_1}).
\begin{proposition}\label{prop:err_gd_GFOM}
	Suppose Assumption \ref{assump:gd_NN} holds for some $K,\Lambda\geq 2$ and $r_0\geq 0$. Then for $\mathfrak{r}\geq 1$, there exists some $c_t=c_t(t,q,L,\mathfrak{r},r_0)>0$ such that 
	\begin{align*}
	& \E^{(0)}\big(n^{-1/2}\pnorm{\Delta \hat{u}^{(t)} }{}\big)^{\mathfrak{r}} + \E^{(0)}\big(n^{-1/2}\pnorm{\Delta \hat{v}_1^{(t)} }{}\big)^{\mathfrak{r}}  +\E^{(0)} \pnorm{\Delta \hat{v}_{[2:L]}^{(t)}}{\infty}^{\mathfrak{r}}\leq \big(K\Lambda\varkappa_{\ast}\big)^{c_t}\cdot n^{-1/c_t}.
	\end{align*}
\end{proposition}

At a high level, the key to the proof of Proposition \ref{prop:err_gd_GFOM} is to provide successive controls for $(\Delta \hat{u}^{(1)},\Delta \hat{v}^{(1)}),\ldots,(\Delta \hat{u}^{(t)},\Delta \hat{v}^{(t)})$ via an alternating residual estimate and a state evolution characterization due to the successive defining nature of the auxiliary GFOM (\ref{def:GFOM_auxiliary_NN}). To get a sense of how this is done, suppose we already have error controls by iteration $t-1$. The error control for $\Delta \hat{u}^{(t)}$ is trivial by that of $\Delta \hat{v}^{(t-1)}$ in view of the definition of $\hat{u}^{(t)},u^{(t)}$. The non-trivial part lies in the control of $\Delta \hat{v}^{(t)}$, and the most non-trivial part is to prove
\begin{align}\label{eqn:err_gd_GFOM_1}
m^{-1}\hat{\mathscr{G}}_{(t-1);\alpha-1}(\hat{u}^{(t)})^\top \Big[\hat{\mathscr{P}}_{ (t-1) }^{(\alpha:L]}\Big(\hat{u}^{(t)},\hat{\mathscr{R}}_{(t-1)}(\hat{u}^{(t)},\hat{u}^{(0)})\Big) \odot \hat{\mathscr{G}}_{(t-1);\alpha}'(\hat{u}^{(t)})\Big]\approx M_\alpha^{(t-1)}.
\end{align}
To prove (\ref{eqn:err_gd_GFOM_1}), we first replace all quantities in the gradient descent on the left hand side of (\ref{eqn:err_gd_GFOM_1}) by the corresponding ones in the auxiliary GFOM,  at the cost of an error term depending on $(\Delta \hat{u}^{(1)},\Delta \hat{v}^{(1)}),\ldots,(\Delta \hat{u}^{(t-1)},\Delta \hat{v}^{(t-1)}), \Delta \hat{u}^{(t)}$. Once this replacement is done, we then use the state evolution characterization in Proposition \ref{prop:se_aux_GFOM} for $(u^{(1)},v^{(1)}),\ldots,(u^{(t-1)},v^{(t-1)}), u^{(t)}$ to relate to $M_\alpha^{(t-1)}$ on the right hand side of (\ref{eqn:err_gd_GFOM_1}). In other words, we may control the error term $\Delta \hat{v}^{(t)}$ by the error terms up to $\Delta \hat{u}^{(t)}$ along with the error estimates in Proposition \ref{prop:se_aux_GFOM}, and thereby completing the inductive step for the error control at iteration $t$. Details of this iterative scheme can be found in Section \ref{subsection:proof_err_gd_GFOM}.

\subsubsection{Putting pieces together}

With Propositions \ref{prop:se_aux_GFOM} and \ref{prop:err_gd_GFOM} at hand, we expect that the state evolution characterization for the auxiliary GFOM iterates $\{(u^{(t)},v^{(t)})\}$ in (\ref{def:GFOM_auxiliary_NN_1}) carries over to the original gradient descent iterates $\{(\hat{u}^{(t)},\hat{v}^{(t)})\}$ in (\ref{def:GD_GFOM_NN}). The details of the proof can be found in Section \ref{subsection:proof_se_gd_NN}.

\section{Proof preliminaries}\label{section:proof_preliminaries}

In this section we provide a few apriori and stability estimates, along with some derivative formulae that will be used in an essential way throughout the proofs. In the sequel, for any $\alpha \in [2:L]$ and any $v_{[2:\alpha]}\in (\mathbb{M}_q)^{[2:\alpha]}$, we write 
\begin{align*}
\kappa_{v_{[2:\alpha]}} \equiv \prod_{\beta \in [2:\alpha]}  \big(1\vee \pnorm{v_\beta}{\op}\big).
\end{align*}
For notational consistency, we write $\kappa_{v_{[2:1]}}\equiv 1$.

\subsection{Apriori and stability estimates for the theoretical functions}
First we provide apriori and stability estimates for the theoretical functions defined in Definitions \ref{def:mean_field_fcn} and \ref{def:mean_field_S_T}; the proof can be found in Section \ref{subsection:proof_apr_est_det_fcn}.
\begin{proposition}\label{prop:apr_est_det_fcn}
	Suppose that $\max_{p=1,2} \max_{\alpha \in [1:L]} \big(\pnorm{\sigma_\alpha^{(p)}}{\infty}+\abs{\sigma_\alpha(0)}\big)\leq \Lambda$ for some $\Lambda\geq 2$. Then the following hold for some universal constant $c_0>0$:
	\begin{enumerate}
		\item For $u\in \R^{m\times q}$, $v_{[2:\alpha]}\in (\mathbb{M}_q)^{[2:\alpha]}$, and any $k \in [m]$,
		\begin{align*}
		&\pnorm{\mathscr{H}_{v_{[2:\alpha]};k}(u_{k\cdot})}{\infty}\vee \pnorm{\mathscr{G}_{v_{[2:\alpha]};k}(u_{k\cdot})}{\infty}\leq (q\Lambda)^{c_0 L}\kappa_{v_{[2:\alpha]}}\cdot (1+\pnorm{u_{k\cdot}}{\infty}).
		\end{align*}
		\item For any $ u,u' \in \R^{m\times q}$, $v_{[2:\alpha]},v_{[2:\alpha]}'\in (\mathbb{M}_q)^{[2:\alpha]}$, and any $k \in [m]$,
		\begin{align*}
		& \max_{\# \in \{\mathscr{H},\mathscr{G},\mathscr{G}'\}}\pnorm{\#_{v_{[2:\alpha]};k}(u_{k\cdot}) - \#_{v_{[2:\alpha]}';k }(u_{k\cdot}')}{}\leq (q\Lambda)^{c_0 L}  \kappa_{v_{[2:\alpha]}}\kappa_{v_{[2:\alpha]}'}\\
		&\qquad \times  \Big[\big(1+\pnorm{u_{k\cdot}}{\infty}\wedge \pnorm{u_{k\cdot}'}{\infty}\big)\cdot  \pnorm{v_{[2:\alpha]}-v_{[2:\alpha]}'}{\infty}+\pnorm{u_{k\cdot}-u_{k\cdot}'}{}\Big].
		\end{align*}
		
		\item For any $u,u' \in \R^{m\times q},v_{[2:\alpha]}, v'_{[2:\alpha]}\in (\mathbb{M}_q)^{[2:\alpha]}$, $z,z' \in \R^{m\times q}$, and any $k \in [m]$,
		\begin{align*}
		&\bigpnorm{\mathscr{P}_{u,v_{[2:\alpha]};k}(z_{k\cdot})- \mathscr{P}_{u',v'_{[2:\alpha]};k}(z_{k\cdot}')}{}\\
		&\leq (q\Lambda)^{c_0 L} \big(\kappa_{v_{[2:\alpha]}}\kappa_{v_{[2:\alpha]}'}\big)^2\cdot \Big\{\pnorm{z_{k\cdot}-z'_{k\cdot}}{}+\big(1+\pnorm{z_{k\cdot}}{\infty}\wedge \pnorm{z_{k\cdot}'}{\infty}\big)\\
		&\qquad \qquad \times  \big[\pnorm{u_{k\cdot}-u'_{k\cdot}}{}+\big(1+\pnorm{u_{k\cdot}}{\infty}\wedge \pnorm{u_{k\cdot}'}{\infty}\big)\cdot \pnorm{v_{[2:\alpha]}-v_{[2:\alpha]}'}{\infty}  \big]  \Big\}.
		\end{align*}
		\item For $u \in \R^{m\times q},v\in (\mathbb{M}_q)^{[2:L]}$, $z \in \R^{m\times q}$, and any $k \in [m]$,
		\begin{align*}
		\bigpnorm{ \mathscr{P}_{v;k}^{(\alpha:L]}(u_{k\cdot},z_{k\cdot}) }{\infty} &\leq (q\Lambda)^{c_0L} \kappa_{v}\cdot \pnorm{z_{k\cdot}}{\infty}.
		\end{align*}
		\item For any $u,u' \in \R^{m\times q},v, v'\in (\mathbb{M}_q)^{[2:L]}$, $z,z' \in \R^{m\times q}$, and any $k \in [m]$,
		\begin{align*}
		&\bigpnorm{\mathscr{P}_{v;k}^{(\alpha:L]}(u_{k\cdot},z_{k\cdot})-\mathscr{P}_{v';k}^{(\alpha:L]}(u_{k\cdot}',z_{k\cdot}')}{}\\
		& \leq (q\Lambda)^{c_0 L} (\kappa_{v}\kappa_{v'})^4 \cdot  \Big\{\pnorm{z_{k\cdot}-z_{k\cdot}'}{\infty}+ \big(1+\pnorm{z_{k\cdot}}{\infty}\wedge \pnorm{z_{k\cdot}'}{\infty}\big) \\
		& \qquad \times  \Big[ \pnorm{u_{k\cdot}-u_{k\cdot}'}{}+ \big(1+\pnorm{u_{k\cdot}}{\infty}\wedge \pnorm{u_{k\cdot}'}{\infty}\big)\cdot \pnorm{v-v'}{\infty}\Big]\Big\}.
		\end{align*}
		\item For $u \in \R^{m\times q},v\in (\mathbb{M}_q)^{[2:L]}$, $w \in \R^{m\times q}$, and $k \in [m]$,
		\begin{align*}
		\pnorm{\mathfrak{S}_{k}\big(u_{k\cdot},w_{k\cdot}, v\big)}{\infty} 
		&\leq (q\Lambda)^{c_0 L} \kappa_{v}^2\cdot\big(1+\pnorm{u_{k\cdot}}{\infty}+\pnorm{\varphi_\ast(w_{k\cdot})}{\infty}+\abs{\xi_k}\big).
		\end{align*}
		The same estimate holds for $(\mathfrak{T}_{\alpha})_k\big(u,w, v\big)$, $\alpha \in [2:L]$.
		\item For any $u,u' \in \R^{m\times q},v, v'\in (\mathbb{M}_q)^{[2:L]}$, $z,z' \in \R^{m\times q}$, and $k \in [m]$,
		\begin{align*}
		&\pnorm{\mathfrak{S}_{k}(u_{k\cdot},w_{k\cdot}, v)-\mathfrak{S}_{k}(u_{k\cdot}',w_{k\cdot}',v')}{}\\
		&\leq (q \Lambda)^{c_0 L} (\kappa_{v}\kappa_{v'})^{c_0}\cdot \Big\{\pnorm{\varphi_\ast(w_{k\cdot})-\varphi_\ast(w_{k\cdot}')}{}\nonumber\\
		&\qquad  +\min_{(\#,\$)\in \{(u,w),(u',w')\}}\big(1+\pnorm{\#_{k\cdot}}{\infty}+\pnorm{\varphi_\ast(\$_{k\cdot})}{\infty}+\abs{\xi_k}\big)\nonumber\\
		&\qquad\qquad \times \Big[ \pnorm{u_{k\cdot}-u_{k\cdot}'}{}+  (1+\pnorm{u_{k\cdot}}{\infty}\wedge \pnorm{u_{k\cdot}'}{\infty})\cdot \pnorm{v-v' }{\infty}\Big]\Big\}.
		\end{align*}
		The same estimate holds for $(\mathfrak{T}_{\alpha})_k\big(u,w, v\big)$, $\alpha \in [2:L]$.
	\end{enumerate}
\end{proposition}

\subsection{Derivative formulae and estimates for the theoretical functions}

The following derivative formula for the theoretical functions will be used throughout the proofs; the proof can be found in Section \ref{subsection:proof_derivative_formula}. 

\begin{proposition}\label{prop:derivative_formula}
	The following hold.
	\begin{enumerate}
		\item With $\big\{\partial_{u_{k\ell}} \mathscr{H}_{v_{[2:1]} }\equiv e_ke_\ell^\top, k \in [m], \ell \in [q]\big\}$, for $\alpha \in [2:L]$ and $k \in [m],\ell \in [q]$, 
		\begin{align*}
		\partial_{u_{k\ell}}  \mathscr{H}_{v_{[2:\alpha]}}(u)=\big[\sigma_{\alpha-1}'\big(\mathscr{H}_{v_{[2:\alpha-1]}}(u)\big)\odot\partial_{u_{k\ell}}  \mathscr{H}_{v_{[2:\alpha-1]} }(u)\big] v_\alpha \in \R^{m\times q}.
		\end{align*}
		Consequently, 
		\begin{align*}
		\partial_{u_{k\ell}} \mathscr{G}_{v_{[2:\alpha]}}(u)&= \sigma_\alpha'\big(\mathscr{H}_{v_{[2:\alpha]}}(u)\big)\odot \partial_{u_{k\ell}} \mathscr{H}_{v_{[2:\alpha]}}(u)\in \R^{m\times q}, \\
		\partial_{u_{k\ell}} \mathscr{G}_{v_{[2:\alpha]}}'(u)&= \sigma_\alpha''\big(\mathscr{H}_{v_{[2:\alpha]}}(u)\big)\odot \partial_{u_{k\ell}}  \mathscr{H}_{v_{[2:\alpha]}}(u)\in \R^{m\times q},
		\end{align*}
		and $\partial_{u_{k\ell}} \mathscr{R}_{v}(u,w)=\partial_{u_{k\ell}} \mathscr{H}_{v_{[2:L]}}(u,w)$.
		\item For $\alpha \in [2:L]$ and $k \in [m],\ell \in [q]$,
		\begin{align*}
		\partial_{u_{k\ell}}\mathscr{P}_{u, v_{[2:\alpha]} }(z)&\equiv \Big(z\odot \sigma_\alpha''\big(\mathscr{H}_{v_{[2:\alpha]}}(u)\big)\odot \partial_{u_{k\ell}}  \mathscr{H}_{v_{[2:\alpha]}}(u)\Big) v_\alpha^\top\in \R^{m\times q},\\
		\partial_{z_{k\ell}} \mathscr{P}_{u, v_{[2:\alpha]} }(z)&\equiv \Big( e_ke_\ell^\top \odot \sigma_\alpha'\big(\mathscr{H}_{v_{[2:\alpha]}}(u)\big)\Big) v_\alpha^\top\in \R^{m\times q}.
		\end{align*}
		\item  For $\alpha \in [1:L-1]$ and $k \in [m],\ell \in [q]$,
		\begin{align*}
		\partial_{u_{k\ell}} \mathscr{P}_{v}^{(\alpha:L]}(u,z)&=\Big[\mathscr{P}_{v}^{(\alpha+1:L]}(u,z)\odot \sigma_{\alpha+1}''\big(\mathscr{H}_{v_{[2:\alpha+1]}}(u)\big)\odot \partial_{u_{k\ell}}  \mathscr{H}_{v_{[2:\alpha+1]}}(u)\\
		&\qquad + \sigma_{\alpha+1}'\big(\mathscr{H}_{v_{[2:\alpha+1]}}(u)\big)\odot \partial_{u_{k\ell}} \mathscr{P}_{v}^{(\alpha+1:L]}(u,z) \Big]v_{\alpha+1}^\top \in \R^{m\times q},\\
		\partial_{z_{k\ell}} \mathscr{P}_{v}^{(\alpha:L]}(u,z)&=\Big(\sigma_{\alpha+1}'\big(\mathscr{H}_{v_{[2:\alpha+1]}}(u)\big)\odot \partial_{z_{k\ell}} \mathscr{P}_{v}^{(\alpha+1:L]}(u,z) \Big)v_{\alpha+1}^\top \in \R^{m\times q}.
		\end{align*}
		Furthermore $\partial_{u_{k\ell}} \mathscr{P}_{v}^{(L:L]}(u,z)\equiv 0_{m\times q}$ and $\partial_{z_{k\ell}} \mathscr{P}_{v}^{(L:L]}(u,z)\equiv e_ke_\ell^\top$.
		\item For $k \in [m],\ell \in [q]$, with $\mathscr{H}_v(\cdot)\equiv \mathscr{H}_{v_{[2:L]}}(\cdot)$,
		\begin{align*}
		&\partial_{u_{k\ell}} \mathfrak{S}\big(u,w, v\big) = \mathscr{P}_{v}^{(1:L]}\big(u,\mathscr{R}_{v}(u,w)\big)\odot\sigma_1''(u)\odot e_k e_\ell^\top + \bigg[\big(\partial_{u_{k\ell}} \mathscr{P}_{v}^{(1:L]}\big)\big(u,\mathscr{R}_v(u,w)\big)\\
		&\qquad\qquad\qquad\qquad + \sum_{\ell' \in [q]} \big(\partial_{z_{k \ell'}} \mathscr{P}_{v}^{(1:L]}\big)\big(u, \mathscr{R}_v(u,w)\big) \partial_{u_{k\ell}} \mathscr{H}_{v;(k,\ell')}(u)\bigg]\odot \sigma_1'(u),\\
		&\partial_{w_{k\ell}} \mathfrak{S}\big(u,w, v\big) = -\Big[\big(\partial_{z_{k\ell}} \mathscr{P}_{v}^{(1:L]}\big)\big(u,\mathscr{R}_v(u,w)\big) \odot \sigma_1'(u)\Big]\cdot \varphi_\ast'(w_{k\ell}).
		\end{align*}
	\end{enumerate}
\end{proposition}

Next we provide bounds for the derivatives in Proposition \ref{prop:derivative_formula}; the proof can be found in Section \ref{subsection:proof_apr_est_derivative}.
\begin{proposition}\label{prop:apr_est_derivative}
	Suppose that (A4)-(A5) in Assumption \ref{assump:gd_NN} hold for some $\Lambda\geq 2$ and $r_0\geq 0$. Then the following hold for some universal constant $c_0>0$:
	\begin{enumerate}
		\item For $u\in \R^{m\times q}$ and $v_{[2:\alpha]}\in (\mathbb{M}_q)^{[2:\alpha]}$, and $k \in [m],\ell \in [q]$,
		\begin{align*}
		\pnorm{\partial_{u_{k\ell}}  \mathscr{H}_{v_{[2:\alpha]}}(u)}{\infty}\vee \pnorm{\partial_{u_{k\ell}}  \mathscr{G}_{v_{[2:\alpha]}}(u)}{\infty}\leq (q\Lambda)^{c_0L}\cdot \kappa_{v_{[2:\alpha]}}.
		\end{align*}
		\item For any $ u,u' \in \R^{m\times q}$, $v_{[2:\alpha]},v_{[2:\alpha]}'\in (\mathbb{M}_q)^{[2:\alpha]}$, and $k \in [m]$, $\ell \in [q]$,
		\begin{align*}
		&\max_{\# \in \{\mathscr{H},\mathscr{G}\}}\bigpnorm{\partial_{u_{k\ell}}  \#_{v_{[2:\alpha]};k}(u_{k\cdot})-	\partial_{u_{k\ell}}  \#_{v_{[2:\alpha]}';k}(u_{k\cdot}')}{} \leq (q\Lambda)^{c_0 L} (\kappa_{v_{[2:\alpha]}}\kappa_{v_{[2:\alpha]}'})^{c_0}\\
		&\qquad \times  \Big[\big(1+\pnorm{u_{k\cdot}}{\infty}\wedge \pnorm{u_{k\cdot}'}{\infty}\big)\cdot  \pnorm{v_{[2:\alpha]}-v_{[2:\alpha]}'}{\infty}+\pnorm{u_{k\cdot}-u_{k\cdot}'}{}\Big]. 
		\end{align*}
		\item For $u\in \R^{m\times q}$,  $v_{[2:\alpha]}\in (\mathbb{M}_q)^{[2:\alpha]}$ and $z \in \R^{m\times q}$, and $k \in [m],\ell \in [q]$,
		\begin{align*}
		\pnorm{\partial_{u_{k\ell}}\mathscr{P}_{u, v_{[2:\alpha]} }(z)}{\infty}\leq (q\Lambda)^{c_0 L} \kappa_{v_{[2:\alpha]}}^2 \pnorm{z}{\infty},\quad \pnorm{\partial_{z_{k\ell}}\mathscr{P}_{u, v_{[2:\alpha]} }(z)}{\infty}\leq q\Lambda \kappa_{v_{[2:\alpha]}}.
		\end{align*}
		\item For any $ u,u' \in \R^{m\times q}$, $v_{[2:\alpha]},v_{[2:\alpha]}'\in (\mathbb{M}_q)^{[2:\alpha]}$, $ z,z' \in \R^{m\times q}$, and $k \in [m]$, $\ell \in [q]$,
		\begin{align*}
		&\max_{\# \in \{u_{k\ell},z_{k\ell}\}}\bigpnorm{ \partial_{\#}\mathscr{P}_{u, v_{[2:\alpha]};k }(z_{k\cdot}) - \partial_{\#}\mathscr{P}_{u', v_{[2:\alpha]}';k }(z_{k\cdot}')  }{}\\
		&\leq (q\Lambda)^{c_0 L} (\kappa_{v_{[2:\alpha]}}\kappa_{v_{[2:\alpha]}'})^{c_0}\cdot \max_{\# \in \{u,z\}}\big(1+\pnorm{\#_{k\cdot}}{\infty}\wedge \pnorm{\#_{k\cdot}'}{\infty}\big)^{c_0}\\
		&\qquad \times  \big(\pnorm{u_{k\cdot}-u_{k\cdot}'}{}+\pnorm{v_{[2:\alpha]}-v_{[2:\alpha]}'}{\infty}+\pnorm{z_{k\cdot}-z_{k\cdot}'}{}\big).
		\end{align*}
		\item For $u \in \R^{m\times q},v\in (\mathbb{M}_q)^{[2:L]}$ and $z \in \R^{m\times q}$, and $k \in [m],\ell \in [q]$,
		\begin{align*}
		\pnorm{\partial_{u_{k\ell}} \mathscr{P}_{v}^{(\alpha:L]}(u,z)}{\infty}\leq (q\Lambda)^{c_0 L}\kappa_{v}^2 \cdot \pnorm{z}{\infty},\, \pnorm{\partial_{z_{k\ell}} \mathscr{P}_{v}^{(\alpha:L]}(u,z)}{\infty}\leq (q\Lambda)^{c_0L} \kappa_{v}.
		\end{align*}
		\item For $u, u' \in \R^{m\times q},v, v'\in (\mathbb{M}_q)^{[2:L]}$ and $z,z' \in \R^{m\times q}$, and $k \in [m],\ell \in [q]$,
		\begin{align*}
		&\max_{\# \in \{u_{k\ell},z_{k\ell}\}} \bigpnorm{\partial_{\#} \mathscr{P}_{v;k}^{(\alpha:L]}(u_{k\cdot},z_{k\cdot})- \partial_{\#} \mathscr{P}_{v';k}^{(\alpha:L]}(u_{k\cdot}',z_{k\cdot}')}{}\\
		&\leq (q\Lambda)^{c_0 L} (\kappa_{v}\kappa_{v'})^{c_0}\cdot \max_{\# \in \{u,z\}}\big(1+\pnorm{\#_{k\cdot}}{\infty}\wedge \pnorm{\#_{k\cdot}'}{\infty}\big)^{c_0}\\
		&\qquad \times   \big(\pnorm{u_{k\cdot}-u_{k\cdot}'}{}+\pnorm{z_{k\cdot}-z_{k\cdot}'}{} + \pnorm{v-v'}{\infty}\big).
		\end{align*}
		\item For $u \in \R^{m\times q},v\in (\mathbb{M}_q)^{[2:L]}$ and $w \in \R^{m\times q}$, and $k \in [m],\ell \in [q]$,
		\begin{align*}
		&\pnorm{\partial_{u_{k\ell}} \mathfrak{S}\big(u,w, v\big)}{\infty}\vee \pnorm{\partial_{w_{k\ell}} \mathfrak{S}\big(u,w, v\big)}{\infty} \\
		&\leq  (q\Lambda)^{c_0 L}\kappa_{v}^{c_0}\cdot (1+\pnorm{\varphi_\ast'(w)}{\infty})\cdot \big(1+\pnorm{u }{\infty}+\pnorm{\varphi_\ast(w)}{\infty}+\pnorm{\xi}{\infty}\big).
		\end{align*}
		\item For $u, u' \in \R^{m\times q},v, v'\in (\mathbb{M}_q)^{[2:L]}$ and $w,w' \in \R^{m\times q}$, and $k \in [m],\ell \in [q]$,
		\begin{align*}
		&\max_{\# \in \{u_{k\ell},w_{k\ell}\}} \bigpnorm{\partial_{\#} \mathfrak{S}_{k}\big(u_{k\cdot},w_{k\cdot}, v\big)-\partial_{\#} \mathfrak{S}_{k}\big(u_{k\cdot}',w_{k\cdot}', v'\big)}{}\\
		&\leq (q\Lambda)^{c_0 L} (\kappa_{v}\kappa_{v'})^{c_0}\cdot \min_{(\#,\$)\in \{(u,w),(u',w')\}}\big(1+\pnorm{\#_{k\cdot}}{\infty}+\pnorm{\varphi_\ast(\$_{k\cdot})}{\infty}+\abs{\xi_k}\big)^{c_0}\\
		&\qquad \times   \big(\pnorm{u_{k\cdot}-u_{k\cdot}'}{}+\pnorm{\varphi_\ast(w_{k\cdot})-\varphi_\ast(w_{k\cdot}')}{}+\pnorm{v-v'}{\infty}\big).
		\end{align*}
	\end{enumerate}
\end{proposition}

\subsection{Apriori estimates for state evolution parameters}

Finally we shall provide apriori estimates for the state evolution parameters in Definition \ref{def:gd_NN_se}; the proof can be found in Section \ref{subsection:proof_apr_est_se}. Recall $\{M_\alpha^{(t-1)}\}_{\alpha \in [2:L]}$ defined in (\ref{def:M_aux_GFOM}).

\begin{proposition}\label{prop:apr_est_se}
	Suppose $\phi^{-1}\leq K$ and (A3)-(A5) in Assumption \ref{assump:gd_NN} hold for some $K,\Lambda\geq 2$ and $r_0\geq 0$. Then there exists some $c_t\equiv c_t(t,q,L,r_0)>1$ such that
	\begin{align*}
	\mathcal{B}^{(t)}\leq  \big(K\Lambda\varkappa_{\ast}\big)^{c_t},
	\end{align*}
	where for $t\geq 1$, 
	\begin{align*}
	\mathcal{B}^{(t)}&\equiv 1+\max_{s \in [0:t]} \Big\{\pnorm{\tau_{t,s}}{\op}+\pnorm{\cov(\mathfrak{U}^{(t)},\mathfrak{U}^{(s)})}{\op}\Big\} + \max_{s \in [1:t]} \Big\{\pnorm{\rho_{t,s}}{\op}+ \pnorm{\Sigma_{t,s}}{\op} + \pnorm{\Omega_{t,s}}{\op} \Big\}\\
	&\qquad  + \frac{\pnorm{D_t}{}^2}{n}+\max_{k \in [m]} \frac{\pnorm{\Theta_{k}^{(t)}(\mathfrak{u}_{k\cdot}^{([0:t])})}{\infty}}{  1+\pnorm{\mathfrak{u}_{k\cdot}^{([1:t])}}{\infty}+ \pnorm{\varphi_\ast(\mathfrak{u}_{k\cdot}^{(0)})}{\infty} }+ \max_{\alpha \in [2:L]} \big\{\pnorm{V^{(t)}_\alpha}{\op}+\pnorm{M_{\alpha}^{(t-1)}}{\op}\big\}.
	\end{align*} 
	Here we write $\tau_{t,0}\equiv \delta_t$ for notational simplicity. Moreover,
	\begin{align*}
	&\max_{k \in [m]}\frac{\bigpnorm{\Theta_{k}^{(t)}(\mathfrak{u}_{k\cdot}^{([0:t])})-\mathfrak{u}_{k\cdot}^{(t)} }{\infty}}{ 1+\pnorm{\mathfrak{u}_{k\cdot}^{([1:t-1])}}{\infty}+ \pnorm{\varphi_\ast(\mathfrak{u}_{k\cdot}^{(0)})}{\infty} }\\
	&\qquad + \max_{s \in [1:t-1]} \pnorm{\tau_{t,s}}{\op} + \max_{s \in [1:t]} \pnorm{\Sigma_{t,s}}{\op}\leq \big(K\Lambda\varkappa_{\ast}\big)^{c_t}\cdot \phi^{-1}.
	\end{align*}
\end{proposition}

\section{Controls for the auxiliary GFOM iterates}\label{section:proof_gfom}

\subsection{$\ell_2$ and $\ell_\infty$ controls for GD and auxiliary GFOM iterates}

We first provide $\ell_2$ controls. 

\begin{proposition}\label{prop:GD_l2_est}
Suppose (A1) and (A3)-(A5) in Assumption \ref{assump:gd_NN} hold for some $K,\Lambda\geq 2$ and $r_0\geq 0$. Then there exists some $c_t\equiv c_t(t,q,L,r_0)>1$ such that for both the gradient descent iterate $(\hat{u}^{(t)},\hat{v}^{(t)})$ in (\ref{def:GD_GFOM_NN}) and the auxiliary GFOM iterate $(u^{(t)},v^{(t)})$ in (\ref{def:GFOM_auxiliary_NN}),
\begin{align*}
&n^{-1/2}\big(\pnorm{u^{(t)}}{}+\pnorm{\hat{u}^{(t)}}{}+ \pnorm{v_1^{(t)}}{}+\pnorm{\hat{v}_1^{(t)}}{}\big)+ \max_{\alpha \in [2:L]} \big(\pnorm{v_\alpha^{(t)}}{\op}+\pnorm{\hat{v}_\alpha^{(t)}}{\op}\big)\\
&\leq  \Big(K\Lambda\varkappa_{\ast} (1+n^{-1/2}\pnorm{X}{\op})(1+\pnorm{X\mu_\ast}{\infty}) \log n\Big)^{c_t}.
\end{align*}
\end{proposition}
\begin{proof}
For the auxiliary GFOM iterate in (\ref{def:GFOM_auxiliary_NN}), we have the following estimates:
\begin{enumerate}
	\item $\pnorm{u^{(t)}}{} \leq \pnorm{A}{\op} \pnorm{v_1^{(t-1)}}{}$.
	\item For $\alpha \in [2:L]$, using the apriori estimate in Proposition \ref{prop:apr_est_se},
	\begin{align*}
	\pnorm{v_\alpha^{(t)}}{\op}&\leq \pnorm{v_\alpha^{(t-1)}}{\op}+ \Lambda \pnorm{M_\alpha^{(t-1)}}{\op}+\pnorm{W_\alpha^{(0)}}{\op}\leq \pnorm{v_\alpha^{(t-1)}}{\op}+\big(K\Lambda\varkappa_{\ast}\log n\big)^{c_t}.
	\end{align*}
	Iterating the estimate and using the initial condition $\pnorm{v_\alpha^{(-1)}}{\op}=0$ to conclude the desired estimate $
	\pnorm{v_\alpha^{(t)}}{\op}\leq \big(K\Lambda\varkappa_{\ast}\log n\big)^{c_t}$.
	\item Using the apriori estimates in Propositions \ref{prop:apr_est_det_fcn}-(6) and \ref{prop:apr_est_se},
	\begin{align*}
	\pnorm{v_1^{(t)}}{}&\leq \pnorm{v_1^{(t-1)}}{}+(K\Lambda) \pnorm{A}{\op}\cdot \pnorm{\mathfrak{S}\big(u^{(t)},u^{(0)}, V^{(t-1)} \big)}{}\\
	&\leq \pnorm{v_1^{(t-1)}}{} + \big(K\Lambda(1+\pnorm{A}{\op})\varkappa_{\ast}\log n\big)^{c_t}\cdot \big(n^{1/2}+\pnorm{u^{(t)}}{}+\pnorm{\varphi_\ast(u^{(0)})}{}\big).
	\end{align*}
	Using the estimate in (1),
	\begin{align*}
	\pnorm{v_1^{(t)}}{}&\leq \big(K\Lambda\varkappa_{\ast} (1+\pnorm{A}{\op}) \log n\big)^{c_t}\cdot \big(n^{1/2}+\pnorm{v_1^{(t-1)}}{}+\pnorm{\varphi_\ast(u^{(0)})}{}\big).
	\end{align*}
	Iterating the estimate and using the initial condition $\pnorm{v^{(-1)}_1}{}=n^{1/2}\sigma_{\mu_\ast}$, we conclude that 
	\begin{align*}
	n^{-1/2}\pnorm{v_1^{(t)}}{}&\leq \big(K\Lambda\varkappa_{\ast} (1+\pnorm{A}{\op}) \log n\big)^{c_t}\cdot \big(1+n^{-1/2}\pnorm{\varphi_\ast(X\mu_\ast)}{}\big)\\
	&\leq \big(K\Lambda\varkappa_{\ast} (1+\pnorm{A}{\op}) (1+\pnorm{X\mu_\ast}{\infty})\log n\big)^{c_t}.
	\end{align*}
\end{enumerate}
Combing the above estimates proves the $\ell_2$ estimates for the auxiliary GFOM iterate in (\ref{def:GFOM_auxiliary_NN}).

For the gradient descent iterate in (\ref{def:GD_GFOM_NN}), we have the following estimates:
\begin{enumerate}
	\item[(1')] $\pnorm{\hat{u}^{(t)}}{} \leq \pnorm{A}{\op} \pnorm{\hat{v}_1^{(t-1)}}{}$.
	\item[(2')] For $\alpha \in [2:L]$, using the apriori estimate in Proposition \ref{prop:apr_est_det_fcn} and the estimate in (1'), by recalling the definition $\kappa_{v}$ in Proposition \ref{prop:apr_est_det_fcn},
	\begin{align*}
	\pnorm{\hat{v}_\alpha^{(t)}}{\op}&\leq \pnorm{\hat{v}_\alpha^{(t-1)}}{\op}+\frac{(q\Lambda)^{c_0}}{m}\cdot  \bigpnorm{\hat{\mathscr{G}}_{(t-1);\alpha-1}(\hat{u}^{(t)})}{} \bigpnorm{\hat{\mathscr{P}}_{ (t-1) }^{(\alpha:L]}\big(\hat{u}^{(t)},\hat{\mathscr{R}}_{(t-1)}(\hat{u}^{(t)},\hat{u}^{(0)})\big)}{}\\
	&\leq \pnorm{\hat{v}_\alpha^{(t-1)}}{\op}+\big(K\Lambda\varkappa_{\ast}\big)^{c_t}\cdot (\kappa_{\hat{v}^{(t-1)}_{[2:L]}})^{c_0}\cdot m^{-1/2} \big(n^{1/2}+\pnorm{\hat{u}^{(t)}}{}+\pnorm{\varphi_\ast(u^{(0)})}{}\big)\\
	&\leq  \big(K\Lambda\varkappa_{\ast} (1+\pnorm{A}{\op}) \big)^{c_t}\cdot (\kappa_{\hat{v}^{(t-1)}_{[2:L]}})^{c_0}\cdot \big(1+\pnorm{X\mu_\ast}{\infty}\big)^{r_0}.
	\end{align*}
	This means that
	\begin{align*}
	\kappa_{\hat{v}^{(t)}_{[2:L]}}\leq \big(K\Lambda\varkappa_{\ast} (1+\pnorm{A}{\op}) (1+\pnorm{X\mu_\ast}{\infty}) \big)^{c_t}\cdot (\kappa_{\hat{v}^{(t-1)}_{[2:L]}})^{c_0}.
	\end{align*}
	Iterating the estimate and using the initial condition $\pnorm{\hat{v}_\alpha^{(-1)}}{\op}=0$ to conclude the desired estimate for $\{\pnorm{\hat{v}_\alpha^{(t)}}{\op}\}_{\alpha \in [2:L]}$.
	\item[(3')] Using the apriori estimate in Proposition \ref{prop:apr_est_det_fcn}, along with (1')-(2') above, 
	\begin{align*}
	\pnorm{\hat{v}_1^{(t)}}{}&\leq \pnorm{\hat{v}_1^{(t-1)}}{}+(K\Lambda)^{c_0} \pnorm{A}{\op}\cdot \bigpnorm{\hat{\mathscr{P}}_{ (t-1) }^{(1:L]}\big(\hat{u}^{(t)},\hat{\mathscr{R}}_{(t-1)}(\hat{u}^{(t)},\hat{u}^{(0)})\big)}{}\\
	&\leq \pnorm{\hat{v}_1^{(t-1)}}{} + \big(K\Lambda(1+\pnorm{A}{\op})(1+\pnorm{X\mu_\ast}{\infty})\varkappa_{\ast}\big)^{c_t}\cdot \big(n^{1/2}+\pnorm{\hat{u}^{(t)}}{}+\pnorm{\varphi_\ast(u^{(0)})}{}\big)\\
	&\leq \big(K\Lambda(1+\pnorm{A}{\op})(1+\pnorm{X\mu_\ast}{\infty})\varkappa_{\ast}\big)^{c_t}\cdot \big(n^{1/2}+\pnorm{\hat{v}_1^{(t-1)}}{}+\pnorm{\varphi_\ast(u^{(0)})}{}\big).
	\end{align*}
	From here we may use the same argument as in (3) above to conclude. 
\end{enumerate}
Combing the above estimates proves the $\ell_2$ estimates for the gradient descent iterate in (\ref{def:GD_GFOM_NN}).
\end{proof}

Next we provide $\ell_\infty$ controls.

\begin{proposition}\label{prop:GD_linfty_est}
Suppose Assumption \ref{assump:gd_NN} holds for some $K,\Lambda\geq 2$ and $r_0\geq 0$. Then there exists some $c_t\equiv c_t(t,q,L,r_0)>1$ such that for the auxiliary GFOM iterate $(u^{(t)},v^{(t)})$ in (\ref{def:GFOM_auxiliary_NN}),
\begin{align*}
\Prob^{(0)}\Big[ \pnorm{u^{(t)}}{\infty}\vee  \pnorm{v^{(t)}}{\infty}\geq \big(K\Lambda\varkappa_{\ast}\log n\big)^{c_t}\Big]\leq n^{-100}.
\end{align*}
\end{proposition}
\begin{proof}
By the definition in (\ref{def:GFOM_auxiliary_NN}) and using Proposition \ref{prop:apr_est_det_fcn}-(6) and Proposition \ref{prop:apr_est_se},
\begin{align*}
&\max\limits_{s \in [t]} \max\limits_{  \# = 1,2   }\max\limits_{k \in [m], \ell \in [n]} \max\limits_{r \in [q]} \big\{\abs{\mathsf{F}^{\langle \# \rangle}_{s,\ell;r}(0)} + \abs{\mathsf{G}^{\langle \# \rangle}_{s,k;r}(0)}\big\}\\
&\leq (K\Lambda)^{c_0}\cdot \Big(\pnorm{n^{1/2}\mu_\ast}{\infty}+\max_{\alpha \in [2:L]}\max_{s \in [t]}\pnorm{M_\alpha^{(s-1)}}{\infty}+\max_{s \in [t]} \pnorm{\mathfrak{S}\big(0,0, V^{(s-1)}\big)}{\infty} \Big)\\
&\leq \big(K\Lambda\varkappa_{\ast}\log n\big)^{c_t}.
\end{align*}
Moreover, by Proposition \ref{prop:apr_est_det_fcn}-(7), uniformly in $s \in [t]$ and $k \in [m], \ell \in [n], r \in [q]$, for $ \# = 1,2$
\begin{align*}
& \frac{ \abs{\mathsf{F}^{\langle \# \rangle}_{s,\ell;r}(x)- \mathsf{F}^{\langle \# \rangle}_{s,\ell;r}(y)} }{\pnorm{x-y}{}}+ \frac{ \abs{\mathsf{G}^{\langle \# \rangle}_{s,k;r}(x)- \mathsf{G}^{\langle \# \rangle}_{s,k;r}(y)} }{(1+\pnorm{x}{}+\pnorm{y}{})\cdot \pnorm{x-y}{}} \leq \big(K\Lambda\varkappa_{\ast}\log n\big)^{c_t}.
\end{align*}
Now we may apply Proposition \ref{prop:GFOM_deloc_asym} to conclude. 
\end{proof}

\subsection{State evolution for the auxiliary GFOM: Proof of Proposition \ref{prop:se_aux_GFOM}}\label{subsection:proof_se_aux_GFOM}

Let $\chi \in C^\infty(\R)$ be a smooth, nondecreasing function such that $\chi\equiv \mathrm{id}$ on $[-1,1]$, $\chi\equiv -2$ on $(-\infty,-2]$ and $\chi\equiv 2$ on $[2,\infty)$. With $\chi_b(\cdot)\equiv b\cdot\chi(\cdot/b)$, let
\begin{align}\label{def:S_bar}
\check{\mathfrak{S}}_{t}(\cdot,\cdot,\cdot) &\equiv \chi_{\mathsf{b}_t}\big(\mathfrak{S}(\cdot,\cdot,\cdot)\big),\quad \check{\mathfrak{T}}_{\alpha;t}(\cdot,\cdot,\cdot) \equiv \chi_{\mathsf{b}_t}\big(\mathfrak{T}_{\alpha}(\cdot,\cdot,\cdot)\big),
\end{align}
where $\mathsf{b}_t\equiv \big(K\Lambda\varkappa_{\ast}\log n\big)^{c_{0;t}}\geq 1$ for some constant $c_{0;t}>0$ to be chosen later on, and the truncation is understood as applied entrywise.

We now define a modified set of state evolution parameters by replacing $\mathfrak{S},\mathfrak{T}_{\alpha}$ by $\check{\mathfrak{S}}_{t},\check{\mathfrak{T}}_{\alpha;t}$ in each step, but starting from the same initialization:
\begin{itemize}
	\item Initialize with $\check{V}^{(0)}\equiv V^{(0)}\in (\mathbb{M}_q)^{[2:L]}$, $\check{\Omega}_{\cdot,0}\equiv \Omega_{\cdot,0}=0_{q\times q}$, $\check{\Omega}_{\cdot,-1}\equiv \Omega_{\cdot,-1}=0_{q\times q}$, $\check{\mathfrak{U}}^{(0)}\equiv \mathfrak{U}^{(0)} \in \R^q$, $\check{D}_{-1}\equiv D_{-1}\in \R^{n\times q}$ and $\check{D}_0\equiv D_0\in \R^{n\times q}$.
	\item For $t=1,2,\ldots$, we iteratively compute the quantities in (S1)-(S5) in Definition \ref{def:gd_NN_se}, now denoted as the mapping $\check{\Theta}^{(t)}: \R^{m\times q[0:t]}\to \R^{m\times q}$, the Gaussian vector $\check{\mathfrak{U}}^{(t)}\in \R^q$, the matrices $\{\check{\tau}_{t,s}\}, \{\check{\rho}_{t,s}\}, \{\check{\Sigma}_{t,s}\}, \{\check{\Omega}_{t,s}\}\subset \mathbb{M}_q$, $\check{D}_t \in \R^{n\times q}$ and $\check{V}^{(t)} \in (\mathbb{M}_q)^{[2:L]}$. 
	\item Let $\check{M}_a^{(t-1)}$ be obtained by replacing the relevant quantities (\ref{def:M_aux_GFOM}) with the aforementioned modified quantities. 
	\item Let $(\check{\mathfrak{V}}^{(1)},\ldots,\check{\mathfrak{V}}^{(t)})\in (\R^q)^{[1:t]}$ be a centered Gaussian matrix whose law is specified by the covariance $\cov(\check{\mathfrak{V}}^{(r)},\check{\mathfrak{V}}^{(s)})=\Sigma_{r,s}$, $1\leq r,s\leq t$.
	\item Let $\check{\Delta}^{(t)}:\R^{n\times q[0:t]}\to \R^{n\times q}$ be defined via
	\begin{align*}
	\check{\Delta}^{(t)}(\check{\mathfrak{v}}_1^{([1:t])} ) = \sum_{s \in [1:t]} \check{\mathfrak{v}}_1^{(s)} \check{\rho}_{t,s}^\top + \check{D}_t \in \R^{n\times q}.
	\end{align*}
\end{itemize}

We define the truncated GFOM iteration $\{(\check{u}^{(t)},\check{v}^{(t)})\}_{t\in \mathbb{Z}_{\geq -1}}$ as follows: with initialization $\check{u}^{(-1)} = 0_{m\times q}$, $\check{v}^{(-1)}\equiv \overline{\mu}_{\ast,n}
\in \R^{n_{L,q}\times q}$, for $t=0,1,2,\ldots,$ let 
\begin{align}\label{def:GFOM_bar_auxiliary_NN}
\begin{cases}
\check{u}^{(t)} = A \check{\mathsf{F}}_t^{\langle 1\rangle}(\check{v}^{([-1:t-1])})+ \check{\mathsf{G}}_{t}^{\langle 1\rangle}(\check{u}^{([-1:t-1])}) \in \R^{m\times q},\\
\check{v}^{(t)}= A^\top \check{\mathsf{G}}_t^{\langle 2\rangle}(\check{u}^{([-1:t])})+\check{\mathsf{F}}_t^{\langle 2\rangle}(\check{v}^{([-1:t-1])})\in \R^{n_{L,q}\times q},
\end{cases}
\end{align}
where for $t\geq 0$, let
\begin{align*}
&\check{\mathsf{F}}_t^{\langle 1\rangle}(\check{v}^{([-1:t-1])})\equiv \check{v}^{(t-1)}\bm{1}_{t\geq 1}+ \overline{\mu}_{\ast,n}\bm{1}_{t=0}, \\
& \check{\mathsf{F}}_t^{\langle 2\rangle}(\check{v}^{([-1:t-1])})\equiv  \bigg[\check{v}^{(t-1)}-
\begin{pmatrix}
0_{n\times q}\\
\big\{\eta_\alpha^{(t-1)} \check{M}_\alpha^{(t-1)}\big\}_{\alpha \in [2:L]}
\end{pmatrix}\bigg] \bm{1}_{t\geq 1}+ \bm{W}_n^{(0)} \bm{1}_{t=0},\\
& \check{\mathsf{G}}_{t}^{\langle 1\rangle}(\check{u}^{([-1:t-1])})\equiv 0_{m\times q},\\
& \check{\mathsf{G}}_t^{\langle 2\rangle}(\check{u}^{([-1:t])})\equiv -\phi^{-1}\eta^{(t-1)}_1\cdot \check{\mathfrak{S}}_{t}\big(\check{u}^{(t)},\check{u}^{(0)}, \check{V}^{(t-1)} \big)\bm{1}_{t\geq 1}.
\end{align*}

We first obtain the state evolution characterization for the  GFOM iteration $\{(\check{u}^{(t)},\check{v}^{(t)})\}$ in (\ref{def:GFOM_bar_auxiliary_NN}).
\begin{lemma}\label{lem:se_truncated_aux_GFOM}
Suppose Assumption \ref{assump:gd_NN} holds for some $K,\Lambda\geq 2$ and $r_0\geq 0$. Fix $\mathfrak{p}\geq 1$ and any $\Psi \in C^3(\R^{ q[0:t]})$ satisfying
\begin{align*}
\max_{a \in \mathbb{Z}_{\geq 0}^{ q[0:t]}, \abs{a}\leq 3} \sup_{x \in \R^{ q[0:t] } }\bigg(\sum_{ \tau \in  q[0:t] }(1+\abs{x_{\tau}})^{\mathfrak{p}}\bigg)^{-1} \abs{\partial^a \Psi(x)} \leq \Lambda_{\Psi}
\end{align*}
for some $\Lambda_{\Psi}\geq 2$.  Then for any $\mathfrak{r}\geq 1$, there exists some $c_t=c_t(c_{0;t},t,\mathfrak{p},q,L,\mathfrak{r},r_0)>0$ such that 
\begin{align*}
&\max_{k \in [m]} \bigabs{\E^{(0)}\Psi\big(\check{u}_k^{([0:t])}\big) -  \E^{(0)}  \Psi\big(\big\{\check{\Theta}_{k}^{(s)}(\check{\mathfrak{U}}^{([0:s])})\big\}_{s \in [0:t]}\big)}\\
&\quad + \max_{\ell \in [n]} \bigabs{\E^{(0)}\Psi\big(\check{v}_{1;\ell}^{([0:t])}\big) -  \E^{(0)}  \Psi\big(\big\{\check{\Delta}_{\ell}^{(s)}(\check{\mathfrak{V}}^{([1:s])} )\big\}_{s \in [0:t]}\big) }\\
&\quad +\max_{\alpha \in [2:L]} \max_{s \in [1:t]}  \E^{(0)} \pnorm{\check{v}_\alpha^{(s)}-\check{V}^{(s)}_\alpha}{\infty}^{\mathfrak{r}} \leq  \big(K\Lambda\varkappa_{\ast}\big)^{c_{t}}\cdot n^{-1/c_t}.
\end{align*}
Moreover, fix a sequence of $\Lambda_\psi$-pseudo-Lipschitz functions $\{\psi_k:\R^{q[0:t]} \to \R\}_{k \in [m\vee n]}$ of order $\mathfrak{p}$, where $\Lambda_\psi\geq 2$. Then for any $\mathfrak{r}\geq 1$, there exists some $c_t'=c_t'(c_{0;t},t,\mathfrak{p},q,L,\mathfrak{r},r_0)>0$ such that 
\begin{align*}
&\E^{(0)} \bigg[\biggabs{\frac{1}{m}\sum_{k \in [m]} \psi_k\big(\check{u}_k^{([0:t])}\big) - \frac{1}{m}\sum_{k \in [m]}  \E^{(0)}  \psi_k\Big(\big\{\check{\Theta}_{k}^{(s)}(\check{\mathfrak{U}}^{([0:s])})\big\}_{s \in [0:t]}\Big)  }^{\mathfrak{r}}\bigg]\\
&\quad + \E^{(0)}  \bigg[\biggabs{\frac{1}{n}\sum_{\ell \in [n]} \psi_\ell\big(\check{v}_{1;\ell}^{([0:t])}\big) - \frac{1}{n}\sum_{\ell \in [n]}  \E^{(0)}  \psi_\ell\Big(\big\{\check{\Delta}_{\ell}^{(s)}(\check{\mathfrak{V}}^{([1:s])} )\big\}_{s \in [0:t]}\Big)   }^{\mathfrak{r}}\bigg]  \\
&\leq \big(K\Lambda\varkappa_{\ast}\big)^{c_t'}\cdot n^{-1/c_t'}. 
\end{align*}

\end{lemma}

\begin{proof}
Recall the notation in (\ref{def:GFOM_notation}). In the proof we shall simply abbreviate $\check{\ast}$ as $\ast$ for notational simplicity. 
	
\noindent (\textbf{Step 1}). We now use Definition \ref{def:GFOM_se_asym} to obtain state evolution. We initialize with $\Phi_{-1}\equiv \mathrm{id}(\R^{m\times q})$, $\Xi_{-1}\equiv \mathrm{id}(\R^{n_{L,q}\times q})$, $\mathfrak{U}^{(-1)}\equiv 0_{m\times q}$ and $\mathfrak{V}^{(-1)}\equiv \overline{\mu}_{\ast,n}\in \R^{n_{L,q}\times q}$. For iteration $t=0$, we have the following:
\begin{itemize}
	\item $\Phi_0: \R^{m\times q[-1:0]}\to \R^{m\times q[-1:0]}$ is defined as 
	\begin{align*}
	\Phi_{0}(\mathfrak{u}^{([-1:0])})\equiv [\mathfrak{u}^{(-1)}\,|\, \mathfrak{u}^{(0)}].
	\end{align*}
	\item The Gaussian laws of $\big\{\mathfrak{U}_k^{(0)}\big\}_{k \in [m]}\subset \R^q$ are determined via
	\begin{align*}
	\cov\big(\mathfrak{U}_k^{(0)}, \mathfrak{U}_k^{(0)}\big)=
	\begin{pmatrix}
	\pnorm{\mu_\ast}{}^2 & 0_{1\times (q-1)}\\
	0_{(q-1)\times 1} & 0
	\end{pmatrix}.
	\end{align*}
	\item $\Xi_0: \R^{n_{L,q}\times q[-1:0]}\to \R^{n_{L,q}\times q[-1:0]}$ is defined as 
	\begin{align*}
	\Xi_{0}(\mathfrak{v}^{([-1:0])})\equiv [\mathfrak{v}^{(-1)}\,|\, \mathfrak{v}^{(0)}+\bm{W}_n^{(0)}].
	\end{align*}
	\item $\big\{\mathfrak{V}^{(0)}_\ell\big\}_{\ell \in [n_{L,q}]} \subset  \R^q$ are degenerate (identically 0).
\end{itemize}
For iteration $t\geq 1$, we have the following state evolution:
\begin{enumerate}
	\item Let $\Phi_{t}:\R^{m\times q[-1:t]}\to \R^{m\times q[-1:t]}$ be defined as follows: for $w \in [0:t-1]$, $\big[\Phi_{t}(\mathfrak{u}^{([-1:t])})\big]_{\cdot,w}\equiv \big[\Phi_{w}(\mathfrak{u}^{([-1:w])})\big]_{\cdot,w}$, and for $w=t$,
	\begin{align*}
	\big[\Phi_{t}(\mathfrak{u}^{([-1:t])})\big]_{\cdot,t} &\equiv \mathfrak{u}^{(t)}-\phi^{-1}\sum_{s \in [1:t-1]} \eta^{(s-1)}_1\cdot \mathfrak{S}\Big(\big[\Phi_{s}(\mathfrak{u}^{([-1:s])})\big]_{\cdot,s} ,\mathfrak{u}^{(0)},V^{(s-1)}\Big)\circ (\mathfrak{f}_{s}^{(t-1)})_\top,
	\end{align*}
	where the correction matrices $\{\mathfrak{f}_{s,k}^{(t-1) } \}_{s \in [1:t-1], k \in [m]}\subset \mathbb{M}_q$ are determined by
	\begin{align*}
	\mathfrak{f}_{s,k}^{(t-1) } \equiv  \frac{1}{n} \sum_{\ell \in [n]}  \E^{(0)} \nabla_{\mathfrak{V}_\ell^{(s)}}  \big[\Xi_{t-1,\ell} (\mathfrak{V}_\ell^{([-1:t-1])})\big]_{\cdot,t-1}\in \mathbb{M}_q, \quad  k \in [m].
	\end{align*}
	\item Let the Gaussian laws of $\{\mathfrak{U}_k^{(t)}\}_{k \in [m]}\subset \R^q$ be determined via the following correlation specification: for $s \in [0:t]$, $k \in [m]$,
	\begin{align*}
	\mathrm{Cov}\big(\mathfrak{U}_k^{(t)}, \mathfrak{U}_k^{(s)} \big)\equiv  \frac{1}{n}\sum_{\ell \in [n]}  \E^{(0)}    \big[\Xi_{t-1,\ell} (\mathfrak{V}_\ell^{([-1:t-1])})\big]_{\cdot,t-1}     \big[\Xi_{s-1,\ell} (\mathfrak{V}_\ell^{([-1:s-1])})\big]_{\cdot,s-1}^\top.
	\end{align*}
	\item Let $\Xi_{t}:\R^{n_{L,q}\times q[-1:t]}\to \R^{n_{L,q}\times q[-1:t]}$ be defined as follows: for $w \in [-1:t-1]$, $\big[\Xi_{t}(\mathfrak{v}^{([0:t])})\big]_{\cdot,w}\equiv \big[\Xi_{w}(\mathfrak{v}^{([0:w])})\big]_{\cdot,w}$, and for $w=t$,
	\begin{align*}
	\big[\Xi_{t}(\mathfrak{v}^{([-1:t])})\big]_{\cdot,t} &\equiv \mathfrak{v}^{(t)}+\sum_{s \in [1:t]}  \big[\Xi_{s-1}(\mathfrak{v}^{([-1:s-1])}) \big]_{\cdot,s-1}\circ (\mathfrak{g}_{s}^{(t)})_\top+ \overline{\mu}_{\ast,n}\circ (\mathfrak{g}_{0}^{(t)})_\top \\
	&\qquad +\big[\Xi_{t-1}(\mathfrak{v}^{([-1:t-1])}) \big]_{\cdot,t-1}- \begin{pmatrix}
	0_{n\times q}\\
	\big\{\eta_\alpha^{(t-1)} M_\alpha^{(t-1)}\big\}_{\alpha \in [2:L]}
	\end{pmatrix},
	\end{align*}
	where the coefficient matrices $\{\mathfrak{g}_{s,\ell}^{(t)}\}_{s \in [0:t],\ell \in [n_{L,q}]}\subset \mathbb{M}_q$ are determined via
	\begin{align*}
	\mathfrak{g}_{s,\ell}^{(t)}\equiv 
	\begin{cases}
	- \frac{\eta_1^{(t-1)}}{\phi n} \sum\limits_{k \in [m]}   \E^{(0)} \nabla_{\mathfrak{U}_k^{(s)}} \mathfrak{S}_{k}\Big(\big[\Phi_{t}(\mathfrak{U}^{([-1:t])})\big]_{k,t} ,\mathfrak{U}_k^{(0)},V^{(t-1)}\Big) , & \ell \in [n];\\
	0_{q\times q}, & \ell \in (n: n_{L,q}].
	\end{cases}
	\end{align*}
	\item Let the Gaussian laws of $\{\mathfrak{V}_\ell^{(t)}\}_{\ell \in [n]}\subset \R^q$ be determined via the following correlation specification: for $s \in [1:t]$, $\ell \in [n]$,
	\begin{align*}
	\mathrm{Cov}(\mathfrak{V}_\ell^{(t)},\mathfrak{V}_\ell^{(s)}) &\equiv \frac{1}{n \phi^2 }\cdot \eta_1^{(t-1)}\eta_1^{(s-1)} \sum_{k \in [m]}  \E^{(0)} \mathfrak{S}_{k}\Big(\big[\Phi_{t}(\mathfrak{U}^{([-1:t])})\big]_{k,t} ,\mathfrak{U}_k^{(0)},V^{(t-1)}\Big)\\
	&\qquad\times  \mathfrak{S}_{k}\Big(\big[\Phi_{s}(\mathfrak{U}^{([-1:s])})\big]_{k,s} ,\mathfrak{U}_k^{(0)},V^{(s-1)}\Big)^\top,
	\end{align*}
	and for $\ell \in (n:n_{L,q}]$, $\mathfrak{V}_\ell^{(t)}\equiv 0$.
\end{enumerate}

\noindent (\textbf{Step 2}). From here, let us make several identifications:
\begin{itemize}
	\item Identify $\{\mathfrak{U}^{(t)}_k\}_{ k \in [m]}$ as $\mathfrak{U}^{(t)}$, and $\{\mathfrak{V}^{(t)}_\ell\}_{ \ell \in [n]}$ as $\mathfrak{V}^{(t)}$. 
	\item The variable $\mathfrak{U}^{(-1)}$ can be dropped for free.
	\item Identify $\Xi_{t}(\mathfrak{V}^{([-1:t])})$ as its last row as an element in $\R^{n\times q}$.
	\item Identify $\Phi_{t}(\mathfrak{U}^{([-1:t])})$ as its last row as an element in $\R^{m\times q}$.
	 \item Identify $\{\mathfrak{f}_{s,k}^{(t-1) } \}_{k \in [m]}$ as 
	 \begin{align*}
	 \mathfrak{f}_{s}^{(t-1) }\equiv \E^{(0)} \nabla_{\mathfrak{V}^{(s)}}  \Xi_{t-1,\pi_n} (\mathfrak{V}^{([-1:t-1])})\in \mathbb{M}_q.
	 \end{align*}
	 \item Identify $\{\mathfrak{g}_{s,\ell}^{(t)}\}_{\ell \in [n]}$ as 
	 \begin{align*}
	 \mathfrak{g}_{s}^{(t)}\equiv -\eta_1^{(t-1)} \E^{(0)} \nabla_{\mathfrak{U}^{(s)}} \mathfrak{S}_{\pi_m}\big(\Phi_{t,\pi_m}(\mathfrak{U}^{([-1:t])}),\mathfrak{U}^{(0)},V^{(t-1)}\big)\in \mathbb{M}_q.
	 \end{align*}
\end{itemize}
 With formal variables $\Delta_0^{(1)}\equiv n^{1/2} W_1^{(0)}\in \R^{n\times q}$, $\Delta_0^{(2)}\equiv \big(W_\alpha^{(0)}\big)_{\alpha \in [2:L]}\in \R^{(n_{L,q}-n)\times q}$,  let $\Delta_t^{(1)}: \R^{n\times q[0:t]}\to \R^{n\times q}$, $\Delta_t^{(2)}: \R^{(n_{L,q}-n)\times q[0:t]}\to \R^{(n_{L,q}-n)\times q}$ be defined by 
\begin{align}\label{ineq:lem:se_truncated_aux_GFOM_1}
\Delta_t^{(1)}(\mathfrak{v}_1^{([-1:t])} )&=\mathfrak{v}_1^{(t)}+ \sum_{s \in [1:t]}  \Delta_{s-1}^{(1)}(\mathfrak{v}_1^{([-1:s-1])})  \big(\mathfrak{g}_{s}^{(t),\top}+I_q\cdot \bm{1}_{s=t}\big) + (\mu_{\ast,n})_{[q]}\mathfrak{g}_0^{(t),\top},\nonumber\\
\Delta_t^{(2)}\big(\mathfrak{v}_{[2:L]}^{([-1:t])} \big)&=\mathfrak{v}_{[2:L]}^{(t)}+ \Delta_{t-1}^{(2)}\big(\mathfrak{v}_{[2:L]}^{([-1:t-1])} \big) -
\big(\eta_\alpha^{(t-1)} M_\alpha^{(t-1)}\big)_{\alpha \in [2:L]},
\end{align}
Clearly $\Delta_t^{(1)}(\mathfrak{v}_1^{([-1:t])} )$ depends on $\mathfrak{v}_1^{([-1:t])}$ only through linear combinations of $\mathfrak{v}_1^{([1:t])}$, so using the identified definition of $\mathfrak{f}_{s}^{(\cdot) }$, for some $\mathfrak{d}_{0}^{(t)} \in \R^{q\times n}$,
\begin{align}\label{ineq:lem:se_truncated_aux_GFOM_2}
\Delta_t^{(1)}(\mathfrak{v}_1^{([-1:t])} ) = \sum_{s \in [1:t]} \mathfrak{v}_1^{(s)} \mathfrak{f}_{s}^{(t),\top } + \mathfrak{d}_{0}^{(t),\top } \in \R^{n\times q}.
\end{align}
Now plugging (\ref{ineq:lem:se_truncated_aux_GFOM_2}) into the first equation of (\ref{ineq:lem:se_truncated_aux_GFOM_1}), we have
\begin{align*}
\mathfrak{f}_{s}^{(t) } &= I_q \bm{1}_{s=t}+\sum_{r \in [s+1:t]} \big(\mathfrak{g}_r^{(t)}+I_q\bm{1}_{r=t}\big)\mathfrak{f}_{s}^{(r-1) } \in \mathbb{M}_q,\quad s \in [1:t],\\
\mathfrak{d}_{0}^{(t) } &= \sum_{r \in [1:t]} \big(\mathfrak{g}_r^{(t)}+I_q\bm{1}_{r=t}\big) \mathfrak{d}_{0}^{(r-1) } + \mathfrak{g}_0^{(t)} (\mu_{\ast,n})_{[q]}^\top \in \R^{q\times n}.
\end{align*}
On the other hand, $\Delta_t^{(2)}\big(\mathfrak{v}_{[2:L]}^{([-1:t])} \big)$ is a  constant shift of $\sum_{s \in [1:t]} \mathfrak{v}^{(s)}_{[2:L]}$:
\begin{align*}
\Delta_t^{(2)}\big(\mathfrak{v}_{[2:L]}^{([-1:t])} \big) = \sum_{s \in [1:t]} \mathfrak{v}^{(s)}_{[2:L]}+\bigg(W_\alpha^{(0)}-\sum_{s \in [1:t]}\eta_\alpha^{(s-1)} M_\alpha^{(s-1)}\bigg)_{\alpha \in [2:L]} \in \R^{(n_{L,q}-n)\times q}.
\end{align*}
Moreover, with some simple calculations,
\begin{align*}
\mathrm{Cov}\big(\mathfrak{U}^{(t)}, \mathfrak{U}^{(s)} \big)&\equiv  \frac{1}{n}\sum_{\ell \in [n]}  \E^{(0)}    \Xi_{t-1,\ell} (\mathfrak{V}^{([-1:t-1])})     \Xi_{s-1,\ell} (\mathfrak{V}^{([-1:s-1])})^\top\\
& = \sum_{r\in [1:t-1],r' \in [1:s-1]}  \mathfrak{f}_{r}^{(t-1) } \Sigma_{r,r'} \mathfrak{f}_{r'}^{(s-1),\top } + \frac{ 1 }{n} \mathfrak{d}_{0}^{(t-1) } \mathfrak{d}_{0}^{(s-1),\top } \in \mathbb{M}_q,
\end{align*}
where $
\Sigma_{t,s}\equiv \mathrm{Cov}(\mathfrak{V}^{(t)},\mathfrak{V}^{(s)})$.

\noindent (\textbf{Step 3}). 
Under a change of notation with $\mathfrak{f}_{s}^{(t)}\equiv \rho_{t,s}$, $\mathfrak{g}_{s}^{(t)}\equiv \tau_{t,s}$ and $D_t\equiv \mathfrak{d}_{0}^{(t),\top}\in \R^{n\times q}$, we now summarize the above identifications into the following state evolution: Recall $\phi=m/n$. Initialize with 
\begin{itemize}
	\item $V^{(0)}=(V_{\alpha}^{(0)})_{\alpha \in [2:L]}\equiv \big(W_\alpha^{(0)}\big)_{\alpha \in [2:L]}\in (\mathbb{M}_q)^{[2:L]}$,
	\item $\mathfrak{U}^{(0)}\equiv \big(\sigma_{\mu_\ast}\mathsf{Z}\,|\,0_{1\times (q-1)}\big)^\top \in \R^q$ where $\sigma_{\mu_\ast}^2\equiv \pnorm{\mu_\ast}{}^2$,
	\item $D_{-1}\equiv (\mu_{\ast,n})_{[q]} = [n^{1/2}\mu_\ast\,|\, 0_{n\times (q-1)}] \in \R^{n\times q}$ and $D_0\equiv n^{1/2}W_1^{(0)} \in \R^{n\times q}$.
\end{itemize}
For $t=1,2,\ldots$, we compute recursively as follows:  
\begin{enumerate}
	\item Let $\Theta^{(t)}:\R^{m\times q[0:t]}\to \R^{m\times q}$ be defined via
	\begin{align*}
	\Theta^{(t)}(\mathfrak{u}^{([0:t])}) &\equiv \mathfrak{u}^{(t)}- \phi^{-1}\sum_{s \in [1:t-1]} \eta^{(s-1)}_1\cdot \mathfrak{S}\big(\Theta^{(s)}(\mathfrak{u}^{([0:s])}) ,\mathfrak{u}^{(0)},V^{(s-1)}\big)\rho_{t-1,s}^\top.
	\end{align*}
	\item Let the Gaussian law of $\mathfrak{U}^{(t)} \in \R^q$ be determined via the following correlation specification: for $s \in [0:t]$, 
	\begin{align*}
	\mathrm{Cov}\big(\mathfrak{U}^{(t)}, \mathfrak{U}^{(s)} \big)\equiv  \sum_{r\in [1:t-1],r' \in [1:s-1]} \rho_{t-1,r} \Sigma_{r,r'} \rho_{s-1,r'}^\top + \frac{ 1 }{n} D_{t-1}^\top  D_{s-1 } \in \mathbb{M}_q.
	\end{align*}
	\item For $s \in [1:t]$, compute $\{\tau_{t,s}\}, \{\rho_{t,s}\}, \{\Sigma_{t,s}\}\subset \mathbb{M}_q$ as follows: 
	\begin{align*}
	\tau_{t,s}&\equiv - \eta_1^{(t-1)}\E^{(0)} \nabla_{\mathfrak{U}^{(s)}} \mathfrak{S}_{\pi_m}\big(\Theta_{\pi_m}^{(t)}(\mathfrak{U}^{([0:t])}),\mathfrak{U}^{(0)},V^{(t-1)}\big) \in \mathbb{M}_q,\\
	\rho_{t,s} &\equiv I_q \bm{1}_{s=t}+\sum_{r \in [s+1:t]}  \big(\tau_{t,r}+I_q\bm{1}_{r=t}\big) \rho_{r-1,s}\in \mathbb{M}_q,\\
	\Sigma_{t,s}&\equiv \phi^{-1}\cdot \eta_1^{(t-1)} \eta_1^{(s-1)}\cdot \E^{(0)}  \mathfrak{S}_{\pi_m}\big(\Theta_{\pi_m}^{(t)}(\mathfrak{U}^{([0:t])}),\mathfrak{U}^{(0)},V^{(t-1)}\big)\\
	&\qquad \times \mathfrak{S}_{\pi_m}\big(\Theta_{\pi_m}^{(s)}(\mathfrak{U}^{([0:s])}),\mathfrak{U}^{(0)},V^{(s-1)}\big)^\top \in \mathbb{M}_q.
	\end{align*}
	\item Compute $\delta_t \in \mathbb{M}_q$ and $D_t \in \R^{n\times q}$ as follows:
	\begin{align*}
	\delta_t&\equiv - \eta_1^{(t-1)}\E^{(0)} \nabla_{\mathfrak{U}^{(0)}} \mathfrak{S}_{\pi_m}\big(\Theta_{\pi_m}^{(t)}(\mathfrak{U}^{([0:t])}),\mathfrak{U}^{(0)},V^{(t-1)}\big)\in \mathbb{M}_q,\\
	D_t&\equiv \sum_{r \in [1:t]} D_{r-1} \big(\tau_{t,r}^\top+I_q\bm{1}_{r=t}\big) + D_{-1}\delta_t^\top \in \R^{n\times q}.
	\end{align*}
	\item Compute $M^{(t-1)}  = (M_{\alpha}^{(t-1)})_{\alpha \in [2:L]}, V^{(t)}  = (V_{\alpha}^{(t)})_{\alpha \in [2:L]}\in (\mathbb{M}_q)^{[2:L]}$ as follows: for $\alpha \in [2:L]$,
	\begin{align*}
	M_{\alpha}^{(t-1)}&\equiv  \E^{(0)}\mathscr{G}_{V^{(t-1)}_{[2:\alpha-1]};\pi_m }\big(\Theta_{\pi_m}^{(t)}(\mathfrak{U}^{([0:t])}) \big) \\
	&\qquad\qquad \times  \big(\mathfrak{T}_{\alpha}\big)_{\pi_m}\big(\Theta_{\pi_m}^{(t)}(\mathfrak{U}^{([0:t])}) ,\mathfrak{U}^{(0)},V^{(t-1)}\big)^\top \in \mathbb{M}_q,\\
	V_{\alpha}^{(t)}&\equiv  V_\alpha^{(t-1)}- \eta_\alpha^{(t-1)} M_\alpha^{(t-1)}=\cdots =W_\alpha^{(0)}-\sum_{s \in [1:t]}\eta_\alpha^{(s-1)} M_{\alpha}^{(s-1)} \in \mathbb{M}_q.
	\end{align*}
\end{enumerate}
Setting $\Upsilon^{(t)}(\mathfrak{u}^{([0:t])})\equiv \phi^{-1} \eta_1^{(t-1)} \mathfrak{S}\big(\Theta^{(t)}(\mathfrak{u}^{([0:t])}) ,\mathfrak{u}^{(0)},V^{(t-1)}\big)$, eliminating $M^{(t)}$, as well as noting the definition for $\Omega_{t,s}$ leads to the desired state evolution. 
\end{proof}

We next control the difference between the state evolution parameters. Let 
\begin{align*}
\Delta\check{\mathcal{B}}_{\mathfrak{p}}^{(t)}&\equiv \max_{s \in [0:t]} \Big\{\pnorm{\Delta \check{\tau}_{t,s}}{\op}+\sup_{\psi \in \mathrm{Lip}(\mathfrak{p})}\bigabs{\E^{(0)}\psi(\check{\mathfrak{U}}^{([0:t])}) - \E^{(0)}\psi(\mathfrak{U}^{([0:t])})  }\Big\}\\
&\qquad  + \max_{s \in [1:t]} \Big\{\pnorm{ \Delta\check{\rho}_{t,s}}{\op}+ \pnorm{\Delta\check{\Sigma}_{t,s}}{\op} + \pnorm{\Delta \check{\Omega}_{t,s}}{\op} \Big\}\\
&\qquad + \frac{\pnorm{\Delta \check{D}_t}{}^2}{n}+ \max_{\alpha \in [2:L]} \big\{\pnorm{\Delta \check{V}^{(t)}_\alpha}{\op}+\pnorm{\Delta\check{M}_{\alpha}^{(t)}}{\op}\big\}.
\end{align*}
Here we let $\Delta \check{\ast}\equiv \check{\ast}-\ast$, and $\check{\mathcal{B}}^{(t)}$ be defined analogous to that in Proposition \ref{prop:apr_est_se}. Moreover, the suprema over $\psi:\R^{[0:t]}\to \R$ runs over all pseudo-Lipschitz functions of order $\mathfrak{p}$.

\begin{lemma}\label{lem:Delta_B_control}
	Suppose Assumption \ref{assump:gd_NN} holds for some $K,\Lambda\geq 2$ and $r_0\geq 0$. There exists a large constant $c_t\equiv c_t(t,q,L,r_0)>0$ and another constant $c_t'\equiv c_t'(t,\mathfrak{p},q,L,r_0)>0$, such that if we choose $\min_{s \in [1:t]}c_{0;s}\geq c_t$ in the definition of $\{\mathsf{b}_s\}_{s \in [1:t]}$ in (\ref{def:S_bar}), then
	\begin{align*}
	\Delta\check{\mathcal{B}}_{\mathfrak{p}}^{(t)}\leq \big(K\Lambda\varkappa_{\ast}\log n\big)^{c_t'}\cdot n^{-50}.
	\end{align*}
\end{lemma}
\begin{proof}
The proof for Proposition \ref{prop:apr_est_se} shows that there exists some $c_{1;t}=c_{1;t}(t,q,L,r_0)>0$ such that for any choice of $\{c_{0;s}\}_{s\in [1:t]}$ in the definition of $\{\mathsf{b}_s\}_{s \in [1:t]}$ in (\ref{def:S_bar}), we have 
\begin{align}\label{ineq:Delta_B_control_1}
\check{\mathcal{B}}^{(t)}+ \max_{s \in [1:t]} \frac{\pnorm{\check{\mathfrak{S}}_{s}\big(\check{\Theta}^{(s)}(\check{\mathfrak{u}}^{([0:s])}),\check{\mathfrak{u}}^{(0)},V^{(s-1)}\big)}{\infty} }{ 1+\pnorm{\check{\mathfrak{u}}^{([1:s])}}{\infty}+ \pnorm{\varphi_\ast(\check{\mathfrak{u}}^{(0)})}{\infty}  }\leq \big(K\Lambda\varkappa_{\ast}\log n\big)^{c_{1;t}}.
\end{align} 
Consequently, there exists some $c_{2;t}=c_{2;t}(t,q,L,r_0)>0$ such that for any choice of $\{c_{0;s}\}_{s\in [1:t]}$, on an event $E_0$ with probability at least $1-c_{2;t} n^{-100}$,
\begin{align}\label{ineq:Delta_B_control_2}
\max_{s \in [0:t]} \big\{ \pnorm{ \mathfrak{U}^{(s)} }{\infty}\vee \pnorm{ \check{\mathfrak{U}}^{(s)} }{\infty} \big\}\leq \big(K\Lambda\varkappa_{\ast}\log n\big)^{c_{2;t}}.
\end{align}
Combining (\ref{ineq:Delta_B_control_1})-(\ref{ineq:Delta_B_control_2}), with the choice $\{c_{0;s}\geq c_{1;t}+r_0 c_{2;t}+2\}_{s \in [1:t]}$ in the definition of $\{\mathsf{b}_s\}_{s \in [1:t]}$ (which means for any given $t$, we may choose the same $\{c_{0;s}\}_{s \in [1:t]}$ above the prescribed threshold), on the event $E_0$ we have 
\begin{align}\label{ineq:Delta_B_control_3}
\check{\Theta}^{(t)}(\check{\mathfrak{U}}^{([0:t])}) &= \check{\mathfrak{U}}^{(t)}-\phi^{-1} \sum_{s \in [1:t-1]} \eta^{(s-1)}_1\cdot \mathfrak{S}\big(\check{\Theta}^{(s)}(\check{\mathfrak{U}}^{([0:s])}),\check{\mathfrak{U}}^{(0)},\check{V}^{(s-1)}\big)\check{\rho}_{t-1,s},\nonumber\\
\check{\Upsilon}^{(t)}(\check{\mathfrak{U}}^{([0:t])}) &= \phi^{-1}\eta_1^{(t-1)} \mathfrak{S}\big(\check{\Theta}^{(t)}(\check{\mathfrak{U}}^{([0:t])}) ,\check{\mathfrak{U}}^{(0)},\check{V}^{(t-1)}\big).
\end{align}
By (S2) and the estimate (\ref{ineq:Delta_B_control_1}), on the event $E_0$,
\begin{align}\label{ineq:Delta_B_control_4}
\sup_{\psi \in \mathrm{Lip}(\mathfrak{p})}\bigabs{\E^{(0)}\psi(\check{\mathfrak{U}}^{([0:t])}) - \E^{(0)}\psi(\mathfrak{U}^{([0:t])})  }\leq \big(K\Lambda\varkappa_{\ast}\log n\big)^{c_{t}}\cdot \Delta\check{\mathcal{B}}_{\mathfrak{p}}^{(t-1)}.
\end{align}
By (S3), we have for $s \in [0:t]$,
\begin{align*}
&\pnorm{\Delta \check{\tau}_{t,s}}{\op}\leq  \eta_1^{(t-1)} \bigpnorm{\E^{(0)}\nabla_{(s)}\mathfrak{S}_{\pi_m}\big(\check{\Theta}^{(t)}(\check{\mathfrak{U}}^{([0:t])}) ,\check{\mathfrak{U}}^{(0)},\check{V}^{(t-1)}\big)\bm{1}_{E_0}\\
&\qquad - \E^{(0)}\nabla_{(s)}\mathfrak{S}_{\pi_m}\big({\Theta}^{(t)}({\mathfrak{U}}^{([0:t])}) ,{\mathfrak{U}}^{(0)},{V}^{(t-1)}\big)\bm{1}_{E_0}  }{\op}+\big(K\Lambda\varkappa_{\ast}\log n\big)^{c_{t}}\cdot n^{-50}.
\end{align*}
With some patient calculations, similar to Proposition \ref{prop:apr_est_derivative}, there exists some universal $c_0>0$ such that for any $k,k' \in [m]$ and $\ell,\ell' \in [q]$,
\begin{align*}
&\pnorm{\partial^2_{(u_{k\ell},u_{k'\ell'}) } \mathfrak{S}\big(u,w, v\big)}{\infty} \leq  (q\Lambda \kappa_{v})^{c_0 L} \cdot \big(1+\pnorm{u }{\infty}+\pnorm{w}{\infty}+\pnorm{\xi}{\infty}\big)^{c_0r_0}.
\end{align*}
Combining the above two displays, we have
\begin{align}\label{ineq:Delta_B_control_5}
\max_{s \in [0:t]}\pnorm{\Delta \check{\tau}_{t,s}}{\op}\leq\big(K\Lambda\varkappa_{\ast}\log n\big)^{c_{t}}\cdot \big(\Delta\check{\mathcal{B}}_{\mathfrak{p}}^{(t-1)}+n^{-50}\big).
\end{align}
Similar (or simpler) arguments combined with (\ref{ineq:Delta_B_control_4}) lead to
\begin{align}\label{ineq:Delta_B_control_6}
&\max_{s \in [1:t]} \Big\{\pnorm{ \Delta\check{\rho}_{t,s}}{\op}+ \pnorm{\Delta\check{\Sigma}_{t,s}}{\op} + \pnorm{\Delta \check{\Omega}_{t,s}}{\op} \Big\} \nonumber\\
&\leq \big(K\Lambda\varkappa_{\ast}\log n\big)^{c_{t}}\cdot \big(\Delta\check{\mathcal{B}}_{\mathfrak{p}}^{(t-1)}+n^{-50}\big).
\end{align}
By (S4), we have
\begin{align}\label{ineq:Delta_B_control_7}
\frac{\pnorm{\Delta \check{D}_t}{}^2}{n}\leq \big(K\Lambda\varkappa_{\ast}\log n\big)^{c_{t}}\cdot \Delta\check{\mathcal{B}}_{\mathfrak{p}}^{(t-1)}.
\end{align}
By (S5), we have similarly 
\begin{align}\label{ineq:Delta_B_control_8}
\max_{\alpha \in [2:L]} \big\{\pnorm{\Delta \check{V}^{(t)}_\alpha}{\op}+\pnorm{\Delta\check{M}_{\alpha}^{(t)}}{\op}\big\}\leq \big(K\Lambda\varkappa_{\ast}\log n\big)^{c_{t}}\cdot \big(\Delta\check{\mathcal{B}}_{\mathfrak{p}}^{(t-1)}+n^{-50}\big).
\end{align}
Now combining (\ref{ineq:Delta_B_control_4})-(\ref{ineq:Delta_B_control_8}), we have
\begin{align*}
\Delta\check{\mathcal{B}}_{\mathfrak{p}}^{(t)}\leq \big(K\Lambda\varkappa_{\ast}\log n\big)^{c_{t}}\cdot \big(\Delta\check{\mathcal{B}}_{\mathfrak{p}}^{(t-1)}+n^{-50}\big).
\end{align*}
Iterating the above estimate and using the initial condition $\Delta\check{\mathcal{B}}_{\mathfrak{p}}^{(1)}\leq \big(K\Lambda\varkappa_{\ast}\log n\big)^{c_{t}}\cdot n^{-50}$ to conclude. 
\end{proof}

We finally compare the GFOM iterations $\{(u^{(t)},v^{(t)})\}$ in (\ref{def:GFOM_auxiliary_NN}) and  $\{(\check{u}^{(t)},\check{v}^{(t)})\}$ in (\ref{def:GFOM_bar_auxiliary_NN}). Let us write $\Delta \check{u}^{(t)}\equiv \check{u}^{(t)}-u^{(t)}$ and $\Delta \check{v}^{(t)}\equiv \check{v}^{(t)}-v^{(t)}$.

\begin{lemma}\label{lem:diff_u_baru}
Suppose Assumption \ref{assump:gd_NN} holds for some $K,\Lambda\geq 2$ and $r_0\geq 0$. There exists a large constant $c_t\equiv c_t(t,q,L,r_0)>0$ and another constant $c_t'\equiv c_t'(t,q,L,r_0)>0$, such that if we choose $\min_{s \in [1:t]}c_{0;s}\geq c_t$ in the definition of $\{\mathsf{b}_s\}_{s \in [1:t]}$ in (\ref{def:S_bar}), then
\begin{align*}
\pnorm{\Delta \check{u}^{(t)}}{}\vee \pnorm{\Delta \check{v}^{(t)}}{}&\leq \big(K\Lambda\varkappa_{\ast}\log n\big)^{c_t'}\cdot n^{-50}.
\end{align*}
\end{lemma}
\begin{proof}
Note that a standard subgaussian estimate shows that on an event $E_0$ with probability at least $1-e^{-n}$, we have $\pnorm{A}{\op}\leq c_0 K$ for some universal $c_0>1$. Moreover, using Proposition \ref{prop:GD_linfty_est} (which equally applies to $\{(\check{u}^{(t)},\check{v}^{(t)})\}$ in (\ref{def:GFOM_bar_auxiliary_NN})), there exists some $c_{1;t}\equiv c_{1;t}(t,q,L,r_0)>1$ such that on an event $E_1$ with probability at least $1-c_{1;t} n^{-100}$, 
\begin{align}\label{ineq:diff_u_baru_1}
\max_{s \in [0:t]} \big\{\pnorm{u^{(s)}}{\infty}\vee \pnorm{v^{(s)}}{\infty}\vee \pnorm{\check{u}^{(s)}}{\infty}\vee \pnorm{\check{v}^{(s)}}{\infty} \big\}\leq \big(K\Lambda\varkappa_{\ast}\log n\big)^{c_{1;t} }.
\end{align}
Now by Proposition \ref{prop:apr_est_det_fcn}-(6), with $\{c_{0;s}\}_{s \in [1:t]}$ chosen as $\{c_{0;s}\geq 100r_0 c_{1;t}\}_{s \in [1:t]}$, on $E_1$, $(\check{u}^{(*)},\check{v}^{(*)})=(u^{(*)},v^{(*)})$ for $*=-1,0$, and for $s=1,\ldots,t$,
\begin{align}\label{ineq:diff_u_baru_2}
\check{u}^{(s)}&= A \check{v}^{(s-1)},\\
\check{v}^{(s)}& =  A^\top 
\Big[
-\phi^{-1}\eta^{(s-1)}_1\cdot  \mathfrak{S}\big(\check{u}^{(s)},\check{u}^{(0)},\check{V}^{(s-1)}\big) \Big]+\check{v}^{(s-1)}-
\begin{pmatrix}
0_{n\times q}\\
\Big\{\eta^{(s-1)}_{\alpha} \check{M}_\alpha^{(s)}\Big\}_{\alpha \in [2:L]}
\end{pmatrix}.\nonumber
\end{align}
Comparing the above display with the alternative definition of the auxiliary GFOM $\{(u^{(t)},v^{(t)})\}$ in (\ref{def:GFOM_auxiliary_NN_1}), on the event $E_0\cap E_1$, for $s \in [1:t]$, we have $\pnorm{\Delta \check{u}^{(s)}}{}\leq c_0 K\cdot  \pnorm{\Delta \check{v}^{(s-1)}}{}$, and
\begin{align*}
\pnorm{\Delta \check{v}^{(s)}}{}&\leq \pnorm{\Delta \check{v}^{(s-1)}}{}+ (K\Lambda)^{c_0}\cdot \bigg(\sum_{\alpha \in [2:L]} \pnorm{\Delta \check{M}_\alpha^{(s-1)}}{}\\
&\qquad + \bigpnorm{\mathfrak{S}\big(\check{u}^{(s)},\check{u}^{(0)},\check{V}^{(s-1)}\big) -\mathfrak{S}\big(u^{(s)},u^{(0)},V^{(s-1)}\big) }{}\bigg).
\end{align*}
Let us choose $\{c_{0;s}\}_{s \in [1:t]}$ be as in Lemma \ref{lem:Delta_B_control} and which satisfies $\{c_{0;s}\geq 100r_0c_{1;t}\}_{s \in [1:t]}$ as above. Using Proposition \ref{prop:apr_est_det_fcn}-(7) and Lemma \ref{lem:Delta_B_control}, along with the estimate in (\ref{ineq:diff_u_baru_1}) and the estimate for $\check{\mathcal{B}}^{(t)}$ in (\ref{ineq:Delta_B_control_1}), on $E_0\cap E_1$ we have for all $s \in [1:t]$,
\begin{align*}
\pnorm{\Delta \check{v}^{(s)}}{}&\leq \pnorm{\Delta \check{v}^{(s-1)}}{}+ \big(K\Lambda\varkappa_{\ast}\log n\big)^{c_t}\cdot \big(n^{-50}+\pnorm{\Delta \check{u}^{(s)}}{}\big).
\end{align*}
Now using the trivial estimate $\pnorm{\Delta \check{u}^{(s)}}{}\leq \pnorm{A}{\op} \pnorm{\Delta \check{v}^{(s-1)}}{}$, on $E_0\cap E_1$ we have for all $s \in [1:t]$,
\begin{align*}
\pnorm{\Delta \check{v}^{(s)}}{}&\leq \big(K\Lambda\varkappa_{\ast}\log n\big)^{c_t}\cdot \big(n^{-50}+\pnorm{\Delta \check{v}^{(s-1)}}{}\big).
\end{align*}
The claimed error bound follows by iterating the above estimate with initial condition $\pnorm{\Delta \check{v}^{(0)}}{}=0$.
\end{proof}

\begin{proof}[Proof of Proposition \ref{prop:se_aux_GFOM}]
The claimed state evolution now follows by a combination of Lemmas \ref{lem:se_truncated_aux_GFOM}, \ref{lem:Delta_B_control} and \ref{lem:diff_u_baru}. 
\end{proof}

\subsection{Error control between the GD and the auxiliary GFOM: Proof of Proposition \ref{prop:err_gd_GFOM}}\label{subsection:proof_err_gd_GFOM}

	In the proof we write $\sigma_1\equiv \sigma$ for notational simplicity. First note that $\hat{u}^{(t)}=u^{(t)}$ and $\hat{v}^{(t)}=v^{(t)}$ for $t=-1,0$. Recall the notation in (\ref{def:hat_P_G_R}). For $t\geq 1$, let
	\begin{align*}
	\mathscr{P}_{  V^{(s)}_{[2:L]} }^{(\alpha:L]}\equiv {\mathscr{P}}_{(s)}^{(\alpha:L]},\quad \mathscr{G}_{V^{(s)}_{[2:\alpha]} }\equiv {\mathscr{G}}_{(s);\alpha},\quad \mathscr{R}_{V^{(s)}_{[2:L]} }\equiv {\mathscr{R}}_{(s)}.
	\end{align*}
	Then the auxiliary GFOM in (\ref{def:GFOM_auxiliary_NN}) can be written as $u^{(t)}\equiv A v^{(t-1)} \in \R^{m\times q}$ and
	\begin{align}\label{ineq:err_gd_GFOM_1}
	v^{(t)}
	&  = A^\top 
	\Big[
	-\phi^{-1}\eta^{(t-1)}_1\cdot \mathscr{P}_{ (t-1) }^{(1:L]}\big(u^{(t)},\mathscr{R}_{(t-1)}(u^{(t)},u^{(0)})\big) \odot \sigma'(u^{(t)})\Big]\nonumber\\
	&\quad +v^{(t-1)}-
	\begin{pmatrix}
	0_{n\times q}\\
	\Big\{\eta^{(t-1)}_{\alpha} M_\alpha^{(t-1)}\Big\}_{\alpha \in [2:L]}
	\end{pmatrix}.
	\end{align}
	Comparing (\ref{ineq:err_gd_GFOM_1}) and $\{(\hat{u}^{(t)},\hat{v}^{(t)})\}$ in (\ref{def:GD_GFOM_NN}), we then have 
	\begin{align}\label{ineq:err_gd_GFOM_2}
	n^{-1/2}\pnorm{\Delta\hat{v}_1^{(t)}}{} &\leq n^{-1/2}\pnorm{\Delta \hat{v}_1^{(t-1)}}{}+K\Lambda \pnorm{A}{\op}\nonumber\\
	&\quad \times n^{-1/2}\bigpnorm{\mathscr{P}_{ (t-1) }^{(1:L]}\big(u^{(t)},\mathscr{R}_{(t-1)}(u^{(t)},u^{(0)})\big) \odot \sigma'(u^{(t)})\nonumber\\
		&\quad\quad\quad  -\hat{\mathscr{P}}_{ (t-1) }^{(1:L]}\big(\hat{u}^{(t)},\hat{\mathscr{R}}_{(t-1)}(\hat{u}^{(t)},\hat{u}^{(0)})\big) \odot \sigma'(\hat{u}^{(t)})  }{},
	\end{align}
	and
	\begin{align}\label{ineq:err_gd_GFOM_3}
	&\pnorm{\Delta  \hat{v}_{[2:L]}^{(t)}}{\infty}\leq \pnorm{\Delta \hat{v}_{[2:L]}^{(t-1)}}{\infty} +\Lambda\cdot \max_{\alpha \in [2:L]}\bigpnorm{m^{-1}\hat{\mathscr{G}}_{(t-1);\alpha-1}(\hat{u}^{(t)})^\top\nonumber\\
	&\qquad  \Big[\hat{\mathscr{P}}_{ (t-1) }^{(\alpha:L]}\Big(\hat{u}^{(t)},\hat{\mathscr{R}}_{(t-1)}(\hat{u}^{(t)},\hat{u}^{(0)})\Big) \odot \hat{\mathscr{G}}_{(t-1);\alpha}'(\hat{u}^{(t)})\Big]-M_\alpha^{(t-1)}}{\infty}.
	\end{align}
	In the following, we shall provide bounds for the right hand sides of (\ref{ineq:err_gd_GFOM_2}) and (\ref{ineq:err_gd_GFOM_3}) that lead to a recursive estimate for $n^{-1/2}\pnorm{\Delta\hat{v}_1^{(t)}}{}$ and $\pnorm{\Delta  \hat{v}_{[2:L]}^{(t)}}{\infty}$.
	
	\noindent (\textbf{Step 1}). In this step, we prove that for any $\mathfrak{r}\geq 1$, there exists some $c_t=c_t(t,q,L,\mathfrak{r},r_0)>0$ such that
	\begin{align}\label{ineq:err_gd_GFOM_step1}
	&\E^{(0)}\Big(n^{-1/2}\bigpnorm{\mathscr{P}_{ (t-1) }^{(1:L]}\big(u^{(t)},\mathscr{R}_{(t-1)}(u^{(t)},u^{(0)})\big) \odot \sigma'(u^{(t)})\nonumber\\
	&\qquad -\hat{\mathscr{P}}_{ (t-1) }^{(1:L]}\big(\hat{u}^{(t)},\hat{\mathscr{R}}_{(t-1)}(\hat{u}^{(t)},\hat{u}^{(0)})\big) \odot \sigma'(\hat{u}^{(t)})  }{}\Big)^{\mathfrak{r}}\nonumber\\
	&\leq \big(K\Lambda\varkappa_{\ast}\log n\big)^{c_t}\cdot\big[\E^{(0)}\big(n^{-1/2}\pnorm{\Delta \hat{v}_1^{(t-1)} }{}\big)^{\mathfrak{r}}+ n^{-1/c_t}\big].
	\end{align}	
	To prove (\ref{ineq:err_gd_GFOM_step1}), note that
	\begin{align}\label{ineq:err_gd_GFOM_step1_1}
	&\bigpnorm{\mathscr{P}_{ (t-1) }^{(1:L]}\big(u^{(t)},\mathscr{R}_{(t-1)}(u^{(t)},u^{(0)})\big) \odot \sigma'(u^{(t)})\nonumber\\
		&\qquad -\hat{\mathscr{P}}_{ (t-1) }^{(1:L]}\big(\hat{u}^{(t)},\hat{\mathscr{R}}_{(t-1)}(\hat{u}^{(t)},\hat{u}^{(0)})\big) \odot \sigma'(\hat{u}^{(t)})  }{}\nonumber\\
	&\leq \pnorm{\mathscr{P}_{ (t-1) }^{(1:L]}\big(u^{(t)},\mathscr{R}_{(t-1)}(u^{(t)},u^{(0)})\big) }{\infty}\cdot \pnorm{\sigma'(\hat{u}^{(t)})-\sigma'(u^{(t)}) }{}\nonumber\\
	&\quad + \pnorm{\sigma'}{\infty}\cdot \bigpnorm{\mathscr{P}_{ (t-1) }^{(1:L]}\big(u^{(t)},\mathscr{R}_{(t-1)}(u^{(t)},u^{(0)})\big) -\hat{\mathscr{P}}_{ (t-1) }^{(1:L]}\big(\hat{u}^{(t)},\hat{\mathscr{R}}_{(t-1)}(\hat{u}^{(t)},\hat{u}^{(0)})\big)  }{}\nonumber\\
	&\equiv I_1+I_2.
	\end{align}
	For $I_1$, using the estimates in Propositions \ref{prop:apr_est_det_fcn} and \ref{prop:apr_est_se}, we have
	\begin{align*}
	I_1&\leq (q\Lambda)^{c_0 L}\cdot \kappa_{V_{[2:L]}^{(t-1)}}\cdot \pnorm{\mathscr{R}_{(t-1)}(u^{(t)},u^{(0)})}{\infty}\cdot \pnorm{\Delta\hat{u}^{(t)} }{}\\
	&\leq (q\Lambda)^{c_0 L}\cdot \kappa_{V_{[2:L]}^{(t-1)}}\cdot\Big(\bigpnorm{ \mathscr{G}_{ V_{[2:L]}^{(t-1)} }(u^{(t)}) }{\infty}+\pnorm{ \varphi_\ast(u^{(0)}) }{\infty}+\pnorm{\xi}{\infty}\Big)\cdot \pnorm{\Delta\hat{u}^{(t)} }{}\\
	&\leq \big(K\Lambda\varkappa_{\ast}\log n\big)^{c_t}\cdot\big(1+\pnorm{ u^{(t)} }{\infty}+\pnorm{u^{(0)} }{\infty}^{r_0}\big)\cdot \pnorm{\Delta \hat{u}^{(t)} }{}.
	\end{align*}
	Using the delocalization estimate in Proposition \ref{prop:GD_linfty_est} and a standard subgaussian estimate for $\pnorm{A}{\op}$, with probability at least $1-n^{-100}$,
	\begin{align}\label{ineq:err_gd_GFOM_step1_2}
	I_1\leq \big(K\Lambda\varkappa_{\ast}\log n\big)^{c_t}\cdot\pnorm{\Delta \hat{v}_1^{(t-1)} }{}.
	\end{align}
	For $I_2$, again using the estimates in Propositions \ref{prop:apr_est_det_fcn} and \ref{prop:apr_est_se}, the $\ell_2$ and $\ell_\infty$ estimates in Propositions \ref{prop:GD_l2_est} and \ref{prop:GD_linfty_est}, and a standard subgaussian estimate for $\pnorm{A}{\op}$, with probability at least $1-n^{-100}$,
	\begin{align}\label{ineq:err_gd_GFOM_step1_3}
	I_2&\leq (q\Lambda)^{c_0 L} (\kappa_{V^{(t-1)}_{[2:L]}}\kappa_{v^{(t-1)}_{[2:L]}})^{c_0}\cdot \bigg\{ \bigpnorm{\mathscr{R}_{(t-1)}(u^{(t)},u^{(0)})-\hat{\mathscr{R}}_{(t-1)}(\hat{u}^{(t)},\hat{u}^{(0)}) }{}\nonumber\\
	&\qquad + \big(1+\pnorm{\hat{\mathscr{R}}_{(t-1)}(\hat{u}^{(t)},\hat{u}^{(0)})}{\infty}\wedge  \pnorm{\mathscr{R}_{(t-1)}(u^{(t)},u^{(0)})}{\infty}\big) \nonumber\\
	&\qquad \qquad \times \Big[ \pnorm{\Delta \hat{u}^{(t)} }{}+ \big(\sqrt{m}+\pnorm{u^{(t)}}{}+\pnorm{\hat{u}^{(t)}}{}\big) \cdot \pnorm{v^{(t-1)}_{[2:L]}-V^{(t-1)}_{[2:L]}}{\infty}\Big]\nonumber\\
	&\leq \big(K\Lambda\varkappa_{\ast}\pnorm{v^{(t-1)}_{[2:L]}}{\infty}\cdot \log n\big)^{c_t}\cdot \big(1+\pnorm{ u^{(t)} }{\infty}+\pnorm{u^{(0)} }{\infty}^{r_0}+m^{-1/2}\pnorm{\hat{u}^{(t)}}{}\big)\nonumber\\
	&\qquad \times \Big[\pnorm{ \Delta \hat{u}^{(t)} }{}+ \sqrt{m}\cdot \pnorm{v^{(t-1)}_{[2:L]}-V^{(t-1)}_{[2:L]}}{\infty}\Big]\nonumber\\
	&\leq \big(K\Lambda\varkappa_{\ast}\log n\big)^{c_t}\cdot\Big[\pnorm{\Delta \hat{v}_1^{(t-1)} }{}+ \sqrt{m}\cdot \pnorm{v^{(t-1)}_{[2:L]}-V^{(t-1)}_{[2:L]}}{\infty}\Big].
	\end{align}
	Now by combining (\ref{ineq:err_gd_GFOM_step1_1})-(\ref{ineq:err_gd_GFOM_step1_3}) and using the state evolution in Proposition \ref{prop:se_aux_GFOM} that provides moment controls for $\pnorm{v^{(t-1)}_{[2:L]}-V^{(t-1)}_{[2:L]}}{\infty}$, on an event $E_1$ with $\Prob^{(0)}(E_1^c)\leq 2n^{-100}$,  we have 
	\begin{align*}
	&\E^{(0)}\Big(n^{-1/2}\bigpnorm{\mathscr{P}_{ (t-1) }^{(1:L]}\big(u^{(t)},\mathscr{R}_{(t-1)}(u^{(t)},u^{(0)})\big) \odot \sigma'(u^{(t)})\nonumber\\
	&\qquad -\hat{\mathscr{P}}_{ (t-1) }^{(1:L]}\big(\hat{u}^{(t)},\hat{\mathscr{R}}_{(t-1)}(\hat{u}^{(t)},\hat{u}^{(0)})\big) \odot \sigma'(\hat{u}^{(t)})  }{}\Big)^{\mathfrak{r}}\bm{1}_{E_1}\nonumber\\
	&\leq \big(K\Lambda\varkappa_{\ast}\log n\big)^{c_t}\cdot\big[\E^{(0)}\big(n^{-1/2}\pnorm{\Delta \hat{v}_1^{(t-1)} }{}\big)^{\mathfrak{r}}+ n^{-1/c_t}\big].
	\end{align*}
	From here we may use the apriori estimates in Propositions \ref{prop:apr_est_det_fcn} and \ref{prop:GD_l2_est} to conclude the claimed inequality in (\ref{ineq:err_gd_GFOM_step1}).
	
	\noindent (\textbf{Step 2}). In this step, we prove that for any $\mathfrak{r}\geq 1$, there exists some $c_t=c_t(t,q,L,\mathfrak{r},r_0)>0$ such that
	\begin{align}\label{ineq:err_gd_GFOM_step2}
	&\E^{(0)}\bigpnorm{m^{-1}\hat{\mathscr{G}}_{(t-1);\alpha-1}(\hat{u}^{(t)})^\top \Big[\hat{\mathscr{P}}_{ (t-1) }^{(\alpha:L]}\Big(\hat{u}^{(t)},\hat{\mathscr{R}}_{(t-1)}(\hat{u}^{(t)},\hat{u}^{(0)})\Big) \odot \hat{\mathscr{G}}_{(t-1);\alpha}'(\hat{u}^{(t)})\Big]-M_\alpha^{(t-1)}}{\infty}^{\mathfrak{r}}\nonumber\\
	&\leq \big(K\Lambda\varkappa_{\ast}\log n\big)^{c_t}\cdot \Big[\E^{(0)}\big(n^{-1/2}\pnorm{\Delta \hat{v}_1^{(t-1)} }{}\big)^{\mathfrak{r}}  +\E^{(0)} \pnorm{\Delta \hat{v}_{[2:L]}^{(t-1)}}{\infty}^{\mathfrak{r}}+n^{-1/c_t}\Big].
	\end{align}
	To this end, note that for $d,d' \in [q]$,
	\begin{align*}
	&\bigg(m^{-1}\hat{\mathscr{G}}_{(t-1);\alpha-1}(\hat{u}^{(t)})^\top \Big[\hat{\mathscr{P}}_{ (t-1) }^{(\alpha:L]}\Big(\hat{u}^{(t)},\hat{\mathscr{R}}_{(t-1)}(\hat{u}^{(t)},\hat{u}^{(0)})\Big) \odot \hat{\mathscr{G}}_{(t-1);\alpha}'(\hat{u}^{(t)})\Big]\bigg)_{dd'}\\
	& = \frac{1}{m}\sum_{k \in [m]} \Big[\mathscr{G}_{\hat{v}_{[2:\alpha-1]}^{(t-1)} }(\hat{u}^{(t)})\Big]_{kd} \Big[\mathscr{P}_{\hat{v}_{[2:L]}^{(t-1)} }^{(\alpha:L]}\Big(\hat{u}^{(t)},{\mathscr{R}}_{\hat{v}_{[2:L]}^{(t-1)} }(\hat{u}^{(t)},\hat{u}^{(0)})\Big)\Big]_{kd'} \Big[\mathscr{G}'_{ \hat{v}_{[2:\alpha]}^{(t-1)}}(\hat{u}^{(t)})\Big]_{kd'}\\
	& = \bigg(\frac{1}{m}\sum_{k \in [m]} \mathscr{G}_{\hat{v}_{[2:\alpha-1]}^{(t-1)};k }(\hat{u}_{k\cdot}^{(t)}) (\mathfrak{T}_{\alpha})_k\big(\hat{u}^{(t)}_{k\cdot},\hat{u}_{k\cdot}^{(0)},\hat{v}_{[2:L]}^{(t-1)}\big)^\top \bigg)_{dd'}\\
	&\equiv \bigg(\frac{1}{m}\sum_{k \in [m]} \psi_{\alpha;k} \big(\hat{u}^{(t)}_{k\cdot},\hat{u}_{k\cdot}^{(0)},\hat{v}_{[2:L]}^{(t-1)}\big) \bigg)_{dd'}.
	\end{align*}
	This means
	\begin{align}\label{ineq:err_gd_GFOM_step2_1}
	\hbox{LHS of (\ref{ineq:err_gd_GFOM_step2})}&\leq \E^{(0)} \biggpnorm{\frac{1}{m}\sum_{k \in [m]} \psi_{\alpha;k}\big(u^{(t)}_{k\cdot},u_{k\cdot}^{(0)},v_{[2:L]}^{(t-1)}\big)-M_\alpha^{(t-1)} }{\infty}^{\mathfrak{r}} \nonumber\\
	&\quad + \E^{(0)}\biggpnorm{\frac{1}{m}\sum_{k \in [m]} \Big(\psi_{\alpha;k}\big(u^{(t)}_{k\cdot},u_{k\cdot}^{(0)},v_{[2:L]}^{(t-1)}\big)- \psi_{\alpha;k}\big(\hat{u}^{(t)}_{k\cdot},\hat{u}_{k\cdot}^{(0)},\hat{v}_{[2:L]}^{(t-1)}\big)\Big) }{\infty}^{\mathfrak{r}}\nonumber\\
	& \equiv \E^{(0)}J_1^{\mathfrak{r}}+\E^{(0)}J_2^{\mathfrak{r}}.
	\end{align}
	First we consider $J_1$. Using Proposition \ref{prop:apr_est_det_fcn}-(2)(7) and the $\ell_\infty$ estimates in (1)(6), along with the apriori estimates in Proposition \ref{prop:apr_est_se} and the $\ell_2/\ell_\infty$ estimates in Propositions \ref{prop:GD_l2_est} and \ref{prop:GD_linfty_est}, we have on an event $E_{2,1}$ with probability at least $1-n^{-100}$,
	\begin{align*}
	&\biggpnorm{\frac{1}{m}\sum_{k \in [m]} \psi_{\alpha;k}\big(u^{(t)}_{k\cdot},u_{k\cdot}^{(0)},v_{[2:L]}^{(t-1)}\big)-\frac{1}{m}\sum_{k \in [m]} \psi_{\alpha;k}\big(u^{(t)}_{k\cdot},u_{k\cdot}^{(0)},V_{[2:L]}^{(t-1)}\big)}{\infty}\\
	&\leq \big(K\Lambda\varkappa_{\ast}\log n\big)^{c_t}\cdot \bigpnorm{v_{[2:L]}^{(t-1)}-V_{[2:L]}^{(t-1)}}{ \infty}. 
	\end{align*}
	This means on $E_{2,1}$,
	\begin{align*}
	J_1 &\leq \biggpnorm{\frac{1}{m}\sum_{k \in [m]} \psi_{\alpha;k}\big(u^{(t)}_{k\cdot},u_{k\cdot}^{(0)},V_{[2:L]}^{(t-1)}\big)-  \psi_{\alpha;k}\big(\Theta_{k}^{(t)}(\mathfrak{U}^{([0:t])}),\mathfrak{U}^{(0)},V_{[2:L]}^{(t-1)}\big)  }{\infty}\\
	&\qquad + \big(K\Lambda\varkappa_{\ast}\log n\big)^{c_t}\cdot \bigpnorm{v_{[2:L]}^{(t-1)}-V_{[2:L]}^{(t-1)}}{ \infty}.
	\end{align*}
	Using the apriori estimates in Proposition \ref{prop:apr_est_det_fcn}, the functions $\{\psi_{\alpha;k}(\cdot,\cdot,V^{(t-1)}_{[2:L]})\}_{k \in [m]}$ are $\big(K\Lambda\varkappa_{\ast}\log n\big)^{c_t}$-pseudo-Lipschitz of order $c_0r_0$, and therefore we may apply Proposition \ref{prop:se_aux_GFOM} to conclude that for any $\mathfrak{r}\geq 1$, $\E^{(0)} J_1^{\mathfrak{r}}\bm{1}_{E_{2,1}}\leq \big(K\Lambda\varkappa_{\ast}\log n\big)^{c_t} \cdot n^{-1/c_t}$. By apriori estimates, we are then led to 
	\begin{align}\label{ineq:err_gd_GFOM_step2_2}
	\E^{(0)} J_1^{\mathfrak{r}}\leq \big(K\Lambda\varkappa_{\ast}\log n\big)^{c_t} \cdot n^{-1/c_t}. 
	\end{align}
	Next we consider $J_2$. Similarly using Proposition \ref{prop:apr_est_det_fcn}-(2)(7) and the $\ell_\infty$ estimates in (1)(6), along with the apriori estimates in Proposition \ref{prop:apr_est_se} and the $\ell_2/\ell_\infty$ estimates in Propositions \ref{prop:GD_l2_est} and \ref{prop:GD_linfty_est}, by further using the trivial estimate $\pnorm{\Delta \hat{u}^{(t)}}{}\leq \pnorm{A}{\op}\pnorm{\Delta \hat{v}^{(t-1)} }{}$, we have on an event $E_{2,2}$ with probability at least $1-n^{-100}$,
	\begin{align*}
	J_2&\leq \big(K\Lambda\varkappa_{\ast}\log n\big)^{c_t}\cdot \Big[n^{-1/2}\pnorm{\Delta \hat{v}_1^{(t-1)}}{}+ \pnorm{\Delta \hat{v}_{[2:L]}^{(t-1)}}{\infty}\Big].
	\end{align*}
	Now using apriori estimates, we conclude that 
	\begin{align}\label{ineq:err_gd_GFOM_step2_3}
	\E^{(0)} J_2^{\mathfrak{r}}&\leq \big(K\Lambda\varkappa_{\ast}\log n\big)^{c_t}\cdot \Big[\E^{(0)}\big(n^{-1/2}\pnorm{\Delta \hat{v}_1^{(t-1)} }{}\big)^{\mathfrak{r}} +\E^{(0)} \pnorm{\Delta \hat{v}_{[2:L]}^{(t-1)} }{\infty}^{\mathfrak{r}}\Big].
	\end{align}
	Combining (\ref{ineq:err_gd_GFOM_step2_1})-(\ref{ineq:err_gd_GFOM_step2_3}) proves the claimed inequality in (\ref{ineq:err_gd_GFOM_step2}).

    \noindent (\textbf{Step 3}). Finally, combining the estimates in (\ref{ineq:err_gd_GFOM_2}), (\ref{ineq:err_gd_GFOM_3}),  (\ref{ineq:err_gd_GFOM_step1}) and (\ref{ineq:err_gd_GFOM_step2}), by further using the standard subgaussian estimate for $\pnorm{A}{\op}$, the apriori estimates in Proposition \ref{prop:apr_est_det_fcn} and the $\ell_2$ estimate in Proposition \ref{prop:GD_l2_est}, we arrive at the recursion:
	\begin{align*}
	& \E^{(0)}\big(n^{-1/2}\pnorm{\Delta \hat{v}_1^{(t)} }{}\big)^{\mathfrak{r}}  +\E^{(0)} \pnorm{\Delta \hat{v}_{[2:L]}^{(t)}}{\infty}^{\mathfrak{r}}\\
	& \leq \big(K\Lambda\varkappa_{\ast}\log n\big)^{c_t}\cdot \Big[\E^{(0)}\big(n^{-1/2}\pnorm{\Delta \hat{v}_1^{(t-1)} }{}\big)^{\mathfrak{r}}  +\E^{(0)} \pnorm{\Delta \hat{v}_{[2:L]}^{(t-1)}}{\infty}^{\mathfrak{r}}+n^{-1/c_t}\Big].
	\end{align*}
	The claim of the proposition follows by iterating the above estimate with the initial condition $\Delta \hat{v}^{(0)}=0$.\qed

\section{Proofs for Section \ref{section:gd_precise_dynamics}}

\subsection{Proof of Theorem \ref{thm:se_gd_NN}}\label{subsection:proof_se_gd_NN}

Recall the gradient descent iterate $\{(\hat{u}^{(t)},\hat{v}^{(t)})\}$ in (\ref{def:GD_GFOM_NN}); in particular, $\hat{u}^{(t)}=XW_1^{(t-1)}$, $\hat{v}^{(t)}_1=n^{1/2} W_1^{(t)}$ and $\hat{v}^{(t)}_\alpha= W_\alpha^{(t)}$ for $\alpha \in [2:L]$.  Note that an application of Propositions \ref{prop:err_gd_GFOM} and \ref{prop:GD_l2_est} shows that
\begin{align*}
&\E^{(0)} \bigg[\biggabs{\frac{1}{m}\sum_{k \in [m]} \psi_k\big(u_{k\cdot}^{([0:t])}\big) - \frac{1}{m}\sum_{k \in [m]} \psi_k\big(\hat{u}_{k\cdot}^{([0:t])}\big) }^{\mathfrak{r}}\bigg]\\
&\leq (K\Lambda)^{c_t }\cdot \E^{(0),1/2}\big(1+n^{-1/2}\pnorm{u^{([0:t])} }{}+n^{-1/2}\pnorm{\hat{u}^{([0:t])} }{}\big)^{2\mathfrak{r}}\\
&\qquad\qquad  \times \E^{(0),1/2}\big(n^{-1/2}\pnorm{ u^{([0:t])}-\hat{u}^{([0:t])}}{}\big)^{2\mathfrak{r}}\\
&\leq \big(K\Lambda\varkappa_{\ast}\log n\big)^{c_t}\cdot n^{-1/c_t}.
\end{align*}
The claimed state evolution for $\hat{u}^{([0:t])}=\big(X\mu_\ast,\{XW_1^{(s-1)}\}_{s \in [1:t]}\big)$ follows from the above display and Proposition \ref{prop:se_aux_GFOM}, i.e., for a sequence of $\Lambda$-pseudo-Lipschitz functions $\{\psi_k:\R^{q[0:t]} \to \R\}_{k \in [m]}$ of order $2$, we have 
\begin{align}\label{ineq:se_gd_NN_1}
&\E^{(0)} \bigg[\biggabs{\frac{1}{m}\sum_{k \in [m]} \psi_k\big(\hat{u}_{k\cdot}^{([0:t])}\big)- \frac{1}{m}\sum_{k \in [m]}  \E^{(0)}  \psi_k\Big( \big\{\Theta_{k}^{(s)}(\mathfrak{U}^{([0:s])})\big\}_{s \in [0:t]}\Big)  }^{\mathfrak{r}}\bigg]\nonumber\\
&\leq \big(K\Lambda\varkappa_{\ast}\log n\big)^{c_t} \cdot n^{-1/c_t}.
\end{align}
Using almost the same arguments, for a sequence of $\Lambda$-pseudo-Lipschitz functions $\{\phi_\ell:\R^{q[0:t]} \to \R\}_{\ell \in [n]}$ of order $2$, we have
\begin{align}\label{ineq:se_gd_NN_2}
&\E^{(0)}  \bigg[\biggabs{\frac{1}{n}\sum_{\ell \in [n]} \phi_\ell\big(\hat{v}_{\ell\cdot}^{([0:t])}\big) - \frac{1}{n}\sum_{\ell \in [n]}  \E^{(0)}  \phi_\ell\Big(\big\{\Delta_{\ell}^{(s)}(\mathfrak{V}^{([1:s])} )\big\}_{s \in [0:t]}\Big)   }^{\mathfrak{r}}\bigg]  \nonumber\\
&\qquad + \max_{\alpha \in [2:L]} \max_{s \in [1:t]} \E^{(0)} \pnorm{W_\alpha^{(s)}-V^{(s)}_\alpha}{\infty}^{\mathfrak{r}} \leq  \big(K\Lambda\varkappa_{\ast}\log n\big)^{c_{t}}\cdot n^{-1/c_t}.
\end{align}
In order to obtain joint distributional characterizations that include $\{U^{(\cdot)}\}$, note that in view of the definitions of $\mathscr{P}_v^{(1:L]},\mathscr{R}_v$ and $\mathfrak{S}$, we have
\begin{align}\label{ineq:se_gd_NN_3}
U^{(t-1)}&=X W_1^{(t-1)} +\phi^{-1}\sum_{s \in [1:t-1]} \eta^{(s-1)}_1\nonumber\\
&\qquad \times  \Big[\mathscr{P}_{W^{(s-1)}_{[2:L]}}^{(1:L]}\Big(XW_1^{(s-1)},\mathscr{R}_{W_{[2:L]}^{(s-1)}}(XW_1^{(s-1)},X\mu_{\ast,[q]})\Big)\odot \sigma_1'(XW_1^{(s-1)})\Big]\,\rho_{t-1,s}^\top\nonumber\\
&=X W_1^{(t-1)} +\phi^{-1}\sum_{s \in [1:t-1]} \eta^{(s-1)}_1 \mathfrak{S}\big(XW_1^{(s-1)},X\mu_{\ast,[q]},W_{[2:L]}^{(s-1)}\big)\,\rho_{t-1,s}^\top\nonumber\\
&=\hat{u}^{(t)} +\phi^{-1}\sum_{s \in [1:t-1]} \eta^{(s-1)}_1 \mathfrak{S}\big(\hat{u}^{(s)},\hat{u}^{(0)},W_{[2:L]}^{(s-1)}\big)\,\rho_{t-1,s}^\top.
\end{align}
This motivates us to define
\begin{align}\label{ineq:se_gd_NN_4}
\bar{U}^{(t-1)}
& \equiv \hat{u}^{(t)} +\phi^{-1}\sum_{s \in [1:t-1]} \eta^{(s-1)}_1 \mathfrak{S}\big(\hat{u}^{(s)},\hat{u}^{(0)},V_{[2:L]}^{(s-1)}\big)\,\rho_{t-1,s}^\top.
\end{align}
By Proposition \ref{prop:apr_est_det_fcn}-(7), the apriori estimates in Proposition \ref{prop:apr_est_se} and the $\ell_2$ estimates in Proposition \ref{prop:GD_l2_est}, with probability at least $1-n^{-100}$,
\begin{align*}
n^{-1/2} \pnorm{ \bar{U}^{(t-1)}-U^{(t-1)}}{}&\leq \big(K\Lambda\varkappa_{\ast}\log n\big)^{c_t}\cdot \max_{s \in [1:t-1]} \bigpnorm{W_{[2:L]}^{(s-1)}-V_{[2:L]}^{(s-1)}}{\infty}.
\end{align*}
Using (\ref{ineq:se_gd_NN_2}) and the aforementioned apriori/$\ell_2$ estimates, we have 
\begin{align}
\E^{(0)}\big(n^{-1/2} \pnorm{ \bar{U}^{(t-1)}-U^{(t-1)}}{}\big)^{\mathfrak{r}}&\leq \big(K\Lambda\varkappa_{\ast}\log n\big)^{c_t}\cdot n^{-1/c_t}.
\end{align}
To obtain state evolution for $\bar{U}^{(\cdot)}$, for $k \in [m]$, we write $\bar{U}_{k\cdot}^{(t-1)}\equiv \bar{\psi}_k(\hat{u}^{([0:t])})$. It is easy to see that on an event $E_0$ with $\Prob^{(0)}(E_0^c)\leq n^{-100}$, 
\begin{align*}
\bar{U}_{k\cdot}^{(t-1)}=\bar{\psi}_k\big(\mathscr{T}_{\varkappa_{\ast}\log n}(\hat{u}^{(0)}),\hat{u}^{([1:t])}\big)\equiv [\bar{\psi}_k]_n (\hat{u}^{([0:t])}).
\end{align*}
Using Proposition \ref{prop:apr_est_det_fcn}-(7) and the apriori estimates in Proposition \ref{prop:apr_est_se}, $\{[\bar{\psi}_k]_n: \R^{q[0:t]}\to \R\}_{k \in [m]}$ are $\big(K\Lambda\varkappa_{\ast}\log n\big)^{c_t}$-pseudo-Lipschitz functions of order $2$. This means we may now apply (\ref{ineq:se_gd_NN_1}) to obtain 
\begin{align*}
\big(\bar{U}^{(t-1)}\big)
 \stackrel{d}{\approx} \bigg(\Theta^{(t)}(\mathfrak{U}^{([0:t])}) +\phi^{-1}\sum_{s \in [1:t-1]} \eta^{(s-1)}_1\cdot \mathfrak{S}\big(\Theta^{(s)}(\mathfrak{U}^{([0:s])}) ,\mathfrak{U}^{(0)},V^{(s-1)}\big)\rho_{t-1,s}^\top\bigg),
\end{align*}
with an error bound $(K\Lambda\varkappa_{\ast}\log n)^{c_t}\cdot n^{-1/c_t}$ in the sense of (\ref{ineq:se_gd_NN_1}) (note that in the above display we slight abuse the notation $\mathfrak{U}^{(\cdot)}$ as a $m\times q$ matrix with i.i.d. rows of $\mathfrak{U}^{(\cdot)}\in \R^q$ in its strict form). Using the state evolution Definition \ref{def:gd_NN_se}-(S1), the right hand side of the above display is $(\mathfrak{U}^{(t)})$. \qed

\subsection{Proof of Proposition \ref{prop:theoretical_gd_se}}\label{subsection:proof_theoretical_gd_se}
We first define the following simplified state evolution.

\begin{definition}\label{def:gd_NN_se_large_sample}
	Initialize with $\bar{V}^{(0)}\equiv V^{(0)} = W_{[2:L]}^{(0)} \in (\mathbb{M}_q)^{[2:L]}$, $\bar{\mathfrak{U}}^{(0)}\stackrel{d}{\equiv } \mathfrak{U}^{(0)}\in \R^q$, and $\bar{D}_{-1}=D_{-1}=n^{1/2}\mu_{\ast,[q]},\bar{D}_0=D_0 = n^{1/2}W_1^{(0)} \in \R^{n\times q}$. For $t=1,2,\ldots$, we compute recursively as follows:  
	\begin{enumerate}
		\item[(SL1)] Let the Gaussian law of $\bar{\mathfrak{U}}^{(t)} \in \R^q$ be determined via the following correlation specification: for $s \in [0:t]$, 
		\begin{align*}
		\mathrm{Cov}(\bar{\mathfrak{U}}^{(t)}, \bar{\mathfrak{U}}^{(s)})\equiv  n^{-1}\cdot \bar{D}_{t-1}^\top  \bar{D}_{s-1 } \in \mathbb{M}_q.
		\end{align*}
		\item[(SL2)] Compute $\bar{D}_t \in \R^{n\times q}$ by
		\begin{align*}
		\bar{D}_t&\equiv \bar{D}_{t-1} \big(\bar{\tau}_{t}^\top+I_q\big)+\bar{D}_{-1}\bar{\delta}_{t}^\top\in \R^{n\times q},
		\end{align*}
		where 
		\begin{itemize}
			\item $\bar{\tau}_{t}\equiv - \eta_1^{(t-1)} \E^{(0)} \nabla_{\bar{\mathfrak{U}}^{(t)}}  \mathfrak{S}_{\pi_m}(\bar{\mathfrak{U}}^{(t)},\bar{\mathfrak{U}}^{(0)}, \bar{V}^{(t-1)}) \in \mathbb{M}_q$,
			\item $\bar{\delta}_{t}\equiv - \eta_1^{(t-1)} \E^{(0)} \nabla_{\bar{\mathfrak{U}}^{(0)}}  \mathfrak{S}_{\pi_m}(\bar{\mathfrak{U}}^{(t)}, \bar{\mathfrak{U}}^{(0)}, \bar{V}^{(t-1)}) \in \mathbb{M}_q$.
		\end{itemize}
		\item[(SL3)] Compute $\bar{V}^{(t)}  = (\bar{V}_{\alpha}^{(t)})_{\alpha \in [2:L]}\in (\mathbb{M}_q)^{[2:L]}$ as follows: for $\alpha \in [2:L]$,
		\begin{align*}
		\bar{V}_{\alpha}^{(t)}&\equiv \bar{V}_\alpha^{(t-1)}-\eta_\alpha^{(t-1)} \cdot \E^{(0)}\mathscr{G}_{\bar{V}^{(t-1)}_{[2:\alpha-1]};\pi_m }(\bar{\mathfrak{U}}^{(t)} )  \big(\mathfrak{T}_{\alpha}\big)_{\pi_m}\big(\bar{\mathfrak{U}}^{(t)}, \bar{\mathfrak{U}}^{(0)}, \bar{V}^{(t-1)}\big)^\top \in \mathbb{M}_q.
		\end{align*}
	\end{enumerate}
\end{definition}

The above state evolution provides a mid-ground for that in Proposition \ref{prop:theoretical_gd_se} and the original state evolution in Definition \ref{def:gd_NN_se}. In fact, Proposition \ref{prop:theoretical_gd_se} follows by the following.

\begin{proposition}\label{prop:population_gd_V}
	Fix $t\in \N$. Suppose $\phi^{-1}\leq K$ and (A2)-(A5) in Assumption \ref{assump:gd_NN} hold for some $K,\Lambda\geq 2$ and $r_0\geq 0$. Then $n^{1/2}\bar{W}_1^{(t)}=\bar{D}_t$ and $\bar{W}_\alpha^{(t)}=\bar{V}_\alpha^{(t)}$ for $\alpha \in [2:L]$.
\end{proposition}
\begin{proof}
	Using Proposition \ref{prop:grad_formula}, with $X\sim \mathsf{Z}_{m\times n}$, for $\alpha \in [1:L]$ and any $\bm{W}$,
	\begin{align*}
	\nabla_{W_\alpha} \overline{\mathsf{L}}(\bm{W})& =  \frac{1}{m}\cdot \E^{(0)}  G_{\bm{W};\alpha-1}(X)^\top \big[\mathsf{P}_{X,\bm{W}}^{(\alpha:L]}\big(G_{\bm{W};L}(X)-Y_{[q]}\big)\odot G_{\bm{W};\alpha}'(X)\big].
	\end{align*}
	Consequently, for $\alpha=1$ and any $t\in \N$, we have
	\begin{align*}
	\nabla_{\bar{W}_1^{(t)}} \overline{\mathsf{L}}(\bar{\bm{W}}^{(t)})& =  \frac{1}{m}\E^{(0)}  X^\top \big[\mathsf{P}_{X,\bar{\bm{W}}^{(t)}}^{(1:L]}\big(G_{\bar{\bm{W}}^{(t)};L}(X)-Y_{[q]}\big)\odot \sigma_1'(X)\big]\\
	& = \frac{1}{m} \E^{(0)} X^\top \mathfrak{S}\big(X \bar{W}_1^{(t)}, X\mu_{\ast,[q]}, \bar{W}_{[2:L]}^{(t)} \big)\\
	& = \E^{(0)} \mathsf{Z}_n \mathfrak{S}_{\pi_m}\big(\bar{W}_1^{(t),\top}\mathsf{Z}_n, \mu_{\ast,[q]}^\top \mathsf{Z}_n, \bar{W}_{[2:L]}^{(t)} \big)^\top.
	\end{align*}
	Now using Gaussian integration-by-parts for the right hand side of the above display, we have for $i \in [n], d \in [q]$, 
	\begin{align*}
	&\Big[\nabla_{\bar{W}_1^{(t)}} \overline{\mathsf{L}}(\bar{\bm{W}}^{(t)})\Big]_{id} =  \E^{(0)} \mathsf{Z}_{n;i} \mathfrak{S}_{(\pi_m,d)}\big(\bar{W}_1^{(t),\top}\mathsf{Z}_n, \mu_{\ast,[q]}^\top \mathsf{Z}_n, \bar{W}_{[2:L]}^{(t)} \big)\\
	& = \sum_{c \in [q]} \E^{(0)} \partial_{u_{\pi_m, c}}\mathfrak{S}_{(\pi_m,d)}\big(\bar{W}_1^{(t),\top}\mathsf{Z}_n, \mu_{\ast,[q]}^\top \mathsf{Z}_n, \bar{W}_{[2:L]}^{(t)} \big)\cdot e_c^\top \bar{W}_1^{(t),\top}e_i\\
	&\qquad  + \sum_{c \in [q]} \E^{(0)} \partial_{w_{\pi_m, c}}\mathfrak{S}_{(\pi_m,d)}\big(\bar{W}_1^{(t),\top}\mathsf{Z}_n, \mu_{\ast,[q]}^\top \mathsf{Z}_n, \bar{W}_{[2:L]}^{(t)} \big)\cdot e_c^\top \mu_{\ast,[q]}^{\top}e_i\\
	& = \Big[ \bar{W}_1^{(t)} \E^{(0)} \nabla_u \mathfrak{S}_{\pi_m}\big(\bar{W}_1^{(t),\top}\mathsf{Z}_n, \mu_{\ast,[q]}^\top \mathsf{Z}_n, \bar{W}_{[2:L]}^{(t)} \big)^\top\\
	&\qquad + \mu_{\ast,[q]} \E^{(0)} \nabla_w \mathfrak{S}_{\pi_m}\big(\bar{W}_1^{(t),\top}\mathsf{Z}_n, \mu_{\ast,[q]}^\top \mathsf{Z}_n, \bar{W}_{[2:L]}^{(t)} \big)^\top\Big]_{id}.
	\end{align*}
	Using the above display in (\ref{def:grad_descent_population}) and comparing with the definition of $\bar{D}_t$ concludes that $n^{1/2}\bar{W}_1^{(t)}=\bar{D}_t$.
	
	For $\alpha \in [2:L]$ and any $t \in \N$, we have 
	\begin{align*}
	\nabla_{\bar{W}_\alpha^{(t)}} \overline{\mathsf{L}}(\bar{\bm{W}}^{(t)})& =  \E^{(0)}\mathscr{G}_{\bar{W}^{(t)}_{[2:\alpha-1]};\pi_m }\big(\bar{W}_1^{(t),\top}\mathsf{Z}_n\big)  \big(\mathfrak{T}_{\alpha}\big)_{\pi_m}\big(\bar{W}_1^{(t),\top}\mathsf{Z}_n, \mu_{\ast,[q]}^\top \mathsf{Z}_n, \bar{W}_{[2:L]}^{(t)}\big)^\top.
	\end{align*}
	Using the above display in (\ref{def:grad_descent_population}) and comparing with the definition of $\bar{V}_\alpha^{(t)}$ concludes that $\bar{W}_\alpha^{(t)}=\bar{V}_\alpha^{(t)}$.
\end{proof}

\subsection{Proof of Theorem \ref{thm:se_gd_NN_large_sample}}\label{subsection:proof_se_gd_NN_large_sample}

	We shall divide the proof into several steps.
	
	\noindent (\textbf{Step 1}). In this step, we derive some apriori estimates for the state evolution parameters in Definition \ref{def:gd_NN_se_large_sample}. In particular, we will show that for $t\geq 1$, there exists some $c_t\equiv c_t(t,q,L,r_0)>1$ such that
	\begin{align}\label{ineq:se_gd_NN_large_sample_step1}
	\bar{\mathcal{B}}^{(t)}&\equiv 1+\max_{s \in [1:t]} \Big\{\pnorm{\bar{\tau}_{s}}{\op}+\pnorm{\bar{\delta}_s}{\op}+n^{-1/2}\pnorm{\bar{D}_s}{}+\max_{\alpha \in [2:L]} \pnorm{\bar{V}^{(s)}_\alpha}{\op}\Big\} \nonumber\\
	&\qquad + \max_{0\leq s\leq r\leq t}\pnorm{\cov(\bar{\mathfrak{U}}^{(r)},\bar{\mathfrak{U}}^{(s)})}{\op} \leq  \big(K\Lambda\varkappa_{\ast}\big)^{c_t}.
	\end{align}
	First, by the definition in (SL1), 
	\begin{align}\label{ineq:se_gd_NN_large_sample_step1_1}
	\max_{0\leq s\leq r\leq t}\pnorm{\cov(\bar{\mathfrak{U}}^{(r)},\bar{\mathfrak{U}}^{(s)})}{\op}\leq \big(K\Lambda\varkappa_{\ast}\cdot \bar{\mathcal{B}}^{(t-1)}\big)^{c_t}.
	\end{align}
	Next, using the definitions of $\bar{\tau}_s,\bar{\delta}_s$, the derivative formulae in Proposition \ref{prop:derivative_formula}, and the apriori estimates in Proposition \ref{prop:apr_est_det_fcn} and (\ref{ineq:se_gd_NN_large_sample_step1_1}),
	\begin{align}\label{ineq:se_gd_NN_large_sample_step1_2}
	\max_{s \in [1:t]} \big\{\pnorm{\bar{\tau}_{s}}{\op}+\pnorm{\bar{\delta}_s}{\op}\big\}\leq \big(K\Lambda\varkappa_{\ast}\cdot \bar{\mathcal{B}}^{(t-1)}\big)^{c_t}.
	\end{align}
	Using the above display and the definition of $\bar{D}_s$, we then have
	\begin{align}\label{ineq:se_gd_NN_large_sample_step1_3}
	\max_{s \in [1:t]} n^{-1/2}\pnorm{\bar{D}_s}{}\leq  \big(K\Lambda\varkappa_{\ast}\cdot \bar{\mathcal{B}}^{(t-1)}\big)^{c_t}.
	\end{align}
	Finally using (SL3) and the apriori estimates in Proposition \ref{prop:apr_est_det_fcn} and (\ref{ineq:se_gd_NN_large_sample_step1_1}), we have 
	\begin{align}\label{ineq:se_gd_NN_large_sample_step1_4}
	\max_{s \in [1:t]}\max_{\alpha \in [2:L]} \pnorm{\bar{V}^{(s)}_\alpha}{\op}\leq \big(K\Lambda\varkappa_{\ast}\cdot \bar{\mathcal{B}}^{(t-1)}\big)^{c_t}.
	\end{align}
	The desired estimate (\ref{ineq:se_gd_NN_large_sample_step1}) follows by combining the estimates (\ref{ineq:se_gd_NN_large_sample_step1_1})-(\ref{ineq:se_gd_NN_large_sample_step1_4}) along with the initial condition $\bar{\mathcal{B}}^{(1)}\leq \big(K\Lambda\varkappa_{\ast}\big)^{c_t}$.
	
	\noindent (\textbf{Step 2}). In this step, we prove that for some $c_t\equiv c_t(t,q,L,r_0)>1$,
	\begin{align}\label{ineq:se_gd_NN_large_sample_step2}
	\Delta \bar{\mathcal{B}}^{(t)}&\equiv \max_{s \in [1:t]} \Big\{\pnorm{\bar{\tau}_s-\tau_{s,s}}{\op}+\pnorm{\bar{\delta}_s-\delta_s}{\op}+ n^{-1/2}\pnorm{\bar{D}_s-D_s}{} +\max_{\alpha \in [2:L]}  \pnorm{\bar{V}_\alpha^{(s)}-V_\alpha^{(s)} }{\op} \Big\} \nonumber\\
	&\quad  + \max_{0\leq s \leq r\leq t} \pnorm{\mathrm{Cov}(\bar{\mathfrak{U}}^{(r)}, \bar{\mathfrak{U}}^{(s)})-\mathrm{Cov}(\mathfrak{U}^{(r)}, \mathfrak{U}^{(s)}) }{\op}\leq \big(K\Lambda\varkappa_{\ast}\big)^{c_t}\cdot \phi^{-1}.
	\end{align}
	First, by the definitions for $\mathrm{Cov}(\bar{\mathfrak{U}}^{(r)}, \bar{\mathfrak{U}}^{(s)})$ and $\mathrm{Cov}({\mathfrak{U}}^{(r)}, {\mathfrak{U}}^{(s)})$, we have
	\begin{align*}
	&\pnorm{\mathrm{Cov}(\bar{\mathfrak{U}}^{(r)}, \bar{\mathfrak{U}}^{(s)})-\mathrm{Cov}(\mathfrak{U}^{(r)}, \mathfrak{U}^{(s)}) }{\op}\\
	&\leq n^{-1}\cdot \pnorm{\bar{D}_{r-1}^\top  \bar{D}_{s-1 }- {D}_{r-1}^\top  {D}_{s-1 }}{\op}+ \pnorm{\Omega_{r-1,s-1}}{\op}.
	\end{align*}
	Using the apriori estimates in (\ref{ineq:se_gd_NN_large_sample_step1}) and Proposition \ref{prop:apr_est_se} for the first term above, and the second estimate in Proposition \ref{prop:apr_est_se} for the second term above, we have 
	\begin{align}\label{ineq:se_gd_NN_large_sample_step2_1}
	&\max_{0\leq s \leq r\leq t}\pnorm{\mathrm{Cov}(\bar{\mathfrak{U}}^{(r)}, \bar{\mathfrak{U}}^{(s)})-\mathrm{Cov}(\mathfrak{U}^{(r)}, \mathfrak{U}^{(s)}) }{\op}\leq  \big(K\Lambda\varkappa_{\ast}\big)^{c_t}\cdot \big(\Delta \bar{\mathcal{B}}^{(t-1)}+\phi^{-1}\big).
	\end{align}
	Next, using the definitions of $\tau_{s,s}$ and $\bar{\tau}_s$, the stability estimate in Proposition \ref{prop:apr_est_derivative}-(8), and the apriori estimates in (\ref{ineq:se_gd_NN_large_sample_step1}) and Proposition \ref{prop:apr_est_se}, we have
	\begin{align*}
	&\pnorm{\bar{\tau}_s-\tau_{s,s}}{\op}\\
	&\leq \big(K\Lambda\varkappa_{\ast}\big)^{c_t}\cdot \Big(\max_{k \in [m]}\E^{(0),1/2} \pnorm{\Theta_{k}^{(s)}(\mathfrak{U}^{([0:s])})-\bar{\mathfrak{U}}^{(s)} }{\infty}^2+ \pnorm{\bar{V}^{(s-1)}-V^{(s-1)}}{\infty}  \Big).
	\end{align*}
	Using the second estimate in Proposition \ref{prop:apr_est_se}, we then conclude that 
	\begin{align}\label{ineq:se_gd_NN_large_sample_step2_2}
	\max_{s \in [1:t]} \pnorm{\bar{\tau}_s-\tau_{s,s}}{\op}\leq \big(K\Lambda\varkappa_{\ast}\big)^{c_t}\cdot \big(\Delta \bar{\mathcal{B}}^{(t-1)}+\phi^{-1}\big).
	\end{align}
	A completely similar argument leads to 
	\begin{align}\label{ineq:se_gd_NN_large_sample_step2_3}
	\max_{s \in [1:t]} \pnorm{\bar{\delta}_s-\delta_{s}}{\op}\leq \big(K\Lambda\varkappa_{\ast}\big)^{c_t}\cdot \big(\Delta \bar{\mathcal{B}}^{(t-1)}+\phi^{-1}\big).
	\end{align}
	From the definitions of $D_s$ and $\bar{D}_s$, we have
	\begin{align*}
	n^{-1/2}\pnorm{\bar{D}_s-D_s}{}& \leq n^{-1/2}\pnorm{D_{-1}\delta_s^\top- \bar{D}_{-1}\bar{\delta}_s^\top}{}+t\cdot \max_{r \in [1:s-1]} n^{-1/2}\pnorm{D_{r-1}}{}\cdot \pnorm{\tau_{s,r}}{\op}\\
	&\qquad + n^{-1/2}\bigpnorm{\bar{D}_{s-1} \big(\bar{\tau}_{s}^\top+I_q\big)- {D}_{s-1} \big({\tau}_{s,s}^\top+I_q\big)}{}.
	\end{align*}
	Using (i) $\bar{D}_{-1}=D_{-1}$ and (\ref{ineq:se_gd_NN_large_sample_step2_3}) for the first term above, (ii) both apriori estimates in Proposition \ref{prop:apr_est_se} for the second term above, and (iii) the apriori estimates in (\ref{ineq:se_gd_NN_large_sample_step1}) and Proposition \ref{prop:apr_est_se}, and the estimate in (\ref{ineq:se_gd_NN_large_sample_step2_2}) for the third term, we have
	\begin{align}\label{ineq:se_gd_NN_large_sample_step2_4}
	\max_{s \in [1:t]} n^{-1/2}\pnorm{\bar{D}_s-D_s}{}& \leq \big(K\Lambda\varkappa_{\ast}\big)^{c_t}\cdot \big(\Delta \bar{\mathcal{B}}^{(t-1)}+\phi^{-1}\big).
	\end{align}
	Finally, using the apriori and stability estimates in Proposition \ref{prop:apr_est_det_fcn}, and the apriori estimates in (\ref{ineq:se_gd_NN_large_sample_step1}) and Proposition \ref{prop:apr_est_se}, we may conclude that 
	\begin{align}\label{ineq:se_gd_NN_large_sample_step2_5}
	\max_{s \in [1:t]} \max_{\alpha \in [2:L]}  \pnorm{\bar{V}_\alpha^{(s)}-V_\alpha^{(s)} }{\op}\leq \big(K\Lambda\varkappa_{\ast}\big)^{c_t}\cdot \big(\Delta \bar{\mathcal{B}}^{(t-1)}+\phi^{-1}\big).
	\end{align}
	Combining (\ref{ineq:se_gd_NN_large_sample_step2_1})-(\ref{ineq:se_gd_NN_large_sample_step2_5}) yields that 
	\begin{align}\label{ineq:se_gd_NN_large_sample_step2_6}
	\Delta \bar{\mathcal{B}}^{(t)} \leq \big(K\Lambda\varkappa_{\ast}\big)^{c_t}\cdot \big(\Delta \bar{\mathcal{B}}^{(t-1)}+\phi^{-1}\big).
	\end{align}
	On the other hands, the same estimates (\ref{ineq:se_gd_NN_large_sample_step2_1})-(\ref{ineq:se_gd_NN_large_sample_step2_5}) yield that $\Delta \bar{\mathcal{B}}^{(1)}\leq \big(K\Lambda\varkappa_{\ast}\big)^{c_t}\cdot \phi^{-1}$, so iterating the above estimate (\ref{ineq:se_gd_NN_large_sample_step2_6}) yields the desired claim (\ref{ineq:se_gd_NN_large_sample_step2}).
	
	\noindent (\textbf{Step 3}). In this step, we prove that there exists some $c_t\equiv c_t(t,q,L,r_0)>1$ such that for fixed sequences of $\Lambda$-pseudo-Lipschitz functions $\{\psi_k:\R^{1+q[1:t]} \to \R\}_{k \in [m]}$ and $\{\phi_\ell:\R^{q[0:t]} \to \R\}_{\ell \in [n]}$ of order $2$, 
	\begin{align}\label{ineq:se_gd_NN_large_sample_step3}
	&\max_{k \in [m]}\bigabs{\E^{(0)}  \psi_k\big(\mathfrak{U}^{(0)}_1, \big\{\Theta_{k}^{(s)}(\mathfrak{U}^{([0:s])})\big\}_{s \in [1:t]}\big)-\E^{(0)}  \psi_k\big(\bar{\mathfrak{U}}^{(0)}_1, \{\bar{\mathfrak{U}}^{(s)}\}_{s \in [1:t]}\big) }\nonumber\\
	&\qquad + \frac{1}{n}\sum_{\ell \in [n]} \bigabs{\E^{(0)}  \phi_\ell\big(\big\{\Delta_{\ell}^{(s)}(\mathfrak{V}^{([1:s])} )\big\}_{s \in [0:t]}\big) -\phi_\ell\big(\big\{\bar{D}_{s;\ell}\big\}_{s \in [0:t]}\big)  }\nonumber\\
	&\leq \big(K\Lambda\varkappa_{\ast}\big)^{c_t}\cdot \phi^{-1}.
	\end{align}
	For the first term in (\ref{ineq:se_gd_NN_large_sample_step3}), as $\psi_k$ is $\Lambda$-pseudo-Lipschitz of order $2$, using the apriori estimates in (\ref{ineq:se_gd_NN_large_sample_step1}) and Proposition \ref{prop:apr_est_det_fcn}, along with the second estimate in Proposition \ref{prop:apr_est_se}, 
	\begin{align*}
	&\hbox{1st term in (\ref{ineq:se_gd_NN_large_sample_step3})} \leq \big(K\Lambda\varkappa_{\ast}\big)^{c_t}\cdot \phi^{-1}\\
	&\quad + \max_{k \in [m]}\bigabs{\E^{(0)}  \psi_k\big(\mathfrak{U}^{(0)}_1, \{{\mathfrak{U}}^{(s)}\}_{s \in [1:t]}\big)-\E^{(0)}  \psi_k\big(\bar{\mathfrak{U}}^{(0)}_1, \{\bar{\mathfrak{U}}^{(s)}\}_{s \in [1:t]}\big) }.
	\end{align*}
	Now by representing $\big(\mathfrak{U}^{(0)}_1, \{{\mathfrak{U}}^{(s)}\}_{s \in [1:t]}\big)$ and $\big(\bar{\mathfrak{U}}^{(0)}_1, \{\bar{\mathfrak{U}}^{(s)}\}_{s \in [1:t]}\big)$ in the canonical Gaussian form, and using the stability estimates for the covariance in (\ref{ineq:se_gd_NN_large_sample_step2}) in conjunction with the square-root estimate for covariance matrices (cf. \cite[Lemma A.3]{bao2023leave}), we arrive at 
	\begin{align}\label{ineq:se_gd_NN_large_sample_step3_1}
	\hbox{1st term in (\ref{ineq:se_gd_NN_large_sample_step3})} \leq \big(K\Lambda\varkappa_{\ast}\big)^{c_t}\cdot \phi^{-1}.
	\end{align}
	For the second term in (\ref{ineq:se_gd_NN_large_sample_step3}), in view of the definition of $\Delta^{(\cdot)}$ in (\ref{def:Delta}), again using the pseudo-Lipschitz property of $\phi_\ell$ and the apriori estimates (as indicated above), 
	\begin{align*}
	\hbox{2nd term in (\ref{ineq:se_gd_NN_large_sample_step3})} &\leq \big(K\Lambda\varkappa_{\ast}\big)^{c_t}\cdot \max_{s \in [1:t]}  \bigg(\frac{1}{n}\sum_{\ell \in [n]} \pnorm{D_{s;\ell}-\bar{D}_{s;\ell}}{}^2+\E^{(0)}\pnorm{\mathfrak{V}^{(s)}}{}^2\bigg)^{1/2}.
	\end{align*}
	Here recall the Gaussian laws of $(\mathfrak{V}^{(s)})$ defined above (\ref{def:Delta}). The first term above can be handled by using (\ref{ineq:se_gd_NN_large_sample_step2}), while the second term above can be handled by noting that $\cov(\mathfrak{V}^{(r)},\mathfrak{V}^{(s)})=\Sigma_{r,s}$ and using the second estimate in Proposition \ref{prop:apr_est_se}. This leads to 
	\begin{align}\label{ineq:se_gd_NN_large_sample_step3_2}
	\hbox{2nd term in (\ref{ineq:se_gd_NN_large_sample_step3})} \leq \big(K\Lambda\varkappa_{\ast}\big)^{c_t}\cdot \phi^{-1}.
	\end{align}
	The claim (\ref{ineq:se_gd_NN_large_sample_step3}) now follows from (\ref{ineq:se_gd_NN_large_sample_step3_1}) and (\ref{ineq:se_gd_NN_large_sample_step3_2}).
	
	\noindent (\textbf{Step 4}). In this step, we shall prove that for fixed sequences of $\Lambda$-pseudo-Lipschitz functions $\{\psi_k:\R^{1+q[1:t]} \to \R\}_{k \in [m]}$ and $\{\phi_\ell:\R^{q[0:t]} \to \R\}_{\ell \in [n]}$ of order $2$, 
	\begin{align}\label{ineq:se_gd_NN_large_sample_step4}
	&\E^{(0)} \bigg[\biggabs{\frac{1}{m}\sum_{k \in [m]} \psi_k\Big((X\mu_{\ast})_{k\cdot},\big\{(XW_1^{(s-1)})_{k\cdot}\big\}_{s \in [1:t]}\Big) - \frac{1}{m}\sum_{k \in [m]}  \E^{(0)}  \psi_k\big(\bar{\mathfrak{U}}^{(0)}_1, \{\bar{\mathfrak{U}}^{(s)}\}_{s \in [1:t]}\big)  }^{\mathfrak{r}}\bigg]\nonumber\\
	&\qquad + \E^{(0)}  \bigg[\biggabs{\frac{1}{n}\sum_{\ell \in [n]} \phi_\ell\big(\big\{n^{1/2}W_{1;\ell\cdot}^{(s)}\big\}_{s \in [0:t]}\big) - \frac{1}{n}\sum_{\ell \in [n]}  \phi_\ell\big(\big\{\bar{D}_{s;\ell}\big\}_{s \in [0:t]}\big)   }^{\mathfrak{r}}\bigg]  \nonumber\\
	&\qquad +\max_{\alpha \in [2:L]} \max_{s \in [1:t]}  \E^{(0)} \pnorm{W_\alpha^{(s)}-\bar{V}^{(s)}_\alpha}{\infty}^{\mathfrak{r}}\nonumber\\
	& \leq  \big(K\Lambda\varkappa_{\ast}\big)^{c_t}\cdot \big[(1\vee\phi)^{c_t}\cdot n^{-1/c_t}+\phi^{-1/c_t}\big].
	\end{align}
	To see this, we shall first apply Theorem \ref{thm:se_gd_NN} with $K$ therein replaced by $K\vee (1\vee \phi)$ which leads to the error estimate $\big(K\Lambda\varkappa_{\ast}\big)^{c_t}\cdot (1\vee\phi)^{c_t}\cdot n^{-1/c_t}$. We then apply the estimate  (\ref{ineq:se_gd_NN_large_sample_step3}) in Step 3 to simplify the state evolution with an error estimate $\big(K\Lambda\varkappa_{\ast}\big)^{c_t}\cdot \phi^{-1}$. This proves the estimate in (\ref{ineq:se_gd_NN_large_sample_step4}).
	
	\noindent (\textbf{Step 5}). Finally, by using the estimate (\ref{ineq:se_gd_NN_large_sample_step4}) in Step 4 and Proposition \ref{prop:population_gd_V}, we obtain the claimed estimates. \qed

\section{Proofs for Section \ref{section:gd_inference}}\label{section:proof_gd_inference}

\subsection{Proof of Theorem \ref{thm:err_gd_NN}}\label{subsection:proof_err_gd_NN}

We first prove the claim for training error.

\begin{proof}[Proof of Theorem \ref{thm:err_gd_NN}: Part 1]
	
	For the training error, note that we may rewrite 
	\begin{align*}
	\mathscr{E}_{\texttt{train}}^{(t)}(X,Y) &\equiv \frac{1}{m} \bigpnorm{Y_{[q]}- \mathscr{G}_{W_{[2:L]}^{(t)}}( X W^{(t)}_1 ) }{}^2.
	\end{align*}
	This motivates us to consider instead
	\begin{align*}
	\bar{\mathscr{E}}_{\texttt{train}}^{(t)}(X,Y) &\equiv \frac{1}{m} \bigpnorm{Y_{[q]}- \mathscr{G}_{V_{[2:L]}^{(t)}}( X W^{(t)}_1 ) }{}^2.
	\end{align*}
	Then using the apriori estimates in Propositions \ref{prop:apr_est_det_fcn} and \ref{prop:apr_est_se}, and the $\ell_2$ estimate in Proposition \ref{prop:GD_l2_est},
	\begin{align*}
	\E^{(0)}\bigabs{\mathscr{E}_{\texttt{train}}^{(t)}(X,Y) - \bar{\mathscr{E}}_{\texttt{train}}^{(t)}(X,Y) }^{\mathfrak{r}}\leq \big(K\Lambda\varkappa_{\ast}\log n\big)^{c_t}\cdot \E^{(0)}\pnorm{W_{[2:L]}^{(s-1)}-V_{[2:L]}^{(s-1)}}{\infty}^{\mathfrak{r}}.
	\end{align*}
	Using Theorem \ref{thm:se_gd_NN} then yields 
	\begin{align}\label{ineq:err_gd_NN_train_1}
	\E^{(0)}\bigabs{\mathscr{E}_{\texttt{train}}^{(t)}(X,Y) - \bar{\mathscr{E}}_{\texttt{train}}^{(t)}(X,Y) }^{\mathfrak{r}}\leq \big(K\Lambda\varkappa_{\ast}\log n\big)^{c_t}\cdot n^{-1/c_t}.
	\end{align}
	On the other hand, for any $k \in [m]$, let
	\begin{align*}
	\bar{\psi}_k(u^{(0)},u^{(t+1)})\equiv \Big(\varphi_\ast\big(\mathscr{T}_{\varkappa_{\ast}\log n}(u^{(0)})\big)+\xi_{k;[q]}-\mathscr{G}_{V_{[2:L]}^{(t)};k}(u^{(t+1)})\Big)^2.
	\end{align*}
	Using the apriori estimates in Propositions \ref{prop:apr_est_det_fcn} and \ref{prop:apr_est_se}, $\{\bar{\psi}_k\}$'s are $\big(K\Lambda\varkappa_{\ast}\log n\big)^{c_t}$-pseudo-Lipschitz of order $2$, and on an event $E_0$ with $\Prob^{(0)}(E_0^c)\leq n^{-100}$, 
	\begin{align*}
	\frac{1}{m}\sum_{k \in [m]} \bar{\psi}_k(\hat{u}^{(0)},\hat{u}^{(t+1)}) &= \bar{\mathscr{E}}_{\texttt{train}}^{(t)}(X,Y),\\
	\frac{1}{m}\sum_{k \in [m]} \E^{(0)} \bar{\psi}_k\big(\mathfrak{U}^{(0)},\Theta_{k}^{(t+1)}(\mathfrak{U}^{(t+1)})\big) &= \E^{(0)} \bigpnorm{\mathscr{R}_{V^{(t)};\pi_m}\big(\Theta_{\pi_m}^{(t+1)}(\mathfrak{U}^{([0:t+1])}),\mathfrak{U}^{(0)} \big)}{}^2.
	\end{align*}
	Therefore we may apply Theorem \ref{thm:se_gd_NN} (in conjunction with a routine argument to remove the effect of $E_0$ via apriori estimates) to obtain
	\begin{align}\label{ineq:err_gd_NN_train_2}
	&\E^{(0)}\bigabs{ \bar{\mathscr{E}}_{\texttt{train}}^{(t)}(X,Y)-\E^{(0)} \bigpnorm{\mathscr{R}_{V^{(t)};\pi_m}\big(\Theta_{\pi_m}^{(t+1)}(\mathfrak{U}^{([0:t+1])}),\mathfrak{U}^{(0)} \big)}{}^2   }^{\mathfrak{r}}\nonumber\\
	& \leq \E^{(0)}\biggabs{ \frac{1}{m}\sum_{k \in [m]} \bar{\psi}_k(\hat{u}^{(0)},\hat{u}^{(t+1)}) - \frac{1}{m}\sum_{k \in [m]} \E^{(0)} \bar{\psi}_k\big(\mathfrak{U}^{(0)},\Theta_{k}^{(t+1)}(\mathfrak{U}^{(t+1)})\big) }^{\mathfrak{r}}\nonumber\\
	&\qquad +\big(K\Lambda\varkappa_{\ast}\log n\big)^{c_t}\cdot n^{-50} \leq  \big(K\Lambda\varkappa_{\ast}\log n\big)^{c_t}\cdot n^{-1/c_t}.
	\end{align}
	The claimed formula for the training error follows from (\ref{ineq:err_gd_NN_train_1}) and (\ref{ineq:err_gd_NN_train_2}).
\end{proof}

Next we prove the claim for the generalization/test error. We need the following derivative control for the neural network function $f_{\bm{W}}$.

\begin{lemma}\label{lem:f_derivative_bound}
	Suppose $\max_{p \in [3]} \max_{\alpha \in [1:L]} \pnorm{\sigma_\alpha^{(p)}}{\infty}\leq \Lambda$ for some $\Lambda\geq 2$. Then there exists some $c_L=c_L(L)>0$ such that
	\begin{align*}
	\max_{i \in [n]}\pnorm{\partial_i^p f_{\bm{W}} }{\infty}\leq \big(q\Lambda\pnorm{W_{[2:L]}}{\infty}\big)^{c_L}\cdot \pnorm{W_1}{\infty}^p,\quad \hbox{for all }p \in [3].
	\end{align*}
\end{lemma}
\begin{proof}
	As $\sigma_L\equiv \mathrm{id}$, we have $\partial_i^{(\cdot)} f_{\bm{W}}(x) = \partial_i^{(\cdot)} h_{\bm{W};L}(x)$. In the proof below we shall write $\partial_i^{(\cdot)} h_{\bm{W};\ell}(x)\equiv \partial_i^{(\cdot)} h_{\ell}$ and $\sigma_\ell^{(\cdot)}(h_{\bm{W};\ell}(x))\equiv \sigma_\ell^{(\cdot)}$ for notational simplicity. 
	
	For the first derivative, we have the following the recursion:
	\begin{align*}
	\partial_i h_{\ell} = W_\ell^\top \big(\sigma_{\ell-1}'\circ \partial_i h_{\ell-1}\big)\in \R^q,\, \partial_i h_1 = (W_1)_{i\cdot}^\top\in \R^q.
	\end{align*}
	This means 
	\begin{align}\label{ineq:f_derivative_bound_1}
	\max_{i \in [n]}\max_{\ell \in [1:L]}\pnorm{\partial_i h_{\ell} }{\infty}\leq \big(q\Lambda\pnorm{W_{[2:L]}}{\infty}\big)^{c_L}\cdot \pnorm{W_1}{\infty}.
	\end{align}
	For the second derivative, we have 
	\begin{align*}
	\partial_i^{2} h_\ell = W_\ell^\top \big(\sigma_{\ell-1}''\circ (\partial_i h_{\ell-1})^2+\sigma_{\ell-1}'\circ \partial_i^2 h_{\ell-1}\big) \in \R^q,\, \partial_i^2 h_1 = 0_q.
	\end{align*}
	Iterating the above relation and using (\ref{ineq:f_derivative_bound_1}), we have
	\begin{align}\label{ineq:f_derivative_bound_2}
	\max_{i \in [n]}\max_{\ell \in [1:L]}\pnorm{\partial_i^2 h_{\ell} }{\infty}\leq \big(q\Lambda\pnorm{W_{[2:L]}}{\infty}\big)^{c_L}\cdot \pnorm{W_1}{\infty}^2.
	\end{align}
	For the third derivative, we have 
	\begin{align*}
	\partial_i^{3} h_\ell = W_\ell^\top \big(\sigma_{\ell-1}'''\circ (\partial_i h_{\ell-1})^3+ 3\sigma_{\ell-1}''\circ \partial_i h_{\ell-1}\circ \partial_i^2 h_{\ell-1} +\sigma_{\ell-1}'\circ \partial_i^3 h_{\ell-1}   \big) \in \R^q,
	\end{align*}
	with $\partial_i^3 h_1 = 0_q$. Iterating the above relation and using both (\ref{ineq:f_derivative_bound_1}) and (\ref{ineq:f_derivative_bound_2}), we are led to 
	\begin{align}\label{ineq:f_derivative_bound_3}
	\max_{i \in [n]}\max_{\ell \in [1:L]}\pnorm{\partial_i^3 h_{\ell} }{\infty}\leq \big(q\Lambda\pnorm{W_{[2:L]}}{\infty}\big)^{c_L}\cdot \pnorm{W_1}{\infty}^3.
	\end{align}
	The proof is complete by combining (\ref{ineq:f_derivative_bound_1})-(\ref{ineq:f_derivative_bound_3}).
\end{proof}

We also need the following Lindeberg's principle due to \cite{chatterjee2006generalization}.

\begin{lemma}\label{lem:lindeberg}
	Let $X=(X_1,\ldots,X_n)$ and $Y=(Y_1,\ldots,Y_n)$ be two random vectors in $\R^n$ with independent components such that $\E X_i^\ell=\E Y_i^\ell$ for $i \in [n]$ and $\ell=1,2$. Then for any $f \in C^3(\R^n)$,
	\begin{align*}
	\bigabs{\E f(X) - \E f(Y)}&\leq \max_{U_i \in \{X_i,Y_i\}}\biggabs{\sum_{i=1}^n\E U_i^3 \int_0^{1} \partial_i^3 f(X_{[1:(i-1)]},tU_i, Y_{[(i+1):n]} )(1-t)^2\,\d{t}}.
	\end{align*}
\end{lemma}

\begin{proof}[Proof of Theorem \ref{thm:err_gd_NN}: Part 2]
	
	The proof is divided into several steps.
	
	\noindent (\textbf{Step 1}). Recall the notation in (\ref{def:GD_GFOM_NN}) and (\ref{def:GFOM_auxiliary_NN_1}); in particular, $W_1^{(t)}=n^{-1/2}\hat{v}_1^{(t)}$ as defined in (\ref{def:GD_GFOM_NN}). Let
	\begin{align*}
	\bar{\mathscr{E}}_{\texttt{test}}^{(t)}(X,Y)\equiv\E^{(0)}\big[ \big(\varphi_\ast(\iprod{X_{\textrm{new}}}{\mu_\ast})+\xi_{\pi_m}- f_{ (n^{-1/2}v_1^{(t)}, V_{[2:L]}^{(t)}) }( X_{\textrm{new}} ) \big)^2 | (X,Y) \big].
	\end{align*}
	In this step we will show that for any $\mathfrak{r}\geq 1$,
	\begin{align}\label{ineq:err_gd_NN_test_step1}
	\E^{(0)}\bigabs{\mathscr{E}_{\texttt{test}}^{(t)}(X,Y) - \bar{\mathscr{E}}_{\texttt{test}}^{(t)}(X,Y) }^{\mathfrak{r}}\leq \big(K\Lambda\varkappa_{\ast}\log n\big)^{c_t}\cdot n^{-1/c_t}.
	\end{align}
	To prove (\ref{ineq:err_gd_NN_test_step1}), we first interpolate the difference as
	\begin{align}\label{ineq:err_gd_NN_test_step1_1}
	&\E^{(0)}\bigabs{\mathscr{E}_{\texttt{test}}^{(t)}(X,Y) - \bar{\mathscr{E}}_{\texttt{test}}^{(t)}(X,Y) }^{\mathfrak{r}}\nonumber\\
	&\lesssim_{\mathfrak{r}} \E^{(0)}\bigabs{\mathscr{E}_{\texttt{test}}^{(t)}(X,Y) - \E^{(0)}\big[ \big(\varphi_\ast(\iprod{X_{\textrm{new}}}{\mu_\ast})+\xi_{\pi_m}- f_{ (W_1^{(t)}, V_{[2:L]}^{(t)}) }( X_{\textrm{new}} ) \big)^2 | (X,Y) \big] }^{\mathfrak{r}}\nonumber\\
	&\quad + \E^{(0)}\bigabs{\bar{\mathscr{E}}_{\texttt{test}}^{(t)}(X,Y) - \E^{(0)}\big[ \big(\varphi_\ast(\iprod{X_{\textrm{new}}}{\mu_\ast})+\xi_{\pi_m}- f_{ (W_1^{(t)}, V_{[2:L]}^{(t)}) }( X_{\textrm{new}} ) \big)^2 | (X,Y) \big] }^{\mathfrak{r}}\nonumber\\
	&\equiv I_1+I_2.
	\end{align}
	For $I_1$, by noting that apriori/continuity estimates for the neural network function $f_{\bm{W}}$ can be obtained directly from row-estimates of $\mathscr{G}_v$), by Theorem \ref{thm:se_gd_NN} we have
	\begin{align}\label{ineq:err_gd_NN_test_step1_2}
	I_1&\leq \big(K\Lambda\varkappa_{\ast}\log n\big)^{c_t}\cdot n^{-1/c_t}.
	\end{align}
	For $I_2$, let $\tilde{X}_1,\ldots,\tilde{X}_m \in \R^n$ be i.i.d. $X_{\mathrm{new}}$ that are independent of all other variables, and $\tilde{X} \in \R^{m\times n}$ whose rows collect $\{\tilde{X}_i\}_{i \in [m]}$. Let $\tilde{Y}_i\equiv \varphi_\ast(\iprod{\tilde{X}_i}{\mu_\ast})+\xi_i$ for $i \in [m]$,  $\tilde{Y}\equiv (\tilde{Y}_i)_{i \in [m]} \in \R^m$ and $\tilde{Y}_{[q]}\equiv [\tilde{Y}\,|\, 0_{m\times (q-1)}]\in \R^{m\times q}$ be the augmented version. Then we may write 
	\begin{align*}
	\hbox{1st term of $I_2$ in $\E^{(0)}$}&=\bar{\mathscr{E}}_{\texttt{test}}^{(t)}(X,Y) = \frac{1}{m} \E^{(0)}\Big[ \bigpnorm{ \tilde{Y}_{[q]}- \mathscr{G}_{V_{[2:L]}^{(t)}}(n^{-1/2} \tilde{X} v^{(t)}_1 ) }{}^2\big| (X,Y) \Big],\\
	\hbox{2nd term of $I_2$ in $\E^{(0)}$}& = \frac{1}{m} \E^{(0)}\Big[ \bigpnorm{ \tilde{Y}_{[q]}- \mathscr{G}_{V_{[2:L]}^{(t)}}(n^{-1/2} \tilde{X} \hat{v}^{(t)}_1 ) }{}^2\big| (X,Y) \Big].
	\end{align*}
	Consequently, using the apriori estimates in Propositions \ref{prop:apr_est_det_fcn} and \ref{prop:apr_est_se}, along with a standard moment estimate for $\pnorm{X}{\op},\pnorm{\tilde{X}}{\op}$ and the $\ell_2$ estimate in Proposition \ref{prop:GD_l2_est}, 
	\begin{align*}
	I_2&\leq \big(K\Lambda\varkappa_{\ast}\log n\big)^{c_t}\cdot \E^{(0)}\big(n^{-1/2}\pnorm{\Delta \hat{v}_1^{(t)} }{}\big)^{\mathfrak{r}}.
	\end{align*}
	Now applying the error estimate in Proposition \ref{prop:err_gd_GFOM}, we arrive at 
	\begin{align}\label{ineq:err_gd_NN_test_step1_3}
	I_2&\leq \big(K\Lambda\varkappa_{\ast}\log n\big)^{c_t}\cdot n^{-1/c_t}.
	\end{align}
	The claimed estimate in (\ref{ineq:err_gd_NN_test_step1}) follows by combining (\ref{ineq:err_gd_NN_test_step1_1}), (\ref{ineq:err_gd_NN_test_step1_2}) and (\ref{ineq:err_gd_NN_test_step1_3}).
	
	\noindent (\textbf{Step 2}).  With $\mathsf{Z}_n\sim \mathcal{N}(0,I_n)$, let
	\begin{align}\label{ineq:err_gd_NN_test_step2_0}
	\bar{\mathscr{E}}_{\texttt{test}}^{(t),\mathsf{Z}}(X,Y)&\equiv\E^{(0)}\big[ \big(\varphi_\ast(\iprod{\mathsf{Z}_n }{\mu_\ast})+\xi_{\pi_m}- f_{ (n^{-1/2}v_1^{(t)}, V_{[2:L]}^{(t)}) }( \mathsf{Z}_n ) \big)^2 | (X,Y) \big].
	\end{align}
	In this step we will show that for any $\mathfrak{r}\geq 1$,
	\begin{align}\label{ineq:err_gd_NN_test_step2}
	\E^{(0)}\bigabs{\bar{\mathscr{E}}_{\texttt{test}}^{(t)}(X,Y) - \bar{\mathscr{E}}_{\texttt{test}}^{(t),\mathsf{Z}}(X,Y) }^{\mathfrak{r}}\leq \big(K\Lambda\varkappa_{\ast}\log n\big)^{c_t}\cdot n^{-1/c_t}.
	\end{align}
	To this end, let $H:\R^n \to \R$ be defined by
	\begin{align*}
	H(z)\equiv  \Big(\varphi_\ast\big(\iprod{z }{\mu_\ast}\big)+\xi_{\pi_m}- f_{ (n^{-1/2}v_1^{(t)}, V_{[2:L]}^{(t)}) }( z ) \Big)^2\equiv H_0^2(z).
	\end{align*}
    Some calculations yield that for any $i \in [n]$,
    \begin{align}\label{ineq:err_gd_NN_test_step2_1}
    \partial_i^3 H(z)&=6\partial_i H_0(z) \partial_i^2 H_0(z)+2H_0(z)\partial_i^3 H_0(z),\nonumber\\
    \partial_i^p H_0(z) &= \varphi_\ast^{(p)}(\iprod{z }{\mu_\ast}) (\mu_\ast)_i^p- \partial_i^p  f_{ (n^{-1/2}v_1^{(t)}, V_{[2:L]}^{(t)}) }(z),\quad p \in \mathbb{Z}_{\geq 0}.
    \end{align}
    Note that $H_0(z)=\varphi_\ast\big(\iprod{z }{\mu_\ast}\big)+\xi_{\pi_m}- \mathscr{H}_{V^{(t)}_{[2:L]};1}\big(n^{-1/2}v_1^{(t),\top} z\big)$, so by the apriori estimates in Propositions \ref{prop:apr_est_det_fcn} and \ref{prop:apr_est_se}, we have
    \begin{align}\label{ineq:err_gd_NN_test_step2_2}
    \abs{H_0(z)}\leq \big(K\Lambda\varkappa_{\ast}\log n\big)^{c_t}\cdot \big[\big(1+\abs{\iprod{z }{\mu_\ast}}\big)^{c_0r_0}+ \pnorm{n^{-1/2}v_1^{(t),\top} z}{\infty}\big].
    \end{align}
     Combining (\ref{ineq:err_gd_NN_test_step2_1})-(\ref{ineq:err_gd_NN_test_step2_2}) and the derivative estimate in Lemma \ref{lem:f_derivative_bound}, 
	\begin{align}\label{ineq:err_gd_NN_test_step2_3}
	\max_{i \in [n]}\abs{\partial_i^3 H(z)}&\leq \big(K\Lambda\varkappa_{\ast}\log n\big)^{c_t}\cdot\big(1+\abs{\iprod{z }{\mu_\ast}}+\pnorm{n^{-1/2}v_1^{(t),\top} z}{\infty}\big)^{c_0 r_0}\nonumber\\
	&\qquad\times \big( \pnorm{\mu_\ast}{\infty}^3+ \pnorm{n^{-1/2} v_1^{(t)}}{\infty}^3\big).
	\end{align}
	Using Lindeberg's principle in Lemma \ref{lem:lindeberg} and the estimate (\ref{ineq:err_gd_NN_test_step2_3}), we have
	\begin{align}\label{ineq:err_gd_NN_test_step2_4}
	&\abs{\E^{(0)} H(\mathsf{Z}_n)-\E^{(0)} H(X_{\mathrm{new}}) }\nonumber\\
	&\leq \big(K\Lambda\varkappa_{\ast}\log n\big)^{c_t}\cdot n \cdot  \E^{(0),1/2} \big(\pnorm{\mu_\ast}{\infty}^3 + \pnorm{n^{-1/2} v_1^{(t)}}{\infty}^3\big)^2\nonumber\\
	&\qquad \times  \sup\nolimits_{\mathsf{W} }\E^{(0),1/2} \big(1+\abs{\iprod{\mathsf{W}}{\mu_\ast}}+\pnorm{n^{-1/2}v_1^{(t),\top} \mathsf{W}}{\infty}\big)^{c_0 r_0}.
	\end{align}
	Here the suprema is taken over all random vectors $\mathsf{W}\in \R^n$ with independent mean $0$, unit variance entries and $\max_{i \in [n]} \pnorm{\mathsf{W}_i}{\psi_2}\leq K$. We note:
	\begin{itemize}
		\item Using the $\ell_\infty$ control for $v_1^{(t)}$ in Proposition \ref{prop:GD_linfty_est} and the fact that $\pnorm{\mu_\ast}{\infty}\leq n^{-1/2} \varkappa_{\ast}$, the first expectation on the right hand side of (\ref{ineq:err_gd_NN_test_step2_4}) can be bounded by $\big(K\Lambda\varkappa_{\ast}\log n\big)^{c_t}\cdot n^{-3/2}$. 
		\item The second expectation on the right hand side of (\ref{ineq:err_gd_NN_test_step2_4}) can be bounded by $\big[K\big(\pnorm{\mu_\ast}{}+\pnorm{n^{-1/2}v_1^{(t)}}{}\big)\big]^{c_t}$, which by Proposition \ref{prop:GD_l2_est}, can be further controlled by $\big(K\Lambda\varkappa_{\ast}\log n\big)^{c_t}$.
	\end{itemize}
	Combining these estimates with (\ref{ineq:err_gd_NN_test_step2_4}), we arrive at
	\begin{align*}
	\hbox{LHS of (\ref{ineq:err_gd_NN_test_step2})}&= \E^{(0)}\abs{\E^{(0)} H(\mathsf{Z}_n)-\E^{(0)} H(X_{\mathrm{new}}) }^{\mathfrak{r}}\leq \big(K\Lambda\varkappa_{\ast}\log n\big)^{c_t}\cdot n^{-1/c_t},
	\end{align*}
	proving (\ref{ineq:err_gd_NN_test_step2}).

	\noindent (\textbf{Step 3}).  In this step we will show that for any $\mathfrak{r}\geq 1$,
	\begin{align}\label{ineq:err_gd_NN_test_step3}
	&\E^{(0)}\bigabs{ \bar{\mathscr{E}}_{\texttt{test}}^{(t),\mathsf{Z}}(X,Y)-\E^{(0)} \bigpnorm{\mathscr{R}_{V_{[2:L]}^{(t)};\pi_m}\big(\mathfrak{U}^{(t+1)},\mathfrak{U}^{(0)} \big)}{}^2 }^{\mathfrak{r}}\nonumber\\
	&\leq \big(K\Lambda\varkappa_{\ast}\log n\big)^{c_t}\cdot n^{-1/c_t}.
	\end{align}
	Recall $(\mu_{\ast,n})_{[q]}=[n^{1/2}\mu_\ast\,|\, 0_{n\times (q-1)}] = D_{-1} \in \R^{n\times q}$. Let $\Sigma,\Sigma_0 \in \R^{2q\times 2q}$ be two covariance matrices defined via
	\begin{align*}
	\Sigma & \equiv 
	\begin{pmatrix}
	n^{-1} \big(v_1^{(t)}\big)^\top v_1^{(t)} & n^{-1} (v_1^{(t)})^\top (\mu_{\ast,n})_{[q]}\\
	n^{-1} (\mu_{\ast,n})_{[q]}^\top v_1^{(t)} & n^{-1} (\mu_{\ast,n})_{[q]}^\top (\mu_{\ast,n})_{[q]}
	\end{pmatrix}
	=
	\begin{pmatrix}
	n^{-1} \big(v_1^{(t)}\big)^\top v_1^{(t)} & n^{-1} (v_1^{(t)})^\top D_{-1}\\
	n^{-1} D_{-1}^\top v_1^{(t)} & n^{-1} D_{-1}^\top D_{-1}
	\end{pmatrix},
	\end{align*}
	and
	\begin{align*}
	\Sigma_0 &\equiv 
	\begin{pmatrix}
	\cov\big(\mathfrak{U}^{(t+1)}, \mathfrak{U}^{(t+1)}\big) & \cov\big(\mathfrak{U}^{(t+1)}, \mathfrak{U}^{(0)}\big)\\
	\cov\big(\mathfrak{U}^{(0)}, \mathfrak{U}^{(t+1)}\big) & \cov\big(\mathfrak{U}^{(0)}, \mathfrak{U}^{(0)}\big)
	\end{pmatrix}
	\stackrel{(\ast)}{=}
	\begin{pmatrix}
	\Omega_{t,t}+ n^{-1}D_t^\top D_t & n^{-1}D_t^\top D_{-1}\\
	n^{-1}D_{-1}^\top D_t & n^{-1}D_{-1}^\top D_{-1}
	\end{pmatrix}.
	\end{align*}
	Here the last identify $(\ast)$ follows from Definition \ref{def:gd_NN_se}-(S2). Now applying (i) the apriori estimate in Proposition \ref{prop:apr_est_se} and the $\ell_2$ estimate in Proposition \ref{prop:GD_l2_est}, and (ii) the state evolution for the auxiliary GFOM $\{(u^{(t)},v^{(t)})\}$ in Proposition \ref{prop:se_aux_GFOM} upon noting that $\E^{(0)}\Delta_{\pi_n}^{(t)}(\mathfrak{V}^{([1:t])} ) \Delta_{\pi_n}^{(t)}(\mathfrak{V}^{([1:t])} )^\top = \Omega_{t,t}+ n^{-1}D_t^\top D_t $, we have
	\begin{align}\label{ineq:err_gd_NN_test_step3_1}
	\pnorm{\Sigma_0}{\infty}^{\mathfrak{r}}+\E^{(0)} \pnorm{\Sigma}{\infty}^{\mathfrak{r}}+n^{1/c_t}\cdot \E^{(0)} \pnorm{\Sigma-\Sigma_0}{\infty}^{\mathfrak{r}}\leq \big(K\Lambda\varkappa_{\ast}\log n\big)^{c_t}.
	\end{align}
	Consequently, in view of the definition of $\bar{\mathscr{E}}_{\texttt{test}}^{(t),\mathsf{Z}}(X,Y)$ in (\ref{ineq:err_gd_NN_test_step2_0}), with $\mathsf{Z}_{2q}\sim \mathcal{N}(0,I_{2q})$, and using the apriori estimates in Propositions \ref{prop:apr_est_det_fcn} and \ref{prop:apr_est_se} along with the first two estimates in (\ref{ineq:err_gd_NN_test_step3_1}) above, we have
	\begin{align}\label{ineq:err_gd_NN_test_step3_2}
	&\bigabs{ \bar{\mathscr{E}}_{\texttt{test}}^{(t),\mathsf{Z}}(X,Y)-\E^{(0)} \bigpnorm{\mathscr{R}_{V_{[2:L]}^{(t)};\pi_m}\big(\mathfrak{U}^{(t+1)},\mathfrak{U}^{(0)} \big)}{}^2 }\nonumber\\
	& = \biggabs{ \E_{\mathsf{Z},\pi_m}\bigpnorm{\mathscr{R}_{V_{[2:L]}^{(t)};\pi_m}\big(\Sigma^{1/2}\mathsf{Z}_{2q}\big)}{}^2 -\E_{\mathsf{Z},\pi_m} \bigpnorm{\mathscr{R}_{V_{[2:L]}^{(t)};\pi_m}\big(\Sigma_0^{1/2}\mathsf{Z}_{2q}\big)}{}^2 }\nonumber\\
	&\leq \big(K\Lambda\varkappa_{\ast}\log n\big)^{c_t}\cdot \E_{\mathsf{Z}}^{1/2} \pnorm{(\Sigma^{1/2}-\Sigma_0^{1/2})\mathsf{Z}_{2q}}{}^2\nonumber\\
	&\stackrel{(\ast\ast)}{\leq} \big(K\Lambda\varkappa_{\ast}\log n\big)^{c_t}\cdot\pnorm{\Sigma-\Sigma_0}{\infty}.
	\end{align}
	The estimate in $(\ast\ast)$ can be obtained by using the standard matrix square root estimate; see, e.g., \cite[Lemma A.3]{bao2023leave}. The claimed estimate in (\ref{ineq:err_gd_NN_test_step3}) now follows by combining (\ref{ineq:err_gd_NN_test_step3_1}) and (\ref{ineq:err_gd_NN_test_step3_2}).
	
	Finally, combining (\ref{ineq:err_gd_NN_test_step1}) in Step 1, (\ref{ineq:err_gd_NN_test_step2}) in Step 2 and (\ref{ineq:err_gd_NN_test_step3}) in Step 3 to conclude the desired estimate.
\end{proof}

\subsection{Equivalent algorithm}

We will present below an algorithm that is equivalent to Algorithm \ref{alg:aug_gd_nn}, but is otherwise easier to make connections to the derivatives formula in Proposition \ref{prop:derivative_formula}.

\begin{algorithmdef}\label{def:equiv_aug_gd_nn}
	Use the same input data and initialization as in Algorithm \ref{alg:aug_gd_nn}. 
	
	\textbf{For $t=1,2,\ldots,t-1$}:
	\begin{enumerate}
		\item \emph{Forward propagation}: For $\alpha = 1,2,\ldots,L$ and $k \in [m], \ell \in [q]$,
		\begin{align*}
		\hat{H}_\alpha^{(t-1)} &\equiv  \sigma_{\alpha-1}( \hat{H}_{\alpha-1}^{(t-1)} ) W_\alpha^{(t-1)},\hbox{ where }\hat{H}_{0}^{(t-1)}=X,\\
		\mathfrak{d}_{k\ell} \hat{H}_\alpha^{(t-1)}&\equiv e_ke_\ell^\top \bm{1}_{\alpha=1}+ \big[\sigma_{\alpha-1}'( \hat{H}_{\alpha-1}^{(t-1)} )\odot \mathfrak{d}_{k\ell} \hat{H}_{\alpha-1}^{(t-1)}\big] W_\alpha^{(t-1)}\bm{1}_{\alpha=2,\cdots,L}.
		\end{align*}
		\item \emph{Backward propagation}: For $\alpha=L,L-1,\ldots,1$ and $k \in [m],\ell \in [q]$,
		\begin{align*}
		&\hat{P}_\alpha^{(t-1)}\equiv \big(\hat{H}_L^{(t-1)} -Y_{[q]}\big)\bm{1}_{\alpha=L}+ \Big(\hat{P}_{\alpha+1}^{(t-1)}\odot \sigma_{\alpha+1}'(\hat{H}_{\alpha+1}^{(t-1)} )\Big) \big(W_{\alpha+1}^{(t-1)}\big)^{\top}\bm{1}_{\alpha=L-1,\cdots,1},\\
		&\binom{
			\mathfrak{d}^{(1)}_{k\ell}\hat{P}_\alpha^{(t-1)} }{
			\mathfrak{d}^{(2)}_{k\ell}\hat{P}_\alpha^{(t-1)}}
		\equiv  \binom{ 0_{m\times q}  }{ e_ke_\ell^\top } \bm{1}_{\alpha=L}+ \Bigg[  \binom{ \hat{P}_{\alpha+1}^{(t-1)}\odot \sigma_{\alpha+1}''\big(\hat{H}_{\alpha+1}^{(t-1)}\big)\odot \mathfrak{d}_{k\ell} \hat{H}_{\alpha+1}^{(t-1)}}{0_{m\times q}}\\
		&\qquad\qquad + \binom{ \sigma_{\alpha+1}'\big(\hat{H}_{\alpha+1}^{(t-1)}\big) }{\sigma_{\alpha+1}'\big(\hat{H}_{\alpha+1}^{(t-1)}\big)}\odot \binom{
			\mathfrak{d}^{(1)}_{k\ell}\hat{P}_{\alpha+1}^{(t-1)}}{
			\mathfrak{d}^{(2)}_{k\ell}\hat{P}_{\alpha+1}^{(t-1)}} 
		\Bigg] \big(W_{\alpha+1}^{(t-1)}\big)^{\top}\bm{1}_{\alpha=L-1,\cdots,1}.
		\end{align*}
		\item \emph{Compute pre-gradient derivative estimates}: For $k \in [m], \ell \in [q]$,
		\begin{align*}
		&\mathfrak{d}_{k\ell} \hat{ \mathfrak{S} }^{(t-1)}\equiv \hat{P}_1^{(t-1)}\odot \sigma_1''(XW_1^{(t-1)})\odot e_ke_\ell^\top \\
		&\qquad + \bigg[\mathfrak{d}^{(1)}_{k\ell}\hat{P}_1^{(t-1)}+\sum_{r \in [q]}  \mathfrak{d}^{(2)}_{k r}\hat{P}_1^{(t-1)} \big(\mathfrak{d}_{k\ell} \hat{H}_L^{(t-1)}\big)_{kr}\bigg]\odot \sigma_1'(XW_1^{(t-1)}).
		\end{align*}
		\item \emph{Compute matrix-variate Onsager correction coefficients}: For $k \in [m]$ and $s \in [1:t]$,
		\begin{align*}
		\hat{\bm{L}}_k^{[t-1]} &\equiv \mathrm{diag}\bigg( \bigg\{ \eta^{(s-1)} \sum_{\ell \in [q]}\big(\mathfrak{d}_{k\ell} \hat{ \mathfrak{S} }^{(s-1)}\big)^\top e_k e_\ell^\top \bigg\}_{s \in [1:t]} \bigg) \in (\mathbb{M}_q)^{t\times t},\, \forall k \in [m],\\
		\hat{\bm{\tau}}^{[t]} & \equiv -\frac{1}{m}\sum_{k \in [m]} \Big((I_{\mathbb{M}_q})_t+\phi^{-1} \hat{\bm{L}}_k^{[t-1]}\mathfrak{O}_{\mathbb{M}_q;t}(\hat{\bm{\rho}}^{[t-1]})\Big)^{-1}\hat{\bm{L}}_k^{[t-1]} \in (\mathbb{M}_q)^{t\times t},\\
		\hat{\bm{\rho}}^{[t]} &\equiv (I_{\mathbb{M}_q})_t+ \big(\hat{\bm{\tau}}^{[t]} + (I_{\mathbb{M}_q})_t\big) \mathfrak{O}_{\mathbb{M}_q;t}(\hat{\bm{\rho}}^{[t-1]})  \in (\mathbb{M}_q)^{t\times t}.
		\end{align*}
		\item Proceed the same as (5) of Algorithm \ref{alg:aug_gd_nn}.
		\item Proceed the same as (6) of Algorithm \ref{alg:aug_gd_nn}.
	\end{enumerate}
	
\end{algorithmdef}

To see the equivalence of Algorithms \ref{alg:aug_gd_nn} and \ref{def:equiv_aug_gd_nn}, we have the following.

\begin{lemma}\label{lem:equiv_alg}
	The elements of the matrices $\mathfrak{d}_{k\ell} \hat{H}_\alpha^{(t-1)}$, $\mathfrak{d}^{(1)}_{k\ell}\hat{P}_\alpha^{(t-1)}$, $\mathfrak{d}^{(2)}_{k\ell}\hat{P}_\alpha^{(t-1)}$, $\mathfrak{d}_{k\ell} \hat{ \mathfrak{S} }^{(t-1)}\in \R^{m\times q} $ are non-zero only in their $k$-th rows, where $k \in [m]$. Moreover, for any further $\ell \in [q]$, 
	\begin{enumerate}
		\item $e_k^\top \mathfrak{d}_{k\ell} \hat{H}_\alpha^{(t-1)}=e_k^\top \mathfrak{\partial}_{\ell} \hat{H}_\alpha^{(t-1)}$.
		\item $e_k^\top \mathfrak{d}^{(b)}_{k\ell}\hat{P}_\alpha^{(t-1)}= e_k^\top \partial^{(b)}_{\ell}\hat{P}_\alpha^{(t-1)}$ for $b=1,2$.
		\item $e_k^\top \mathfrak{d}_{k\ell} \hat{ \mathfrak{S} }^{(t-1)}=e_k^\top \mathfrak{\partial}_{\ell} \hat{ \mathfrak{S} }^{(t-1)}$.
	\end{enumerate}
\end{lemma}
\begin{proof}
	(1), (2) follow from the definitions immediately. For (3), first note that 
	\begin{align*}
	e_k^\top \Big(\hat{P}_1^{(t-1)}\odot \sigma_1''(XW_1^{(t-1)})\odot e_ke_\ell^\top \Big) = e_k^\top \bigg[\big(\hat{P}_1^{(t-1)}\odot \sigma_1''(XW_1^{(t-1)})\big)\,e_\ell e_\ell^\top   \bigg].
	\end{align*}
	On the other hand, with $M_r\equiv \partial^{(2)}_{r}\hat{P}_1^{(t-1)}$ and $N_\ell\equiv \partial_{\ell} \hat{H}_L^{(t-1)}$,
	\begin{align*}
	&e_k^\top \bigg(\sum_{r \in [q]}  \mathfrak{d}^{(2)}_{k r}\hat{P}_1^{(t-1)} \big(\mathfrak{d}_{k\ell} \hat{H}_L^{(t-1)}\big)_{kr}\bigg) =  \sum_{r \in [q]}  e_k^\top\partial^{(2)}_{r}\hat{P}_1^{(t-1)} \big(\partial_{\ell} \hat{H}_L^{(t-1)}\big)_{kr} \\
	& =  \sum_{r \in [q]}  e_k^\top M_r \bigg(\sum_{d \in [q]} e_d e_d^\top \big(N_\ell\big)_{kr} \bigg) = \sum_{d, r \in [q]}  e_k^\top M_r e_d e_r^\top N_\ell^\top e_k e_d^\top= \sum_{d \in [q]} (\mathcal{Q}_d)_{kk} e_d^\top.
	\end{align*}
	Here $\mathcal{Q}_d\equiv \sum_{r \in [q]} M_r e_d e_r^\top N_\ell^\top$. Consequently, 
	\begin{align*}
	\sum_{r \in [q]}  \mathfrak{d}^{(2)}_{k r}\hat{P}_1^{(t-1)} \big(\mathfrak{d}_{k\ell} \hat{H}_L^{(t-1)}\big)_{kr}&=
	\begin{pmatrix}
	(\mathcal{Q}_1)_{11}& (\mathcal{Q}_2)_{11} &\cdots & (\mathcal{Q}_q)_{11}\\
	(\mathcal{Q}_1)_{22}& (\mathcal{Q}_2)_{22} &\cdots & (\mathcal{Q}_q)_{22}\\
	\vdots &\vdots &\ddots &\vdots \\
	(\mathcal{Q}_1)_{qq}& (\mathcal{Q}_2)_{qq} &\cdots & (\mathcal{Q}_q)_{qq}
	\end{pmatrix} = \sum_{d \in [q]}\mathrm{vecdiag}\,(\mathcal{Q}_d) e_d^\top.
	\end{align*}
	The right hand side of the above display is exactly equal to $\hat{Q}_\ell^{(t-1)}$ as defined in Algorithm \ref{alg:aug_gd_nn}.
\end{proof}

In plain words, Lemma \ref{lem:equiv_alg} above says that the $k$-th row of $\mathfrak{\partial}_{\ell} \hat{H}_\alpha^{(t-1)}\in \R^{m\times q}$ (resp. $\partial^{(b)}_{\ell}\hat{P}_\alpha^{(t-1)},\mathfrak{\partial}_{\ell} \hat{ \mathfrak{S} }^{(t-1)}\in \R^{m\times q}$) in Algorithm \ref{alg:aug_gd_nn}, as a $q$-dimensional vector, corresponds exactly to the (only non-trivial) $k$-th row of $\mathfrak{d}_{k\ell} \hat{H}_\alpha^{(t-1)}$ (resp. $ \mathfrak{d}^{(b)}_{k\ell}\hat{P}_\alpha^{(t-1)}, \mathfrak{d}_{k\ell} \hat{ \mathfrak{S} }^{(t-1)}\in \R^{m\times q}$) in Algorithm \ref{def:equiv_aug_gd_nn}. Henceforth, we will focus on analyzing Algorithm \ref{def:equiv_aug_gd_nn}.

\subsection{Consistency of $\hat{\bm{\tau}}^{[t]}$ and $\hat{\bm{\rho}}^{[t]}$}\label{subsection:proof_tau_rho_est_err}

The goal of this subsection is to prove the consistency of $\hat{\bm{\tau}}^{[t]}$ and $\hat{\bm{\rho}}^{[t]}$ in Algorithm \ref{alg:aug_gd_nn} (or its equivalent Algorithm \ref{def:equiv_aug_gd_nn}).
\begin{proposition}\label{prop:tau_rho_est_err}
	Suppose Assumption \ref{assump:gd_NN} holds for some $K,\Lambda\geq 2$ and $r_0\geq 0$. Then for any $\mathfrak{r}\geq 1$, there exists some $c_t=c_t(t,q,L,\mathfrak{r},r_0)>0$ such that 
	\begin{align*}
	\E^{(0)} \pnorm{\hat{\bm{\tau}}^{[t]}-\bm{\tau}^{[t]} }{\infty}^{\mathfrak{r}}+ \E^{(0)}\pnorm{  \hat{\bm{\rho}}^{[t]}- \bm{\rho}^{[t]}}{\infty}^{\mathfrak{r}}\leq \big(K\Lambda\varkappa_{\ast}\big)^{c_{t}}\cdot n^{-1/c_t}.
	\end{align*}
\end{proposition}

The proof of Proposition \ref{prop:tau_rho_est_err} relies the following alternating estimates. 

\begin{lemma}\label{lem:rho_hat_error}
	Suppose Assumption \ref{assump:gd_NN} holds for some $K,\Lambda\geq 2$ and $r_0\geq 0$. Then there exists some $c_t=c_t(t,q,L,r_0)>0$ such that 
	\begin{align*}
	\pnorm{\hat{\bm{\rho}}^{[t]}-\bm{\rho}^{[t]} }{\infty}\leq \Big[K\Lambda\varkappa_{\ast}\log n\cdot \big(1+\pnorm{ \hat{\bm{\tau}}^{[t]} - \bm{\tau}^{[t]}}{\infty}\big)\Big]^{c_{t}}\cdot \pnorm{ \hat{\bm{\tau}}^{[t]} - \bm{\tau}^{[t]}}{\infty}.
	\end{align*}
\end{lemma}
\begin{proof}
	Comparing the definition of $\hat{\bm{\rho}}^{[t]}$ in Algorithm \ref{def:equiv_aug_gd_nn} and $\bm{\rho}^{[t]}$ in (\ref{def:rho_se_matrix}), we have
	\begin{align*}
	\pnorm{ \hat{\bm{\rho}}^{[t]} - \bm{\rho}^{[t]}}{\infty}/c_t&\leq \pnorm{ \bm{\rho}^{[t]}}{\infty}\cdot \pnorm{ \hat{\bm{\tau}}^{[t]} - \bm{\tau}^{[t]}}{\infty}+\big(1+\pnorm{\hat{\bm{\tau}}^{[t]}}{\infty}\big)\cdot \pnorm{ \hat{\bm{\rho}}^{[t-1]} - \bm{\rho}^{[t-1]}}{\infty}.
	\end{align*}
	Using the simple estimate $\pnorm{\hat{\bm{\tau}}^{[t]}}{\infty}\leq \pnorm{\hat{\bm{\tau}}^{[t]}-\bm{\tau}^{[t]} }{\infty}+\pnorm{ \bm{\tau}^{[t]}}{\infty}$ and the apriori estimate in Proposition \ref{prop:apr_est_se}, we have
	\begin{align*}
	\pnorm{ \hat{\bm{\rho}}^{[t]} - \bm{\rho}^{[t]}}{\infty}&\leq \big(K\Lambda\varkappa_{\ast}\log n\big)^{c_{t}}\cdot\Big[\pnorm{ \hat{\bm{\tau}}^{[t]} - \bm{\tau}^{[t]}}{\infty}\\
	&\qquad +\big(1+\pnorm{ \hat{\bm{\tau}}^{[t]} - \bm{\tau}^{[t]}}{\infty}\big)\cdot \pnorm{ \hat{\bm{\rho}}^{[t-1]} - \bm{\rho}^{[t-1]}}{\infty}\Big].
	\end{align*}
	Iterating the above estimate and using the trivial initial condition $\pnorm{\hat{\bm{\rho}}^{[t-1]} - \bm{\rho}^{[t-1]}}{\infty}=0$ to conclude. 
\end{proof}

\begin{lemma}\label{lem:tau_hat_error}
	Suppose Assumption \ref{assump:gd_NN} holds for some $K,\Lambda\geq 2$ and $r_0\geq 0$. Then for any $\mathfrak{r}\geq 1$, there exists some $c_t=c_t(t,q,L,\mathfrak{r},r_0)>0$ such that 
	\begin{align*}
	\E^{(0)} \pnorm{\hat{\bm{\tau}}^{[t]}-\bm{\tau}^{[t]} }{\infty}^{\mathfrak{r}}\leq \big(K\Lambda\varkappa_{\ast}\log n\big)^{c_{t}}\cdot \Big(n^{-1/c_t}+ \E^{(0),1/2}\pnorm{  \hat{\bm{\rho}}^{[t-1]}- \bm{\rho}^{[t-1]}}{\infty}^{2\mathfrak{r}}\Big). 
	\end{align*}
\end{lemma}
\begin{proof}
	Recall $\bm{\rho}^{[t]}\in (\mathbb{M}_q)^{t\times t}$ defined in (\ref{def:tau_rho_matrix}). Let
	\begin{align}\label{ineq:tau_hat_error_def_tau_bar}
	\bar{\bm{\tau}}^{[t]} &\equiv -\frac{1}{m}\sum_{k \in [m]} \Big((I_{\mathbb{M}_q})_t+\phi^{-1} \hat{\bm{L}}_k^{[t-1]}\mathfrak{O}_{\mathbb{M}_q;t}({\bm{\rho}}^{[t-1]})\Big)^{-1}\hat{\bm{L}}_k^{[t-1]} \in (\mathbb{M}_q)^{t\times t}.
	\end{align}
	\noindent (\textbf{Step 1}). In this step, we shall prove that for any $\mathfrak{r}\geq 1$,
	\begin{align}\label{ineq:tau_hat_error_step1}
	\E^{(0)} \pnorm{\bar{\bm{\tau}}^{[t]}-\bm{\tau}^{[t]} }{\infty}^{\mathfrak{r}}\leq \big(K\Lambda\varkappa_{\ast}\log n\big)^{c_{t}}\cdot n^{-1/c_t}. 
	\end{align}
	Recall $\bm{\mathfrak{L}}^{(t)}_k\big(u^{([0:t])},V^{([0:t-1])}\big)$ defined in (\ref{def:L_u_V}). For $k \in [m]$, consider the mapping $\mathfrak{H}^{(t)}_k: \R^{q[0:t]}\times \big((\mathbb{M}_q)^{[2:L]}\big)^{[0:t-1]}\to (\mathbb{M}_q)^{t\times t}$ defined by
	\begin{align}\label{ineq:tau_hat_error_step1_H}
	&\mathfrak{H}^{(t)}_k\big(u^{([0:t])},V^{([0:t-1])}\big)\\
	&\equiv - \Big( (I_{\mathbb{M}_q})_t+\phi^{-1}\bm{\mathfrak{L}}^{(t)}_k\big(u^{([0:t])},V^{([0:t-1])}\big) \mathfrak{O}_{\mathbb{M}_q;t}(\bm{\rho}^{[t-1]})\Big)^{-1} \bm{\mathfrak{L}}^{(t)}_k\big(u^{([0:t])},V^{([0:t-1])}\big).\nonumber
	\end{align}
	Then Proposition \ref{prop:tau_repres} implies that
	\begin{align}\label{ineq:tau_hat_error_step1_1}
	\bm{\tau}^{[t]}&= \E^{(0)} \mathfrak{H}^{(t)}_{\pi_m} \Big[\Big\{\Theta^{(s)}_{\pi_m}\big(\mathfrak{U}^{([0:s])}\big)\Big\}_{s \in [0:t]},V^{([0:t-1])}\Big].
	\end{align}
	\noindent \emph{Claim 1}: We claim that 
	\begin{align}\label{ineq:tau_hat_error_step1_2}
	\bar{\bm{\tau}}^{[t]}& = \E^{(0)} \mathfrak{H}^{(t)}_{\pi_m} \Big[\Big((X\mu_{\ast,[q]})_{\pi_m\cdot},\big\{\big(XW_1^{(s-1)}\big)_{\pi_m\cdot}\big\}_{s \in [1:t]}\Big),\bm{W}_{[2:L]}^{([0:t-1])}\Big].
	\end{align}
	To see (\ref{ineq:tau_hat_error_step1_2}), first note that for  Algorithm \ref{def:equiv_aug_gd_nn} we have the following identification:
	\begin{align*}
	\hat{H}_\alpha^{(t-1)} &=\mathscr{H}_{W_{[2:L]}^{(t-1)}}(X W_1^{(t-1)}),\nonumber\\
	 \mathfrak{d}_{k\ell} \hat{H}_\alpha^{(t-1)} &=\partial_{u_{k\ell}}\mathscr{H}_{W_{[2:L]}^{(t-1)}}(X W_1^{(t-1)}),\nonumber\\
	\hat{P}_\alpha^{(t-1)} &=\mathscr{P}^{(\alpha:L]}_{W_{[2:L]}^{(t-1)}}\Big(X W_1^{(t-1)}, \mathscr{R}_{W_{[2:L]}^{(t-1)}}\big(XW_1^{(t-1)},X\mu_{\ast, [q]}\big)\Big),\nonumber\\
	\mathfrak{d}^{(1)}_{k\ell}\hat{P}_\alpha^{(t-1)}  &= \partial_{u_{k\ell}}\mathscr{P}^{(\alpha:L]}_{W_{[2:L]}^{(t-1)}}\Big(X W_1^{(t-1)}, \mathscr{R}_{W_{[2:L]}^{(t-1)}}\big(XW_1^{(t-1)},X\mu_{\ast, [q]}\big)\Big),\nonumber\\
	\mathfrak{d}^{(2)}_{k\ell}\hat{P}_\alpha^{(t-1)}  &= \partial_{z_{k\ell}}\mathscr{P}^{(\alpha:L]}_{W_{[2:L]}^{(t-1)}}\Big(X W_1^{(t-1)}, \mathscr{R}_{W_{[2:L]}^{(t-1)}}\big(XW_1^{(t-1)},X\mu_{\ast, [q]}\big)\Big).
	\end{align*}
	So using the derivative formula in Proposition \ref{prop:derivative_formula}-(4), $\mathfrak{d}_{k\ell} \hat{ \mathfrak{S} }^{(t-1)}$ in Algorithm \ref{def:equiv_aug_gd_nn} can be written as 
	\begin{align*}
	\mathfrak{d}_{k\ell} \hat{ \mathfrak{S} }^{(t-1)}= \partial_{u_{k\ell}} \mathfrak{S}\big(XW_1^{(t-1)},X\mu_{\ast, [q]},W_{[2:L]}^{(t-1)}\big).
	\end{align*}
	This means for $s \in [1:t]$,
	\begin{align}\label{ineq:tau_hat_error_step1_3}
	\big(\hat{\bm{L}}_k^{[t-1]} \big)_{ss}& = \eta^{(s-1)} \sum_{\ell \in [q]}\big(\partial_{k\ell} \hat{ \mathfrak{S} }^{(s-1)}\big)^\top e_k e_\ell^\top \nonumber\\
	& = \eta^{(s-1)} \sum_{\ell \in [q]} \partial_{u_{\ell}} \mathfrak{S}_{k}\big( (XW_1^{(s-1)})_{k\cdot}, (X\mu_{\ast, [q]})_{k\cdot},W_{[2:L]}^{(s-1)}\big) e_\ell^\top \nonumber\\
	& = \bigg(\bm{\mathfrak{L}}^{(t)}_k\Big[\Big((X\mu_{\ast,[q]})_{k\cdot},\big\{\big(XW_1^{(s-1)}\big)_{k\cdot}\big\}_{s \in [1:t]}\Big),\bm{W}_{[2:L]}^{([0:t-1])}\Big]\bigg)_{ss},
	\end{align}
	proving (\ref{ineq:tau_hat_error_step1_2}) of Claim 1, in view of the definition of $\mathfrak{H}_k^{(t)}$ in (\ref{ineq:tau_hat_error_step1_H}).
	
	\noindent \emph{Claim 2}: We claim that $\{\mathfrak{H}^{(t)}_k\}_{k \in [m]}$ satisfies
	\begin{align}\label{ineq:tau_hat_error_step1_4}
	&\max_{k \in [m]}\bigpnorm{\mathfrak{H}^{(t)}_k\big(u_1^{([0:t])},V_1^{([0:t-1])}\big)- \mathfrak{H}^{(t)}_k\big(u_2^{([0:t])},V_2^{([0:t-1])}\big) }{}\\
	&\leq \Big[K\Lambda \varkappa_{\ast}\cdot (\pnorm{V_1^{([0:t-1])}}{\infty}+\pnorm{V_2^{([0:t-1])}}{\infty})\cdot \min_{\#=1,2} \Big(1+\pnorm{u_\#^{([1:t])}}{\infty}+ \pnorm{\varphi_\ast(u_\#^{(0)}) }{\infty} \Big)\Big]^{c_t}\nonumber\\
	&\qquad \times \Big(\pnorm{\varphi_\ast(u_1^{(0)})-\varphi_\ast(u_2^{(0)})}{\infty}+\pnorm{u_1^{([1:t])}-u_2^{([1:t])}}{\infty}+\pnorm{V_1^{([0:t-1])}-V_2^{([0:t-1])}}{\infty}\Big).\nonumber
	\end{align}	
To see this, note that in the definition (\ref{ineq:tau_hat_error_step1_H}), as the matrix $\bm{\mathfrak{L}}^{(t)}_k\big(u^{([0:t])},V^{([0:t-1])}\big) \mathfrak{O}_{\mathbb{M}_q;t}(\bm{\rho}^{[t-1]})$ is lower triangular with diagonal elements $0$, then its power to the order $qt$ must become a zero matrix.  So we may write 
\begin{align*}
&\mathfrak{H}^{(t)}_k\big(u^{([0:t])},V^{([0:t-1])}\big)\\
&= \sum_{r=0}^{qt} (-1)^{r+1} \bigg[\phi^{-1}\bm{\mathfrak{L}}^{(t)}_k\big(u^{([0:t])},V^{([0:t-1])}\big) \mathfrak{O}_{\mathbb{M}_q;t}(\bm{\rho}^{[t-1]})\bigg]^r \bm{\mathfrak{L}}^{(t)}_k\big(u^{([0:t])},V^{([0:t-1])}\big).
\end{align*}
Using the definition (\ref{def:L_u_V}) and Proposition \ref{prop:apr_est_derivative}-(8), $\bm{\mathfrak{L}}^{(t)}_k$ satisfies an estimate in the form of (\ref{ineq:tau_hat_error_step1_4}), and thereby concluding the proof of Claim 2.

Now replacing $\{\hat{u}^{(s)}\}_{s \in [0:t]}$ and $\bm{W}_{[2:L]}^{([0:t-1])}$ in (\ref{ineq:tau_hat_error_step1_2}) with the auxiliary GFOM iterates $\{u^{(s)}\}_{s \in [0:t]}$ and $V^{([0:t-1])}$, in view of (\ref{ineq:tau_hat_error_step1_4}), by using the error estimate in Proposition \ref{prop:err_gd_GFOM} along with the delocalization estimate in Proposition \ref{prop:GD_linfty_est} for $\{u^{(s)}\}_{s \in [0:t]}$ and the global estimates for $\big\{(\hat{v}^{(s)},v^{(s)})\big\}_{s \in [0:t]}$ in Proposition \ref{prop:GD_l2_est}, 
\begin{align*}
\bigpnorm{\bar{\bm{\tau}}^{[t]}- \E^{(0)} \mathfrak{H}^{(t)}_{\pi_m} \big(\{(u^{(s)})_{\pi_m\cdot}\}_{s \in [0:t]},V^{([0:t-1])}\big)}{\infty}\leq \big(K\Lambda\varkappa_{\ast}\log n\big)^{c_{t}}\cdot n^{-1/c_t}. 
\end{align*}
The claimed estimate in (\ref{ineq:tau_hat_error_step1}) now follows by applying Proposition \ref{prop:se_aux_GFOM} to the left hand side of the above display and (\ref{ineq:tau_hat_error_step1_1}).

	\noindent (\textbf{Step 2}). In this step, we shall prove that for any $\mathfrak{r}\geq 1$,
	\begin{align}\label{ineq:tau_hat_error_step2}
	\E^{(0)} \pnorm{\bar{\bm{\tau}}^{[t]}-\bm{\hat{\tau}}^{[t]} }{\infty}^{\mathfrak{r}}\leq \big(K\Lambda\varkappa_{\ast}\log n\big)^{c_{t}}\cdot \E^{(0),1/2}\pnorm{  \hat{\bm{\rho}}^{[t-1]}- \bm{\rho}^{[t-1]}}{\infty}^{2\mathfrak{r}}.
	\end{align}
	To this end, from the definition (\ref{ineq:tau_hat_error_def_tau_bar}), we may write 
	\begin{align}\label{ineq:tau_hat_error_step2_0}
	\bar{\bm{\tau}}^{[t]} &\equiv  \sum_{r=0}^{qt} (-1)^{r+1} \phi^{-r}\cdot \bigg[ \frac{1}{m}\sum_{k \in [m]} \Big( \hat{\bm{L}}_k^{[t-1]}\mathfrak{O}_{\mathbb{M}_q;t}({\bm{\rho}}^{[t-1]})\Big)^r\hat{\bm{L}}_k^{[t-1]}\bigg].
	\end{align}
	A similar representation holds for $\hat{\bm{\tau}}^{[t]}$ by replacing ${\bm{\rho}}^{[t-1]}$ in the above display by $\hat{\bm{\rho}}^{[t-1]}$. This means that 
	\begin{align}\label{ineq:tau_hat_error_step2_1}
	\pnorm{\bar{\bm{\tau}}^{[t]}-\bm{\hat{\tau}}^{[t]} }{\infty}&\leq \big[K\Lambda\cdot (\pnorm{{\bm{\rho}}^{[t-1]}}{\infty}+\pnorm{\hat{\bm{\rho}}^{[t-1]}}{\infty}) \big]^{c_t}\cdot \E_{\pi_m} \pnorm{\hat{\bm{L}}_{\pi_m}^{[t-1]}}{\infty}^{c_t}\nonumber\\
	&\qquad\qquad \times \pnorm{  \hat{\bm{\rho}}^{[t-1]}- \bm{\rho}^{[t-1]}}{\infty}.
	\end{align}
	We first estimate $\pnorm{\hat{\bm{\rho}}^{[t-1]}}{\infty}$. Using the representation (\ref{ineq:tau_hat_error_step2_0}) for $\hat{\bm{\tau}}^{[t]}$, 
	\begin{align*}
	\pnorm{\hat{\bm{\tau}}^{[t]}}{\infty}\leq (K\Lambda)^{c_t}\cdot \E_{\pi_m} \pnorm{\hat{\bm{L}}_{\pi_m}^{[t-1]}}{\infty}^{c_t}\cdot \pnorm{\hat{\bm{\rho}}^{[t-1]}}{\infty}^{c_t}.
	\end{align*}
	Using the definition of $\hat{\bm{\rho}}^{[t]}$ in Algorithm \ref{def:equiv_aug_gd_nn}, we have
	\begin{align*}
	\pnorm{\hat{\bm{\rho}}^{[t]}}{\infty}&\leq 1+(K\Lambda)^{c_t}\cdot \big(1+\pnorm{\hat{\bm{\tau}}^{[t]}}{\infty}\big)\cdot \pnorm{\hat{\bm{\rho}}^{[t-1]}}{\infty}.
	\end{align*}
	Combining the above two displays, we have
	\begin{align*}
	\pnorm{\hat{\bm{\rho}}^{[t]}}{\infty}&\leq (K\Lambda)^{c_t}\cdot \big(1+\E_{\pi_m} \pnorm{\hat{\bm{L}}_{\pi_m}^{[t-1]}}{\infty}^{c_t}\big)^{c_t}\cdot \big(1+\pnorm{\hat{\bm{\rho}}^{[t-1]}}{\infty}\big)^{c_t}.
	\end{align*}
	Iterating the above estimate with the trivial initial condition $\pnorm{\hat{\bm{\rho}}^{[0]}}{\infty}=0$, we have
	\begin{align}\label{ineq:tau_hat_error_step2_2}
	\pnorm{\hat{\bm{\rho}}^{[t]}}{\infty}&\leq (K\Lambda)^{c_t}\cdot \big(1+\E_{\pi_m} \pnorm{\hat{\bm{L}}_{\pi_m}^{[t-1]}}{\infty}^{c_t}\big)^{c_t}.
	\end{align}
	Next we provide moment estimates for $\E_{\pi_m} \pnorm{\hat{\bm{L}}_{\pi_m}^{[t-1]}}{\infty}^{c_t}$. By (\ref{ineq:tau_hat_error_step1_3}) and using a similar argument as in Step 1 that replaces $\{\hat{u}^{(s)}\}_{s \in [0:t]}$ and $\bm{W}_{[2:L]}^{([0:t-1])}$ with the auxiliary GFOM iterates $\{u^{(s)}\}_{s \in [0:t]}$ and $V^{([0:t-1])}$, it suffices to provide moment controls for $\E_{\pi_m}\pnorm{\bm{\mathfrak{L}}^{(t)}_{\pi_m} (\{(u^{(s)})_{\pi_m\cdot}\}_{s \in [0:t]},V^{([0:t-1])})}{\infty}^{c_t}$. As $\{\bm{\mathfrak{L}}^{(t)}_k\}$'s are a $(K\Lambda \varkappa_{\ast})^{c_t}$-pseudo-Lipschitz of order $c_0 r_0$, by Proposition \ref{prop:se_aux_GFOM} and the apriori estimate in Proposition \ref{prop:apr_est_se}, for any $\mathfrak{r}\geq 0$,
	\begin{align*}
	\E^{(0)} \big(\E_{\pi_m}\pnorm{\bm{\mathfrak{L}}^{(t)}_{\pi_m} (\{(u^{(s)})_{\pi_m\cdot}\}_{s \in [0:t]},V^{([0:t-1])})}{\infty}^{c_t}\big)^{ \mathfrak{r}}\leq \big(K\Lambda\varkappa_{\ast}\log n\big)^{c_{t}}.
	\end{align*}
	Summarizing the above arguments, we obtain
	\begin{align}\label{ineq:tau_hat_error_step2_3}
	\E^{(0)} \big(\E_{\pi_m} \pnorm{\hat{\bm{L}}_{\pi_m}^{[t-1]}}{\infty}^{c_t}\big)^{\mathfrak{r}}\leq \big(K\Lambda\varkappa_{\ast}\log n\big)^{c_{t}}.
	\end{align}
	The claimed estimate in (\ref{ineq:tau_hat_error_step2}) now follows from a combination of (\ref{ineq:tau_hat_error_step2_1}), (\ref{ineq:tau_hat_error_step2_2}) and (\ref{ineq:tau_hat_error_step2_3}), along with the apriori estimates in proposition \ref{prop:apr_est_se}.
	
	\noindent (\textbf{Step 3}). The desired estimate now follows from (\ref{ineq:tau_hat_error_step1}) and (\ref{ineq:tau_hat_error_step2}). 
\end{proof}

\begin{proof}[Proof of Proposition \ref{prop:tau_rho_est_err}]
	Combining Lemmas \ref{lem:rho_hat_error} and \ref{lem:tau_hat_error}, for any $\mathfrak{r}\geq 1$,
	\begin{align*}
	\E^{(0)} \pnorm{\hat{\bm{\rho}}^{[t]}-\bm{\rho}^{[t]} }{\infty}^{\mathfrak{r}} &\leq  \big(K\Lambda\varkappa_{\ast}\log n\big)^{c_{t}}\cdot \Big(1+ \E^{(0),1/4}\pnorm{  \hat{\bm{\rho}}^{[t-1]}- \bm{\rho}^{[t-1]}}{\infty}^{4\mathfrak{r}}\Big)\\
	&\qquad \times \Big(n^{-1/c_t}+ \E^{(0),1/4}\pnorm{  \hat{\bm{\rho}}^{[t-1]}- \bm{\rho}^{[t-1]}}{\infty}^{4\mathfrak{r}}\Big).
	\end{align*}
	Now iterating the estimate until the trivial initial condition $\pnorm{\hat{\bm{\rho}}^{[0]}- \bm{\rho}^{[0]}}{\infty}=0$ to conclude the estimate for $\E^{(0)}\pnorm{  \hat{\bm{\rho}}^{[t]}- \bm{\rho}^{[t]}}{\infty}^{\mathfrak{r}}$. The estimate for $\E^{(0)} \pnorm{\hat{\bm{\tau}}^{[t]}-\bm{\tau}^{[t]} }{\infty}^{\mathfrak{r}}$ follows from another application of Lemma \ref{lem:tau_hat_error}.
\end{proof}

\subsection{Proof of Theorem \ref{thm:est_test_err}}\label{subsection:proof_est_test_err}

Recall $U^{(t)}$ defined in (\ref{def:U_debiased}), and $\hat{U}^{(t)}$  defined in Algorithm \ref{def:equiv_aug_gd_nn}.
\begin{lemma}\label{lem:U_apr_err_est}
	Suppose Assumption \ref{assump:gd_NN} holds for some $K,\Lambda\geq 2$ and $r_0\geq 0$. Then for any $\mathfrak{r}\geq 1$, there exists some $c_t=c_t(t,q,L,\mathfrak{r},r_0)>0$ such that 
	\begin{enumerate}
		\item $\E^{(0)}\big(n^{-1/2}\pnorm{U^{(t)}}{}\big)^{\mathfrak{r}}+\E^{(0)}\big(n^{-1/2}\pnorm{\hat{U}^{(t)}}{}\big)^{\mathfrak{r}}\leq \big(K\Lambda\varkappa_{\ast}\log n\big)^{c_{t}}$.
		\item $\E^{(0)}\big(n^{-1/2}\pnorm{U^{(t)}-\hat{U}^{(t)} }{}\big)^{\mathfrak{r}}\leq \big(K\Lambda\varkappa_{\ast}\big)^{c_{t}} \cdot n^{-1/c_t}$.
	\end{enumerate}
\end{lemma}
\begin{proof}
	\noindent (1). By Proposition \ref{prop:tau_rho_est_err}, we only need to verify the claim for $\E^{(0)}\big(n^{-1/2}\pnorm{U^{(t)}}{}\big)^{\mathfrak{r}}$. By the definition (\ref{def:U_debiased}), we have
	\begin{align}\label{ineq:U_apr_err_est_1}
	\frac{\pnorm{U^{(t)}}{}}{\sqrt{n}}&\leq \big(K\Lambda (1+\pnorm{A}{\op})\big)^{c_{t}}\cdot  \pnorm{\bm{\rho}^{[t]}}{\infty}\nonumber\\
	&\qquad \times \max_{s \in [1:t]}\bigg( \pnorm{W_1^{(t)}}{}+ \frac{1}{\sqrt{n}}\bigpnorm{\mathsf{P}_{X,\bm{W}^{(s-1)}}^{(1:L]}\big(G_{\bm{W}^{(s-1)};L}(X)-Y\big)}{} \bigg).
	\end{align}
	Using that for any $s \in [1:t]$,
	\begin{align*}
	\mathsf{P}_{X,\bm{W}^{(s-1)}}^{(1:L]}\big(G_{\bm{W}^{(s-1)};L}(X)-Y\big) = \mathscr{P}_{W^{(s-1)}_{[2:L]}}^{(1:L]}\Big(XW_1^{(s-1)}, \mathscr{R}_{W^{(s-1)}_{[2:L]}}\big(XW_1^{(s-1)},X\mu_{\ast,[q]}\big)\Big),
	\end{align*}
	by the apriori estimate in Proposition \ref{prop:apr_est_det_fcn},
	\begin{align}\label{ineq:U_apr_err_est_2}
	&\frac{1}{\sqrt{n}}\bigpnorm{\mathsf{P}_{X,\bm{W}^{(s-1)}}^{(1:L]}\big(G_{\bm{W}^{(s-1)};L}(X)-Y\big)}{}\nonumber\\
	&\leq (K\Lambda)^{c_t}\cdot \pnorm{W^{(s-1)}_{[2:L]}}{\infty}^{c_0}\cdot \frac{1}{\sqrt{n}} \bigpnorm{\mathscr{R}_{W^{(s-1)}_{[2:L]}}\big(XW_1^{(s-1)},X\mu_{\ast,[q]}\big)}{}\nonumber\\
	&\leq \big(K\Lambda \varkappa_{\ast}(1+\pnorm{A}{\op}+\pnorm{X\mu_\ast}{\infty})\big)^{c_{t}}\cdot \pnorm{W^{(s-1)}_{[2:L]}}{\infty}^{c_0}\cdot \big(1+\pnorm{W^{(s-1)}_1}{}\big).
	\end{align}
	Combining (\ref{ineq:U_apr_err_est_1}) and (\ref{ineq:U_apr_err_est_2}) and the apriori estimate in Proposition \ref{prop:apr_est_se}, 
	\begin{align*}
	\frac{\pnorm{U^{(t)}}{}}{\sqrt{n}}&\leq \big(K\Lambda\varkappa_{\ast} (1+\pnorm{A}{\op}+\pnorm{X\mu_\ast}{\infty})\big)^{c_{t}}\cdot \max_{s \in [1:t]}\Big(1+\pnorm{W^{(s-1)}_{[2:L]}}{\infty}+\pnorm{W^{(s-1)}_1}{}\Big)^{c_0}.
	\end{align*}
	The claimed estimate now follows from the apriori $\ell_2$ control in Proposition \ref{prop:GD_l2_est}.
	
	\noindent (2). By the definition of $U^{(t)}$ and $\hat{U}^{(t)}$,
	\begin{align*}
	\frac{\pnorm{U^{(t)}-\hat{U}^{(t)} }{}}{\sqrt{n}}&\leq \max_{s \in [1:t]} \frac{(K\Lambda)^{c_t}}{\sqrt{n}}\bigpnorm{\mathsf{P}_{X,\bm{W}^{(s-1)}}^{(1:L]}\big(G_{\bm{W}^{(s-1)};L}(X)-Y\big)}{}\cdot \pnorm{\bm{\rho}^{[t]}- \hat{\bm{\rho}}^{[t]}}{\infty}.
	\end{align*}
	Now we may use the estimate (\ref{ineq:U_apr_err_est_2}) along with Proposition \ref{prop:tau_rho_est_err} to conclude. 
\end{proof}

\begin{proof}[Proof of Theorem \ref{thm:est_test_err}]
	We define the intermediate quantity
	\begin{align}\label{ineq:est_test_err_0}
	\tilde{\mathscr{E}}_{\texttt{test}}^{(t)}(X,Y)\equiv 
	\frac{1}{m} \bigpnorm{Y_{[q]}- \mathscr{G}_{V_{[2:L]}^{(t)}}(U^{(t)})}{ }^2.
	\end{align}
	
	\noindent (\textbf{Step 1}). In this step, we shall prove that for any $\mathfrak{r}\geq 1$,
	\begin{align}\label{ineq:est_test_err_step1}
	\E^{(0)}\bigabs{\tilde{\mathscr{E}}_{\texttt{test}}^{(t)}(X,Y)-\hat{\mathscr{E}}_{\texttt{test}}^{(t)}(X,Y)}^{\mathfrak{r}}\leq \big(K\Lambda\varkappa_{\ast}\log n\big)^{c_{t}} \cdot n^{-1/c_t}.
	\end{align}
	To this end, using the apriori and stability estimates in Proposition \ref{prop:apr_est_det_fcn}-(1)(2), and the apriori estimate in Proposition \ref{prop:apr_est_se},
	\begin{align*}
	&\bigabs{\tilde{\mathscr{E}}_{\texttt{test}}^{(t)}(X,Y)-\hat{\mathscr{E}}_{\texttt{test}}^{(t)}(X,Y)}\\
	&\leq \frac{1}{m}\cdot\Big(2\pnorm{Y}{}+\bigpnorm{\mathscr{G}_{V_{[2:L]}^{(t)}}(U^{(t)})}{}+\bigpnorm{\mathscr{G}_{W_{[2:L]}^{(t)}}(\hat{U}^{(t)})}{}\Big)\cdot \bigpnorm{\mathscr{G}_{V_{[2:L]}^{(t)}}(U^{(t)})-\mathscr{G}_{W_{[2:L]}^{(t)}}(\hat{U}^{(t)})}{}\\
	&\leq \big(K\Lambda\varkappa_{\ast}\log n\cdot (1+\pnorm{X\mu_\ast}{\infty})\big)^{c_{t}}\cdot \big(1+\pnorm{V^{(t)}_{[2:L]}}{\infty}+\pnorm{ W^{(t)}_{[2:L]}}{\infty}\big)^{c_t}\\
	&\qquad \times \big(1+m^{-1/2}\pnorm{U^{(t)}}{}+m^{-1/2}\pnorm{\hat{U}^{(t)}}{}\big)^{c_0}\cdot  \big(m^{-1/2}\pnorm{U^{(t)}-\hat{U}^{(t)}}{}+\pnorm{V^{(t)}_{[2:L]}- W^{(t)}_{[2:L]}}{\infty}\big).
	\end{align*}
	The claimed estimate (\ref{ineq:est_test_err_step1}) now follows by an application of the state evolution Theorem \ref{thm:se_gd_NN}, the apriori estimate in Proposition \ref{prop:apr_est_se}, the $\ell_2$ estimate in Proposition \ref{prop:GD_l2_est} and Lemma \ref{lem:U_apr_err_est}.

	\noindent (\textbf{Step 2}). In this step, we shall prove that for any $\mathfrak{r}\geq 1$,
	\begin{align}\label{ineq:est_test_err_step2}
	\E^{(0)}\bigabs{\tilde{\mathscr{E}}_{\texttt{test}}^{(t)}(X,Y)-\E^{(0)} \pnorm{\mathscr{R}_{V^{(t)};\pi_m}\big(\mathfrak{U}^{(t+1)},\mathfrak{U}^{(0)} \big)}{}^2 }^{\mathfrak{r}}\leq \big(K\Lambda\varkappa_{\ast}\log n\big)^{c_{t}} \cdot n^{-1/c_t}.
	\end{align}
	To see this, for $k \in [m]$, with the test function $\tilde{\psi}_k:\R^q\times \R^q\to \R$ defined as
	\begin{align*}
	\tilde{\psi}_k(u^{(0)},u^{(t)})\equiv \bigpnorm{\varphi_\ast\big(\mathscr{T}_{\varkappa_{\ast}\log n}(u^{(0)})\big)+\xi_{k;[q]}- \mathscr{G}_{V_{[2:L]}^{(t)};k }(u^{(t)})}{ }^2,
	\end{align*}
	we have on an event $E_0$ with $\Prob^{(0)}(E_0^c)\leq n^{-100}$,
	\begin{align*}
	\tilde{\mathscr{E}}_{\texttt{test}}^{(t)}(X,Y)\equiv 
	\frac{1}{m} \sum_{k \in [m]} \tilde{\psi}_k\big((X\mu_{\ast,[q]}))_{k\cdot},U^{(t)}_{k\cdot}\big).
	\end{align*}
	Moreover, the apriori and stability estimates in Proposition \ref{prop:apr_est_det_fcn}-(1)(2) along with the apriori estimate in Proposition \ref{prop:apr_est_se} show that $\{\tilde{\psi}_k\}_{k \in [m]}$ are $\big(K\Lambda\varkappa_{\ast}\log n\big)^{c_{t}}$-pseudo-Lipschitz of order $2$. Consequently we may conclude (\ref{ineq:est_test_err_step2}) by (i) applying the state evolution Theorem \ref{thm:se_gd_NN} and (i) using a routine argument to remove the effect of $E_0$ via apriori estimates.

	\noindent (\textbf{Step 3}). Finally, the claimed estimate with $\#=\E^{(0)} \pnorm{\mathscr{R}_{V^{(t)};\pi_m}\big(\mathfrak{U}^{(t+1)},\mathfrak{U}^{(0)} \big)}{}^2$ follows by combining (\ref{ineq:est_test_err_step1}) and (\ref{ineq:est_test_err_step2}). The other claim follows by a further application of Theorem \ref{thm:err_gd_NN}.
\end{proof}

\section{Proofs for Section \ref{section:gaussian_representation_f_W}}\label{section:proof_gaussian_representation_f_W}

\subsection{Proof of Theorem \ref{thm:represent_learning}}\label{subsection:proof_represent_learning}

We shall first prove the following marginal distributional characterization of the first-layer weight $W_1^{(t)}$.

\begin{proposition}\label{prop:feature_learning_grad}
	Fix $t\in \N$. Suppose Assumption \ref{assump:gd_NN} holds for some $K,\Lambda\geq 2$ and $r_0\geq 0$. Fix a sequence of $\Lambda$-pseudo-Lipschitz functions $\{\phi_\ell:\R^{q} \to \R\}_{\ell \in [n]}$ of order $2$. Then for any $\mathfrak{r}\geq 1$, there exists some $c_t=c_t(t,q,L,\mathfrak{r},r_0)>0$ such that 
	\begin{align*}
	&\E^{(0)}  \bigg[\biggabs{\frac{1}{n}\sum_{\ell \in [n]} \phi_\ell\big(n^{1/2}W_{1;\ell\cdot}^{(t)}\big) - \frac{1}{n}\sum_{\ell \in [n]}  \E^{(0)}  \phi_\ell\big(\mathfrak{b}_{\ast,\ell}^{(t)}+\Omega_{t,t}^{1/2}\mathsf{Z}_q\big)   }^{\mathfrak{r}}\bigg] \leq  \big(K\Lambda\varkappa_{\ast}\big)^{c_{t}}\cdot n^{-1/c_t}.
	\end{align*}
	Here
	\begin{align}\label{eqn:feature_learning_grad_1}
	\mathfrak{b}_{\ast,\ell}^{(t)}\equiv n^{1/2}(\mu_\ast)_\ell\cdot \mathfrak{m}_W^{(t)} +   \mathfrak{M}_W^{(t)} \big(n^{1/2}W_1^{(0),\top} e_\ell\big) \in \R^q.
	\end{align}
\end{proposition}

\begin{proof}
	By the definition of $\Delta_{\ell}^{(t)}(\mathfrak{V}^{([1:t])} )$ in (\ref{def:Delta}), we have
	\begin{align}\label{ineq:feature_learning_grad_dyn_1}
	\frac{1}{n}\sum_{\ell \in [n]} \phi_\ell\big(n^{1/2}W_{1;\ell\cdot}^{(t)}\big) \approx \frac{1}{n}\sum_{\ell \in [n]}  \E^{(0)}  \phi_\ell\big(D_t^{\top}e_\ell+\Omega_{t,t}^{1/2}\mathsf{Z}_q\big),
	\end{align}
	where $\approx$ stands for the error term (in arbitrary moments) as in Theorem \ref{thm:se_gd_NN}.
	
	Now we shall compute $D_t^{\top}e_\ell$. Using the recursive definition of $D_t$ in (S4), as $\delta_t D_{-1}^\top e_\ell = n^{1/2}(\mu_\ast)_\ell\cdot \delta_t e_1$, we have for $t\geq 1$,
	\begin{align}\label{ineq:feature_learning_grad_dyn_2}
	D_t^\top e_\ell & = n^{1/2}(\mu_\ast)_\ell\cdot \delta_t e_1+ \sum\nolimits_{r \in [1:t]}  \big(\tau_{t,r}+I_q\bm{1}_{r=t}\big) D_{r-1}^\top e_\ell. 
	\end{align}
	As $D_0^\top e_\ell= n^{1/2} W_1^{(0),\top} e_\ell$, this means $D_t^{\top}e_\ell$ must take the form 
	\begin{align*}
	D_t^{\top}e_\ell = n^{1/2}(\mu_\ast)_\ell\cdot \mathfrak{m}_W^{(t)} +   \mathfrak{M}_W^{(t)} \big(n^{1/2}W_1^{(0),\top} e_\ell\big)
	\end{align*}
	for some $\mathfrak{m}_W^{(t)} \in \R^q$ and $\mathfrak{M}_W^{(t)} \in \mathbb{M}_q$. Moreover, $\mathfrak{m}_W^{(0)}=0_q$ and $\mathfrak{M}_W^{(0)}=I_q$. Using (\ref{ineq:feature_learning_grad_dyn_2}), we may obtain the recursive formula for $\mathfrak{m}_W^{(t)}$ and $\mathfrak{M}_W^{(t)}$ in (\ref{eqn:feature_learning_grad_2}). 
\end{proof}

\begin{proof}[Proof of Theorem \ref{thm:represent_learning}]
	We note that $f_{\bm{W}^{(t)}}(x) =\mathfrak{h}_{W^{(t)}_{[2:L]}}\big(W_1^{(t),\top} x\big)$. 
	
	\noindent (\textbf{Step 1}). In this step, we show that
	\begin{align}\label{ineq:represent_learning_step1}
	\E^{(0)}\Big[\mathfrak{d}_{\mathrm{BL}}^{(X,Y)}\Big(f_{\bm{W}^{(t)}}(\mathsf{Z}_n),\mathfrak{h}_{V^{(t)}_{[2:L]}}\big(W_1^{(t),\top} \mathsf{Z}_n\big)\Big)\Big]^{\mathfrak{r}} \leq \big(K\Lambda\varkappa_{\ast}\big)^{c_t}\cdot n^{-1/c_t}.
	\end{align}
	Here $\mathsf{Z}_n\sim \mathcal{N}(0,I_n)$ is independent of all other random variables.
	
	To prove (\ref{ineq:represent_learning_step1}), note that by the apriori estimates in Propositions \ref{prop:apr_est_det_fcn} and \ref{prop:apr_est_se} along with Theorem \ref{thm:se_gd_NN},
		\begin{align*}
		&\E^{(0)}\bigpnorm{\mathfrak{h}_{W^{(t)}_{[2:L]}}\big(W_1^{(t),\top} \mathsf{Z}_n\big)- \mathfrak{h}_{V^{(t)}_{[2:L]}}\big(W_1^{(t),\top} \mathsf{Z}_n\big)}{}^{\mathfrak{r}}\\
		&\leq \E^{(0)} \Big\{\Big[K\Lambda\varkappa_{\ast}\cdot \max_{\alpha \in [2:L]} \big(1+\pnorm{W_\alpha^{(t)}}{\infty}\big) \Big]^{c_t}\cdot \big(1+\pnorm{W_1^{(t),\top} \mathsf{Z}_n}{\infty}\big)^{\mathfrak{r}}\cdot \max_{\alpha \in [2:L]}\pnorm{W_\alpha^{(s)}-V^{(s)}_\alpha}{\infty}^{\mathfrak{r}}\Big\}\\
		&\leq \big(K\Lambda\varkappa_{\ast}\log n \big)^{c_t}\cdot n^{-1/c_t}.
		\end{align*}
	The claimed estimate (\ref{ineq:represent_learning_step1}) now follows from the above display upon noting that by Jensen's inequality, for $\mathfrak{r}\geq 1$, 
	\begin{align*}
	&\E^{(0)}\Big[\mathfrak{d}_{\mathrm{BL}}^{(X,Y)}\Big(f_{\bm{W}^{(t)}}(\mathsf{Z}_n),\mathfrak{h}_{V^{(t)}_{[2:L]}}\big(W_1^{(t),\top} \mathsf{Z}_n\big)\Big)\Big]^{\mathfrak{r}} \leq \E^{(0)} \bigpnorm{f_{\bm{W}^{(t)}}(\mathsf{Z}_n)- \mathfrak{h}_{V^{(t)}_{[2:L]}}\big(W_1^{(t),\top} \mathsf{Z}_n\big)}{}^{\mathfrak{r}}.
	\end{align*}

	\noindent (\textbf{Step 2}). In this step, we show that with
	\begin{align}\label{ineq:represent_learning_Sigma}
	\Sigma_{W_1^{(t)}}&\equiv \pnorm{\mu_\ast}{}^2\cdot  \mathfrak{m}_W^{(t)} \mathfrak{m}_W^{(t),\top}+ \mathfrak{M}_W^{(t)}W_1^{(0),\top}W_1^{(0)} \mathfrak{M}_W^{(t),\top} \nonumber\\
	&\qquad  +\mathfrak{m}_W^{(t)}\mu_\ast^\top W_1^{(0)}\mathfrak{M}_W^{(t),\top}+ \mathfrak{M}_W^{(t)}W_1^{(0),\top}\mu_\ast \mathfrak{m}_W^{(t),\top}+ \Omega_{t,t} \in \mathbb{M}_q,
	\end{align}
	we have 
	\begin{align}\label{ineq:represent_learning_step2}
	\E^{(0)}\pnorm{W_1^{(t),\top} W_1^{(t)}- \Sigma_{W_1^{(t)}} }{\infty}^{\mathfrak{r}}\leq \big(K\Lambda\varkappa_{\ast}\big)^{c_{t}}\cdot n^{-1/c_t}.
	\end{align}
	To see (\ref{ineq:represent_learning_step2}), for fixed $d,d' \in [q]$, let  $\phi(w_1,\ldots,w_q)\equiv w_d w_{d'}$. Clearly $\phi$ is $c_0$-pseudo-Lipschitz of order $2$. Then  $n^{-1}\sum_{\ell \in [n]} \phi\big(n^{1/2}W_{1;\ell\cdot}^{(t)}\big)=e_d^\top W_1^{(t),\top} W_1^{(t)} e_{d'}$, and 
	\begin{align*}
	&\frac{1}{n}\sum_{\ell \in [n]}  \E^{(0)}  \phi\big(\mathfrak{b}_{\ast,\ell}^{(t)}+\Omega_{t,t}^{1/2}\mathsf{Z}_q\big)\\
	&= \frac{1}{n}\sum_{\ell \in [n]} e_d^\top \mathfrak{b}_{\ast,\ell}^{(t)} \mathfrak{b}_{\ast,\ell}^{(t),\top } e_{d'}+ \E^{(0)} \big(e_{d}^\top \Omega_{t,t}^{1/2}\mathsf{Z}_q\big)\big(\mathsf{Z}_q^\top \Omega_{t,t}^{1/2} e_{d'} \big)\\
	& =  \Big(\pnorm{\mu_\ast}{}^2\cdot  \mathfrak{m}_W^{(t)} \mathfrak{m}_W^{(t),\top}+ \mathfrak{M}_W^{(t)}W_1^{(0),\top}W_1^{(0)} \mathfrak{M}_W^{(t),\top} \\
	&\qquad  +\mathfrak{m}_W^{(t)}\mu_\ast^\top W_1^{(0)}\mathfrak{M}_W^{(t),\top}+ \mathfrak{M}_W^{(t)}W_1^{(0),\top}\mu_\ast \mathfrak{m}_W^{(t),\top}+ \Omega_{t,t}\Big)_{dd'}=\big(\Sigma_{W_1^{(t)}}\big)_{dd'}.
	\end{align*}
Now we may use Proposition \ref{prop:feature_learning_grad} to conclude the desired estimate (\ref{ineq:represent_learning_step2}). 
	
	\noindent (\textbf{Step 3}). To see the desired representation, first note that for any bounded, Lipschitz test function $\psi:\R^q\to \R$,
	\begin{align}\label{ineq:represent_learning_step3_1}
	\E_{\mathsf{Z}_n} \psi\Big( \mathfrak{h}_{V^{(t)}_{[2:L]}}\big( W_1^{(t),\top} \mathsf{Z}_n\big)\Big) = \E_{\mathsf{Z}_q} \psi\Big( \mathfrak{h}_{V^{(t)}_{[2:L]}}\big((W_1^{(t),\top} W_1^{(t)})^{1/2}\mathsf{Z}_q\big)\Big).
	\end{align}
	On the other hand, using the stability estimate in Proposition \ref{prop:apr_est_det_fcn} and the apriori estimates in Proposition \ref{prop:apr_est_se},
	\begin{align}\label{ineq:represent_learning_step3_2}
	&\bigabs{ \E_{\mathsf{Z}_q} \psi\Big( \mathfrak{h}_{V^{(t)}_{[2:L]}}\big((W_1^{(t),\top} W_1^{(t)})^{1/2}\mathsf{Z}_q\big)\Big)- \E_{\mathsf{Z}_q} \psi\Big( \mathfrak{h}_{V^{(t)}_{[2:L]}}\big(\Sigma_{W_1^{(t)}}^{1/2}\mathsf{Z}_q\big)\Big) }\nonumber\\
	&\leq  \big(K\Lambda\varkappa_{\ast}\big)^{c_t}\cdot \big(1+ \pnorm{W_{1}^{(t)} }{}\big)\cdot \E_{\mathsf{Z}_q} (1+\pnorm{\mathsf{Z}_q}{})\cdot \bigpnorm{\big((W_1^{(t),\top} W_1^{(t)})^{1/2}- \Sigma_{W_1^{(t)}}^{1/2}\big) \mathsf{Z}_q}{}\nonumber\\
	&\leq \big(K\Lambda\varkappa_{\ast}\big)^{c_t}\cdot \big(1+ \pnorm{W_{1}^{(t)} }{}\big)\cdot\bigpnorm{W_1^{(t),\top} W_1^{(t)}- \Sigma_{W_1^{(t)}}}{\infty}^{1/2}.
	\end{align}
	Here in the last inequality we used the square-root estimate in \cite[Lemma A.3]{bao2023leave}. Now applying the apriori estimate in Proposition \ref{prop:GD_l2_est} and (\ref{ineq:represent_learning_step2}) to (\ref{ineq:represent_learning_step3_2}), we conclude in view of (\ref{ineq:represent_learning_step3_1}) that 
	\begin{align}\label{ineq:represent_learning_step3_3}
	&\E^{(0)} \bigabs{ \E_{\mathsf{Z}_n} \psi\Big( \mathfrak{h}_{V^{(t)}_{[2:L]}}\big( W_1^{(t),\top} \mathsf{Z}_n\big)\Big) - \E_{\mathsf{Z}_q} \psi\Big( \mathfrak{h}_{V^{(t)}_{[2:L]}}\big(\Sigma_{W_1^{(t)}}^{1/2}\mathsf{Z}_q\big)\Big) }^{\mathfrak{r}}\leq \big(K\Lambda\varkappa_{\ast}\big)^{c_{t}} n^{-1/c_t}.
	\end{align}
	Recall $U_{\ast,\texttt{eff}}^{(t)}$ defined in (\ref{def:effective_signal}). Note that
	\begin{align}\label{ineq:represent_learning_Sigma_tilde}
	&U_{\ast,\texttt{eff}}^{(t),\top} U_{\ast,\texttt{eff}}^{(t)}+\Omega_{t,t} \nonumber\\
	&=\big( \mathfrak{m}_W^{(t)}\mu_\ast^\top +\mathfrak{M}_W^{(t)}W_1^{(0),\top}\big)  \big(\mu_\ast \mathfrak{m}_W^{(t),\top}+W_1^{(0)}\mathfrak{M}_W^{(t),\top}\big)+\Omega_{t,t} = \Sigma_{W_1^{(t)}} \in \mathbb{M}_q.
	\end{align}
    This means
	\begin{align}\label{ineq:represent_learning_step3_4}
	\E_{\mathsf{Z}_q} \psi\Big( \mathfrak{h}_{V^{(t)}_{[2:L]}}\big(\Sigma_{W_1^{(t)}}^{1/2}\mathsf{Z}_q\big)\Big)  =\E_{(x,\mathsf{Z}_q)} \psi\Big( \mathfrak{h}_{V^{(t)}_{[2:L]}}\big(U_{\ast,\texttt{eff}}^{(t),\top} x+ \Omega_{t,t}^{1/2}\mathsf{Z}_{q}\big)\Big).
	\end{align}
	The claimed estimate now follows by combining (\ref{ineq:represent_learning_step1}) in Step 1, and (\ref{ineq:represent_learning_step3_3}) and (\ref{ineq:represent_learning_step3_4}) above.
\end{proof}

\subsection{Proof of Theorem \ref{thm:represent_learning_large_sample}}\label{subsection:proof_represent_learning_large_sample}

\begin{lemma}\label{lem:diff_m_M_bar}
	Fix $t\in \N$. Suppose $\phi^{-1}\leq K$ and (A3)-(A5) in Assumption \ref{assump:gd_NN} hold for some $K,\Lambda\geq 2$ and $r_0\geq 0$. Then there exists some $c_t=c_t(t,q,L,r_0)>0$ such that  
	\begin{align*}
	\pnorm{\bar{\mathfrak{m}}_W^{(t)} -{\mathfrak{m}}_W^{(t)}}{}\vee \pnorm{\bar{\mathfrak{M}}_W^{(t)}-\mathfrak{M}_W^{(t)}}{}\leq \big(K\Lambda\varkappa_{\ast}\big)^{c_t}\cdot \phi^{-1}. 
	\end{align*}
\end{lemma}
\begin{proof}
	Let $\Delta \bar{\mathfrak{m}}_W^{(t)} \equiv \bar{\mathfrak{m}}_W^{(t)} -{\mathfrak{m}}_W^{(t)}$, $\Delta \bar{\mathfrak{M}}_W^{(t)}=\bar{\mathfrak{M}}_W^{(t)}-{\mathfrak{M}}_W^{(t)}$, and $\Delta \bar{\delta}_t\equiv  \bar{\delta}_t-\delta_t$. Comparing the definitions in (\ref{eqn:feature_learning_grad_2}) and (\ref{eqn:feature_learning_grad_large_sample}),
	\begin{align}\label{ineq:diff_m_M_bar_1}
	\pnorm{\Delta \bar{\mathfrak{m}}_W^{(t)}}{} &\leq  \pnorm{\Delta \bar{\delta}_t e_1}{}+ \pnorm{\bar{\tau}_{t} \bar{\mathfrak{m}}_W^{(t-1)}-\tau_{t,t} \mathfrak{m}_W^{(t-1)}  }{} \nonumber\\
	&\qquad + \pnorm{\Delta \bar{\mathfrak{m}}_W^{(t-1)}}{}  + t \max_{s \in [1:t-1]} \pnorm{\tau_{t,s}}{\op} \pnorm{\mathfrak{m}_W^{(s-1)} }{}.
	\end{align}
	It is easy to prove, by the apriori estimates in Proposition \ref{prop:apr_est_se} and (\ref{ineq:se_gd_NN_large_sample_step1}), that
	\begin{align}\label{ineq:diff_m_M_bar_2}
	\pnorm{\mathfrak{m}_W^{(t)} }{}\vee \pnorm{\bar{\mathfrak{m}}_W^{(t)} }{} \leq \big(K\Lambda\varkappa_{\ast}\big)^{c_t}.
	\end{align}
	Combining (\ref{ineq:diff_m_M_bar_1}), (\ref{ineq:diff_m_M_bar_2}) the stability estimate in (\ref{ineq:se_gd_NN_large_sample_step2}) and the second estimate in Proposition \ref{prop:apr_est_se}, we have
	\begin{align*}
	\pnorm{\Delta \bar{\mathfrak{m}}_W^{(t)}}{} &\leq \big(K\Lambda\varkappa_{\ast}\big)^{c_t}\cdot \big(\phi^{-1}+\pnorm{\Delta \bar{\mathfrak{m}}_W^{(t-1)}}{} \big).
	\end{align*} 
	Iterating the above display concludes the desired estimate for $\pnorm{\Delta \bar{\mathfrak{m}}_W^{(t)}}{} $. The estimate for $\pnorm{\Delta \bar{\mathfrak{M}}_W^{(t)}}{}$ can be obtained in a similar fashion. 
\end{proof}

\begin{proof}[Proof of Theorem \ref{thm:represent_learning_large_sample}]
	Fix a bounded Lipschitz function $\psi:\R^q\to \R$. With $(x,\mathsf{Z}_q)\sim \mathcal{N}(0,I_n)\otimes \mathcal{N}(0,I_q)$, by the apriori estimates in Proposition \ref{prop:apr_est_det_fcn}, Proposition \ref{prop:apr_est_se} and (\ref{ineq:se_gd_NN_large_sample_step1}),
	\begin{align*}
	&\bigabs{\E_{(x,\mathsf{Z}_q)} \psi\Big( \mathfrak{h}_{V^{(t)}_{[2:L]}}\big(U_{\ast,\texttt{eff}}^{(t),\top} x+ \Omega_{t,t}^{1/2}\mathsf{Z}_{q}\big)\Big)- \E_x \psi \Big( \mathfrak{h}_{\bar{W}^{(t)}_{[2:L]}}\big(\bar{U}_{\ast,\texttt{eff}}^{(t),\top} x\big) \Big)}\\
	&\leq \big(K\Lambda\varkappa_{\ast}\big)^{c_t}\cdot \E_{(x,\mathsf{Z}_q)} \big(1+ \pnorm{U_{\ast,\texttt{eff}}^{(t),\top} x}{}+\pnorm{\bar{U}_{\ast,\texttt{eff}}^{(t),\top} x}{}+ \pnorm{\mathsf{Z}_q}{} \big)\\
	&\qquad \times \Big[\max_{\alpha \in [2:L]}\pnorm{\bar{W}^{(t)}_{\alpha}-V^{(t)}_{\alpha}}{}+ \bigpnorm{(U_{\ast,\texttt{eff}}^{(t)}-\bar{U}_{\ast,\texttt{eff}}^{(t)})^\top x}{}+\pnorm{\Omega_{t,t}}{\op}^{1/2}\cdot \pnorm{\mathsf{Z}_q}{}\Big].
	\end{align*}
	Now for the terms in the bracket, by using (i) the stability estimate in (\ref{ineq:se_gd_NN_large_sample_step2}) for the first term, (ii) Lemma \ref{lem:diff_m_M_bar} for the second term, and (iii) the second estimate in Proposition \ref{prop:apr_est_se} for the third term, we have
	\begin{align*}
	&\bigabs{\E_{(x,\mathsf{Z}_q)} \psi\Big( \mathfrak{h}_{V^{(t)}_{[2:L]}}\big(U_{\ast,\texttt{eff}}^{(t),\top} x+ \Omega_{t,t}^{1/2}\mathsf{Z}_{q}\big)\Big)- \E_x \psi \Big( \mathfrak{h}_{\bar{W}^{(t)}_{[2:L]}}\big(\bar{U}_{\ast,\texttt{eff}}^{(t),\top} x\big) \Big)}\leq \big(K\Lambda\varkappa_{\ast}\big)^{c_t}\cdot  \phi^{-1}. 
	\end{align*}
	The claimed estimate follows. 
\end{proof}

\section{Remaining proofs}\label{section:remaining_proof}

\subsection{Proof of Proposition \ref{prop:grad_formula}}\label{subsection:proof_grad_formula}

	In the proof we shall write $G_{\alpha}\equiv G_{\bm{W};\alpha}(X)$, $G_{\alpha}'\equiv G_{\bm{W};\alpha}'(X)$ and $\mathsf{P}_{X,\bm{W};\alpha}\equiv \mathsf{P}_{\alpha}$. Fix  $i \in [n\bm{1}_{\alpha=1}+q\bm{1}_{\alpha \in [2:L]}], j \in [q]$ and let $\partial_{ij}\equiv \partial_{(W_\alpha)_{ij}}$. For $\beta \in (\alpha:L]$, by using the definition of $G_\beta$, for any $\epsilon_\beta \in \R^{m\times q_\beta}$, 
	\begin{align*}
	\bigiprod{\epsilon_\beta}{ \partial_{ij} G_\beta} & = \sum_{k \in [m], \ell \in [q]} (\epsilon_\beta)_{k\ell}\cdot  \Big(G_\beta'\odot \big[\partial_{ij} G_{\beta-1} W_\beta\big]\Big)_{k\ell} \\
	&=  \sum_{k \in [m], \ell \in [q], r \in [q]} (\epsilon_\beta)_{k\ell} (G_\beta')_{k\ell} \cdot  (\partial_{ij}G_{\beta-1})_{k r} (W_\beta)_{r \ell}\\
	& = \sum_{k \in [m], r \in [q]} \big((\epsilon_\beta\odot G_\beta')W_\beta^\top \big)_{kr}  (\partial_{ij}G_{\beta-1})_{kr} =  \bigiprod{\mathsf{P}_\beta(\epsilon_\beta)}{ \partial_{ij} G_{\beta-1} }.
	\end{align*}
	For $\alpha \in [1:L-1]$, iterating the above display  from $\beta=L$ with $\epsilon_L\equiv G_{\bm{W};L}(X)-Y_{[q]}$ until $\beta=\alpha+1$, by further using $\partial_{ij}G_\alpha=G_\alpha'\odot (G_{\alpha-1} e_i e_j^\top) $, we have
	\begin{align*}
	m\cdot \frac{\partial \mathsf{L}}{\partial (W_\alpha)_{ij}} (\bm{W}) &=\bigiprod{\epsilon_L}{\partial_{ij} G_L} = \cdots = \bigiprod{\mathsf{P}^{(\alpha:L]} (\epsilon_L) }{\partial_{ij} G_\alpha}\\
	& = \sum_{k \in [m], r \in [q_\alpha]} \mathsf{P}_{kr}^{(\alpha:L]} (\epsilon_L) (G_\alpha')_{kr}  (G_{\alpha-1} e_i e_j^\top)_{kr}\\
	& = \sum_{k \in [m]} (G_{\alpha-1}^\top)_{i k} \mathsf{P}_{k j}^{(\alpha:L]} (\epsilon_L) (G_\alpha')_{k j} = \Big(G_{\alpha-1}^\top \big[ \mathsf{P}^{(\alpha:L]} (\epsilon_L) \odot G_\alpha' \big]\Big)_{ij},
	\end{align*}
	proving the desired gradient formula for $\alpha \in [1:L-1]$. The case $\alpha = L$ is trivial. \qed

\subsection{Proof of Proposition \ref{prop:tau_repres}}\label{subsection:proof_tau_repres}

	Note that $\Upsilon^{(t)}:\R^{m\times q[0:t]}\to \R^{m\times q}$ in (S1) can be defined alternatively via
	\begin{align*}
	\Upsilon^{(t)}(\mathfrak{u}^{([0:t])})\equiv \phi^{-1} \eta_1^{(t-1)} \cdot \mathfrak{S}\bigg(\mathfrak{u}^{(t)}-\sum_{s \in [1:t-1]} \Upsilon^{(s)}(\mathfrak{u}^{([0:s])})\rho_{t-1,s}^\top ,\mathfrak{u}^{(0)},V^{(t-1)}\bigg).
	\end{align*}
	Consequently, for any $k \in [m]$, we may identify $\Upsilon_{k}^{(t)}:\R^{q[0:t]}\to \R^q$ as
	\begin{align*}
	\Upsilon^{(t)}_{k}(u^{([0:t])})\equiv \phi^{-1} \eta_1^{(t-1)} \cdot \mathfrak{S}_{k}\bigg(u^{(t)}-\sum_{s \in [1:t-1]} \rho_{t-1,s} \Upsilon^{(s)}_{k}(u^{([0:s])}) ,u^{(0)},V^{(t-1)}\bigg).
	\end{align*}
	This means for $s \in [1:t]$ and $r \in [q]$,
	\begin{align*}
	\partial_{u^{(s)}_r} \Upsilon^{(t)}_{k}(u^{([0:t])})&\equiv \phi^{-1} \eta_1^{(t-1)} \cdot \sum_{\ell \in [q]} \bigg[ \partial_\ell^{(1)} \mathfrak{S}_{k}\big(\Theta^{(t)}_{k}(u^{([0:t])}),u^{(0)},V^{(t-1)}\big)\\
	&\qquad \times \bigg(\delta_{s,t} \delta_{r,\ell}- \sum_{\tau \in [1:t-1]} e_\ell^\top \rho_{t-1,\tau} \partial_{u^{(s)}_r}  \Upsilon^{(\tau)}_{k}(u^{([0:\tau])}) \bigg)  \bigg]\in \R^q.
	\end{align*}
	Equivalently,
	\begin{align*}
	&\partial_{u^{(s)}_r} \Upsilon^{(t)}_{k}(u^{([0:t])})= \phi^{-1} \eta_1^{(t-1)}\cdot \bigg\{ \big[(\partial_{u_{kr}} \mathfrak{S})^\top e_k\big]\big(\Theta^{(t)}_{k}(u^{([0:t])}),u^{(0)},V^{(t-1)}\big) \delta_{s,t}\\
	&\qquad  - \sum_{\ell \in [q]} \big[(\partial_{u_{k\ell}} \mathfrak{S})^\top e_k\big] \big(\Theta^{(t)}_{k}(u^{([0:t])}),u^{(0)},V^{(t-1)}\big) e_\ell^\top  \sum_{\tau \in [1:t-1]} \rho_{t-1,\tau} \partial_{u^{(s)}_r}  \Upsilon^{(\tau)}_{k}(u^{([0:\tau])})\bigg\}.
	\end{align*}
	Consequently, we may represent 
	$\nabla \bm{\Upsilon}^{[t]}_{k}(u^{([0:t])})\equiv \big(\nabla_{u^{(s)}} \Upsilon^{(\tau)}_{k}(u^{([0:\tau])})\big)_{\tau,s \in [t]}\in (\mathbb{M}_q)^{t\times t}$, where $\nabla_{u^{(s)}} \Upsilon^{(\tau)}_{k}(u^{([0:\tau])})=\sum_{r \in [q]} \partial_{u^{(s)}_r} \Upsilon^{(\tau)}_{k}(u^{([0:\tau])}) e_r^\top \in \mathbb{M}_q$, as follows:
	\begin{align*}
	\phi \cdot \nabla \bm{\Upsilon}^{[t]}_{k}(u^{([0:t])})&=   \bm{\mathfrak{L}}^{(t)}_k\Big(\Big\{\Theta^{(s)}_{k}\big(u^{([0:s])}\big)\Big\}_{s \in [0:t]},V^{([0:t-1])}\Big) \\
	&\quad - \bm{\mathfrak{L}}^{(t)}_k\Big(\Big\{\Theta^{(s)}_{k}\big(u^{([0:s])}\big)\Big\}_{s \in [0:t]},V^{([0:t-1])}\Big)  \mathfrak{O}_{\mathbb{M}_q;t}(\bm{\rho}^{[t-1]}) \nabla \bm{\Upsilon}^{[t]}_{k}(u^{([0:t])}).
	\end{align*}
	This means 
	\begin{align*}
	\nabla \bm{\Upsilon}^{[t]}_{k}(u^{([0:t])}) &= \bigg(\phi\cdot (I_{\mathbb{M}_q})_t+\bm{\mathfrak{L}}^{(t)}_k\Big(\Big\{\Theta^{(s)}_{k}\big(u^{([0:s])}\big)\Big\}_{s \in [0:t]},V^{([0:t-1])}\Big)  \mathfrak{O}_{\mathbb{M}_q;t}(\bm{\rho}^{[t-1]})\bigg)^{-1}\\
	&\qquad \times  \bm{\mathfrak{L}}^{(t)}_k\Big(\Big\{\Theta^{(s)}_{k}\big(u^{([0:s])}\big)\Big\}_{s \in [0:t]},V^{([0:t-1])}\Big).
	\end{align*}
	The claimed formula for $\bm{\tau}^{[t]}$ now follows by the definition of $\tau_{t,s}$ in Definition \ref{def:gd_NN_se}.\qed

\subsection{Proof of Proposition \ref{prop:apr_est_det_fcn}}\label{subsection:proof_apr_est_det_fcn}

	\noindent (1). For $u\in \R^{m\times q}$ and $v_{[2:\alpha]}\in (\mathbb{M}_q)^{[2:\alpha]}$, by definition 
	\begin{align*}
	\pnorm{\mathscr{H}_{v_{[2:\alpha]};k}(u_{k\cdot})}{\infty}\leq q \Lambda \pnorm{v_\alpha}{\infty}\cdot \big(1+\pnorm{\mathscr{H}_{v_{[2:\alpha-1]};k}(u_{k\cdot})}{\infty}\big).
	\end{align*}
	Iterating the estimate and using the initial condition $\pnorm{\mathscr{H}_{v_{[2:1]};k}(u_{k\cdot})}{\infty}=\pnorm{u_{k\cdot}}{\infty}$ to conclude. The estimate for  $\pnorm{\mathscr{G}_{v_{[2:\alpha]};k}(u_{k\cdot})}{\infty}$  now follows trivially.

	\noindent (2). For $u,u' \in \R^{m\times q}$ and $v_{[2:\alpha]},v_{[2:\alpha]}'\in (\mathbb{M}_q)^{[2:\alpha]}$, by definition
	\begin{align*}
	&\pnorm{\mathscr{H}_{v_{[2:\alpha]};k}(u_{k\cdot}) - \mathscr{H}_{v_{[2:\alpha]}';k }(u_{k\cdot}')}{}\leq \bigpnorm{\sigma_{\alpha-1}\big(\mathscr{H}_{v_{[2:\alpha-1]};k}(u_{k\cdot})\big)}{}\cdot \pnorm{v_\alpha-v_\alpha'}{\infty}\\
	&\qquad + \bigpnorm{\sigma_{\alpha-1}\big(\mathscr{H}_{v_{[2:\alpha-1]};k}(u_{k\cdot})\big) - \sigma_{\alpha-1}\big(\mathscr{H}_{v_{[2:\alpha-1]}';k}(u_{k\cdot}')\big) }{}\cdot \pnorm{v_\alpha'}{\infty}\\
	& \leq (q\Lambda)^{c_0 L}\kappa_{v_{[2:\alpha]}} \big(1+\pnorm{u_{k\cdot}}{\infty}\big)\cdot \pnorm{v_\alpha-v_\alpha'}{\infty}+ \Lambda \pnorm{v_\alpha'}{\infty}\cdot \pnorm{\mathscr{H}_{v_{[2:\alpha-1]};k}(u_{k\cdot}) - \mathscr{H}_{v_{[2:\alpha-1]}';k }(u_{k\cdot}')}{}.
	\end{align*}
	Iterating the estimate and using the initial condition $\pnorm{\mathscr{H}_{v_{[2:1]};k}(u_{k\cdot}) - \mathscr{H}_{v_{[2:1]}';k }(u_{k\cdot}')}{}=\pnorm{u_{k\cdot}-u_{k\cdot}'}{}$ to conclude. 
	
	\noindent (3). For $u,u' \in \R^{m\times q},v_{[2:\alpha]}, v'_{[2:\alpha]}\in (\mathbb{M}_q)^{[2:\alpha]}$ and $z,z' \in \R^{m\times q}$, note that
	\begin{align*}
	&\bigpnorm{\mathscr{P}_{u,v_{[2:\alpha]};k}(z_{k\cdot})- \mathscr{P}_{u',v'_{[2:\alpha]};k}(z_{k\cdot}')}{}\\
	& = \bigpnorm{v_\alpha \big(z_{k\cdot}\odot \mathscr{G}_{v_{[2:\alpha]};k }'(u_{k\cdot})\big) -v_\alpha' \big(z_{k\cdot}'\odot \mathscr{G}_{v'_{[2:\alpha]};k }'(u_{k\cdot}')\big)  }{}\\
	&\leq \pnorm{z_{k\cdot}\odot \mathscr{G}_{v_{[2:\alpha]};k }'(u_{k\cdot})}{}\cdot \pnorm{v_\alpha-v_\alpha'}{\infty}+\pnorm{v_\alpha'}{\op}\cdot  \bigpnorm{z_{k\cdot}\odot \mathscr{G}_{v_{[2:\alpha]};k }'(u_{k\cdot})-z_{k\cdot}'\odot \mathscr{G}_{v'_{[2:\alpha]};k }'(u_{k\cdot}')}{}\\
	&\leq \Lambda\cdot  \pnorm{z_{k\cdot}}{} \cdot \pnorm{v_\alpha-v_\alpha'}{\infty}+ \pnorm{v_\alpha'}{\op}\cdot  \Big(\pnorm{z_{k\cdot}}{\infty}  \bigpnorm{\mathscr{G}_{v_{[2:\alpha]};k }'(u_{k\cdot})-\mathscr{G}_{v'_{[2:\alpha]};k }'(u_{k\cdot}')}{} + \Lambda\cdot \pnorm{z_{k\cdot}-z_{k\cdot}'}{}\Big).
	\end{align*}
	Now using the estimates in (2), the right hand side of the above display can be further bounded by
	\begin{align*}
	& \Lambda\cdot \big(\pnorm{z_{k\cdot}}{}\cdot \pnorm{v_\alpha-v_\alpha'}{\infty}+\pnorm{v_\alpha'}{\op} \pnorm{z_{k\cdot}-z_{k\cdot}'}{}\big) + \pnorm{v_\alpha'}{\op}\cdot \pnorm{z_{k\cdot}}{\infty}\\
	&\qquad \times (q\Lambda)^{c_0 L}\kappa_{v_{[2:\alpha]}}\kappa_{v_{[2:\alpha]}'} \cdot \big[\pnorm{u_{k\cdot}-u_{k\cdot}'}{}+\big(1+\pnorm{u_{k\cdot}}{\infty}\wedge \pnorm{u_{k\cdot}'}{\infty}\big)\cdot\pnorm{v_{[2:\alpha]}-v_{[2:\alpha]}'}{\infty}\big].
	\end{align*}
	The claim follows. 
	
	\noindent (4). For $u \in \R^{m\times q},v \in (\mathbb{M}_q)^{[2:L]}$ and $z \in \R^{m\times q}$,
	\begin{align*}
	\bigpnorm{ \mathscr{P}_{v;k}^{(\alpha:L]}(u_{k\cdot},z_{k\cdot}) }{\infty}&= \bigpnorm{ \mathscr{P}_{u,v_{[2:\alpha+1]};k}\big(\mathscr{P}_{v;k}^{(\alpha+1:L]}(u_{k\cdot},z_{k\cdot})\big)}{\infty}\\
	&\leq q\cdot \Lambda\pnorm{v_{\alpha+1}}{\infty} \bigpnorm{\mathscr{P}_{v;k}^{(\alpha+1:L]}(u_{k\cdot},z_{k\cdot})}{\infty}.
	\end{align*}
	Iterating the bound until reaching $\pnorm{ \mathscr{P}_{v;k}^{(L:L]}(u_{k\cdot},z_{k\cdot}) }{\infty}=\pnorm{z_{k\cdot}}{\infty}$ to conclude.
	
	\noindent (5). For $u,u' \in \R^{m\times q},v, v'\in (\mathbb{M}_q)^{[2:L]}$ and $z,z' \in \R^{m\times q}$, by using (3),
	\begin{align*}
	&\bigpnorm{\mathscr{P}_{v;k}^{(\alpha:L]}(u_{k\cdot},z_{k\cdot})-\mathscr{P}_{v';k}^{(\alpha:L]}(u_{k\cdot}',z_{k\cdot}')}{}\\
	& = \bigpnorm{ \mathscr{P}_{u,v_{[2:\alpha+1]};k}\big(\mathscr{P}_{v;k}^{(\alpha+1:L]}(u_{k\cdot},z_{k\cdot})\big)- \mathscr{P}_{u',v'_{[2:\alpha+1]};k}\big(\mathscr{P}_{v';k}^{(\alpha+1:L]}(u_{k\cdot}',z_{k\cdot}')\big)  }{}\\
	&\leq (q\Lambda)^{c_0 L}(\kappa_{v}\kappa_{v'})^2\cdot  \Big\{\bigpnorm{\mathscr{P}_{v;k}^{(\alpha+1:L]}(u_{k\cdot},z_{k\cdot})-\mathscr{P}_{v';k}^{(\alpha+1:L]}(u_{k\cdot}',z_{k\cdot}')}{} \\
	&\qquad + \Big(1+ \pnorm{\mathscr{P}_{v';k}^{(\alpha+1:L]}(u_{k\cdot}',z_{k\cdot}')}{\infty}\wedge \pnorm{\mathscr{P}_{v;k}^{(\alpha+1:L]}(u_{k\cdot},z_{k\cdot})}{\infty}\Big)  \\
	&\qquad \qquad \times  \Big[ \pnorm{u_{k\cdot}-u_{k\cdot}'}{}+ \big(1+\pnorm{u_{k\cdot}}{\infty}\wedge \pnorm{u_{k\cdot}'}{\infty}\big)\cdot \pnorm{v-v'}{\infty}\Big].
	\end{align*}
	Using (4), we have 
	\begin{align*}
	&\bigpnorm{\mathscr{P}_{v;k}^{(\alpha:L]}(u_{k\cdot},z_{k\cdot})-\mathscr{P}_{v';k}^{(\alpha:L]}(u_{k\cdot}',z_{k\cdot}')}{} \leq (q\Lambda)^{c_0 L} (\kappa_{v}\kappa_{v'})^3\\
	&\qquad \times \Big\{ \pnorm{\mathscr{P}_{v;k}^{(\alpha+1:L]}(u_{k\cdot},z_{k\cdot})-\mathscr{P}_{v';k}^{(\alpha+1:L]}(u_{k\cdot}',z_{k\cdot}')}{} +   \big(1+\pnorm{z_{k\cdot}}{\infty}\wedge \pnorm{z_{k\cdot}'}{\infty}\big) \\
	& \qquad \qquad \times  \big[ \pnorm{u_{k\cdot}-u_{k\cdot}'}{}+ \big(1+\pnorm{u_{k\cdot}}{\infty}\wedge \pnorm{u_{k\cdot}'}{\infty}\big)\cdot \pnorm{v-v'}{\infty}\big]\Big\}.
	\end{align*}
	Iterating the estimate and using the initial condition 
	\begin{align*}
	\pnorm{\mathscr{P}_{v;k}^{(L:L]}(u_{k\cdot},z_{k\cdot})-\mathscr{P}_{v';k}^{(L:L]}(u_{k\cdot}',z_{k\cdot}')}{}  =\pnorm{z_{k\cdot}-z_{k\cdot}'}{}
	\end{align*}
	to conclude. 
	
	\noindent (6). For $u \in \R^{m\times q},v\in (\mathbb{M}_q)^{[2:L]}$ and $w \in \R^{m\times q}$, using the estimate in (4), 
	\begin{align*}
	\pnorm{\mathfrak{S}_{k}(u_{k\cdot},w_{k\cdot}, v)}{\infty} &\leq \Lambda\cdot \bigpnorm{\mathscr{P}_{v;k}^{(1:L]}\big(u_{k\cdot},\mathscr{R}_{v;k }(u_{k\cdot},w_{k\cdot})\big)}{\infty}\\
	&\leq (q\Lambda)^{c_0 L} \kappa_{v}\cdot \pnorm{\mathscr{R}_{v;k}(u_{k\cdot},w_{k\cdot})}{\infty}\\
	&\leq (q\Lambda)^{c_0 L} \kappa_{v }\cdot\big(\pnorm{ \mathscr{G}_{v;k}(u_{k\cdot}) }{\infty}+\pnorm{\varphi_\ast(w_{k\cdot})}{\infty}+\abs{\xi_k}\big).
	\end{align*}
	Now using the estimate in (1) to conclude.
	
	\noindent (7). For $u,u' \in \R^{m\times q},v, v'\in (\mathbb{M}_q)^{[2:L]}$ and $z,z' \in \R^{m\times q}$, using the estimates in (4) and then (5),
	\begin{align}\label{ineq:apr_est_det_fcn_1}
	&\pnorm{\mathfrak{S}_{k}(u_{k\cdot},w_{k\cdot}, v)-\mathfrak{S}_{k}(u_{k\cdot}',w_{k\cdot}',v')}{}\nonumber\\
	&\leq (q\Lambda)^{c_0 L} \kappa_{v}\cdot \pnorm{\mathscr{R}_{v;k}(u_{k\cdot},w_{k\cdot})}{\infty}\cdot \pnorm{u_{k\cdot}-u_{k\cdot}'}{}\nonumber\\
	&\qquad + \Lambda\cdot \bigpnorm{\mathscr{P}_{v;k}^{(1:L]}\big(u_{k\cdot},\mathscr{R}_{v;k }(u_{k\cdot},w_{k\cdot})\big)-\mathscr{P}_{v';k}^{(1:L]}\big(u_{k\cdot}',\mathscr{R}_{v';k }(u_{k\cdot}',w_{k\cdot}')\big)}{}\nonumber\\
	&\equiv I_1+I_2.
	\end{align}
	Similar to the derivation in (6), 
	\begin{align}\label{ineq:apr_est_det_fcn_2}
	I_1\leq (q\Lambda)^{c_0 L} \kappa_{v}^2\cdot\big(1+\pnorm{u_{k\cdot}}{\infty}+\pnorm{\varphi_\ast(w_{k\cdot})}{\infty}+\abs{\xi_k}\big)\cdot \pnorm{u_{k\cdot}-u_{k\cdot}'}{}.
	\end{align}
	For $I_2$, using the estimate in (5), 
	\begin{align}\label{ineq:apr_est_det_fcn_3}
	I_2
	&\leq (q \Lambda)^{c_0 L} (\kappa_{v}\kappa_{v'})^{c_0}\cdot \Big\{\pnorm{\varphi_\ast(w_{k\cdot})-\varphi_\ast(w_{k\cdot}')}{}\nonumber\\
	&\qquad  +\big(1+(\pnorm{u_{k\cdot}}{\infty}+\pnorm{\varphi_\ast(w_{k\cdot})}{\infty})\wedge (\pnorm{u_{k\cdot}'}{\infty}+\pnorm{\varphi_\ast(w_{k\cdot}')}{\infty})+\abs{\xi_k}\big)\nonumber\\
	&\qquad\qquad \times \big[ \pnorm{u_{k\cdot}-u_{k\cdot}'}{}+  (1+\pnorm{u_{k\cdot}}{\infty}\wedge \pnorm{u_{k\cdot}'}{\infty})\cdot \pnorm{v-v' }{\infty}\big]\Big\}.
	\end{align}
	Combining the above displays (\ref{ineq:apr_est_det_fcn_1})-(\ref{ineq:apr_est_det_fcn_3}) to conclude. \qed

\subsection{Proof of Proposition \ref{prop:derivative_formula}}\label{subsection:proof_derivative_formula}

	\noindent (1)-(2). The claimed formulae in (1) follow directly from an application of the chain rule to the definition. The formulae in (2) follow by using those derived in (1).
	
	\noindent (3). Note that for $\alpha \in [1:L-1]$ we have 
	\begin{align}\label{ineq:derivative_formula_1}
	\mathscr{P}_{v}^{(\alpha:L]}(u,z)=\mathscr{P}_{u,v_{[2:\alpha+1]} }\big(\mathscr{P}_{v}^{(\alpha+1:L]}(u,z)\big).
	\end{align}
	By using (2) and the chain rule to (\ref{ineq:derivative_formula_1}), for $c \in [m], d\in [q]$,
	\begin{align*}
	&\partial_{u_{k\ell}} \mathscr{P}_{v;(c,d)}^{(\alpha:L]}(u,z) = \big(\partial_{u_{k\ell}} \mathscr{P}_{u,v_{[2:\alpha+1]}; (c,d) }\big)\big(\mathscr{P}_{v}^{(\alpha+1:L]}(u,z)\big)\\
	&\qquad + \sum_{k' \in [m], \ell' \in [q]} \Big[\Big( e_{k'}e_{\ell'}^\top \odot \sigma_{\alpha+1}'\big(\mathscr{H}_{v_{[2:\alpha+1]}}(u)\big)\Big) v_{\alpha+1}^\top\Big]_{cd} \cdot \partial_{u_{k\ell}} \mathscr{P}_{v;(k',\ell')}^{(\alpha+1:L]}(u,z)\\
	&=\bigg[\Big(\mathscr{P}_{v}^{(\alpha+1:L]}(u,z)\odot \sigma_{\alpha+1}''\big(\mathscr{H}_{v_{[2:\alpha+1]}}(u)\big)\odot \partial_{u_{k\ell}}  \mathscr{H}_{v_{[2:\alpha+1]}}(u)\Big) v_{\alpha+1}^\top\\
	&\qquad + \Big(\sigma_{\alpha+1}'\big(\mathscr{H}_{v_{[2:\alpha+1]}}(u)\big)\odot \partial_{u_{k\ell}} \mathscr{P}_{v}^{(\alpha+1:L]}(u,z) \Big)v_{\alpha+1}^\top \bigg]_{cd},
	\end{align*}
	where the last identity follows as 
	\begin{align}\label{ineq:derivative_formula_2}
	\sum_{k',\ell'} (e_{k'}e_{\ell'}^\top \odot M) v^\top\cdot N_{k',\ell'}=\sum_{k',\ell'} M_{k',\ell'}N_{k',\ell'} e_{k'}e_{\ell'}^\top v^\top = (M\odot N) v^\top.
	\end{align}
	Next, using the chain rule to (\ref{ineq:derivative_formula_1}), for $c \in [m], d \in [q]$,
	\begin{align*}
	&\partial_{z_{k\ell}} \mathscr{P}_{v;(c,d)}^{(\alpha:L]}(u,z) \\
	&= \sum_{k' \in [m], \ell' \in [q]} \Big[\Big( e_{k'}e_{\ell'}^\top \odot \sigma_{\alpha+1}'\big(\mathscr{H}_{v_{[2:\alpha+1]}}(u)\big)\Big) v_{\alpha+1}^\top\Big]_{cd} \cdot \partial_{z_{k\ell}} \mathscr{P}_{v;(k',\ell')}^{(\alpha+1:L]}(u,z)\\
	& = \bigg[\Big(\sigma_{\alpha+1}'\big(\mathscr{H}_{v_{[2:\alpha+1]}}(u)\big)\odot \partial_{z_{k\ell}} \mathscr{P}_{v}^{(\alpha+1:L]}(u,z) \Big)v_{\alpha+1}^\top\bigg]_{cd}.
	\end{align*}
	Here in the last line we used (\ref{ineq:derivative_formula_2}) again.
	
	\noindent (4). Note that for $c \in [m], d \in [q]$, 
	\begin{align}\label{ineq:derivative_formula_3}
	\partial_{u_{k\ell}} \mathfrak{S}_{(c,d)}\big(u,w, v\big) &=\Big[\partial_{u_{k\ell}} \mathscr{P}_{v;(c,d)}^{(1:L]}\big(u,\mathscr{R}_{v}(u,w)\big)\Big] \sigma_1'(u_{cd}) \nonumber\\
	&\qquad + \mathscr{P}_{v;(c,d)}^{(1:L]}\big(u,\mathscr{R}_{v}(u,w)\big)\sigma_1''(u_{cd}) \delta_{kc}\delta_{\ell d}.
	\end{align}
	On the other hand, using (3) we have
	\begin{align*}
	&\partial_{u_{k\ell}} \mathscr{P}_{v;(c,d)}^{(1:L]}\big(u,\mathscr{R}_{v }(u,w)\big)=\big(\partial_{u_{k\ell}} \mathscr{P}_{v;(c,d)}^{(1:L]}\big)\big(u,\mathscr{R}_v(u,w)\big)\\
	&\qquad + \sum_{k' \in [m], \ell' \in [q]} \big(\partial_{z_{k'\ell'}} \mathscr{P}_{v;(c,d)}^{(1:L]}\big)\big(u,\mathscr{R}_v(u,w)\big) \partial_{u_{k\ell}} \mathscr{R}_{v;(k',\ell')}(u,w).
	\end{align*}
	Using (1), we have $\partial_{u_{k\ell}} \mathscr{R}_{v;(k',\ell')}(u,w)=\partial_{u_{k\ell}} \mathscr{H}_{v;(k',\ell')}(u)\delta_{kk'}$. This means 
	\begin{align}\label{ineq:derivative_formula_4}
	&\partial_{u_{k\ell}} \mathscr{P}_{v;(c,d)}^{(1:L]}\big(u,\mathscr{R}_{v }(u,w)\big)=\bigg[\big(\partial_{u_{k\ell}} \mathscr{P}_{v}^{(1:L]}\big)\big(u,\mathscr{R}_v(u,w)\big)\nonumber\\
	&\qquad + \sum_{\ell' \in [q]} \big(\partial_{z_{k \ell'}} \mathscr{P}_{v}^{(1:L]}\big)\big(u,\mathscr{R}_v(u,w)\big) \partial_{u_{k\ell}} \mathscr{H}_{v;(k,\ell')}(u)\bigg]_{cd}.
	\end{align}
	Combining (\ref{ineq:derivative_formula_3})-(\ref{ineq:derivative_formula_4}) to conclude the derivative formula for $\partial_{u_{k\ell}} \mathfrak{S}\big(u,w, v\big)$.
	
	Finally, using
	\begin{align*}
	\partial_{w_{k\ell}} \mathfrak{S}_{(c,d)}\big(u,w, v\big) &= \Big[\partial_{w_{k\ell}} \mathscr{P}_{v;(c,d)}^{(1:L]}\big(u,\mathscr{R}_{v}(u,w)\big)\Big] \sigma_1'(u_{cd}),
	\end{align*}
	and
	\begin{align*}
	\partial_{w_{k\ell}} \mathscr{P}_{v;(c,d)}^{(1:L]}\big(u,\mathscr{R}_{v }(u,w)\big) &= \sum_{k' \in [m], \ell' \in [q]} \big(\partial_{z_{k'\ell'}} \mathscr{P}_{v;(c,d)}^{(1:L]}\big)\big(u,\mathscr{R}_v(u,w)\big) \partial_{w_{k\ell}} \mathscr{R}_{v;(k',\ell')}(u,w)\\
	& = - \big(\partial_{z_{k\ell}} \mathscr{P}_{v;(c,d)}^{(1:L]}\big)\big(u,\mathscr{R}_v(u,w)\big) \varphi_\ast'(w_{k\ell}),
	\end{align*}
	we arrive at the derivative formula for $\partial_{w_{k\ell}} \mathfrak{S}\big(u,w, v\big)$.\qed

\subsection{Proof of Proposition \ref{prop:apr_est_derivative}}\label{subsection:proof_apr_est_derivative}

	\noindent (1). Using the recursion in Proposition \ref{prop:derivative_formula}-(1), 
	\begin{align*}
	\pnorm{\partial_{u_{k\ell}}  \mathscr{H}_{v_{[2:\alpha]}}(u)}{\infty}\leq q\Lambda  \pnorm{v_\alpha}{\infty}\cdot \pnorm{\partial_{u_{k\ell}}  \mathscr{H}_{v_{[2:\alpha-1]}}(u)}{\infty}.
	\end{align*}
	Iterating the estimate and using the initial condition $\pnorm{\partial_{u_{k\ell}}  \mathscr{H}_{v_{[2:1]}}(u)}{\infty}=1$ to conclude the estimate for $\pnorm{\partial_{u_{k\ell}}  \mathscr{H}_{v_{[2:\alpha]}}(u)}{\infty}$. The estimate for $\pnorm{\partial_{u_{k\ell}}  \mathscr{G}_{v_{[2:\alpha]}}(u)}{\infty}$ follows trivially. 
	
	\noindent (2). Using the recursive formula in Proposition \ref{prop:derivative_formula}-(1), we have 
	\begin{align*}
	\partial_{u_{k\ell}}  \mathscr{H}_{v_{[2:\alpha]};k}(u_{k\cdot})=v_\alpha^\top \big[\sigma_{\alpha-1}'\big(\mathscr{H}_{v_{[2:\alpha-1]};k}(u_{k\cdot})\big)\odot\partial_{u_{k\ell}}  \mathscr{H}_{v_{[2:\alpha-1]};k }(u_{k\cdot})\big].
	\end{align*}
	Consequently, using the apriori estimate in (1) and Proposition \ref{prop:apr_est_det_fcn}-(2),
	\begin{align*}
	&\bigpnorm{\partial_{u_{k\ell}}  \mathscr{H}_{v_{[2:\alpha]};k}(u_{k\cdot})-	\partial_{u_{k\ell}}  \mathscr{H}_{v_{[2:\alpha]}';k}(u_{k\cdot}')}{}\\
	& \leq (q\Lambda)^{c_0 L} \big(\kappa_{v_{[2:\alpha]}}\kappa_{v_{[2:\alpha]}'}\big)^{c_0}\cdot \Big[\pnorm{u_{k\cdot}-u_{k\cdot}'}{}+\big(1+\pnorm{u_{k\cdot}}{\infty}\wedge \pnorm{u_{k\cdot}'}{\infty}\big)\cdot\pnorm{v_{[2:\alpha]}-v_{[2:\alpha]}'}{\infty}\Big]\\
	&\qquad + (q\Lambda)^{c_0}\cdot \pnorm{v_\alpha}{\infty}\cdot \bigpnorm{\partial_{u_{k\ell}}  \mathscr{H}_{v_{[2:\alpha-1]};k}(u_{k\cdot})-	\partial_{u_{k\ell}}  \mathscr{H}_{v_{[2:\alpha-1]}';k}(u_{k\cdot}')}{}.
	\end{align*}
	Now we may iterate the above estimate and using the initial trivial condition to conclude.	A similar estimate holds for $\pnorm{\partial_{u_{k\ell}}  \mathscr{G}_{v_{[2:\alpha]};k}(u_{k\cdot})-	\partial_{u_{k\ell}}  \mathscr{G}_{v_{[2:\alpha]};k}(u_{k\cdot}')}{}$.
	
	\noindent (3). Using the formula in Proposition \ref{prop:derivative_formula}-(2) for $\alpha \in [2:L]$,
	\begin{align*}
	\pnorm{\partial_{u_{k\ell}}\mathscr{P}_{u, v_{[2:\alpha]} }(z)}{\infty}\leq q \Lambda \pnorm{v_\alpha}{\infty}\cdot \pnorm{z}{\infty}\cdot  \pnorm{\partial_{u_{k\ell}}  \mathscr{H}_{v_{[2:\alpha]}}(u)}{\infty}.
	\end{align*}
	Now using (1) to conclude. The case for $\alpha=1$ is trivial. The bound for $\pnorm{\partial_{z_{k\ell}}\mathscr{P}_{u, v_{[2:\alpha]} }(z)}{\infty}$ follows trivially from the formula proved in Proposition \ref{prop:derivative_formula}-(2).
	
	\noindent (4). Using the recursive formula in Proposition \ref{prop:derivative_formula}-(2), we have 
	\begin{align*}
	\partial_{u_{k\ell}}\mathscr{P}_{u, v_{[2:\alpha]};k }(z_{k\cdot})&\equiv v_\alpha \Big(z_{k\cdot}\odot \sigma_\alpha''\big(\mathscr{H}_{v_{[2:\alpha]};k}(u_{k\cdot})\big)\odot \partial_{u_{k\ell}}  \mathscr{H}_{v_{[2:\alpha]};k}(u_{k\cdot})\Big).
	\end{align*}
	Consequently, using the apriori estimate in (1), the stability estimate in (2) and Proposition \ref{prop:apr_est_det_fcn}-(2),
	\begin{align*}
	&\bigpnorm{ \partial_{u_{k\ell}}\mathscr{P}_{u, v_{[2:\alpha]};k }(z_{k\cdot}) - \partial_{u_{k\ell}}\mathscr{P}_{u', v_{[2:\alpha]}';k }(z_{k\cdot}')  }{}\\
	&\leq (q\Lambda)^{c_0 L} \big(\kappa_{v_{[2:\alpha]}}\kappa_{v_{[2:\alpha]}'}\big)^{c_0}\cdot  \max_{\# \in \{u,z\}}\big(1+\pnorm{\#_{k\cdot}}{\infty}\wedge \pnorm{\#_{k\cdot}'}{\infty}\big)^2\\
	&\qquad \times  \big(\pnorm{u_{k\cdot}-u_{k\cdot}'}{}+\pnorm{v_{[2:\alpha]}-v_{[2:\alpha]}'}{\infty}+\pnorm{z_{k\cdot}-z_{k\cdot}'}{}\big).
	\end{align*}
	A simpler argument leading to the same estimate can be derived for $\pnorm{ \partial_{z_{k\ell}}\mathscr{P}_{u, v_{[2:\alpha]};k }(z_{k\cdot}) - \partial_{z_{k\ell}}\mathscr{P}_{u', v_{[2:\alpha]}';k }(z_{k\cdot}')  }{}$.
	
	\noindent (5). Using the formula in Proposition \ref{prop:derivative_formula}-(3) for $\alpha \in [1:L-1]$,
	\begin{align*}
	&\pnorm{\partial_{u_{k\ell}} \mathscr{P}_{v}^{(\alpha:L]}(u,z)}{\infty}\leq q\Lambda \pnorm{v_{\alpha+1}}{\infty}\cdot \Big[ \pnorm{\partial_{u_{k\ell}} \mathscr{P}_{v}^{(\alpha+1:L]}(u,z)}{\infty}\\
	&\qquad + \pnorm{\mathscr{P}_{v}^{(\alpha+1:L]}(u,z)}{\infty}\cdot \pnorm{\partial_{u_{k\ell}}  \mathscr{H}_{v_{[2:\alpha+1]}}(u)}{\infty} \Big]\\
	&\leq q\Lambda \pnorm{v_{\alpha+1}}{\infty}\cdot  \pnorm{\partial_{u_{k\ell}} \mathscr{P}_{v}^{(\alpha+1:L]}(u,z)}{\infty}+ (q\Lambda)^{c_0 L}\kappa_v^2 \cdot \pnorm{z}{\infty}.
	\end{align*}
	Iterating the estimate and using the initial condition $\pnorm{\partial_{u_{k\ell}} \mathscr{P}_{v}^{(L:L]}(u,z)}{\infty}=0$ to conclude. The case for $\alpha=L$ is trivial. Similarly, by using the second formula in Proposition \ref{prop:derivative_formula}-(3), we arrive at
	\begin{align*}
	\pnorm{\partial_{z_{k\ell}} \mathscr{P}_{v}^{(\alpha:L]}(u,z)}{\infty}\leq q\Lambda \pnorm{v_{\alpha+1}}{\infty}\cdot \pnorm{\partial_{z_{k\ell}} \mathscr{P}_{v}^{(\alpha+1:L]}(u,z)}{\infty}.
	\end{align*} 
	Iterating the estimate and using instead the initial condition $\pnorm{\partial_{z_{k\ell}} \mathscr{P}_{v}^{(L:L]}(u,z)}{\infty}=1$ to conclude. 
	
	\noindent (6). Using the recursive formula in Proposition \ref{prop:derivative_formula}-(3), we have 
	\begin{align*}
	&\partial_{u_{k\ell}} \mathscr{P}_{v;k}^{(\alpha:L]}(u_{k\cdot},z_{k\cdot})\\
	&=v_{\alpha+1}\Big[\mathscr{P}_{v;k}^{(\alpha+1:L]}(u_{k\cdot},z_{k\cdot})\odot \sigma_{\alpha+1}''\big(\mathscr{H}_{v_{[2:\alpha+1]};k }(u_{k\cdot})\big)\odot \partial_{u_{k\ell}}  \mathscr{H}_{v_{[2:\alpha+1]};k}(u_{k\cdot})\\
	&\qquad + \sigma_{\alpha+1}'\big(\mathscr{H}_{v_{[2:\alpha+1]};k }(u_{k\cdot})\big)\odot \partial_{u_{k\ell}} \mathscr{P}_{v;k}^{(\alpha+1:L]}(u_{k\cdot},z_{k\cdot}) \Big].
	\end{align*}
	Consequently, using the apriori estimates in (1) and (3) along with the stability estimates in (2) and (4), as well as the estimates in Proposition \ref{prop:apr_est_det_fcn},
	\begin{align*}
	&\bigpnorm{\partial_{u_{k\ell}} \mathscr{P}_{v;k}^{(\alpha:L]}(u_{k\cdot},z_{k\cdot})- \partial_{u_{k\ell}} \mathscr{P}_{v';k}^{(\alpha:L]}(u_{k\cdot}',z_{k\cdot}')}{}\\
	&\leq (q\Lambda)^{c_0 L} (\kappa_{v}\kappa_{v'})^{c_0}\cdot \max_{\# \in \{u,z\}}\big(1+\pnorm{\#_{k\cdot}}{\infty}\wedge \pnorm{\#_{k\cdot}'}{\infty}\big)^{c_0}\\
	&\qquad \times   \big(\pnorm{u_{k\cdot}-u_{k\cdot}'}{}+\pnorm{z_{k\cdot}-z_{k\cdot}'}{} + \pnorm{v-v'}{\infty}\big).
	\end{align*}
	A simpler argument leading to the same estimate holds for $\pnorm{\partial_{z_{k\ell}} \mathscr{P}_{v;k}^{(\alpha:L]}(u_{k\cdot},z_{k\cdot})- \partial_{z_{k\ell}} \mathscr{P}_{v';k}^{(\alpha:L]}(u_{k\cdot}',z_{k\cdot}')}{}$. 
	
	\noindent (7). Using the formula in Proposition \ref{prop:derivative_formula}-(4), we have
	\begin{align*}
	&\pnorm{\partial_{u_{k\ell}} \mathfrak{S}\big(u,w, v\big)}{\infty}\leq  \Lambda\cdot \pnorm{\mathscr{P}_{v}^{(1:L]}\big(u,\mathscr{R}_{v}(u,w)\big)}{\infty} + \Lambda\cdot \bigg[ \pnorm{\partial_{u_{k\ell}} \mathscr{P}_{v}^{(1:L]}(u,\mathscr{R}_v(u,w))}{\infty} \\
	&\qquad + \sum_{\ell' \in [q]} \pnorm{\partial_{z_{k \ell'}} \mathscr{P}_{v}^{(1:L]}(u, \mathscr{R}_v(u,w))}{\infty}\cdot  \pnorm{\partial_{u_{k\ell}} \mathscr{H}_{v}(u)}{\infty}\bigg].
	\end{align*}
	Now using the estimates obtained in (1)-(3) and the apriori estimate in Proposition \ref{prop:apr_est_det_fcn}-(4), we have 
	\begin{align*}
	\pnorm{\partial_{u_{k\ell}} \mathfrak{S}\big(u,w, v\big)}{\infty}&\leq  (q\Lambda)^{c_0 L}\kappa_{v}^2\cdot  \big(1+\pnorm{\mathscr{R}_{v}(u,w) }{\infty}\big)\\
	&\leq (q\Lambda)^{c_0 L}\kappa_{v}^2\cdot \big(1+\pnorm{\mathscr{G}_{v}(u) }{\infty}+\pnorm{\varphi_\ast(w)}{\infty}+\pnorm{\xi}{\infty}\big).
	\end{align*}
	The claimed estimate for $\pnorm{\partial_{u_{k\ell}} \mathfrak{S}\big(u,w, v\big)}{\infty}$ now follows from the apriori estimate in Proposition \ref{prop:apr_est_det_fcn}-(1). 
	
	Finally, using the second formula in in Proposition \ref{prop:derivative_formula}-(4), 
	\begin{align*}
	\pnorm{\partial_{w_{k\ell}} \mathfrak{S}\big(u,w, v\big)}{\infty} &\leq  \Lambda\cdot \pnorm{\mathscr{R}_{v}(u,w) }{\infty} \cdot  \pnorm{\varphi_\ast'(w)}{\infty}\\
	&\leq \Lambda\cdot \pnorm{\varphi_\ast'(w)}{\infty} \cdot \big(\pnorm{\mathscr{G}_{v}(u) }{\infty}+\pnorm{\varphi_\ast(w)}{\infty}+\pnorm{\xi}{\infty}\big).
	\end{align*}
	The claimed estimate follows.
	
	\noindent (8). Using the formula in Proposition \ref{prop:derivative_formula}-(4), we have 
	\begin{align*}
	\partial_{u_{k\ell}} \mathfrak{S}_{k}\big(u,w, v\big) &= \mathscr{P}_{v;k}^{(1:L]}\big(u_{k\cdot},\mathscr{R}_{v;k}(u_{k\cdot},w_{k\cdot})\big)\odot\sigma_1''(u_{k\cdot})\odot e_\ell  \\
	&  \qquad +\sigma_1'(u_{k\cdot})\odot \bigg[\partial_{u_{k\ell}} \mathscr{P}_{v;k}^{(1:L]}\big(u_{k\cdot},\mathscr{R}_{v;k}(u_{k\cdot},w_{k\cdot})\big) \\
	&\qquad + \sum_{\ell' \in [q]} \partial_{z_{k \ell'}} \mathscr{P}_{v;k}^{(1:L]}\big(u_{k\cdot},\mathscr{R}_{v;k}(u_{k\cdot},w_{k\cdot})\big) \partial_{u_{k\ell}} \mathscr{H}_{v;(k,\ell')}(u_{k\cdot})\bigg].
	\end{align*}
	Now we may use the apriori estimates and the stability estimates in (1)-(7) above, as well as the estimates in Proposition \ref{prop:apr_est_det_fcn} to conclude. We omit the tedious calculations and similar arguments for the other derivative. \qed

\subsection{Proof of Proposition \ref{prop:apr_est_se}}\label{subsection:proof_apr_est_se}

\begin{proof}[Proof of Proposition \ref{prop:apr_est_se}, Part I]
	With slight abuse of notation, in the proof we simply identify $\mathcal{B}^{(t)}$ as $\max_{s \in [1:t]} \mathcal{B}^{(s)}$ to avoid unnecessarily complicated notation. By (S1) and Proposition \ref{prop:apr_est_det_fcn}-(6), 
	\begin{align*}
	&\bigpnorm{\Theta_{k}^{(t)}(\mathfrak{u}^{([0:t])})}{\infty} \\
	&\leq \pnorm{\mathfrak{u}_{k\cdot}^{(t)}}{\infty}+qK\Lambda \sum_{s \in [1:t-1]}  \bigpnorm{ \mathfrak{S}_{k}\big(\Theta_{k}^{(s)}(\mathfrak{u}_{k\cdot}^{([0:s])}),\mathfrak{u}_{k\cdot}^{(0)},V^{(s-1)}\big) }{\infty}\cdot \pnorm{\rho_{t-1,s}}{\infty}\\
	&\leq \pnorm{\mathfrak{u}_{k\cdot}^{(t)}}{\infty}+K(q\Lambda)^{c_0L}\cdot \max_{s \in [1:t-1]} \kappa_{V_{[2:\alpha]}^{(s-1)}}^2\cdot \Big(1+ \bigpnorm{\Theta_{k}^{(s)}(\mathfrak{u}_{k\cdot}^{([0:s])})}{\infty}+ \pnorm{\varphi_\ast(\mathfrak{u}_{k\cdot}^{(0)})}{\infty}+\pnorm{\xi}{\infty}\Big).
	\end{align*}
	Iterating the bound and using the initial condition $\pnorm{\Theta_{k}^{(1)}(\mathfrak{u}_{k\cdot}^{([0:1])})}{\infty}=\pnorm{\mathfrak{u}_{k\cdot}^{(1)}}{\infty}$, 
	\begin{align*}
	\bigpnorm{\Theta_{k}^{(t)}(\mathfrak{u}_{k\cdot}^{([0:t])})}{\infty} \leq \big(K\Lambda\varkappa_{\ast} \mathcal{B}^{(t-1)}\big)^{c_t} \cdot \big(1+\pnorm{\mathfrak{u}_{k\cdot}^{([1:t])}}{\infty}+ \pnorm{\varphi_\ast(\mathfrak{u}_{k\cdot}^{(0)})}{\infty}\big).
	\end{align*}
	This means
	\begin{align}\label{ineq:apr_est_se_1}
	\frac{ \bigpnorm{\Theta_{k}^{(t)}(\mathfrak{u}_{k\cdot}^{([0:t])})}{\infty} }{  1+\pnorm{\mathfrak{u}_{k\cdot}^{([1:t])}}{\infty}+ \pnorm{\varphi_\ast(\mathfrak{u}_{k\cdot}^{(0)})}{\infty} }\leq \big(K\Lambda\varkappa_{\ast} \mathcal{B}^{(t-1)}\big)^{c_t}.
	\end{align}
	By (S2), it is easy to obtain
	\begin{align}\label{ineq:apr_est_se_2}
	\max_{s \in [0:t]} \pnorm{\cov(\mathfrak{U}^{(t)},\mathfrak{U}^{(s)})}{\op}\leq 2\mathcal{B}^{(t-1)}.
	\end{align}
	By (S3), using Proposition \ref{prop:apr_est_derivative}-(7) and (\ref{ineq:apr_est_se_1}), for $s \in [0:t]$,
	\begin{align*}
	\pnorm{\tau_{t,s}}{\op}&\leq \Lambda\cdot \bigpnorm{\E^{(0)} \nabla_{\mathfrak{U}^{(s)}} \mathfrak{S}_{\pi_m}\big(\Theta_{\pi_m}^{(t)}(\mathfrak{U}^{([0:t])}) ,\mathfrak{U}^{(0)},V^{(t-1)}\big) }{\op}\\
	&\leq \Lambda^{c_t} (\mathcal{B}^{(t-1)})^{c_0}\cdot \Big(\E^{(0)} \pnorm{\Theta_{\pi_m}^{(t)}(\mathfrak{U}^{([0:t])})}{\infty}^{c_0 r_0}+ \E^{(0)}\pnorm{\mathfrak{U}^{(0)}}{\infty}^{c_0 r_0}+\pnorm{\xi}{\infty}^{c_0r_0}\Big)\\
	&\leq \big(\Lambda \varkappa_{\ast}\mathcal{B}^{(t-1)}\big)^{c_t} \cdot \E^{(0)} \pnorm{\mathfrak{U}^{([0:t])}}{\infty}^{c_0r_0}.
	\end{align*}
	Now using the estimate in (\ref{ineq:apr_est_se_2}), for $s \in [0:t]$,
	\begin{align}\label{ineq:apr_est_se_3}
	\pnorm{\tau_{t,s}}{\op}\leq \big(K\Lambda \varkappa_{\ast} \mathcal{B}^{(t-1)}\big)^{c_t}.
	\end{align}
	Next, for $\{\rho_{t,s}\}_{s \in [1:t]}$, by definition and using the above display (\ref{ineq:apr_est_se_3}), 
	\begin{align}\label{ineq:apr_est_se_4}
	\pnorm{\rho_{t,s}}{\op}&\leq \big(K\Lambda \varkappa_{\ast} \mathcal{B}^{(t-1)}\big)^{c_t}.
	\end{align} 
	For $\{\Sigma_{t,s}\}_{s \in [1:t]}$, using the estimates in Proposition \ref{prop:apr_est_det_fcn}-(6) and (\ref{ineq:apr_est_se_1}),
	\begin{align}\label{ineq:apr_est_se_5}
	\pnorm{\Sigma_{t,s}}{\op}&\leq c_t K\Lambda \cdot \E^{(0)} \max_{\tau \in \{t,s\} } \bigpnorm{\mathfrak{S}_{\pi_m}\big(\Theta_{\pi_m}^{(\tau)}(\mathfrak{U}^{([0:\tau])}),\mathfrak{U}^{(0)},V^{(\tau-1)}\big)}{\infty} \nonumber\\
	&\leq (K\Lambda \varkappa_{\ast}\mathcal{B}^{(t-1)})^{c_t}\cdot \big(1+\E^{(0)} \pnorm{\mathfrak{U}^{([0:t])}}{\infty}^{c_0r_0}\big)\leq \big(K\Lambda \varkappa_{\ast} \mathcal{B}^{(t-1)}\big)^{c_t}.
	\end{align}
	For $\{\Omega_{t,s}\}_{s \in [1:t]}$, by definition and the estimates in (\ref{ineq:apr_est_se_4})-(\ref{ineq:apr_est_se_5}),
	\begin{align}\label{ineq:apr_est_se_6}
	\pnorm{\Omega_{t,s}}{\op}&\leq \big(K\Lambda \varkappa_{\ast} \mathcal{B}^{(t-1)}\big)^{c_t}.
	\end{align}
	By (S4) and the estimate in (\ref{ineq:apr_est_se_3}), 
	\begin{align}\label{ineq:apr_est_se_7}
	n^{-1} \pnorm{D_t}{}^2\leq \big(K\Lambda \varkappa_{\ast} \mathcal{B}^{(t-1)}\big)^{c_t}.
	\end{align}
	By (S5) and the apriori estimates in Proposition \ref{prop:apr_est_det_fcn}-(1)(6) and (\ref{ineq:apr_est_se_1}),
	\begin{align}\label{ineq:apr_est_se_8}
	\max_{\alpha \in [2:L]}\pnorm{V_{\alpha}^{(t)}}{\infty}&\leq \big(\Lambda \varkappa_{\ast}\mathcal{B}^{(t-1)}\big)^{c_t} \cdot \big(1+\E^{(0)} \pnorm{\mathfrak{U}^{([0:t])}}{\infty}^{c_0r_0}\big) \leq \big(K\Lambda \varkappa_{\ast} \mathcal{B}^{(t-1)}\big)^{c_t}.
	\end{align}
	A similar estimate holds for $M_{\alpha}^{(t-1)}$:
	\begin{align}\label{ineq:apr_est_se_9}
	\max_{\alpha \in [2:L]}\pnorm{M_{\alpha}^{(t-1)}}{\infty}\leq \big(K\Lambda \varkappa_{\ast} \mathcal{B}^{(t-1)}\big)^{c_t}.
	\end{align}
	Combining (\ref{ineq:apr_est_se_1})-(\ref{ineq:apr_est_se_9}), we have $
	\mathcal{B}^{(t)}\leq \big(K\Lambda\varkappa_{\ast} \mathcal{B}^{(t-1)}\big)^{c_t}$. 
	It is easy to establish that $\mathcal{B}^{(1)}\leq \big(K\Lambda\varkappa_{\ast}\big)^{c_0}$, so we may conclude by iterating the recursive estimate. 
\end{proof}

\begin{proof}[Proof of Proposition \ref{prop:apr_est_se}, Part II]
We now provide refined estimates for some of the state evolution parameters. By (S1), (\ref{ineq:apr_est_se_1}) and the proven estimate for $\mathcal{B}^{(t)}$, 
\begin{align}\label{ineq:apr_est_se_part2_1}
&\bigpnorm{\Theta_{k}^{(t)}(\mathfrak{u}_{k\cdot}^{([0:t])})-\mathfrak{u}_{k\cdot}^{(t)} }{\infty} \nonumber\\
&\leq q\Lambda \phi^{-1}\cdot  \sum_{s \in [1:t-1]}  \bigpnorm{ \mathfrak{S}_{k}\big(\Theta_{k}^{(s)}(\mathfrak{u}_{k\cdot}^{([0:s])}),\mathfrak{u}_{k\cdot}^{(0)},V^{(s-1)}\big) }{\infty}\cdot \pnorm{\rho_{t-1,s}}{\infty}\nonumber\\
&\leq \big(K\Lambda\varkappa_{\ast}\big)^{c_t}\cdot\big(1+\pnorm{\mathfrak{u}_{k\cdot}^{([1:t-1])}}{\infty}+ \pnorm{\varphi_\ast(\mathfrak{u}_{k\cdot}^{(0)})}{\infty}\big)\cdot \phi^{-1}.
\end{align}
Next, for $s \in [1:t-1]$, as $\partial_{\mathfrak{u}^{(s)}_{k\ell}} \mathfrak{S}_{k}\big(\mathfrak{u}_{k\cdot}^{(t)}, \mathfrak{u}_{k\cdot}^{(0)}, V^{(t-1)}\big)=0$, using the estimates in Proposition \ref{prop:apr_est_derivative}-(8), (\ref{ineq:apr_est_se_1}) and (\ref{ineq:apr_est_se_part2_1}), we have 
\begin{align*}
& \bigpnorm{ \partial_{\mathfrak{u}^{(s)}_{k\ell}} \mathfrak{S}_{k}\big(\Theta_{k}^{(t)}(\mathfrak{u}_{k\cdot}^{([0:t])}), \mathfrak{u}_{k\cdot}^{(0)}, V^{(t-1)}\big) }{}\\
&=\bigpnorm{ \partial_{\mathfrak{u}^{(s)}_{k\ell}} \mathfrak{S}_{k}\big(\Theta_{k}^{(t)}(\mathfrak{u}_{k\cdot}^{([0:t])}), \mathfrak{u}_{k\cdot}^{(0)}, V^{(t-1)}\big)- \partial_{\mathfrak{u}^{(s)}_{k\ell}} \mathfrak{S}_{k}\big(\mathfrak{u}_{k\cdot}^{(t)}, \mathfrak{u}_{k\cdot}^{(0)}, V^{(t-1)}\big)  }{}\\
&\leq \big(K\Lambda\varkappa_{\ast}\big)^{c_t}\cdot \big(1+\pnorm{\mathfrak{u}_{k\cdot}^{([0:t])}}{\infty}\big)^{c_0r_0}\cdot \bigpnorm{\Theta_{k}^{(t)}(\mathfrak{u}_{k\cdot}^{([0:t])})-\mathfrak{u}^{(t)} }{\infty}\\
&\leq \big(K\Lambda\varkappa_{\ast}\big)^{c_t}\cdot \big(1+\pnorm{\mathfrak{u}_{k\cdot}^{([0:t])}}{\infty}\big)^{c_0r_0}\cdot\phi^{-1}.
\end{align*}
Now using the definition of $\{\tau_{t,s}\}$ and the proven estimate for $\mathcal{B}^{(t)}$, we have 
\begin{align}\label{ineq:apr_est_se_part2_2}
\max_{s \in [1:t-1]} \pnorm{\tau_{t,s}}{\op}\leq \big(K\Lambda\varkappa_{\ast}\big)^{c_t}\cdot \phi^{-1}.
\end{align}
Moreover, by using the definition of $\{\Sigma_{t,s}\}$ in (S3), we conclude a similar estimate
\begin{align}\label{ineq:apr_est_se_part2_3}
\max_{s \in [1:t]} \pnorm{\Sigma_{t,s}}{\op}\leq \big(K\Lambda\varkappa_{\ast}\big)^{c_t}\cdot \phi^{-1}.
\end{align}
The claimed estimate now follows by combining (\ref{ineq:apr_est_se_part2_1}), (\ref{ineq:apr_est_se_part2_2}) and (\ref{ineq:apr_est_se_part2_3}).
\end{proof}

\appendix

\section{Common notation}\label{section:notation}
For any two integers $m,n$, let $[m:n]\equiv \{m,m+1,\ldots,n\}$. We sometimes write for notational convenience $[n]\equiv [1:n]$. When $m>n$, it is understood that $[m:n]=\emptyset$.  

For $a,b \in \R$, $a\vee b\equiv \max\{a,b\}$ and $a\wedge b\equiv\min\{a,b\}$. For $a \in \R$, let $a_\pm \equiv (\pm a)\vee 0$. For $a \in \R$ and $M\geq 0$, let $\mathscr{T}_M(x)\equiv (x\wedge M)\vee (-M)$.  For a multi-index $a \in \mathbb{Z}_{\geq 0}^n$, let $\abs{a}\equiv \sum_{i \in [n]}a_i$. For $x \in \R^n$, let $\pnorm{x}{p}$ denote its $p$-norm $(0\leq p\leq \infty)$, and $B_{n;p}(R)\equiv \{x \in \R^n: \pnorm{x}{p}\leq R\}$. We simply write $\pnorm{x}{}\equiv\pnorm{x}{2}$ and $B_n(R)\equiv B_{n;2}(R)$. For $x \in \R^n$, let $\mathrm{diag}(x)\equiv (x_i\bm{1}_{i=j})_{i,j \in [n]} \in \R^{n\times n}$.

For a matrix $M \in \R^{m\times n}$, let $\pnorm{M}{\op},\pnorm{M}{F}$ denote the spectral and Frobenius norm of $M$ respectively, and let $\pnorm{M}{\infty}\equiv \max_{i\in [m], j\in [n]} \abs{M_{ij}}$. For two matrices $M, N \in \R^{m\times n}$, let $M\odot N\equiv (M_{ij}N_{ij})_{i \in [m], j \in [n]}$ be their Hadamard product. For a square matrix $M\in \R^{n\times n}$, we write $\mathrm{diag}(M)\equiv (M_{ii}\bm{1}_{i=j})_{i,j \in [n]} \in \R^{n\times n}$ and $\mathrm{vecdiag}(M)\equiv (M_{ii})_{i \in [n]}\in \R^n$. $I_n$ is reserved for an $n\times n$ identity matrix, written simply as $I$ (in the proofs) if no confusion arises. 

Let $\mathbb{M}_q\equiv \mathbb{M}_q(\R)=\R^{q\times q}$ denote all $q\times q$ matrices with elements in $\R$. We shall write $(\mathbb{M}_q)^{t\times t}$ for all $t\times t$ matrices with elements in $\mathbb{M}_q$ (or viewed as block matrices). Let $0_{\mathbb{M}_q}=0_{q\times q}\in \mathbb{M}_q = \R^{q\times q}$ and 
$(I_{\mathbb{M}_q})_t \equiv \mathrm{diag}(I_q,\ldots,I_q) \in (\mathbb{M}_q)^{t\times t}$. For any (block) matrix $N = (N_{r,s})_{r,s \in [t]}\in (\mathbb{M}_q)^{t\times t}$, we write $N^\top\equiv (N_{s,r}^\top)_{r,s \in [t]}$ as the usual matrix transpose, and $N_\top\equiv (N_{r,s}^\top)_{r,s \in [t]} \in (\mathbb{M}_q)^{t\times t}$ as the matrix obtained by transposing every of its elements. For any (block) matrix $N \in (\mathbb{M}_q)^{t\times t}$, let $\mathfrak{O}_{\mathbb{M}_q;t+1}(N): (\mathbb{M}_q)^{t\times t}\to (\mathbb{M}_q)^{(t+1)\times (t+1)}$ be defined as
\begin{align}\label{def:mat_O}
\mathfrak{O}_{\mathbb{M}_q;t+1}(N)\equiv 
\begin{pmatrix}
(0_{\mathbb{M}_q})_{1\times t} & 0_{\mathbb{M}_q}\\
N & (0_{\mathbb{M}_q})_{t\times 1}
\end{pmatrix}
\in (\mathbb{M}_q)^{(t+1)\times (t+1)}.
\end{align}
For notational consistency, we write $\mathfrak{O}_{\mathbb{M}_q;1}(\emptyset)=0_{\mathbb{M}_q}$. 

We use $C_{x}$ to denote a generic constant that depends only on $x$, whose numeric value may change from line to line unless otherwise specified. $a\lesssim_{x} b$ and $a\gtrsim_x b$ mean $a\leq C_x b$ and $a\geq C_x b$, abbreviated as $a=\bigo_x(b), a=\Omega_x(b)$ respectively;  $a\asymp_x b$ means $a\lesssim_{x} b$ and $a\gtrsim_x b$. $\bigo$ and $\smallo$ (resp. $\mathcal{O}_{\mathbf{P}}$ and $\mathfrak{o}_{\mathbf{P}}$) denote the usual big and small O notation (resp. in probability). By convention, sum and product over an empty set are understood as $\Sigma_{\emptyset}(\cdots)=0$ and $\Pi_{\emptyset}(\cdots)=1$. 

For a random variable $X$, we use $\Prob_X,\E_X$ (resp. $\Prob^X,\E^X$) to indicate that the probability and expectation are taken with respect to $X$ (resp. conditional on $X$).

For $\Lambda>0$ and $\mathfrak{p}\in \N$, a measurable map $f:\R^n \to \R$ is called \emph{$\Lambda$-pseudo-Lipschitz of order $\mathfrak{p}$} iff 
\begin{align}\label{cond:pseudo_lip}
\abs{f(x)-f(y)}\leq \Lambda\cdot  (1+\pnorm{x}{}+\pnorm{y}{})^{\mathfrak{p}-1}\cdot\pnorm{x-y}{},\quad \hbox{for all } x,y \in \R^{n}.
\end{align}
Moreover, $f$ is called \emph{$\Lambda$-Lipschitz} iff $f$ is $\Lambda$-pseudo-Lipschitz of order $1$, and in this case we often write $\pnorm{f}{\mathrm{Lip}}\leq L$, where $\pnorm{f}{\mathrm{Lip}}\equiv \sup_{x\neq y} \abs{f(x)-f(y)}/\pnorm{x-y}{}$. For a proper, closed convex function $f$ defined on $\R^n$, its \emph{proximal operator} $\prox_f(\cdot)$  is defined by $\prox_f(x)\equiv \argmin_{z \in \R^n} \big\{\frac{1}{2}\pnorm{x-z}{}^2+f(z)\big\}$.

\section{Matrix-variate GFOM state evolution}\label{section:GFOM_theory}

\subsection{Symmetric version}

Consider a symmetric GFOM initialized with $z^{(0)} \in \R^{n\times q}$ and subsequently updated according to
\begin{align}\label{def:GFOM_sym}
z^{(t)}&=A \mathsf{F}_t(z^{([0:t-1])})+\mathsf{G}_t(z^{([0:t-1])}) \in \R^{n\times q}.
\end{align}
Here $A \in \R^{n\times n}$ is a symmetric random matrix, and $\{\mathsf{F}_t,\mathsf{G}_t\}$ are \emph{row-separate}, in the sense that for some measurable functions $\{\mathsf{F}_{t,\ell},\mathsf{G}_{t,\ell}: \R^{q[0:t-1]} \to \R^q\}_{\ell \in [n]}$,
\begin{align}\label{def:row_sep}
\mathsf{F}_t(z^{([0:t-1])})=\big(\mathsf{F}_{t,\ell}(z_\ell^{([0:t-1])})\big)_{\ell \in [n]},\quad \mathsf{G}_t(z^{([0:t-1])})=\big(\mathsf{G}_{t,\ell}(z_\ell^{([0:t-1])})\big)_{\ell \in [n]}.
\end{align}

For the symmetric GFOM in (\ref{def:GFOM_sym}), its state evolution is iteratively described in the following definition by two objects, namely, (i) a row-separate map $\Theta_t: \R^{n\times q[0:t]} \to \R^{n\times q[0:t]}$, and (ii) a sequence of centered Gaussian matrices $\{\mathfrak{Z}_k^{([1:\infty))}\}_{k \in [n]}\in (\R^q)^{[1:\infty)}$.

\begin{definition}\label{def:GFOM_se_sym}
	Initialize with $\Theta_0\equiv \mathrm{id}(\R^{n\times q})$ and $\mathfrak{Z}^{(0)}\equiv z^{(0)}\in \R^{n\times q}$. For $t=1,2,\ldots$, with $\E^{(0)}\equiv \E\big[\cdot | \mathfrak{Z}^{(0)}\big]$, execute the following steps:
	\begin{enumerate}
		\item  Let $\Theta_t: \R^{n\times q[0:t]} \to \R^{n\times q[0:t]}$ be defined as follows: for $w \in [0:t-1]$, let $\big[\Theta_t(\mathfrak{z}^{([0:t])})\big]_{\cdot,w}\equiv [\Theta_w(\mathfrak{z}^{([0:w])})]_{\cdot,w}$, and for $w=t$,
		\begin{align*}
		\big[\Theta_t(\mathfrak{z}^{([0:t])})\big]_{\cdot,t} \equiv  \mathfrak{z}^{(t)}+ \sum_{s \in [1:t-1]}  \mathsf{F}_{s}\big(\Theta_{s-1}(\mathfrak{z}^{([0:s-1])})\big) (\mathfrak{b}_{s}^{(t)})_\top + \mathsf{G}_t\big(\Theta_{t-1}(\mathfrak{z}^{([0:t-1])})\big),
		\end{align*} 
		where the correction matrices $\big\{\mathfrak{b}_s^{(t)}=(\mathfrak{b}_{s,k}^{(t)})_{k \in [n]} \in (\mathbb{M}_q)^n \big\}_{s \in [1:t-1]}$ are determined by 
		\begin{align*}
		\mathfrak{b}_{s,k}^{(t)} \equiv \sum_{\ell \in [n]} \E A_{k\ell}^2\cdot  \E^{(0)}\nabla_{\mathfrak{Z}_\ell^{(s)} } \big(\mathsf{F}_{t,\ell}\circ \Theta_{t-1,\ell} \big) \big(\mathfrak{Z}^{([0:t-1])}_\ell\big),\quad k \in [n].
		\end{align*}
		\item Let the Gaussian law of $\{\mathfrak{Z}_k^{(t)}\}_{k \in [n]}\in \R^q$ be determined via the following correlation specification: for $s \in [1:t]$ and $k \in [n]$,
		\begin{align*}
		\cov(\mathfrak{Z}^{(t)}_k, \mathfrak{Z}^{(s)}_k)\equiv \sum_{\ell \in [n]} \E A_{k\ell}^2 \cdot \E^{(0)}\mathsf{F}_{t,\ell} \Big(\Theta_{t-1,\ell}\big(\mathfrak{Z}^{([0:t-1])}_\ell\big)\Big) \mathsf{F}_{s,\ell} \Big(\Theta_{s-1,\ell}\big(\mathfrak{Z}^{([0:s-1])}_\ell\big)\Big)^\top.
		\end{align*}
	\end{enumerate}
\end{definition}

\begin{theorem}\label{thm:GFOM_se_sym}
	Fix $t \in \N$ and $n \in \N$. Suppose the following hold:
	\begin{enumerate}
		\item[(D1)] $A\equiv A_0/\sqrt{n}$, where $A_0$ is a symmetric matrix whose upper triangle entries are independent mean $0$ variables such that $\max_{i,j \in [n]}\pnorm{A_{0,ij}}{\psi_2}\leq K$ holds for some $K\geq 2$.
		\item[(D2)] For all $s \in [t],\ell \in [n]$, $r\in [q]$, $\mathsf{F}_{s,\ell;r},\mathsf{G}_{s,\ell;r} \in C^3(\R^{q[0:s-1]})$. Moreover, there exists some $\Lambda\geq 2$ such that 
		\begin{align*}
		\max_{s \in [t]}\max_{\mathsf{E}_s \in \{\mathsf{F}_s,\mathsf{G}_s\}}\max_{\ell \in [n]} \max_{r \in [q]} \Big\{\abs{\mathsf{E}_{s,\ell;r}(0)}+\max_{a \in \mathbb{Z}_{\geq 0}^{q[0:s-1]}, \abs{a}\leq 3} \pnorm{ \partial^a \mathsf{E}_{s,\ell;r} }{\infty}\Big\}\leq \Lambda.
		\end{align*}
	\end{enumerate}
	Then for any $\Psi \in C^3(\R^{ q[0:t]})$ satisfying
	\begin{align}\label{cond:Psi_asym}
	\max_{a \in \mathbb{Z}_{\geq 0}^{ q[0:t]}, \abs{a}\leq 3} \sup_{x \in \R^{ q[0:t] } }\bigg(\sum_{ \tau \in  q[0:t] }(1+\abs{x_{\tau}})^{\mathfrak{p}}\bigg)^{-1} \abs{\partial^a \Psi(x)} \leq \Lambda_{\Psi}
	\end{align}
	for some $\Lambda_{\Psi}\geq 2$, it holds for some $c_0=c_0(q)>0$ and $c_1\equiv c_1(\mathfrak{p},q)>0$, such that with $\E^{(0)}\equiv \E\big[\cdot | z^{(0)}\big]$
	\begin{align*}
	& \max_{k \in [n]}\bigabs{\E^{(0)}\Psi\big[z_k^{([0:t])}(A)\big]-\E^{(0)}\Psi\big[\Theta_{t,k}(\mathfrak{Z}_k^{([0:t])})\big] } \\
	&\qquad \leq \Lambda_\Psi \cdot \big(K\Lambda \log n\cdot (1+\pnorm{z^{(0)}}{\infty})\big)^{c_1 t^{5}} \cdot n^{-1/c_0^t}.
	\end{align*}
\end{theorem}

\begin{theorem}\label{thm:GFOM_se_sym_avg}
	Fix $t \in \N$ and $n \in \N$. Suppose (D1) in Theorem \ref{thm:GFOM_se_sym} and (D2) therein replaced by
	\begin{enumerate}
		\item[(D2)'] 
		$
		\max\limits_{s \in [t]}\max\limits_{\mathsf{E}_s \in \{\mathsf{F}_s,\mathsf{G}_s\}}\max\limits_{\ell \in [n]} \max\limits_{r \in [q]}\big\{\pnorm{\mathsf{E}_{s,\ell;r}}{\mathrm{Lip}}+\abs{\mathsf{E}_{s,\ell;r}(0)}\big\}\leq \Lambda$ for some $\Lambda\geq 2$.
	\end{enumerate}
    Fix a sequence of $\Lambda_\psi$-pseudo-Lipschitz functions $\{\psi_k:\R^{q[0:t]} \to \R\}_{k \in [n]}$ of order $\mathfrak{p}$, where $\Lambda_\psi\geq 2$. Then for any $\mathfrak{r}\geq 1$, there exists some $C_0=C_0(\mathfrak{p},q,\mathfrak{r})>0$ such that with $\E^{(0)}\equiv \E\big[\cdot | z^{(0)}\big]$,
	\begin{align*}
	&\E^{(0)} \biggabs{\frac{1}{n}\sum_{k \in [n]} \psi_k\big(z_k^{([0:t])}(A)\big) - \frac{1}{n}\sum_{k \in [n]}  \E^{(0)} \psi_k\big[\Theta_{t,k}(\mathfrak{Z}_k^{([0:t])})\big]  }^{\mathfrak{r}} \\
	&\leq \big(K\Lambda\Lambda_\psi\log n\cdot (1+\pnorm{z^{(0)}}{\infty})\big)^{C_0 t^{5} }\cdot n^{-1/C_0^t}. 
	\end{align*}
\end{theorem}

The following delocalization estimate will be useful.

\begin{proposition}\label{prop:GFOM_deloc_sym}
	Fix $t \in \N$ and $n \in \N$. Suppose (D1) in Theorem \ref{thm:GFOM_se_sym} and (D2) therein replaced by the following:
	\begin{enumerate}
		\item[(D2)''] There exists some some $\Lambda\geq 2$ and $\kappa\geq 0$ such that
		\begin{align*}
		\max\limits_{s \in [t]}\max\limits_{\mathsf{E}_s \in \{\mathsf{F}_s,\mathsf{G}_s\}}\max\limits_{\ell \in [n]} \max\limits_{r \in [q]} \bigg\{\max\limits_{x\neq y} \biggabs{ \frac{\mathsf{E}_{s,\ell;r}(x)-\mathsf{E}_{s,\ell;r}(y)}{(1+\pnorm{x}{}+\pnorm{y}{})^{\kappa}\pnorm{x-y}{} }}+ \abs{\mathsf{E}_{s,\ell;r}(0)}  \bigg\}\leq \Lambda.
		\end{align*}	 
	\end{enumerate}
	Then there exists some constant $c_t\equiv c_t(t,q,\kappa)>0$ such that for $0\leq x\leq n$,
	\begin{align*}
	\max_{k \in [n]}\Prob^{(0)}\Big[\pnorm{z^{(t)}_{k\cdot}}{}\geq \big(K\Lambda (1+\pnorm{z^{(0)}}{\infty})\cdot x\big)^{c_t} \Big]\leq c_te^{-x/c_t}.
	\end{align*}
\end{proposition}
\begin{proof}
	The proof is a modification of that of \cite[Proposition 6.2]{han2024entrywise}.	For notational simplicity, we consider
	\begin{align*}
	z^{(t)}&=A \mathsf{F}_t(z^{(t-1)})+\mathsf{G}_t(z^{(t-1)}) \in \R^{n\times q},
	\end{align*}
	and its leave-one-out version: for $k \in [n]$, let
	\begin{align*}
	z^{(t)}_{[-k]}&=A_{[-k]} \mathsf{F}_t(z_{[-k]}^{(t-1)})+\mathsf{G}_t(z_{[-k]}^{(t-1)})\in \R^{n\times q},
	\end{align*}
	where $z^{(0)}_{[-k]}\equiv z^{(0)}$, and $A_{[-k]}$ is obtained by setting the elements of the $k$-th row and column of $A$ to be $0$. Before proceeding with adaptation of the arguments in \cite[Proposition 6.2]{han2024entrywise}, it is convenient to note that from (2) we have for any $x, y \in \R^{n\times q}$ and $\mathsf{E}_t \in \{\mathsf{F}_t,\mathsf{G}_t\}$,
	\begin{align*}
	\pnorm{ \mathsf{E}_t(x)- \mathsf{E}_t(y)}{}^2&= \sum_{\ell \in [n],r \in [q]} \abs{\mathsf{E}_{s,\ell;r}(x_{\ell\cdot})-\mathsf{E}_{s,\ell;r}(y_{\ell\cdot})}^2\\
	&\leq q\cdot \sum_{\ell \in [n]} \big(1+\pnorm{x_{\ell\cdot} }{}+\pnorm{y_{\ell\cdot} }{}\big)^{2\kappa}\pnorm{x_{\ell\cdot}-y_{\ell\cdot} }{}^{2}\\
	&\leq q\cdot \Big[\max_{\ell \in [n]} \big(1+\pnorm{x_{\ell\cdot} }{}+\pnorm{y_{\ell\cdot} }{}\big)^{2\kappa}\Big]\cdot \pnorm{x-y}{}^{2}.
	\end{align*}
	Let us work on the event $
	\mathscr{E}_0\equiv \big\{\pnorm{A}{\op}\vee \pnorm{A_{[-k]}}{\op}\vee\max_{k \in [n]}\pnorm{A_{k\cdot}}{}\leq c_0 K e_n(x)\big\}$ for some large enough $c_0>2$, and $e_n(x)\equiv 1+\sqrt{x/n}$. By choosing $c_0$ large, in view of \cite[Theorem 4.4.5]{vershynin2018high} and the subsequent remarks, we have $\Prob(\mathscr{E}_0^c)\leq c e^{-x/c}$. Now note that on the event $\mathscr{E}_0$, 
	\begin{align*}
	&\bigpnorm{ z^{(t)}-z_{[-k]}^{(t)}  }{}\\
	&\leq \bigpnorm{A \mathsf{F}_t(z^{(t-1)})- A_{[-k]} \mathsf{F}_t\big(z^{(t-1)}_{[-k]}\big)}{}+ \bigpnorm{\mathsf{G}_t(z^{(t-1)})- \mathsf{G}_t\big(z^{(t-1)}_{[-k]}  }{}\nonumber\\
	&\leq \pnorm{A}{\op}\bigpnorm{\mathsf{F}_t(z^{(t-1)})- \mathsf{F}_t\big(z^{(t-1)}_{[-k]}\big)}{}+ \bigpnorm{ \big(A- A_{[-k]}\big) \mathsf{F}_t\big(z^{(t-1)}_{[-k]}\big)}{} + \bigpnorm{\mathsf{G}_t(z^{(t-1)})- \mathsf{G}_t\big(z^{(t-1)}_{[-k]}}{}\nonumber\\
	&\leq c_q K\Lambda e_n(x)\cdot \max_{\ell \in [n]} \Big(1+\pnorm{z_{\ell\cdot}^{(t-1)}}{}+\pnorm{z_{[-k],\ell\cdot}^{(t-1)}}{}\Big)^\kappa  \bigpnorm{ z^{(t-1)}-z_{[-k]}^{(t-1)}  }{} + \bigpnorm{ A_{[k]} \mathsf{F}_t\big(z^{(t-1)}_{[-k]}\big)}{}.\nonumber
	\end{align*}
	Here $A_{[k]}\equiv A-A_{[-k]}$. By writing $\Delta z^{(t)}_{[-k]}\equiv  z^{(t)}-z_{[-k]}^{(t)}$, we have for $0\leq x\leq n$,
	\begin{align}\label{ineq:gfom_deloc_1}
	\pnorm{\Delta z^{(t)}_{[-k]}}{}&\leq (K\Lambda)^{c_t}\cdot \big(1+\pnorm{z^{(t-1)}}{\infty}\big)\cdot \big(1+\pnorm{\Delta z^{(t-1)}_{[-k]}}{}\big)^{\kappa+1}+\bigpnorm{ A_{[k]} \mathsf{F}_t\big(z^{(t-1)}_{[-k]}\big)}{}.
	\end{align}
	On the other hand, 
	\begin{align}\label{ineq:gfom_deloc_2}
	1+\bigpnorm{z^{(t-1)}_{[-k],k\cdot}}{} \leq c_q\Lambda\cdot \Big( 1+\bigpnorm{z^{(t-2)}_{[-k],k\cdot}}{}\Big) \leq\cdots\leq (c_q\Lambda)^{t-1} \Big(1+\bigpnorm{z^{(0)}_{k\cdot}}{}\Big),
	\end{align}
	we have on the event $\mathscr{E}_0$, for $0\leq x\leq n$,
	\begin{align}\label{ineq:gfom_deloc_3}
	\bigpnorm{ A_{[k]} \mathsf{F}_t\big(z^{(t-1)}_{[-k]}\big)}{}&= \bigg\{\sum_{\ell \neq k} \abs{A_{k,\ell}}^2\pnorm{ \mathsf{F}_{t,k}\big(z^{(t-1)}_{[-k],k\cdot} \big)}{}^2+\bigpnorm{ A_{k,\cdot} \mathsf{F}_t\big(z^{(t-1)}_{[-k]}\big)  }{}^2\bigg\}^{1/2}\nonumber\\
	&\leq \pnorm{A_{k,\cdot}}{}\cdot\bigpnorm{ \mathsf{F}_{t,k}\big(z^{(t-1)}_{[-k],k\cdot}\big)  }{}+ \bigpnorm{ A_{k,\cdot} \mathsf{F}_t\big(z^{(t-1)}_{[-k]}\big)  }{}\nonumber\\
	&\leq c_q K\cdot \Lambda\Big(1+\bigpnorm{z^{(t-1)}_{[-k],k\cdot}}{}^\kappa\Big)+\bigpnorm{ A_{k,\cdot} \mathsf{F}_t\big(z^{(t-1)}_{[-k]}\big)  }{}\nonumber\\
	&\leq (K\Lambda)^{c_t} \cdot\big(1+\pnorm{z^{(0)}}{\infty}\big)^{\kappa}+ \bigpnorm{ A_{k,\cdot} \mathsf{F}_t\big(z^{(t-1)}_{[-k]}\big)  }{}.
	\end{align}
	As $A_{k\cdot}$ is independent of $z_{[-k]}^{(t-1)}$, for $0\leq x\leq n$, on an event $\mathscr{E}_{1,t}\cap \mathscr{E}_0$, where $\Prob(\mathscr{E}_{1,t}^c)\leq c_qe^{-x/c_q}$, 
	\begin{align}\label{ineq:gfom_deloc_4}
	\bigpnorm{ A_{k,\cdot} \mathsf{F}_t\big(z^{(t-1)}_{[-k]}\big)  }{}&\leq c_q \sqrt{x}\cdot K n^{-1/2}\bigpnorm{ \mathsf{F}_t\big(z_{[-k]}^{(t-1)}\big)  }{}\nonumber\\
	&\leq c_q \sqrt{x}\cdot K \Lambda\cdot  \max_{k \in [n]}\Big(1+\bigpnorm{z_{[-k],k\cdot}^{(t-1)} }{}\Big)^{\kappa}\nonumber\\
	&\leq (K\Lambda)^{c_t}\sqrt{x}\cdot \big(1+\pnorm{z^{(0)}}{\infty}\big)^\kappa.
	\end{align}
	Combining the above two displays (\ref{ineq:gfom_deloc_3})-(\ref{ineq:gfom_deloc_4}), for $0\leq x\leq n$, on the event $\mathscr{E}_0\cap \mathscr{E}_{1,t}$, 
	\begin{align*}
	\bigpnorm{ A_{[k]} \mathsf{F}_t\big(z^{(t-1)}_{[-k]}\big)}{}&\leq (K\Lambda)^{c_t}\sqrt{x}\cdot \big(1+\pnorm{z^{(0)}}{\infty}\big)^\kappa.
	\end{align*}
	Consequently, combined with (\ref{ineq:gfom_deloc_1}), on the event $\cap_{s \in [t]}\big(\mathscr{E}_0\cap \mathscr{E}_{1,s}\big)$, for $0\leq x\leq n$, 
	\begin{align*}
	1+\bigpnorm{ \Delta z_{[-k]}^{(t)}  }{}
	&\leq (K\Lambda)^{c_t}\cdot \big(1+\pnorm{z^{(t-1)}}{\infty}\big)\cdot \big(1+\pnorm{\Delta z^{(t-1)}_{[-k]}}{}\big)^{\kappa+1}\\
	&\qquad\qquad +(K\Lambda)^{c_t}\sqrt{x}\cdot \big(1+\pnorm{z^{(0)}}{\infty}\big)^\kappa \\
	&\leq \cdots \leq \Big[K\Lambda \Big(1+\max_{s \in [t]}\pnorm{z^{(s-1)}}{\infty}\Big)\cdot x\Big]^{c_{t}}.
	\end{align*}
	By noting that $\pnorm{z^{(t)}_{k\cdot}}{}\leq \bigpnorm{ z^{(t)}-z_{[-k]}^{(t)}  }{}+\pnorm{z^{(t)}_{[-k],k\cdot}}{}$ and using the estimate (\ref{ineq:gfom_deloc_2}), we arrive at the recursion
	\begin{align*}
	1+\max_{s \in [t+1]}\pnorm{z^{(s-1)}}{\infty}\leq \Big[K\Lambda \Big(1+\max_{s \in [t]}\pnorm{z^{(s-1)}}{\infty}\Big)\cdot x\Big]^{c_{t}}.
	\end{align*}
	Iterating the above estimate to conclude.
\end{proof}

\subsection{Asymmetric version}
Consider an asymmetric GFOM initialized with $(u^{(0)},v^{(0)})\in \R^{m\times q}\times \R^{n\times q}$, and subsequently updated according to 
\begin{align}\label{def:GFOM_asym}
\begin{cases}
u^{(t)} = A \mathsf{F}_t^{\langle 1\rangle}(v^{([0:t-1])})+ \mathsf{G}_{t}^{\langle 1\rangle}(u^{([0:t-1])}) \in \R^{m\times q},\\
v^{(t)}= A^\top \mathsf{G}_t^{\langle 2\rangle}(u^{([0:t])})+\mathsf{F}_t^{\langle 2\rangle}(v^{([0:t-1])})\in \R^{n\times q}.
\end{cases}
\end{align}
Here we denote $A$ as an $m\times n$ random matrix, and the row-separate functions $\mathsf{F}_t^{\langle 1 \rangle}, \mathsf{F}_t^{\langle 2 \rangle}:\R^{n\times q[0:(t-1)]} \to \R^{n\times q}$, $\mathsf{G}_t^{\langle 1 \rangle}: \R^{m\times q[0: (t-1)]} \to \R^{m\times q}$ and $\mathsf{G}_t^{\langle 2 \rangle}:\R^{m\times q[0: t]} \to \R^{m\times q}$ are understood as applied row-wise.

The state evolution for the asymmetric GFOM (\ref{def:GFOM_asym}) is iteratively described---in the following definition---by (i) two row-separate maps $\Phi_t: \R^{m\times q[0:t]} \to \R^{m\times q[0:t]}$ and $\Xi_t: \R^{n\times q[0:t]} \to \R^{n\times q[0:t]}$, and (ii) two sequences of centered Gaussian matrices $\{\mathfrak{U}_k^{([1:\infty))}\}_{k \in [m]}\in (\R^q)^{ [1:\infty)}$ and $\{\mathfrak{V}_\ell^{([1:\infty))}\}_{\ell \in [n]}\in (\R^q)^{ [1:\infty)}$. 

\begin{definition}\label{def:GFOM_se_asym}
	Initialize with $\Phi_0=\mathrm{id}(\R^{m\times q})$, $\Xi_0\equiv \mathrm{id}(\R^{n\times q})$, and $\mathfrak{U}^{(0)}\equiv u^{(0)}$, $\mathfrak{V}^{(0)}\equiv v^{(0)}$. For $t=1,2,\ldots$, with $\E^{(0)}\equiv \E\big[\cdot|(\mathfrak{U}^{(0)},\mathfrak{V}^{(0)})\big]$, and $\pi_n$ denoting the uniform distribution on $[n]$, we execute the following steps:
	\begin{enumerate}
		\item Let $\Phi_{t}:\R^{m\times q[0:t]}\to \R^{m\times q[0:t]}$ be defined as follows: for $w \in [0:t-1]$, $\big[\Phi_{t}(\mathfrak{u}^{([0:t])})\big]_{\cdot,w}\equiv \big[\Phi_{w}(\mathfrak{u}^{([0:w])})\big]_{\cdot,w}$, and for $w=t$,
		\begin{align*}
		\big[\Phi_{t}(\mathfrak{u}^{([0:t])})\big]_{\cdot,t} \equiv \mathfrak{u}^{(t)}+\sum_{s \in [1:t-1]}  \mathsf{G}_{s}^{\langle 2\rangle}\big(\Phi_{s}(\mathfrak{u}^{([0:s])}) \big)\circ (\mathfrak{f}_{s}^{(t-1)})_\top +\mathsf{G}_{t}^{\langle 1\rangle}\big(\Phi_{t-1}(\mathfrak{u}^{([0:t-1])}) \big),
		\end{align*}
		where the correction matrices $\big\{\mathfrak{f}_{s}^{(t-1) } = (\mathfrak{f}_{s,k}^{(t-1) } )_{k \in [m]} \in (\mathbb{M}_q)^m\big\}_{s \in [1:t-1]}$ are determined by
		\begin{align*}
		\mathfrak{f}_{s,k}^{(t-1) } \equiv  \sum_{\ell \in [n]} \E A_{k\ell}^2\cdot \E^{(0)} \nabla_{\mathfrak{V}_\ell^{(s)}} \big\{ \mathsf{F}_{t,\ell}^{\langle 1\rangle}\circ \Xi_{t-1,\ell}\big\} (\mathfrak{V}_\ell^{([0:t-1])})\in \mathbb{M}_q, \quad  k \in [m].
		\end{align*}
		\item Let the Gaussian laws of $\{\mathfrak{U}_k^{(t)}\}\subset \R^q$ be determined via the following correlation specification: for $s \in [1:t]$, $k \in [m]$,
		\begin{align*}
		\mathrm{Cov}\big(\mathfrak{U}_k^{(t)}, \mathfrak{U}_k^{(s)} \big)\equiv  \sum_{\ell \in [n]} \E A_{k\ell}^2 \cdot \E^{(0)}  \big\{ \mathsf{F}_{t,\ell}^{\langle 1\rangle}\circ \Xi_{t-1,\ell} \big\} (\mathfrak{V}_\ell^{([0:t-1])})   \big\{ \mathsf{F}_{s,\ell}^{\langle 1\rangle}\circ \Xi_{s-1,\ell}\big\} (\mathfrak{V}_\ell^{([0:s-1])})^\top.
		\end{align*}
		\item Let $\Xi_{t}:\R^{n\times q[0:t]}\to \R^{n\times q[0:t]}$ be defined as follows: for $w \in [0:t-1]$, $\big[\Xi_{t}(\mathfrak{v}^{([0:t])})\big]_{\cdot,w}\equiv \big[\Xi_{w}(\mathfrak{v}^{([0:w])})\big]_{\cdot,w}$, and for $w=t$,
		\begin{align*}
		\big[\Xi_{t}(\mathfrak{v}^{([0:t])})\big]_{\cdot,t} \equiv \mathfrak{v}^{(t)}+\sum_{s \in [1:t]}  \mathsf{F}_{s}^{\langle 1\rangle}\big(\Xi_{s-1}(\mathfrak{v}^{([0:s-1])}) \big)\circ \mathfrak{g}_{s}^{(t),\top}+\mathsf{F}_{t}^{\langle 2\rangle}\big(\Xi_{t-1}(\mathfrak{v}^{([0:t-1])}) \big),
		\end{align*}
		where  the correction matrices $\big\{\mathfrak{g}_{s}^{(t)}=(\mathfrak{g}_{s,\ell}^{(t)})_{\ell \in [n]} \in (\mathbb{M}_q)^n\big\}_{s \in [1:t]}$ are determined by
		\begin{align*}
		\mathfrak{g}_{s,\ell}^{(t)}\equiv \sum_{k \in [m]} \E A_{k\ell}^2 \cdot \E^{(0)} \nabla_{\mathfrak{U}_k^{(s)}} \big\{ \mathsf{G}_{t,k}^{\langle 2\rangle}\circ\Phi_{t,k}\big\}  (\mathfrak{U}_k^{([0:t])})  \in \mathbb{M}_q,\quad \ell \in [n].
		\end{align*}
		\item Let the Gaussian law of $\{\mathfrak{V}^{(t)}_\ell\}$ be determined via the following correlation specification: for $s \in [1:t]$, $\ell \in [n]$,
		\begin{align*}
		\mathrm{Cov}(\mathfrak{V}_\ell^{(t)},\mathfrak{V}_\ell^{(s)})\equiv  \sum_{k \in [m]} \E A_{k\ell}^2\cdot \E^{(0)} \big\{ \mathsf{G}_{t,k}^{\langle 2\rangle}\circ\Phi_{t,k}\big\}  (\mathfrak{U}_k^{([0:t])}) \big\{ \mathsf{G}_{s,k}^{\langle 2\rangle}\circ\Phi_{s,k}\big\} (\mathfrak{U}_k^{([0:s])})^\top.
		\end{align*}
	\end{enumerate}
\end{definition}

\begin{theorem}\label{thm:GFOM_se_asym}
	Fix $t \in \N$ and $n \in \N$. Suppose the following hold:
	\begin{enumerate}
		\item[($D^\ast$1)] $A\equiv A_0/\sqrt{n}$, where the entries of $A_0\in \R^{m\times n}$ are independent mean $0$ variables such that $\max_{i,j \in [n]}\pnorm{A_{0,ij}}{\psi_2}\leq K$ holds for some $K\geq 2$.
		\item[($D^\ast$2)] For all $s \in [t], \#\in \{1,2\}, k \in [m], \ell \in [n], r\in [q]$, $\mathsf{F}_{s,\ell;r}^{\langle \# \rangle},\mathsf{G}_{s,k;r}^{\langle 1 \rangle} \in C^3(\R^{q[0:s-1]})$ and $\mathsf{G}_{s,k;r}^{\langle 2 \rangle} \in C^3(\R^{q[0:s]})$. Moreover, there exists some $\Lambda\geq 2$ and $\mathfrak{p}\in \N$ such that 
		\begin{align*}
		&\max_{s \in [t]}\max_{\# =1,2}\max_{k\in[m], \ell \in [n]}\max_{r \in [q]}\Big\{ \abs{ \mathsf{F}_{s,\ell;r}^{\langle \# \rangle}(0) }+ \abs{\mathsf{G}_{s,k;r}^{\langle \# \rangle}(0)}\\
		&\quad +\max_{a \in \mathbb{Z}_{\geq 0}^{q[0:s-1]},b \in \mathbb{Z}_{\geq 0}^{q[0:s]},\abs{a}\vee \abs{b}\leq 3 } \Big(\bigpnorm{ \partial^a \mathsf{F}_{s,\ell;r}^{\langle \# \rangle} }{\infty}+\bigpnorm{ \partial^a \mathsf{G}_{s,k;r}^{\langle 1 \rangle}  }{\infty}+ \bigpnorm{  \partial^b \mathsf{G}_{s,k;r}^{\langle 2 \rangle}   }{\infty} \Big)   \Big\}\leq \Lambda.
		\end{align*}
	\end{enumerate}
	Further suppose $1/K\leq m/n\leq K$. Then for any $\Psi \in C^3(\R^{ q[0:t]})$ satisfying (\ref{cond:Psi_asym}), for some $\Lambda_{\Psi}\geq 2$, it holds for some $c_0=c_0(q)>0$ and $c_1\equiv c_1(\mathfrak{p},q)>0$, such that with $\E^{(0)}\equiv \E[\cdot|(u^{(0)},v^{(0)})]$, 
	\begin{align*}
	& \max_{k \in [m]}\bigabs{ \E^{(0)} \Psi\big(u_{k}^{([0:t])}(A)\big)-\E^{(0)}\Psi\big(\Phi_{t,k}\big(\mathfrak{U}_k^{([0:t])}\big)\big) } \\
	& \quad \vee \max_{\ell \in [n]} \bigabs{\E^{(0)} \Psi\big(v_{\ell}^{([0:t])}(A)\big) -\E \Psi\big(\Xi_{t,\ell}\big(\mathfrak{V}_\ell^{([1:t])}\big)\big) } \\
	& \leq \Lambda_\Psi \cdot \big(K\Lambda \log n\cdot (1+\pnorm{u^{(0)}}{\infty}+\pnorm{v^{(0)}}{\infty})\big)^{c_1 t^5} \cdot n^{-1/c_0^t}.
	\end{align*}
\end{theorem}

\begin{theorem}\label{thm:GFOM_se_asym_avg}
	Fix $t \in \N$ and $n \in \N$, and suppose $1/K\leq m/n\leq K$ for some $K\geq 2$. Suppose $(D^\ast 1)$ in Theorem \ref{thm:GFOM_se_asym} holds and $(D^\ast2)$ therein is replaced by
	\begin{enumerate}
		\item[($D^\ast $2)'] 
		$
		\max\limits_{s \in [t]}\max\limits_{  \# = 1,2   }\max\limits_{k \in [m], \ell \in [n]} \max\limits_{r \in [q]}\big\{\pnorm{ \mathsf{F}_{s,\ell;r}^{\langle \# \rangle} }{\mathrm{Lip}}+ \pnorm{\mathsf{G}_{s,k;r}^{\langle \# \rangle}}{\mathrm{Lip}}+\abs{\mathsf{F}_{s,\ell;r}^{\langle \# \rangle} (0)}+ \abs{\mathsf{G}_{s,k;r}^{\langle \# \rangle} (0)}\big\}\leq \Lambda$ for some $\Lambda\geq 2$.
	\end{enumerate}
	Fix a sequence of $\Lambda_\psi$-pseudo-Lipschitz functions $\{\psi_k:\R^{q[0:t]} \to \R\}_{k \in [m\vee n]}$ of order $\mathfrak{p}$, where $\Lambda_\psi\geq 2$. Then for any $\mathfrak{r}\geq 1$, there exists some $C_0=C_0(\mathfrak{p},q,\mathfrak{r})>0$ such that with $\E^{(0)}\equiv \E[\cdot|(u^{(0)},v^{(0)})]$,  
	\begin{align*}
	&\E^{(0)} \bigg[\biggabs{\frac{1}{m}\sum_{k \in [m]} \psi_k\big(u_k^{([0:t])}(A)\big) - \frac{1}{m}\sum_{k \in [m]}  \E^{(0)}  \psi_k\big[\Phi_{t,k}\big(\mathfrak{U}_k^{([0:t])}\big)\big]  }^{\mathfrak{r}}\bigg]\\
	&\quad \vee \E^{(0)}  \bigg[\biggabs{\frac{1}{n}\sum_{\ell \in [n]} \psi_\ell\big(v_\ell^{([0:t])}(A)\big) - \frac{1}{n}\sum_{\ell \in [n]}  \E^{(0)}  \psi_\ell\big[\Xi_{t,\ell}\big(\mathfrak{V}_\ell^{([0:t])}\big)\big]   }^{\mathfrak{r}}\bigg]  \\
	&\leq \big(K\Lambda\Lambda_\psi\log n\cdot (1+\pnorm{u^{(0)}}{\infty}+\pnorm{v^{(0)}}{\infty})\big)^{C_0 t^5}\cdot n^{-1/C_0^t}. 
	\end{align*}
\end{theorem}
\begin{figure}[t!]
	\centering
	\includegraphics[width=0.99\linewidth]{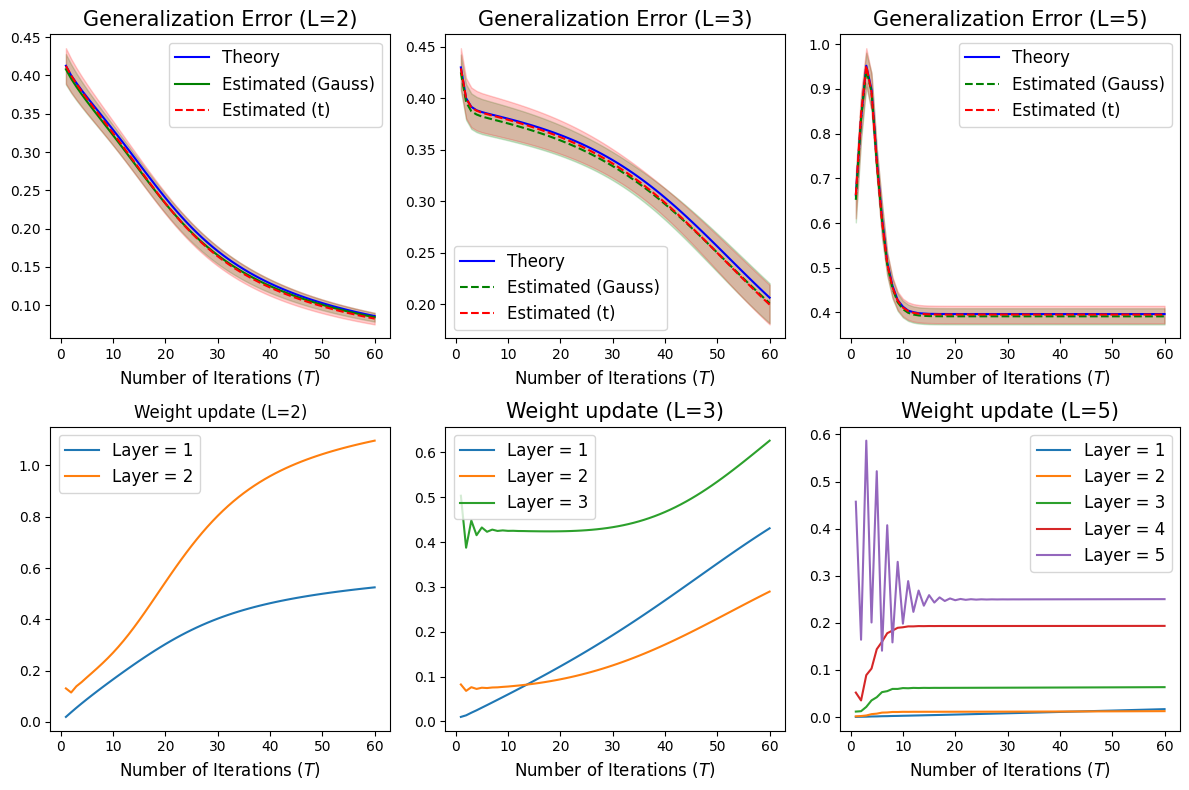}
	\caption{Algorithmic estimation of the generalization error (noiseless). }
	\label{fig:estimation_error_smallnoise}
\end{figure}

We have the following analogue of Proposition \ref{prop:GFOM_deloc_sym}; the proof is omitted for simplicity.

\begin{proposition}\label{prop:GFOM_deloc_asym}
Fix $t \in \N$ and $n \in \N$, and suppose $1/K\leq m/n\leq K$ for some $K\geq 2$. Suppose $(D^\ast 1)$ in Theorem \ref{thm:GFOM_se_asym} holds and $(D^\ast2)$ therein is replaced by the following:
\begin{enumerate}
	\item[(D*2)''] There exists some some $\Lambda\geq 2$ and $\kappa\geq 0$ such that
	\begin{align*}
	&\max\limits_{s \in [t]} \max\limits_{  \# = 1,2   }\max\limits_{k \in [m], \ell \in [n]} \max\limits_{r \in [q]} \bigg\{\max\limits_{x\neq y} \biggabs{ \frac{\mathsf{F}^{\langle \# \rangle}_{s,\ell;r}(x)-\mathsf{F}^{\langle \# \rangle}_{s,\ell;r}(y)}{(1+\pnorm{x}{}+\pnorm{y}{})^\kappa \pnorm{x-y}{} }}+ \abs{\mathsf{F}^{\langle \# \rangle}_{s,\ell;r}(0)}\\
	&\qquad\qquad +  \max\limits_{x\neq y} \biggabs{ \frac{\mathsf{G}^{\langle \# \rangle}_{s,k;r}(x)-\mathsf{G}^{\langle \# \rangle}_{s,k;r}(y)}{(1+\pnorm{x}{}+\pnorm{y}{})^\kappa\pnorm{x-y}{} }}+ \abs{\mathsf{G}^{\langle \# \rangle}_{s,k;r}(0)} \bigg\}\leq \Lambda.
	\end{align*}	 
\end{enumerate}
Then there exists some constant $c_t\equiv c_t(t,q,\kappa)>0$ such that for $0\leq x\leq n$,
\begin{align*}
\max_{k \in [m],\ell \in [n]}\Prob^{(0)}\Big[\pnorm{u^{(t)}_{k\cdot}}{}\vee \pnorm{v^{(t)}_{\ell\cdot}}{}\geq \big(K\Lambda (1+\pnorm{u^{(0)}}{\infty}+\pnorm{v^{(0)}}{\infty})\cdot x\big)^{c_t}\Big]\leq c_te^{-x/c_t}.
\end{align*}
\end{proposition}

\section{Additional simulation results}\label{section:simulation_additional}

In this appendix, we present analogues of Figures \ref{fig:estimation_error} and \ref{fig:multi-index} in the noiseless setting, i.e., $\xi = 0_m$. The simulations parameters are otherwise the same as those used in Sections \ref{subsection:alg_est_gen_err} and \ref{subsection:alg_est_gen_err_multi}, and Algorithm~\ref{alg:aug_gd_nn} is run for $60$ iterations with $80$ Monte Carlo repetitions. Specifically:
\begin{itemize}
    \item In Figure \ref{fig:estimation_error_smallnoise}, we present the noiseless analogue of Figure \ref{fig:estimation_error}.
    \item In Figure \ref{fig:multi-index_noiseless}, we present the noiseless analogue of Figure \ref{fig:multi-index}.
\end{itemize}

Not surprisingly, in both Figures \ref{fig:estimation_error_smallnoise} and \ref{fig:multi-index_noiseless} we observe that the error bars are substantially smaller than their counterparts in Figures \ref{fig:estimation_error} and \ref{fig:multi-index}. This shows that the accuracy of Algorithm~\ref{alg:aug_gd_nn} can be substantially improved for large signal-to-noise ratios, which can be achieved either by reducing the noise level or by increasing the sample size.

\section*{Acknowledgments}
The research of Q. Han is partially supported by NSF grant DMS-2143468. The research of M. Imaizumi is supported by JSPS KAKENHI (24K02904), JST CREST (JPMJCR21D2), and JST FOREST (JPMJFR216I).

\begin{figure}[t!]
    \centering
    \includegraphics[width=0.99\linewidth]{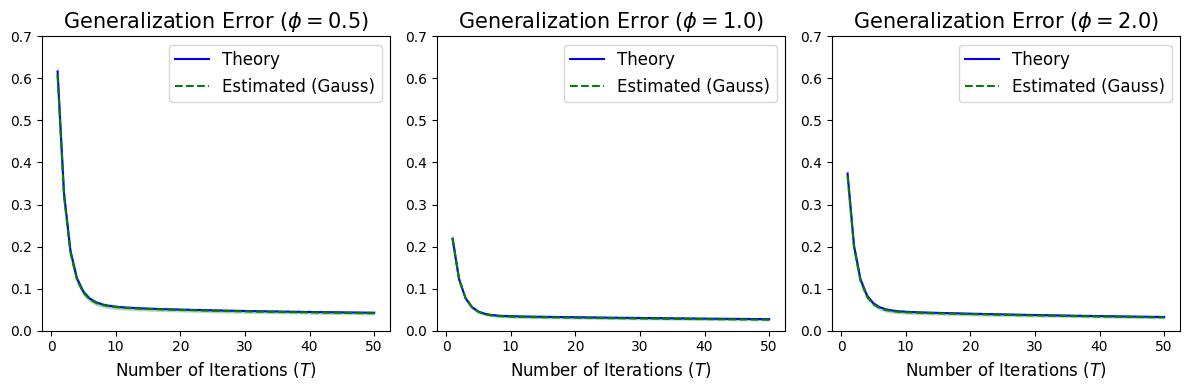}
    \caption{Algorithmic estimation of the generalization error with
a multi-index model (noiseless).}
    \label{fig:multi-index_noiseless}
\end{figure}

\bibliographystyle{alpha}
\bibliography{mybib}

\end{document}